%% file: jmlr-paper.tex
\documentclass[twoside,11pt]{article}

\usepackage{jmlr2e}
\usepackage{flushend}
\usepackage{cite}
\usepackage{amsmath,amssymb}    
\usepackage{graphicx}   
\usepackage{verbatim}   
\usepackage{color}      
\usepackage{subfigure}  
\usepackage{booktabs}
\usepackage{multirow}
\usepackage{graphics}
\usepackage{epsfig}
\usepackage{float}
\usepackage{algorithm}
\usepackage{algorithmic}
\usepackage{subfigure}

\usepackage{xcolor}
\usepackage{ctable}
\usepackage{colortbl}

\usepackage{tabularx}
\newcolumntype{Y}{>{\centering\arraybackslash}X}

\newcommand*{\rom}[1]{\expandafter\@slowromancap\romannumeral #1@}

\newcommand{\si}[0]{\mathbf{x_i}}

\ShortHeadings{A Spectral Approach for the Design of Experiments: Design, Analysis and Algorithms}{anonymous}
\firstpageno{1}

\begin{document}

\title{A Spectral Approach for the Design of Experiments: Design, Analysis and Algorithms}

\author{\name Bhavya Kailkhura \email kailkhura1@llnl.gov \\
       \addr Center for Applied Scientific Computing\\
       Lawrence Livermore National Lab\\
       Livermore, CA 94550, USA
       \AND
       \name Jayaraman J.\ Thiagarajan \email jjayaram@llnl.gov  \\
       \addr Center for Applied Scientific Computing\\
       Lawrence Livermore National Lab\\
       Livermore, CA 94550, USA
       \AND
       \name Charvi\ Rastogi \email charvirastogi@iitb.ac.in \\
       \addr Department of EECS\\
       Indian Institute of Technology\\
       Bombay, MH 400076, India
       \AND
       \name Pramod K.\ Varshney \email varshney@syr.edu \\
       \addr Departments of EECS\\
       Syracuse University\\
       Syracuse, NY 13244, USA
       \AND
       \name Peer-Timo \ Bremer \email bremer5@llnl.gov \\
       \addr Center for Applied Scientific Computing\\
       Lawrence Livermore National Lab\\
       Livermore, CA 94550, USA}


\editor{TBD}

\maketitle

\begin{abstract}%
This paper proposes a new approach to construct high quality space-filling sample designs. First, we propose a novel technique to quantify the space-filling property and optimally trade-off uniformity and randomness in sample designs in arbitrary dimensions. 
Second, we connect the proposed metric (defined in the spatial domain) to the objective measure of the design performance (defined in the spectral domain). 
This connection serves as an analytic framework for evaluating the qualitative properties of space-filling designs in general. 
Using the theoretical insights provided by this spatial-spectral analysis, we derive the notion of optimal space-filling designs, which we refer to as space-filling spectral designs. 
Third, we propose an efficient estimator to evaluate the space-filling properties of sample designs in arbitrary dimensions and use it to develop an optimization framework to generate high quality space-filling designs. 
Finally, we carry out a detailed performance comparison on two different applications in $2$ to $6$ dimensions: a) image reconstruction and b) surrogate
modeling on several benchmark optimization functions and an inertial confinement fusion (ICF) simulation code.
We demonstrate that the propose spectral designs significantly outperform existing approaches especially in high dimensions.
\end{abstract}

\section{Introduction}
\label{intro}
\input{intro}

\section{Related Work}
\label{sec:realted}
\input{related}

\section{A Metric for Assessing Space-filling Property}
\label{sec:metric}
\input{metric}


\section{Space-filling Spectral Designs}
\label{sfsd}
\input{pds}

\section{Qualitative Analysis of Space-filling Spectral Designs}
\label{pds_analysis_sec}
\input{pds_analysis}

\section{Optimization of Step PCF based Space-filling Spectral Designs}

\input{opt_pds}

\section{Space-filling Spectral Designs with Improved Coverage}
\input{stair}

\section{Synthesis of Space-filling Spectral Designs}
\input{synthesis}

\section{Experiments}
\input{exp}


\section{Conclusion and Future Directions}
\input{conclusion}

\acks{Acknowledgments}
This work was performed under the auspices of the U.S. Department of Energy by Lawrence Livermore National Laboratory under Contract DE-AC52-07NA27344. LLNL-JRNL-743060

\appendix
\section{Benchmark Optimization Functions}\label{apd:first}
\input{appendix}

\bibliography{references}

\end{document}

%% file: intro.tex
Exploratory analysis and inference in high dimensional parameter
spaces is a ubiquitous problem in science and engineering. As a
result, a wide-variety of machine learning tools and optimization
techniques have been proposed to address this challenge. In its most
generic formulation, one is interested in analyzing a high-dimensional
function $f: \mathcal{D} \rightarrow \mathbb{R}$ defined on the
$d$-dimensional domain $\mathcal{D}$. A typical approach for such an
analysis is to first create an initial sampling
$\mathcal{X}=\{\mathbf{x_i}\in \mathcal{D}\}_{i=1}^N$ of
$\mathcal{D}$, evaluate $f$ at all $\mathbf{x_i}$, and perform
subsequent analysis and learning using only the resulting tuples
$\{(\mathbf{x_i},f(\mathbf{x_i}))\}_{i=1}^N$. Despite the widespread
use of this approach, a critical question that still persists is: how
should one obtain a high quality initial sampling $\mathcal{X}$ for
which the data $f(\mathcal{X})$ is collected? This challenge is
typically refered to as Design of Experiments (DoE) and solutions have
been proposed as early as ~\citep{fisher} to optimize agricultural
experiments. Subsequently, DoE has received significant attention from
researchers in different fields~\citep{doe:review} as it is an
important \textit{building block} for a wide variety of applications,
such as surrogate modeling, image reconstruction, reinforcement
learning, or data analysis. In several scenarios, it has been shown
that success crucially depends on the quality of the initial sampling
$\mathcal{X}$. Currently, a plethora of sampling solutions exist in
the literature with a wide-range of assumptions and statistical
guarantees; see~\citep{doe:review, Sampling:survey} for a detailed
review of related methods. Conceptually, most approaches aim to cover
the sampling domain as uniformly as possible, in order to generate the
so called \textit{space-filling} experimental designs. However, it is
well know that uniformity alone does not necessarily lead to high
performance. For example, optimal sphere packings lead to highly
uniform designs, yet are well known to cause strong aliasing artifacts
most easily perceived in many computer graphics applications. Instead,
a common assumption is that a good design should balance uniformity
and randomness. Unfortunately, an exact definition for what should be
considered a good space-filling design remains elusive.


Most common approaches use various scalar metrics to encapsulate
different notions of ideal sampling properties. One popular metric is
the discrepancy of an experimental design, defined as an appropriate
$\ell_p$ norm of the ratio of points within all (hyper-rectangular)
sub-volumes of $\mathcal{D}$ and the corresponding volume ratio. In
other words, discrepancy quantifies the non-uniformity of a sample
design. The most prominent examples of so called {\it discrepancy
  sequences} are Quasi-Monte Carlo (QMC) methods
and their variants~\citep{caflisch_1998}. In their classical form, discrepancy
sequences are deterministic though extensions to incorporate randomess
have been proposed, for example, using digital
scrambling~\citep{Owen:1995}. Nevertheless, by optimizing for
discrepancy these techniques focus almost exclusively on uniformity,
and consequently even optimized QMC patterns can be quite structured and create
aliasing artifacts. Furthermore, even the fastest known strategies for
evaluating popular discrepancy measures require $O(N^2d)$ operations
making evaluation, let alone optimization, for discrepancy difficult even
for moderate dimensions. Finally, for most discrepancy measures, the
optimal achievable values are not known. This makes it difficult to
determine whether a poorly performing sample design is due to the
insufficiency of the chosen discrepancy measure or due to ineffective optimization.


Another class of metrics to describe sample designs are based on
geometric distances. These can be used directly by, for example,
optimizing the maximin or minimax distance of a sample
design~\citep{Schlomer:2011} or indirectly by enforcing empty disk
conditions. The latter is the basis for the so-called Poisson disk
samples~\citep{Lagae:2008}, which aim to generate random points such
that no two samples can be closer than a given minimal distance
$r_{min}$, i.e.\ enforcing an empty disc of radius $r_{min}$ around
each sample. Typically, Poisson-type samples are characterized by the
{\it relative radius}, $\rho$, defined as the ratio of the minimum
disk radius $r_{min}$ and the maximum possible disk radius $r_{max}$
for $N$ samples to cover the sampling domain. Similar to the
discrepancy sequences, maximin and minimax designs exclusively
consider uniformity, are difficult to optimize for especially in
higher dimensions, and often lead to very regular patterns. Poisson
disk samples use $\rho$ to trade off randomness (lower $\rho$ values)
and uniformity (higher $\rho$ values). A popular recommendation in
$2$-d is to aim for $0.65 \le \rho \le 0.85$ as a good
compromise. However, there does not exist any theoretical guidance for choosing
$\rho$ and hence, optimal values for higher dimensions are not
known. As discussed in more detail in Section~\ref{sec:realted}, there
also exist a wide variety of techniques that combine different metrics
and heuristics. For example, Latin Hypercube sampling (LHS) aims to
spread the sample points uniformly by stratification, and one can
potentially optimize the resulting design using maximin or minimax
techniques~\citep{lhs:dist}.

In general, scalar metrics to evaluate the quality of a sample design
tend not to be very descriptive. Especially in high dimensions
different designs with, for example, the same $\rho$ can exhibit
widely different performance and for some discrepancy sequences the
optimal designs converge to random samples in high dimensions~\citep{qmc:ana,lds:hd}.
Furthermore, one rarely knows the best achievable value of the metric, i.e.\ the lowest
possible discrepancy, for a given problem which makes evaluating and
comparing sampling designs difficult. Finally, most metrics are expensive to
compute and not easily optimized. This makes it challenging in
practice to create good designs in high dimensions and with large sample sizes. 

To alleviate this problem, we propose a new technique to quantify the
space-filling property, which enables us to systematically trade-off
uniformity and randomness, consequently producing better quality
sampling designs. More specifically, we use tools from statistical
mechanics to connect the qualitative performance (in the spectral
domain) of a sampling pattern with its spatial properties
characterized by the pair correlation function (PCF). The PCF measures
the distribution of point samples as a function of distances, thus,
providing a holistic view of the space-filling property (See
Figure~\ref{pcfpoisson}). Furthermore, we establish the connection
between the PCF and the power spectral density (PSD) via the $1-$D
Hankel transform in arbitrary dimensions, thus providing a relation
between the PCF and an objective measure of sampling quality to help
subsequent design and analysis.

Using insights from the analysis of space-filling designs in the spectral
domain, we provide design guidelines to systematically trade-off
uniformity and randomness for a good sampling pattern. The analytical
tractability of the PCF enables us to perform theoretical analysis in the
spectral domain to derive the structure of optimal space-filling
designs, referred to as \textit{space-filling spectral design} in the
rest of this paper.  Next, we develop an edge corrected kernel density
estimator based technique to measure the space-filling property via
PCFs in arbitrary dimensions. In contrast to existing PCF estimation
techniques, the proposed PCF estimator is both accurate and
computationally efficient. Based on this estimator, we
develop a systematic optimization framework and a novel algorithm to
synthesize space-filling spectral designs. In particular, we propose to
employ a weighted least-squares based gradient descent optimization,
coupled with the proposed PCF estimator, to accurately match the
optimal space-filling spectral design defined in terms of the PCF.

Note that, there is a strong connection between the proposed
space-filling spectral designs and coverage based designs such as
Poisson Disk Sampling (PDS)~\citep{Gamito:2009}. However, the major
difference lies in the metric/criterion these techniques use to
estimate and optimize the space-filling designs. Furthermore, existing
works on PDS focus primarily on algorithmic issues, such as worst-case
running times and numerical issues associated with providing
high-quality implementations. However, different PDS methods often
demonstrate widely different performances which raises the questions
of how to evaluate the qualitative properties of different PDS
patterns and how to define an optimal PDS pattern? We argue that,
coverage $(\rho)$ based metrics alone are insufficient for
understanding the statistical aspects of PDS. This makes it difficult
to generate high quality PDS patterns. As we will demonstrate below,
existing PDS approaches largely ignore the randomness objective and
instead concentrate exclusively on the coverage objective resulting in
inferior sampling patterns compared to space-filling spectral designs,
especially in high dimensions. Note that, on the other hand, the
proposed PCF based metric does not have these limitations and enables
a comprehensive analysis of statistical properties of space-filling
designs (including PDS), while producing higher quality sampling
patterns compared to the state-of-the-art PDS approaches.

In~\citep{Kailkhura:2016}, we use the PCF to understand the nature of PDS
and provided theoretical bounds on the sample size of achievable PDS. Here we
significantly extend our previous work and provide a more
comprehensive analysis of the problem along with a novel space-filling
spectral designs, an edge corrected PCF estimator, an optimization approach
to synthesize the space-filling spectral designs and a detailed
evaluation of the performance of the proposed sample design.  The main
contributions of this paper can be summarized as follows:

\begin{itemize}
	\item We provide a novel technique to quantify the space-filling property of sample designs in arbitrary dimensions and systematically trade-off uniformity and randomness. 
	\item We use tools from statistical mechanics to connect the qualitative performance (in the spectral domain) of a sample design  with its spatial properties characterized by the PCF. 
	\item We develop a computationally efficient edge corrected kernel density estimator based technique to estimate the space-filling property in arbitrary dimensions.
	\item Using theoretical insights obtained via spectral analysis of point distributions, we provide design guidelines for optimal space-filling designs.
	\item We devise a systematic optimization framework and a gradient descent optimization algorithm to generate high quality space-filling designs.
	\item We demonstrate the superiority of proposed space-filling spectral samples compared to existing  space-filling approaches
	through rigorous empirical studies on two
	different applications: $a)$ image
	reconstruction and $b)$ surrogate modeling on several benchmark optimization
	functions and an inertial confinement fusion (ICF) simulation code. 
\end{itemize}

%% file: related.tex
\begin{figure}[t]
	\centering
	\subfigure[]{
		\includegraphics[%
		width=0.35\textwidth,clip=true]{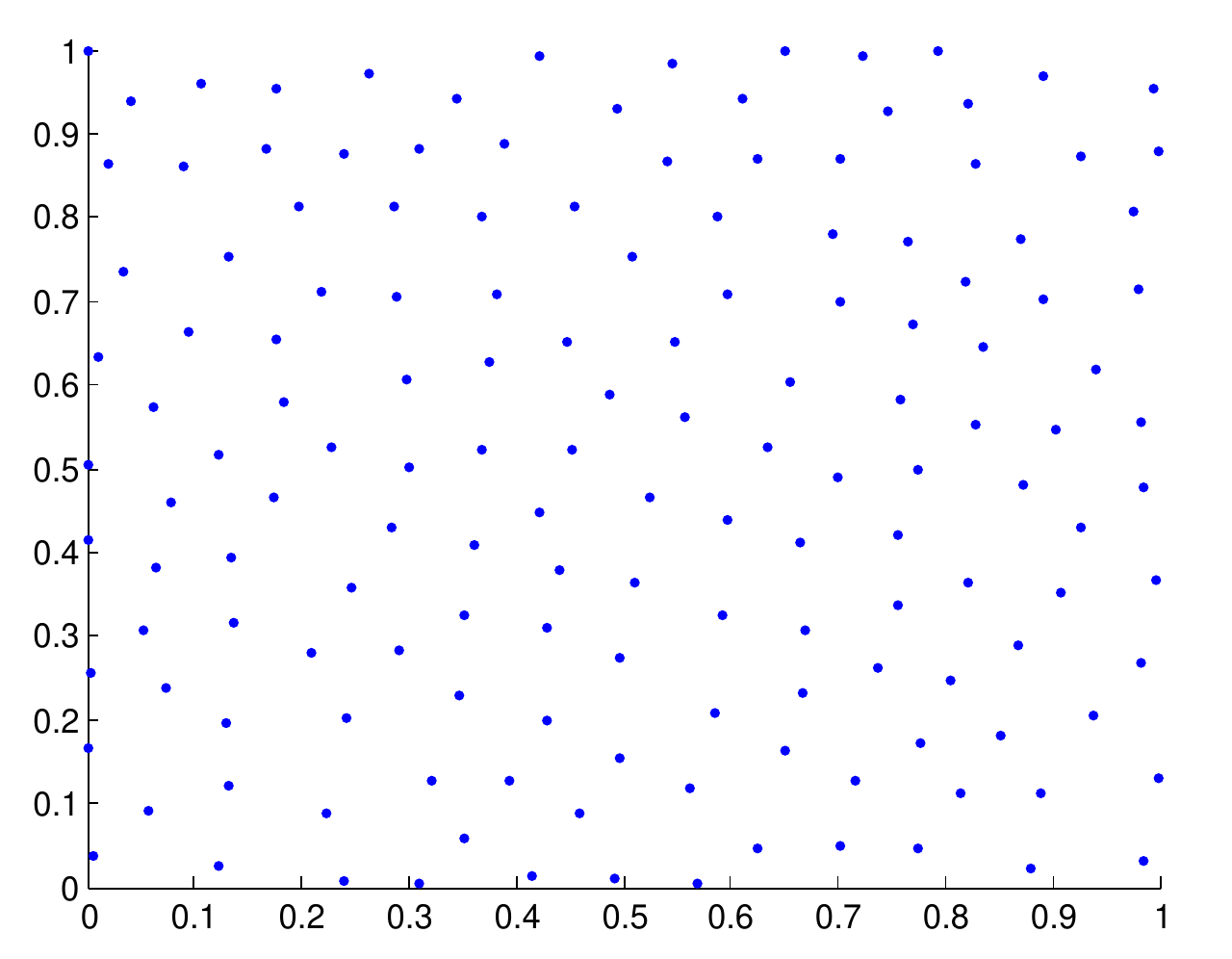}
		\label{point}}
	\hspace{0.1in}
	\subfigure[] {
		\includegraphics[%
		width=0.45\textwidth,clip=true]{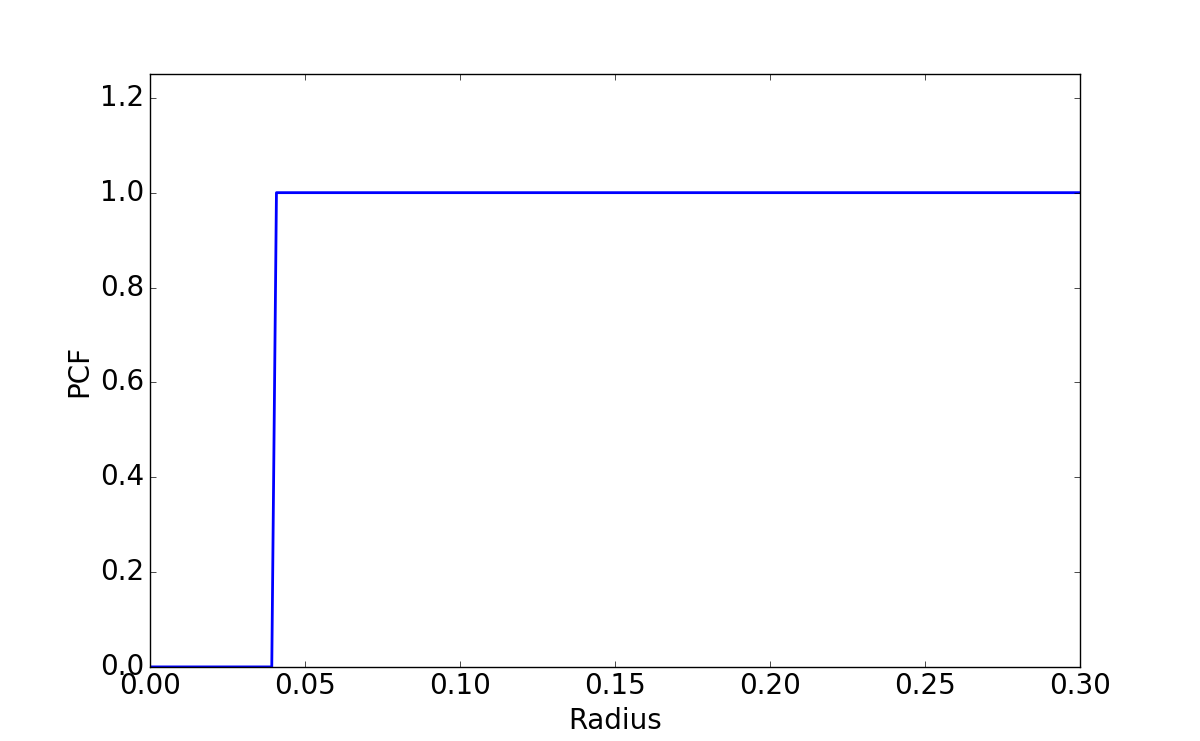}
		\label{pcfpoisson} }
	\caption{A sample design that balances \textit{randomness} and \textit{uniformity}. \subref{point} Point distribution, and \subref{pcfpoisson} Pair correlation function.}
	\label{PDSfig}
\end{figure}

In this section, we provide a brief overview of existing approaches for creating space-filling sampling patterns. Note that, the prior art for this long-studied research area is too extensive, and hence we recommend interested readers to refer to \citep{doe:review, Sampling:survey} for a more comprehensive review.

\subsection{Latin Hypercube Sampling}
Monte-Carlo methods form an important class of techniques for
space-filling sample design. However, it is well known that
Monte-Carlo methods are characterized by high variance in the
resulting sample distributions. Consequently, variance reduction
methods are employed in practice to improve the performance of simple
Monte Carlo techniques. One example is stratified sampling using the
popular Latin Hypercube Sampling (LHS)~\citep{McKay:1992,lhs:var}.
Since its inception, several variants of LHS have been proposed with
the goal of achieving better space-filling properties, in addition to
reducing variance. A notable improvement in this regard are techniques
that achieve LHS space filling not only in one-dimensional
projections, but also in higher dimensions. For example,
Tang~\citep{oa1:lhs,oa2:lhs} introduced orthogonal-array-based Latin
hypercube sampling to improve space-filling in higher dimensional
subspaces. Furthermore, a variety of space-filling criteria such as
entropy, integrated mean square error, minimax and maximin distances,
have been utilized for optimizing LHS~\citep{lhs:dist}. A particularly
effective and widely adopted metric is the maximin distance criterion,
which maximizes the minimal distance between points to avoid designs
with points too close to one another~\citep{lhs:maximin}. A detailed
study on LHS and its variants can be found in~\citep{owen:book}.

\subsection{Quasi Monte Carlo Sampling}
Following the success of Monte-Carlo methods, Quasi-Monte Carlo (QMC)
sampling was introduced in ~\citep{Halton:1964} and since then has
become the de facto solution in a wide-range of
applications~\citep{caflisch_1998, lds:hd}. The core idea of QMC
methods is to replace the random or pseudo-random samples in
Monte-Carlo methods with well-chosen deterministic points. These
deterministic points are chosen such that they are highly uniform,
which can be quantified using the measure of discrepancy.
Low-discrepancy sequences along with bounds on their discrepancy were
introduced in the 1960's by Halton~\citep{Halton:1964} and
Sobol~\citep{Sobol:1967}, and are still in use today. However, despite their
effectiveness, a critical limitation of QMC methods is that error
bounds and statistical confidence bounds of the resulting designs
cannot be obtained due to the deterministic nature of low-discrepancy
sequences. In order to alleviate this challenge, randomized
quasi-Monte Carlo (RQMC) sampling has been proposed~\citep{rqmc}, and in
many cases shown to be provably better than the classical QMC
techniques~\citep{owen2005quasi}. This has motivated the development of
other randomized quasi-Monte Carlo techniques, for example, methods based
on digital scrambling~\citep{Owen:1995}.

\subsection{Poisson Disk Sampling}
While discrepancy-based designs have been popular among
uncertainty quantification researchers, the computer graphics
community has had long-standing success with coverage-based designs.
In particular, Poisson disk sampling (PDS) is widely used in applications such as
image/volume rendering. The authors in~\citep{Dippe:1985,Cook:1986}
were the first to introduce PDS for turning regular aliasing patterns
into featureless noise, which makes them perceptually less visible.
Their work was inspired by the seminal work of Yellott
\textit{et.al.}~\citep{Yellott:1983}, who observed that the
photo-receptors in the retina of monkeys and humans are distributed
according to a Poisson disk distribution, thus explaining its
effectiveness in imaging.

Due to the broad interest created by the initial work on PDS, a large
number of approaches to generate Poisson disk distributions have been
developed over the last decade
\citep{Gamito:2009,Ebeida:2012,Ebeida:2011,Amitabh,
  Bridson:2007,Oztireli:2012,Heck:2013,Wei:2008,Dunbar:2006,Wei:2010,Balzer:2009,Bo,Yan:2012,Yan:2013,Ying:2013,Ying2014,Hou2013,Ying2013,Guo,Wachtel:2014,Xu,
  Ebeida2014,deGoes:2012,Zhou:2012}. Most Poisson disk sample
generation methods are based on dart
throwing~\citep{Dippe:1985,Cook:1986}, which attempts to generate as
many darts as necessary to cover the sampling domain while not
violating the Poisson disk criterion. Given the desired disk size
$r_{min}$ (or coverage $\rho$), dart throwing generates random samples
and rejects or accepts each sample based on its distance to the
previously accepted samples. Despite its effectiveness, its primary
shortcoming is the choice of termination condition, since the
algorithm does not know whether or not the domain is fully covered.
Hence, in practice, the algorithm has poor convergence, which in turn
makes it computationally expensive. On the other hand, dart throwing
is easy to implement and applicable to any sampling domain, even
non-Euclidean. For example, Anirudh \textit{et.al.} use a dart
throwing technique to generate Poisson disk samples on the
Grassmannian manifold of low-dimensional linear
subspaces~\citep{Anirudh:2017}.

Reducing the computational complexity of PDS generation, particularly
in low and moderate dimensions, has been the central focus of many
existing efforts. To this end, approximate techniques that produce
sample sets with characteristics similar to Poisson disk have been
developed. Early examples~\citep{McCool:1992} are relatively simple
and can be used for a wide range of sampling domains, but the gain in
computational efficacy is limited. Other methods partition the space
into grid cells in order to allow parallelization across the different
cells and achieve linear time algorithms~\citep{Bridson:2007}. Another
class of approaches, referred to as \textit{tile-based} methods, have
been developed for generating a large number of Poisson disk samples
in $2$-D. Broadly, these methods either start with a smaller set of
samples, often obtained using other PDS techniques, and tile these
samples~\citep{Wachtel:2014}, or alternatively use a regular tile
structure for placing each sample~\citep{Ostromoukhov:2004}. With the
aid of efficient data structures, these methods can generate a large
number of samples efficiently. Unfortunately, these approximations can
lead to low sample quality due to artifacts induced at tile boundaries
and the inherent non-random nature of tilings. More recently, many
researchers have explored the idea of partitioning the sampling space
in order to avoid generating new samples that will be ultimately
rejected by dart throwing. While some of these methods only work in
$2-$D~\citep{Dunbar:2006, Ebeida:2011}, the efficiency of other
methods that are designed for higher
dimensions~\citep{Gamito:2009,Ebeida:2012} drops exponentially with
increasing dimensions. Finally, relaxation methods that iteratively
increase the Poisson disk radius of a sample set~\citep{McCool:1992}
by re-positioning the samples also exist. However, these methods have
the risk of converging to a regular pattern with tight packing unless
randomness is explicitly enforced~\citep{Balzer:2009,Schlomer:2011}.

A popular variant of PDS is the maximal PDS (MPDS) distribution, where
the \textit{maximality} constraint requires that the sample disks
overlap, in the sense that they cover the whole domain leaving no room
to insert an additional point. In practice, maximal PDS tends to
outperform traditional PDS due to better coverage. However,
algorithmically guaranteeing maximality requires expensive checks
causing the resulting algorithms to be slow in moderate ($2$-$5$) and
practically unfeasible in higher ($7$ and above) dimensions. Though
strategies to alleviate this limitation have been proposed
in~\citep{Ebeida:2012}, the inefficiency of MPDS algorithms in higher
dimensions still persists. Interestingly, a common limitation of all
existing MPDS approaches is that there is no direct control over the
number of samples produced by the algorithm, which makes the use of
these algorithms difficult in practice, since optimizing samples for a
given sample budget is the most common approach.

As discussed in Section~\ref{intro}, the metrics used by the
space-filling designs discussed above do not provide insights into how to
systematically trade-off uniformity and randomness. Thereby, making the
design and optimization of sampling pattern a cumbersome process. To
alleviate this problem, we propose a novel metric for assessing the
space-filling property and connect the proposed metric (defined in the
spatial domain) to the objective measure of design performance
(defined in the spectral domain).


%% file: metric.tex
\begin{figure*}[!htb]
\centering
\subfigure[Random]{
\includegraphics[trim=2.2cm 1.5cm 3cm 1.5cm,clip=true, width=0.22\linewidth]{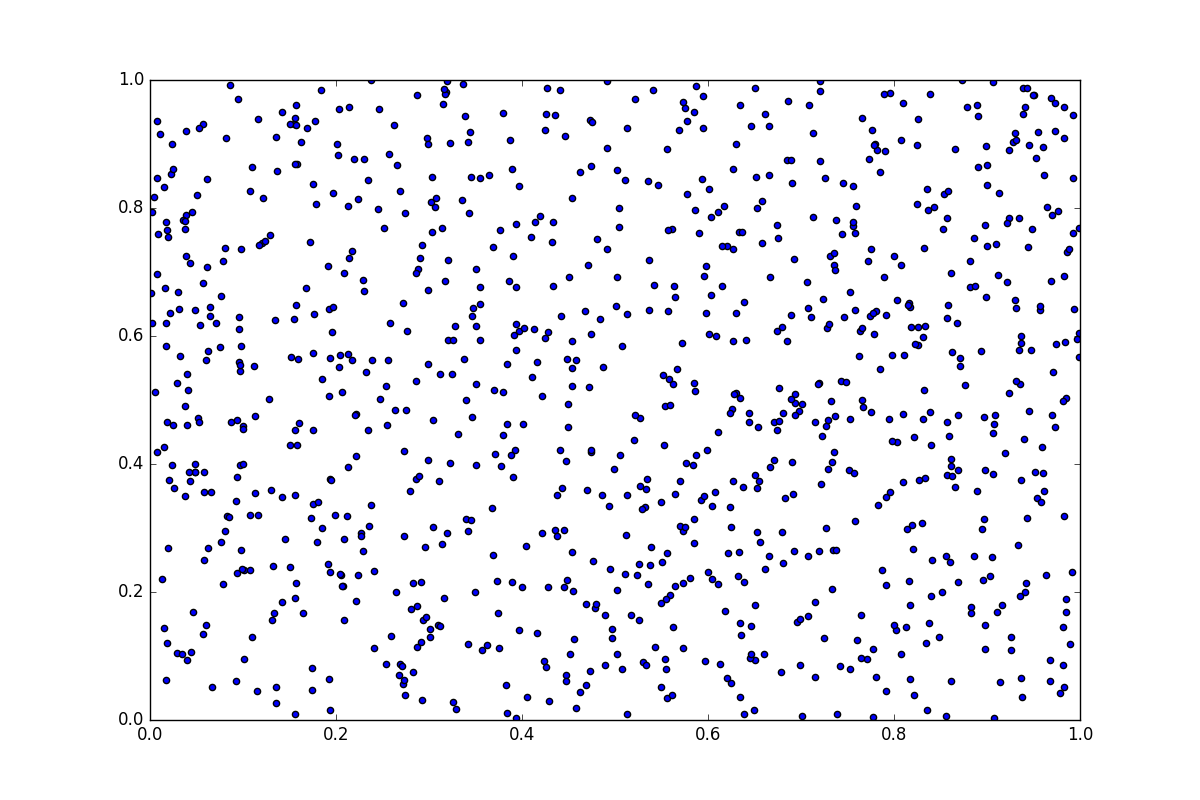}
}\subfigure[Regular]{
\includegraphics[trim=2.2cm 1.5cm 3cm 1.5cm,clip=true, width=0.22\linewidth]{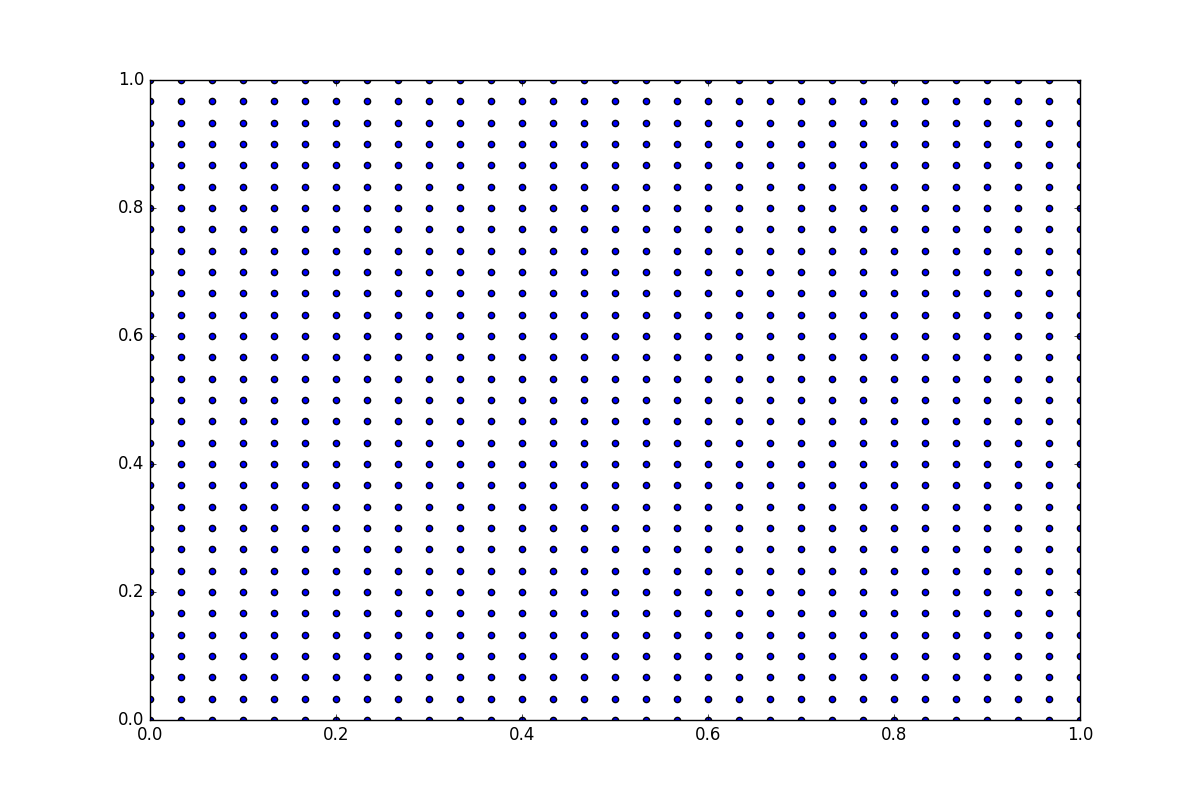}
}
\subfigure[Sobol]{
\includegraphics[trim=2.2cm 1.5cm 3cm 1.5cm,clip=true, width=0.22\linewidth]{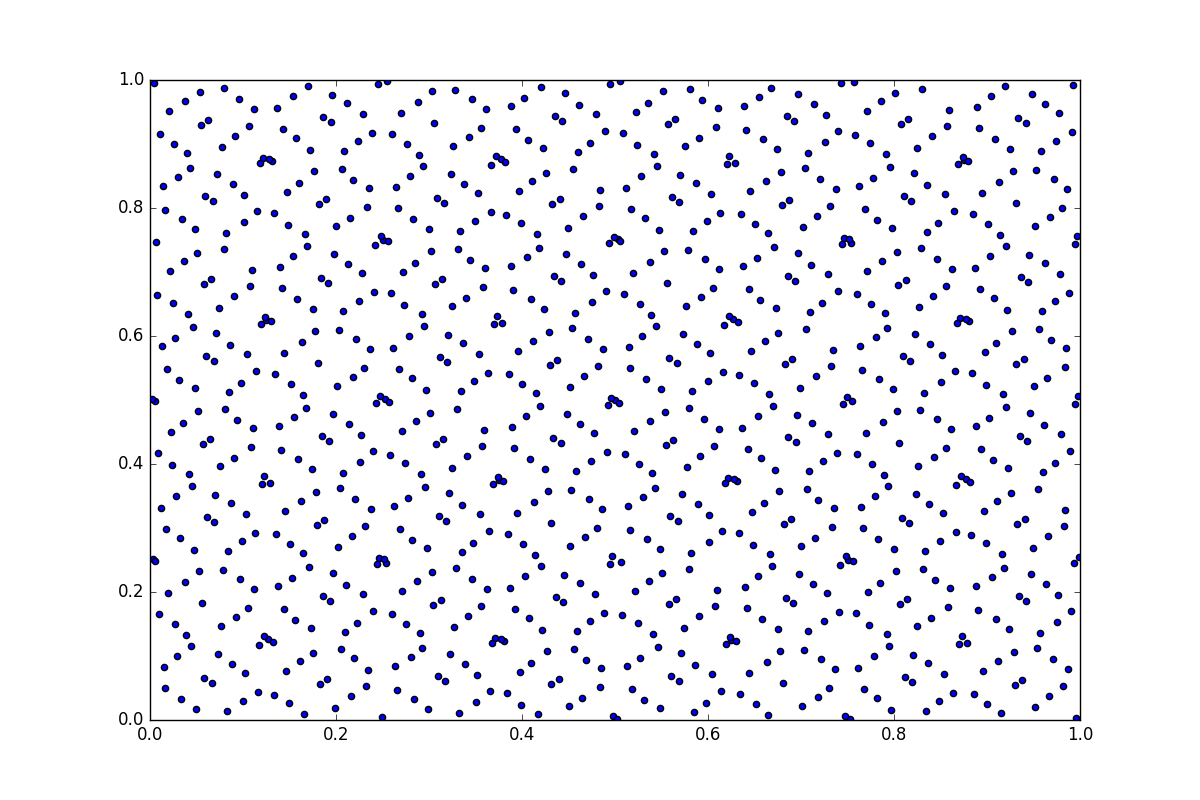}
   }
\subfigure[Halton]{
\includegraphics[clip=true, width=0.22\linewidth]{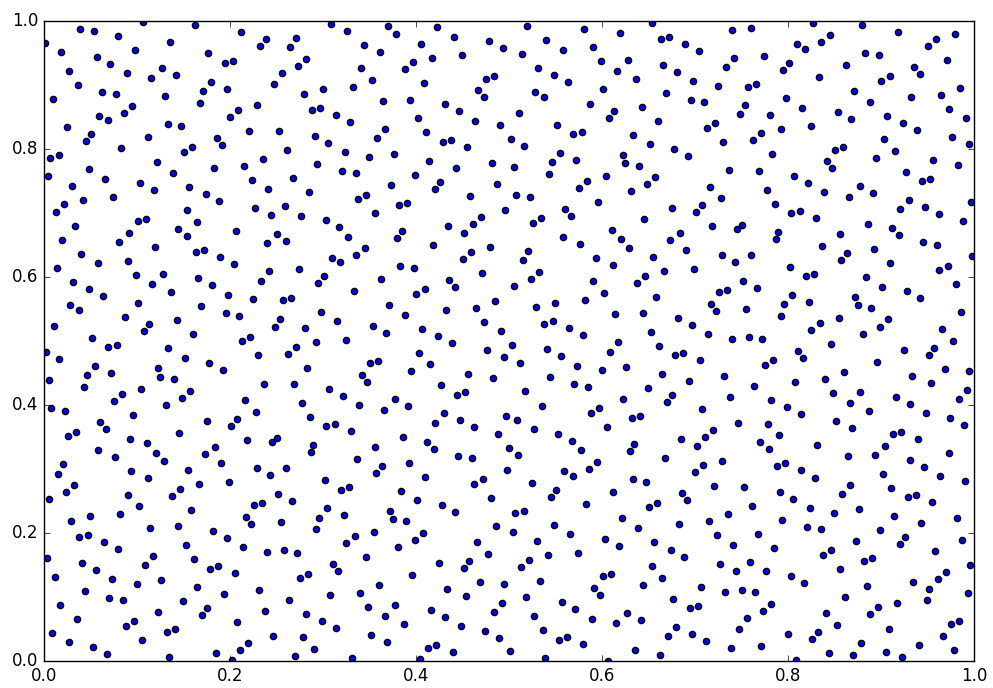}
}
\vfill
\subfigure[LHS]{
\includegraphics[trim=2.2cm 1.5cm 3cm 1.5cm,clip=true, width=0.22\linewidth]{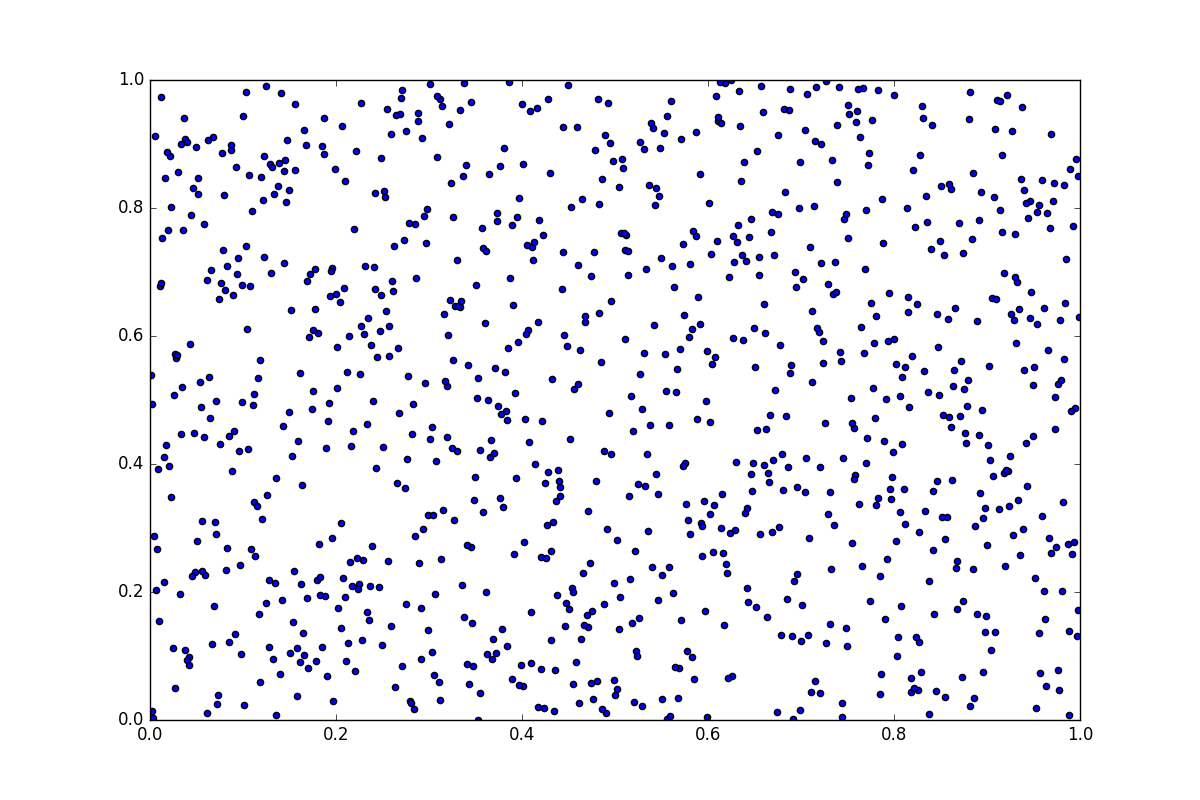}
 }
\subfigure[MPDS]{
\includegraphics[trim=2.2cm 1.5cm 3cm 1.5cm,clip=true, width=0.22\linewidth]{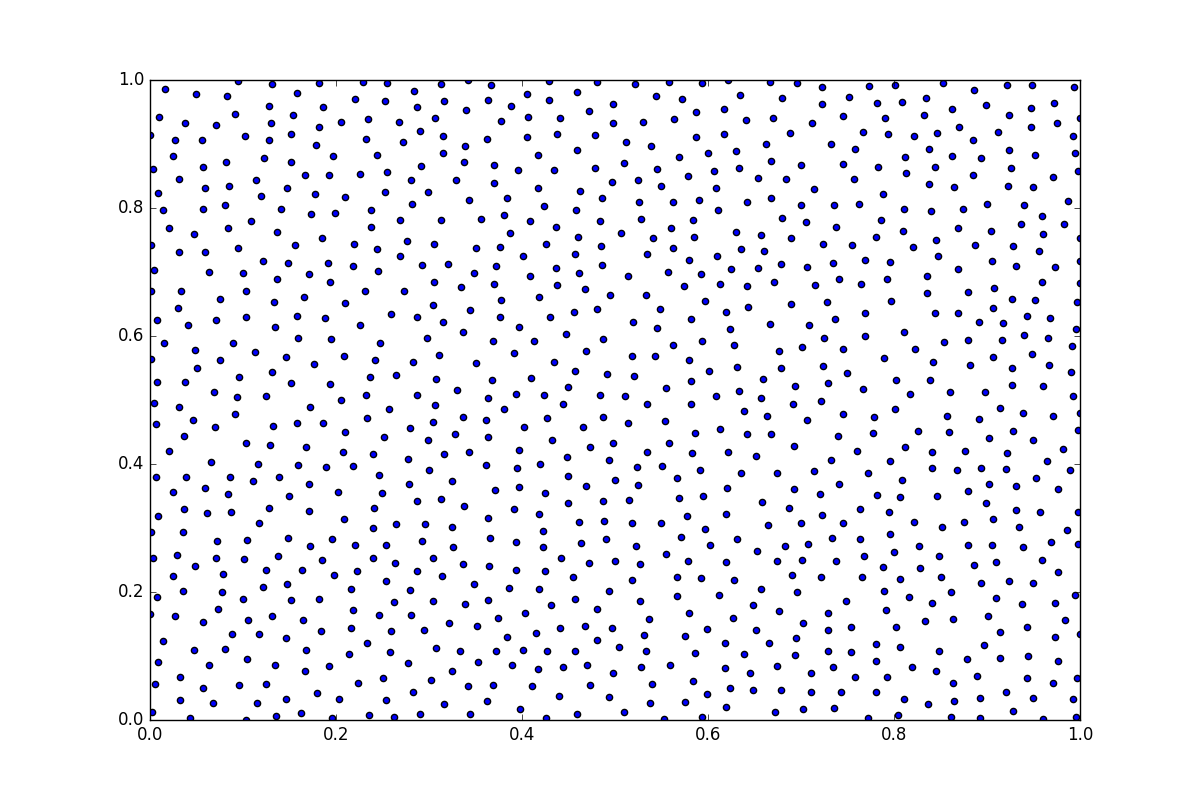}
}
\subfigure[Step PCF]{
\includegraphics[trim=2.2cm 1.5cm 3cm 1.5cm,clip=true, width=0.22\linewidth]{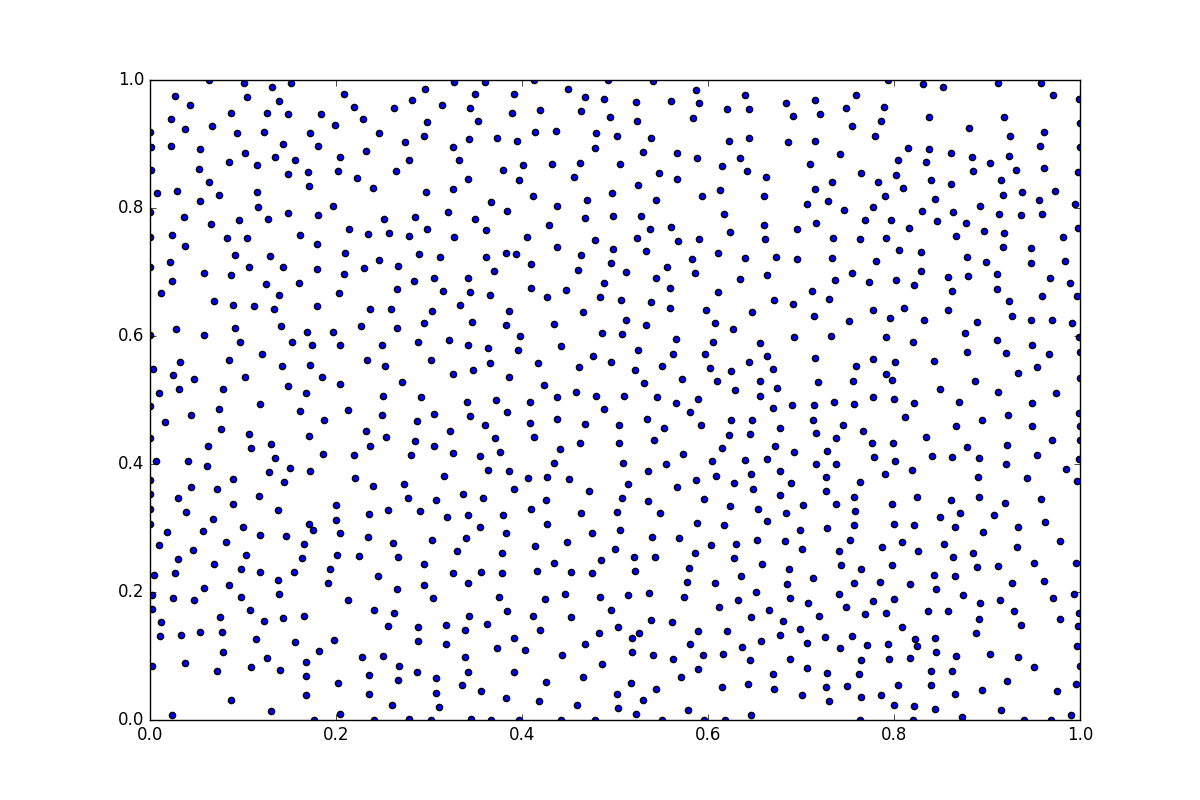}
 }
\subfigure[Stair PCF]{
\includegraphics[trim=2.2cm 1.5cm 3cm 1.5cm,clip=true, width=0.22\linewidth]{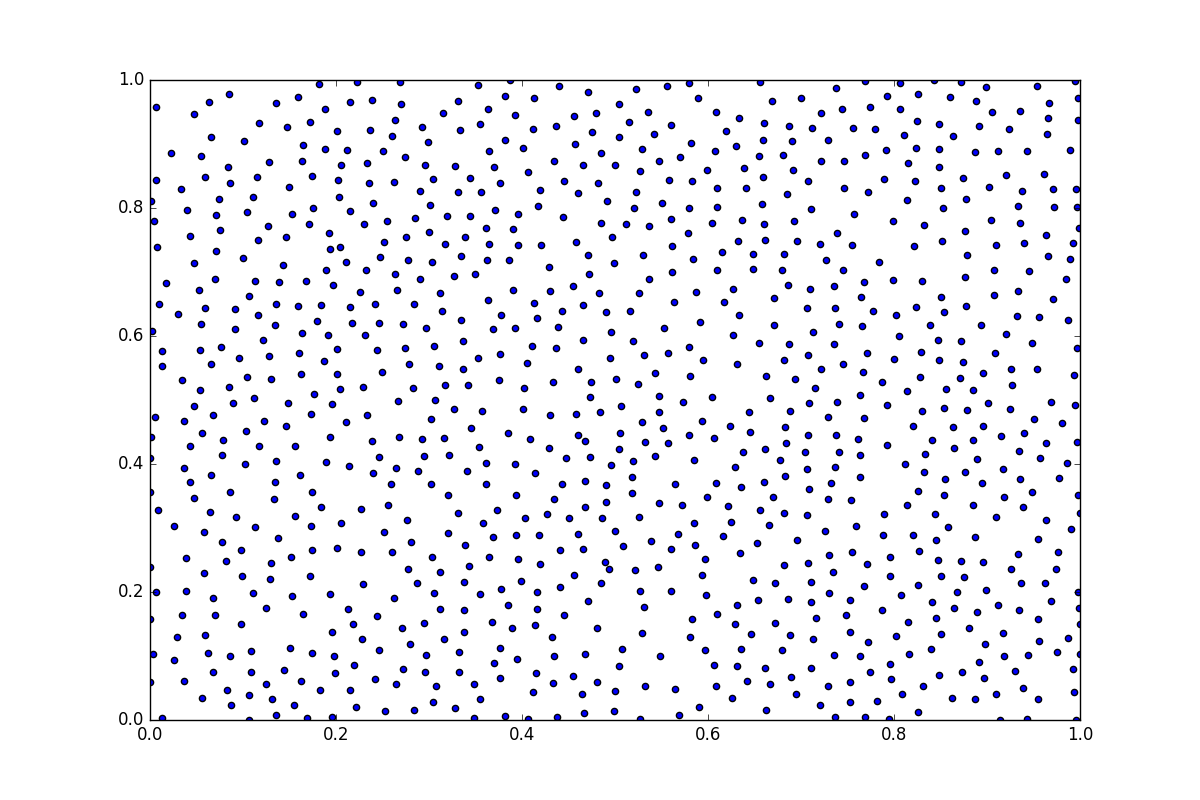}}
\caption{Visualization of $2-$d point distributions obtained using different sample design techniques. In all cases, the number of samples $N$ was fixed at $1000$.}
\label{point_samp}
\end{figure*}

\begin{figure*}[!htb]
\centering
\subfigure[Random]{
\includegraphics[trim=2.2cm 0.5cm 2.2cm 1.5cm,
  width=0.21\textwidth,clip=true]{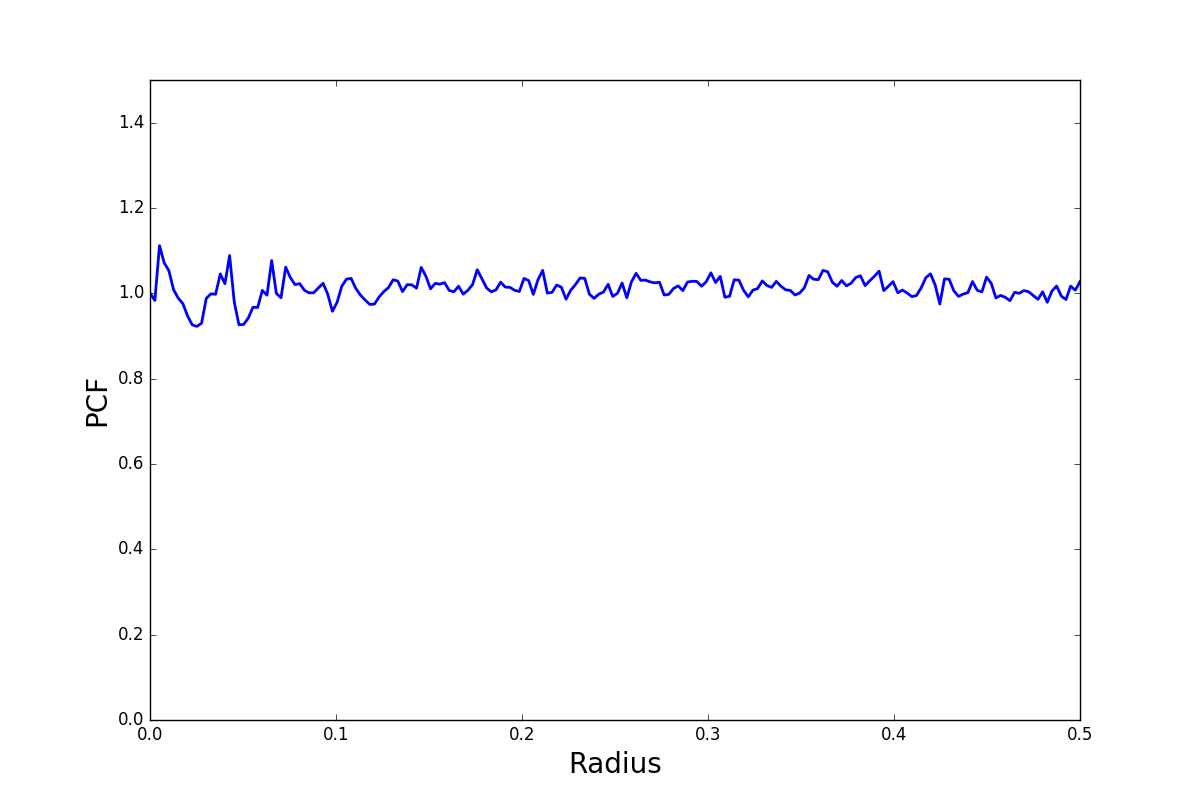}
}\subfigure[Regular]{
\includegraphics[trim=2.2cm 0.5cm 2.2cm 1.5cm,
  width=0.21\textwidth,clip=true]{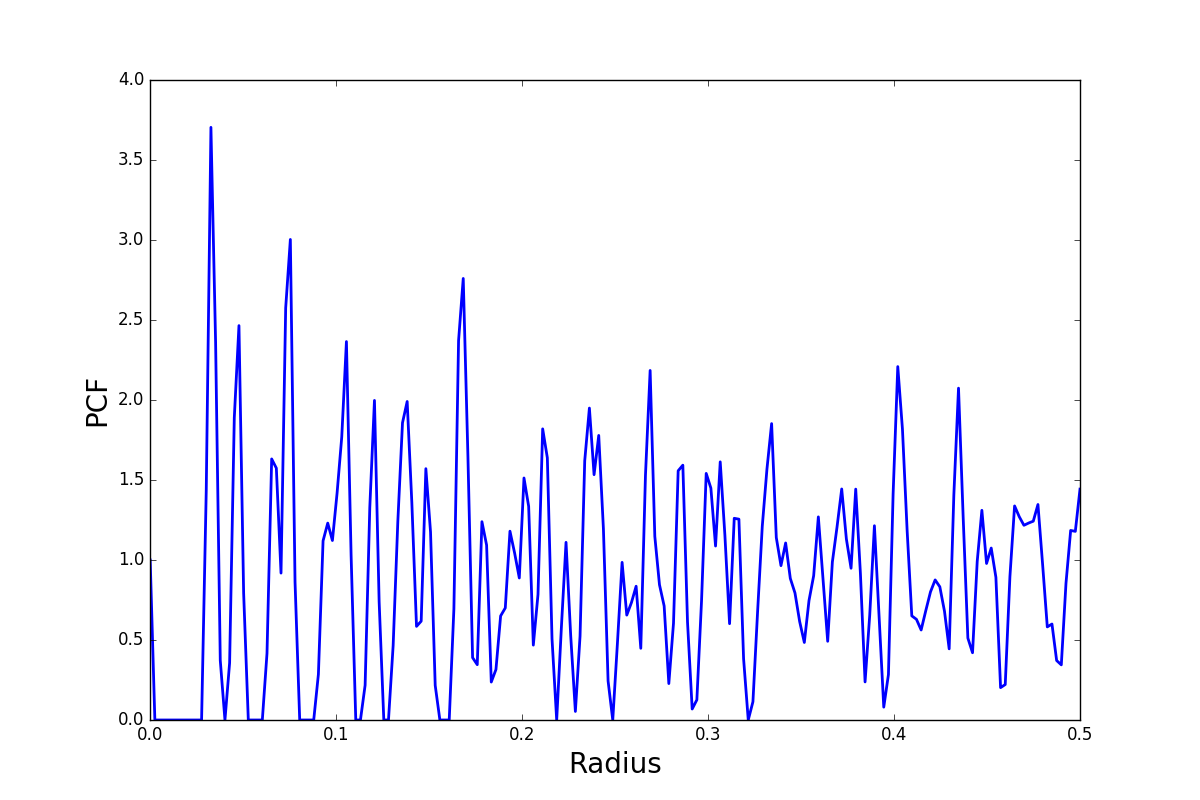}
}
\subfigure[Sobol]{
\includegraphics[
  width=0.22\textwidth,clip=true]{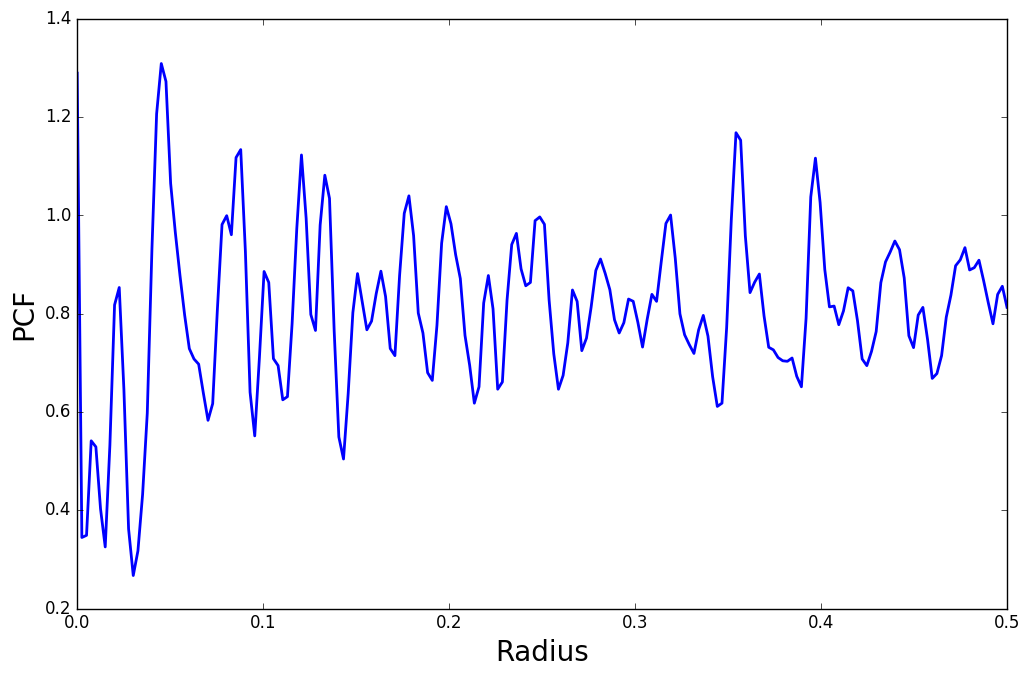}
   }
\subfigure[Halton]{
\includegraphics[
  width=0.22\textwidth,clip=true]{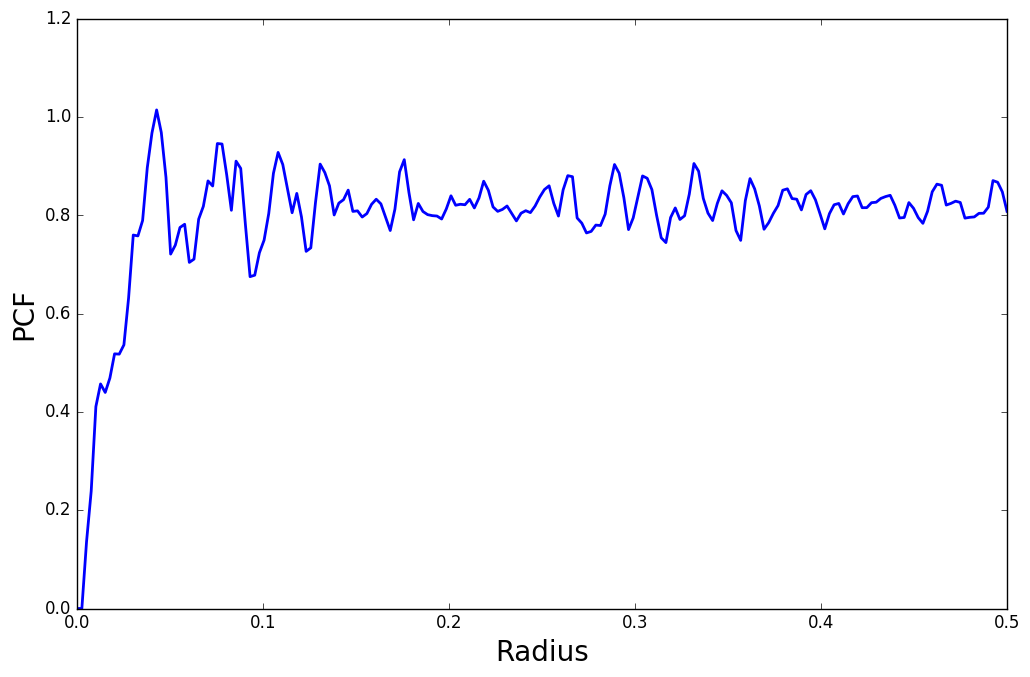}
}
\vfill
\subfigure[LHS]{
\includegraphics[
  width=0.22\textwidth,clip=true]{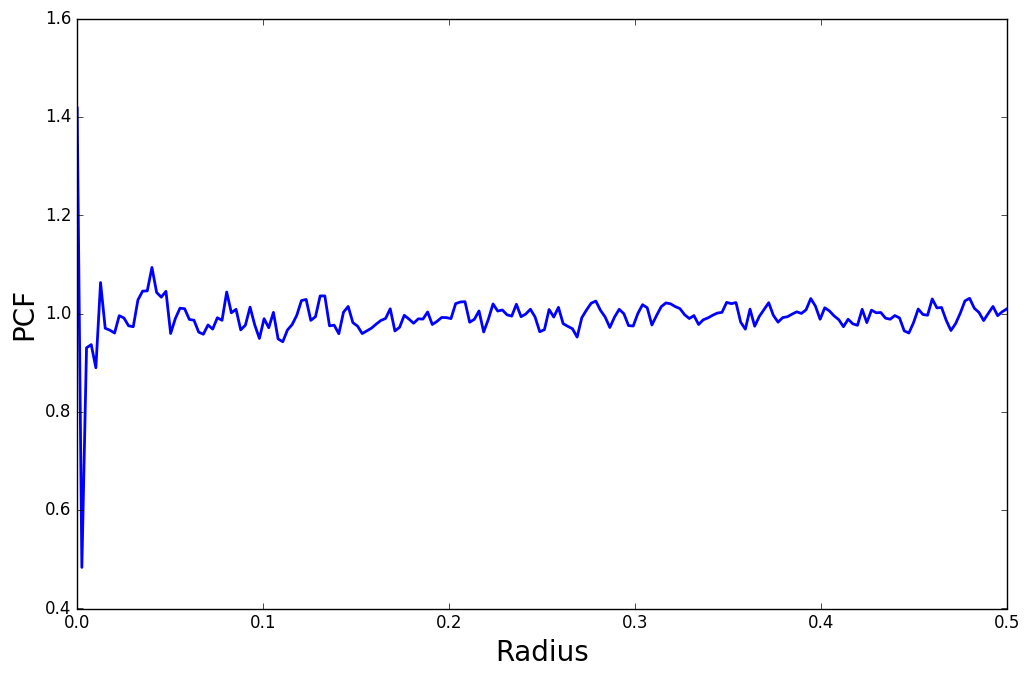}
 }
\subfigure[MPDS]{
\includegraphics[
  width=0.22\textwidth,clip=true]{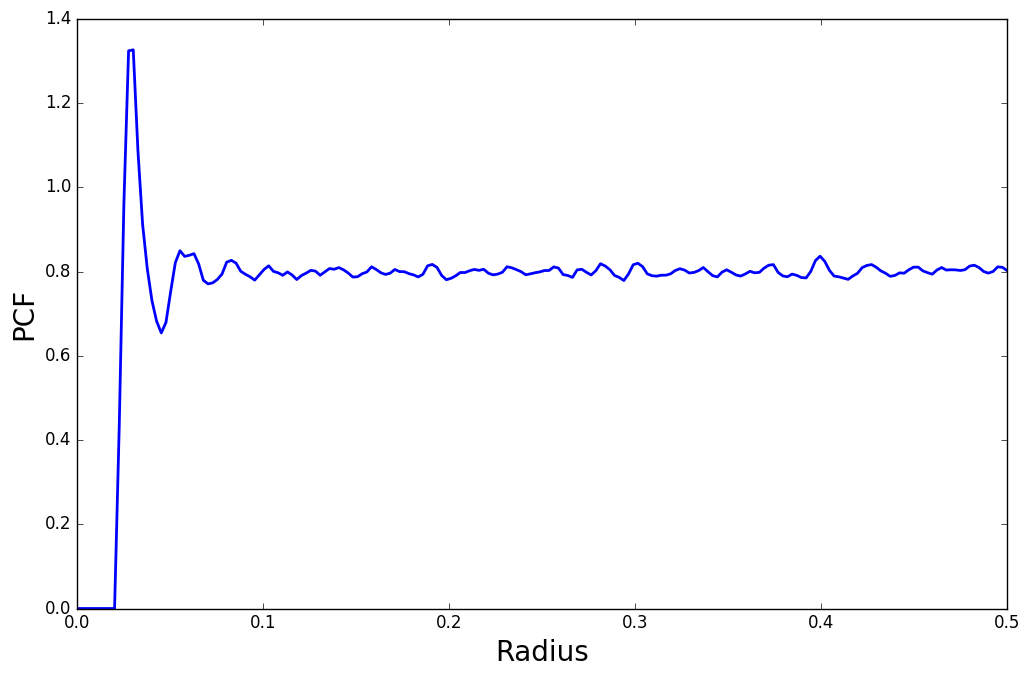}
}
\subfigure[Step PCF]{
\includegraphics[
  width=0.22\textwidth,clip=true]{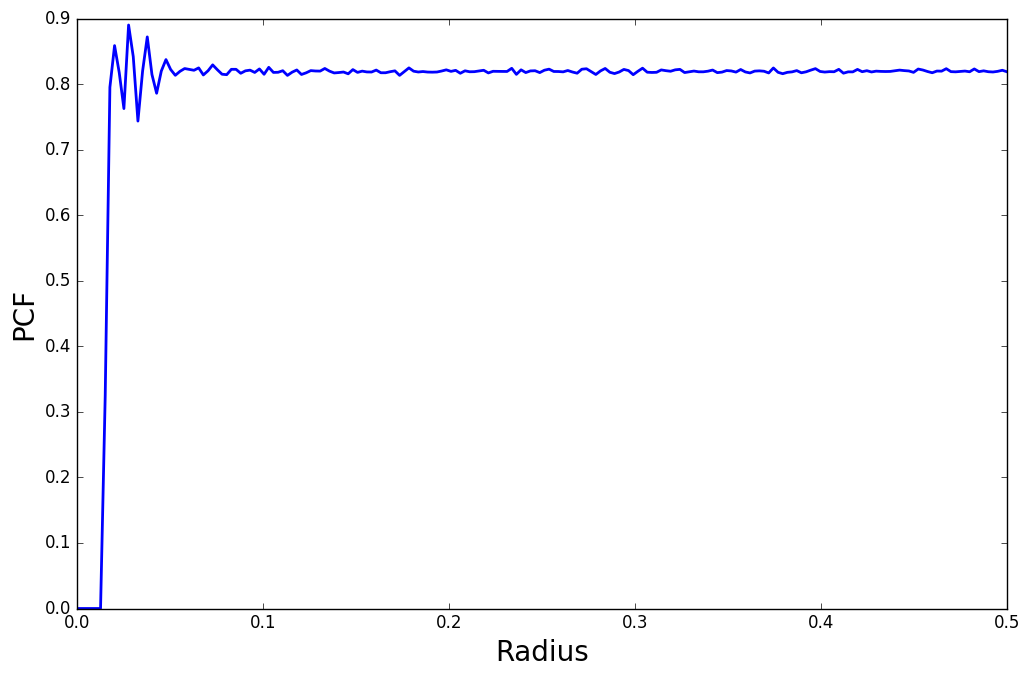}
 }
\subfigure[Stair PCF]{
\includegraphics[
  width=0.22\textwidth,clip=true]{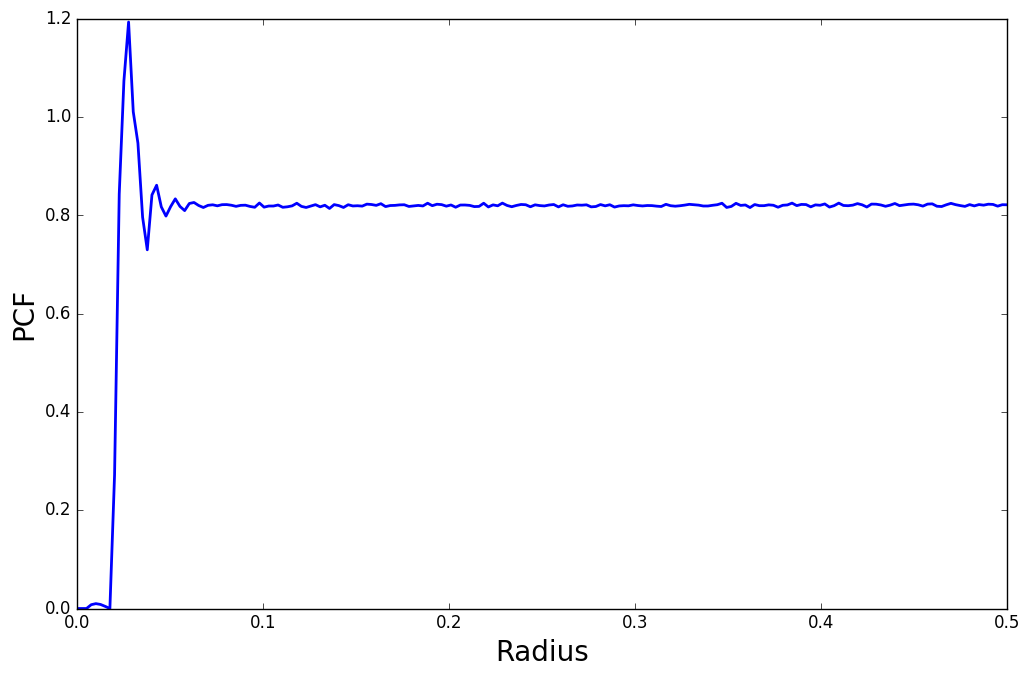}
 }
\caption{\textit{Space-filling Metric}: Pair correlation functions, corresponding to the samples in Figure \ref{point_samp}, characterize the coverage (and randomness) of point distributions obtained using different techniques.}
\label{pcf_samp}
\end{figure*}

\begin{figure*}[!htb]
\centering
\subfigure[Random]{
\includegraphics[%
  width=0.22\textwidth,clip=true]{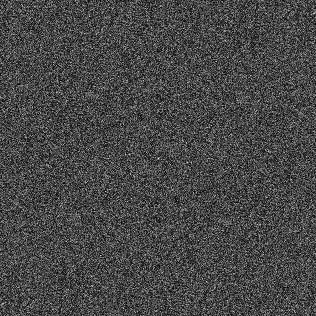}
}\subfigure[Regular]{
\includegraphics[%
  width=0.22\textwidth,clip=true]{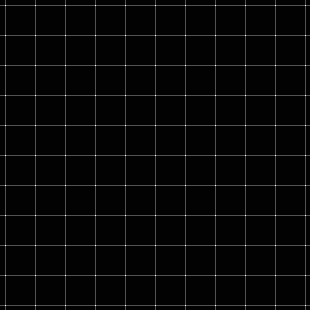}
}
\subfigure[Sobol]{
\includegraphics[%
  width=0.22\textwidth,clip=true]{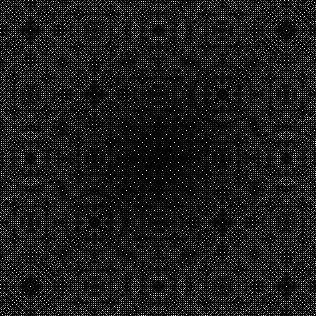}
   }
\subfigure[Halton]{
\includegraphics[%
  width=0.22\textwidth,clip=true]{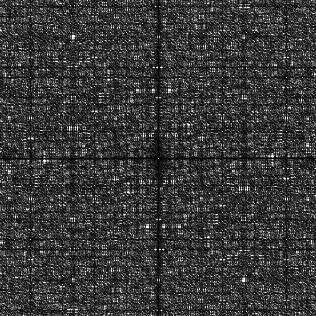}
}
\vfill
\subfigure[LHS]{
\includegraphics[%
  width=0.22\textwidth,clip=true]{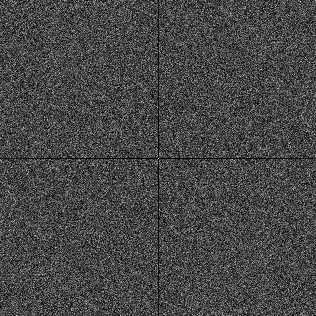}
 }
\subfigure[MPDS]{
\includegraphics[%
  width=0.22\textwidth,clip=true]{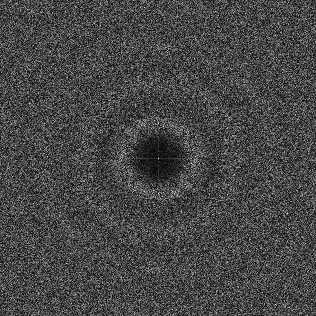}
}
\subfigure[Step PCF]{
\includegraphics[%
  width=0.22\textwidth,clip=true]{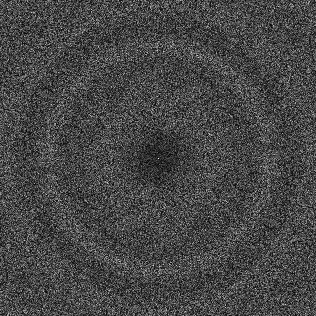}
 }
\subfigure[Stair PCF]{
\includegraphics[%
  width=0.22\textwidth,clip=true]{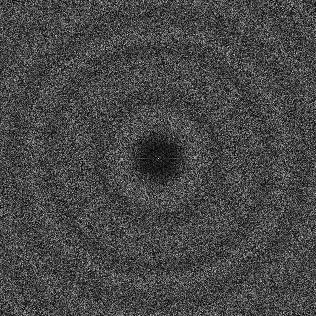}}
\caption{\textit{Performance Quality Metric}: Power spectral density is used to characterize the effectiveness of sample designs, through the distribution of power in different frequencies.}
\label{psd_samp}
\end{figure*}

In this section we first provide a definition of a space-filling
design. Subsequently, we propose a metric to quantify space-filling
properties of sample designs.

\subsection{ Space-filling Designs}
Without any prior knowledge of the function $f$ of interest, a reasonable
objective when creating $\mathcal{X}$ is that the samples should be
random to provide an equal chance of finding features of interest,
e.g., local minima in an optimization problem, anywhere in
$\mathcal{D}$. However, to avoid sampling only parts of the parameter
space, a second objective is to cover the space in $\mathcal{D}$
uniformly, in order to guarantee that all sufficiently large features
are found. Therefore, a good space-filling design can be defined as follows:

\begin{definition}
\label{sp-def}
	A space-filling design is a set of $\si$'s
	that are randomly distributed (Objective $1$: Randomness) but no two
	samples are closer than a given minimum distance $r_{min}$ (Objective
	$2$: Coverage).
\label{def:sp-def}
\end{definition}

Next, we describe the metric that we use to quantify the space-filling
property of a sample design.  The proposed metric is based on the
spatial statistic, \textit{pair correlation function }(PCF) and we
will show that this metric is directly linked to an objective measure of
design performance defined in the spectral domain.

\subsection{Pair Correlation Function as a Space-filling Metric}
In contrast to existing scalar space-filling metrics such as
discrepancy, and coverage, the PCF characterizes the distribution of
sample distances, thus providing a comprehensive description of the
sample designs. More specifically, PCF describes the joint probability
of having sampling points at a certain distance apart. A precise
definition of the PCF can be given in terms of the intensity $\lambda$
and product density $\beta$ of a point process~\citep{illian2008,
  Oztireli:2012}.

\begin{definition}
  Let us denote the intensity of a point process $\mathcal{X}$ as
  $\lambda(\mathcal{X})$, which is the average number of points in an
  infinitesimal volume around $\mathcal{X}$. For isotropic point
  processes, this is a constant value. To define the product density
  $\beta$, let $\{B_i\}$ denote the set of infinitesimal spheres
  around the points, and $\{dV_i\}$ indicate the volume measures of
  $B_i$. Then, we have
  $P(\mathbf{x_1}, \cdots, \mathbf{x_N}) = \beta(\mathbf{x_1}, \cdots,
  \mathbf{x_N})dV_1\cdots dV_N$.
  In the isotropic case, for a pair of points, $\beta$ depends only on
  the distance between the points, hence one can write
  $\beta(\mathbf{x_i},\mathbf{x_j})=\beta(||\mathbf{x_i}-\mathbf{x_j}||)=\beta(r)$
  and $P(r)=\beta(r)dxdy$. The PCF is then defined as
\begin{equation}
\label{pcf}
G(r)=\dfrac{\beta}{\lambda^2}. 
\end{equation}
\end{definition}
Note that, the PCF characterizes spatial properties of a
sample design. However, in several cases, it is easier to link
the objective performance of a sample design to its spectral
properties. Therefore, we establish a connection between the spatial
property of a sample design defined in PCF space to its spectral
properties.

\subsection{Connecting Spatial Properties and Spectral Properties of Space-filling Designs}
Fourier analysis is a standard approach for understanding the objective properties of sampling patterns. Hence, we propose to analyze the spectral properties of sample designs, using tools such as the power spectral density, in order to assess their quality. For isotropic samples, a quality metric of interest is the radially-averaged power spectral density, which describes how the signal power is distributed over different frequencies.

\begin{definition}
  For a finite set of $N$ points, $\{\mathbf{x}_j\}_{j=1}^N$, in a
  region with unit volume, the radially-averaged power spectral
  density (PSD) is formally defined as
	\begin{equation}
	P(k) = \frac{1}{N} |S(k)|^2 = \frac{1}{N} \sum_{j,\ell} e^{-2\pi i k(\mathbf{x}_{\ell} - \mathbf{x}_j)},  
	\end{equation}
	where $S(k)$ denotes the Fourier transform of the sample function.
\end{definition}Next, we show that the connection between spectral properties of a $d$-dimensional isotropic sample design and its corresponding pair correlation function can be obtained via the $d$-dimensional Fourier transform or more efficiently using the $1$-d Hankel transform. 

\begin{proposition}
\label{ht1}
For an isotropic sample design with $N$ points,
$\{\mathbf{x}_j\}_{j=1}^N$, in a $d$-dimensional region with unit
volume, the pair correlation function $G(r)$ and radially averaged
power spectral density $P(k)$ are related as follows:
\begin{equation}
G(r)= 1+\dfrac{V}{2\pi N} H\left[P({k})-1\right]
\end{equation}
where $V$ is the volume of the sampling region and $H[.]$ denotes the $1$-d Hankel transform, defined as
\begin{equation*}
H(f(k))=\int_{0}^{\infty}k J_0(kr)f(k) dk,
\end{equation*}
with $J_0(.)$ denoting the Bessel function of order zero.
\end{proposition}
\begin{proof}
Note that, PSD and PCF of a sample design are related via the $d$-dimensional Fourier transform as follows:
\begin{eqnarray*}
P(\mathbf{k})&=& 1+\frac{N}{V} F\left(G(\mathbf{r})-1\right)\\
&=& 1+\frac{N}{V} \int_{\mathbb{R}^d} \left(G(\mathbf{r})-1\right)\exp(-i\mathbf{k.r}) d\mathbf{r}.
\end{eqnarray*}
It can be shown that, for radially symmetric or isotropic functions, the above relationship simplifies to
\begin{equation*} 
P(k)=1+2\pi\frac{N}{V} H\left[G({r})\right]. 
\end{equation*}
Next, using the inverse property of the Hankel transform, i.e.,
\begin{equation*}
H_0^{-1}(f(r))=\int_{0}^{\infty}r J_0(kr)f(r) dr,
\end{equation*}
we have
\begin{equation}
G(r)= 1+\dfrac{V}{2\pi N} H\left[P({k})-1\right]. 
\end{equation}
\end{proof}
Proposition~\ref{ht1} is important as it enables us to qualitatively understand space-filling designs by first mapping them into the PCF space constructed based on spatial distances between points and, then, evaluating and understanding spectral properties of sample designs.

In Figure~\ref{pcf_samp}, we show the PCF\footnote{Note that, for
  non-isotropic sample designs, $d$-dimensional PCF~\citep{illian2008}
  can be more descriptive.} of some commonly used $2$-d sample designs
($N=1000$) illustrated in Figure~\ref{point_samp}. As can be
observed, both regular grid samples and QMC sequences have significant
oscillations in their PCFs, which can be attributed to their
structured nature. Regular grid sample design demonstrates a large
disk radius $r_{min}$ ($G(r)=0$ for $0 \leq r \leq r_{min}$) as every
sample is at least $r_{min}$ apart from the rest of the samples, which in
turn implies a better coverage. However, in practice, they perform
poorly compared to randomized sample designs and this can be
understood by studying their spectral properties. In contrast, random
sample (Monte-Carlo) designs have a constant PCF with nearly no
oscillations, since point samples are uncorrelated, thus,
$P(r) = \lambda dx \lambda dy$ and theoretically have
$G(r)=1,\;\forall r$. Furthermore, the LHS design has a similar PCF as
random designs with the exception of a small, yet non-zero, $r_{min}$.

Other variants of PDS like MPDS, Step PCF and Stair PCF designs attempt to
trade-off between coverage ($G(r)=0$ for $0 \leq r \leq r_{min}$) and
randomness $G(r)=1,\;\text{for}\; r>r_{min}$. Note that, the Step and
the Stair PCF methods are space-filling spectral designs proposed
later in this paper. However, upon a careful comparison, it can be
seen that MPDS has a larger peak and more oscillations in its PCF
compared to the proposed designs. In fact, our empirical studies show
that the amount of oscillations in the PCF of the MPDS design
significantly increases with dimensions.

Next, in Figure~\ref{psd_samp}, we show the corresponding PSDs of the
different sample designs. It can be seen that, oscillations in PCF
directly correspond to oscillations in PSDs. For example, the
oscillatory behavior of the PCF for regular and QMC sequences cause a
non-uniform distribution of power in their corresponding
PSDs. Furthermore, the larger peak height in the PCF of MPDS implies
that a large amount of power is concentrated in a small frequency band
instead of power being distributed over all frequencies. In
Section~\ref{pds_analysis_sec}, we will analyze the effect of the
shape of PCF on the performance of a sample design in detail.

It is important to note that, not every PCF (or PSD) is physically
realizable by a sample design. In fact, there are two necessary
mathematical conditions~\footnote{Whether or not these two conditions
  are not only necessary but also sufficient is still an open question
  (however, no counterexamples are known).} that a sample design must
satisfy to be realizable.

\begin{definition}[Realizability]
\label{real}
A PCF can be defined to be realizable through a sample design, if it satisfies the following conditions:
\begin{itemize}
\item its PCF must be non-negative, i.e., $G(r)\geq 0,\;\forall r$, and
\item its corresponding PSD must be non-negative, i.e., $P(k)\geq 0,\;\forall k$.
\end{itemize}
\end{definition}
As both the PSD and the PCF characteristics are strongly tied to each other (as shown in Proposition~\ref{ht1}), these two conditions limit the space of realizable space-filling spectral designs. The results from this section will serve as tools for qualitatively understanding and, thus, designing optimal space-filling spectral designs in the following sections.


%% file: pds.tex
In this section, we first formalize desired characteristics of a good space-filling design, as given in Definition~\ref{def:sp-def}. Next, we will describe the proposed framework for creating space-filling spectral designs.

\begin{definition}
	A set of $N$ point samples $\mathcal{X}$ in a sampling domain $\mathcal{D}$ can be characterized as a space-filling design, if $\mathcal{X}=\{\mathbf{x_i}\in \mathcal{D};\;i=1,\cdots N\}$ satisfy the following two objectives:
	\begin{itemize}
		\item $\forall \mathbf{x_i}\in \mathcal{X},\;\forall \triangle \mathcal{D} \subseteq \mathcal{D}\;:\;P(\mathbf{x_i}\in \triangle \mathcal{D})=\int_{\triangle \mathcal{D}}\mathbf{dx}$
		\item $\forall \mathbf{x_i},\mathbf{x_j}\in \mathcal{X}\;:\;||\mathbf{x_i}-\mathbf{x_j}||\geq r_{\text{min}}$
	\end{itemize}
	where $r_{\text{min}}$ is referred to as the coverage radius.
\end{definition}In the above definition, the first objective states that the probability of a uniformly distributed random sample $\mathbf{x_i}\in \mathcal{X}$  falling inside a subset $\triangle \mathcal{D}$ of $\mathcal{D}$ is equal to the hyper-volume of $\triangle \mathcal{D}$. The
second condition enforces the minimum distance constraint between
point sample pairs for improving coverage. 

A Poisson design enforces the first condition alone, in which case the number of samples that fall inside any subset $\triangle \mathcal{D} \subseteq \mathcal{D}$ obeys a discrete Poisson distribution. Though easier to implement, Poisson sampling often produces distributions where the samples are grouped into clusters and leaves holes in possibly the regions of interest. In other words, this increases the risk of missing important features, when only the samples are used for analysis. Consequently, a sample design that distributes random samples in a uniform manner across $\mathcal{D}$ is preferred, so that clustering patterns are not observed. The coverage condition explicitly eliminates the clustering behavior by preventing samples from being closer than $r_{\text{min}}$. 
A space-filling design can be defined conveniently in the PCF domain and we refer to this as the space-filling spectral design, due to its direct connection to the spectral domain properties.


\subsection{Defining a Space-filling Spectral Design in Spatial Domain}
\label{sec:pcf}

For Poisson design, point locations are not correlated and, therefore, $P(r)=\lambda dx \lambda dy$. This implies that for Poisson designs $G(r)=1$. Similarly, for space-filling designs, due to the minimum distance constraint between the point sample pairs, we do not have any point samples in the region $ 0\leq r <  r_{\text{min}} $. Consequently, space-filling spectral designs are defined as a step pair correlation function in the spatial domain (\textit{Step PCF}).

\begin{proposition}
\label{PDSinPCF}
Given the desired coverage radius $r_{\text{min}}$, a space-filling spectral design is defined in the spatial domain as
\[ G(r-r_{\text{min}}) = \left\{ \begin{array}{rll}
				0  & \mbox{if}\ r< r_{\text{min}} \\
			   1  & \mbox{if}\ r \geq r_{\text{min}}.
				\end{array}\right. 
\]
\end{proposition}As a consequence of Proposition~\ref{ht1}, space-filling spectral designs can equivalently be defined in the spectral domain. 

\subsection{Defining a Space-filling Spectral Design in Spectral Domain}
\label{sec:psd}

We derive the power spectral density of the space-filling spectral design using the connection established in Section~\ref{sec:metric}. Following our earlier notation, we denote the $d$-dimensional power spectral density by $P(\mathbf{k})$ and $d$-dimensional PCF by $G(\mathbf{r})$.

\begin{proposition}
\label{PDSinPSD}
Given the desired coverage radius $r_{\text{min}}$, a $d$-dimensional space-filling spectral design $\mathcal{X}$, with $N$ sample points in a sampling domain $\mathcal{D}$ of volume $V$, can be defined in the PSD domain as
$$P(k)=1-\frac{N}{V} \left(\dfrac{2\pi r_{\text{min}}}{k}\right)^{\frac{d}{2}}J_{\frac{d}{2}}(kr_{\text{min}})$$
where $J_{\frac{d}{2}}(.)$ is the Bessel function of order $d/2$.
\end{proposition}
\begin{proof}
We know that, 
\begin{eqnarray}
P(\mathbf{k})&=& 1+\frac{N}{V} F\left(G(\mathbf{r})-1\right),\\
&=& 1+\frac{N}{V} \int_{\mathbb{R}^d} \left(G(\mathbf{r})-1\right)\exp(-i\mathbf{k.r}) d\mathbf{r},
\end{eqnarray}where $F(.)$ denotes the $d$-dimensional Fourier transform. Note that, for the radially symmetric or isotropic functions, i.e., $G(r)$ where $r=||\mathbf{r}||$, the above relationship simplifies to
\begin{equation}
\label{PSDht1}
P(k)=1+\frac{N}{V}\;{(2\pi)^{\frac{d}{2}}} k^{1-\frac{d}{2}}H_{\frac{d}{2}-1}\left(r^{\frac{d}{2}-1}(G({r})-1)\right), 
\end{equation}where 
$$H_v(f(r))(k)=\int_{0}^{\infty}r J_v(kr)f(r) dr$$ is the $1-$d Hankel transform of order $v$ with $J$ being the Bessel function.
To derive the PSD of a step function, we first evaluate the Hankel transform of $f(r)=(G(r)-1)$ where $G(r)$ is a step function.
\begin{eqnarray*}
H_{\frac{d}{2}-1}\left(r^{\frac{d}{2}-1}(G(r)-1)\right)&=& \int_{0}^{\infty}r^{\frac{d}{2}} J_{\frac{d}{2}-1}(kr)\left(G(r)-1\right) dr\\
&=&- \int_{0}^{r_{\text{min}}}r^{\frac{d}{2}} J_{\frac{d}{2}-1}(kr) dr\\
&=&- \dfrac{^{r_{\text{min}}^\frac{d}{2}}}{k}J_{\frac{d}{2}}(kr_{\text{min}})
\end{eqnarray*}
Using this expression in (\ref{PSDht1}), we obtain
\begin{equation}
\label{PSDstep}
P(k)=1-\frac{N}{V} \left(\dfrac{2\pi r_{\text{min}}}{k}\right)^{\frac{d}{2}}J_{\frac{d}{2}}(kr_{\text{min}}).
\end{equation}
\end{proof}
This proposition connects the spatial properties of a space-filling spectral design, defined via the PCF, to its spectral properties. The motivation for this is the fact that in several cases, it is easier to link the objective performance of a sample design to its spectral properties. In the next section, we will develop the relation between spectral properties and an objective measure of the performance, which in turn provides us guidelines for designing better space-filling spectral sampling patterns. 


%% file: pds_analysis.tex
In this section, we derive insights regarding the objective performance of space-filling spectral designs. To this end, we analyze the impact of the shape of the PCF on the reconstruction performance. Further, for a tractable analysis, we consider the task of recovering the class of periodic functions using space-filling spectral designs and analyze the reconstruction error 
as a function of their spectral properties. The analysis presented in this section will clarify how the shape of the PCF of a sample design directly impacts its reconstruction performance.

\subsection{Analysis of Reconstruction Error for Periodic Functions}
\label{sec:anal-per}
Let us denote the Fourier transform of the sample design $\mathcal{X}$ by $\mathcal{S}$. The function to be sampled and its corresponding Fourier representation are denoted by $\mathcal{I}$ and $\hat {\mathcal{I}}(k)$ respectively. 
Now, the spectrum of the sampled function is given by $\hat {\mathcal{I}}_{s}(k) = \mathcal{S} \ast \hat {\mathcal{I}}(k)$. Note that, a sampling pattern with a finite number of points is comprised of two components, a DC peak at the origin and a noisy remainder $\mathcal{\bar{S}}$. Thus, equivalently, we have $\hat {\mathcal{I}}_{s}(k) = \{ n\delta(k) + \mathcal{\bar{S}} \} \ast \hat {\mathcal{I}}(k).$
The error introduced in the process of function reconstruction is the difference between the reconstructed and the original functions: 
$$ \mathcal{E}(k) = | \hat {\mathcal{I}}_{s}(k)/N - \mathcal{I}(k) |^{2} = |\mathcal{\bar{S}} \ast \hat {\mathcal{I}}(k)/N |^{2}$$
where we have divided the R.H.S. by $N$ to normalize the energy of $\mathcal{I}_{s}$. For error analysis, we focus on the low frequency content of the error term, since the high frequency components are removed during the reconstruction process.

Denoting the power spectrum without the DC component by $\mathcal{P'}(k)$, for a constant function the error simplifies to
\begin{equation}
\mathcal{E}(k)  \propto |\mathcal{S'}(k)|^2 \propto \mathcal{P'}(k).
\end{equation}
This, as stated above, allows for the characterization of the error in terms of the spectral properties of the sampling pattern used.

Next, we consider an important class of functions, the family of periodic functions, for further analysis. All periodic functions with a finite period can be expressed as a Fourier series, which is a summation of sine and cosine terms
$$\mathcal{I}(x) = a_0 + \sum_{m=1}^{M} a_{m}cos(2\pi mx) + \sum_{m=1}^{M} b_{m}sin(2\pi mx).$$
The Fourier transform of this function is equivalently a summation of pulses:
$$\hat{\mathcal{I}}(k) = a_0 \delta(k) + \sum_{m=1}^{M} a_{m}\left(\dfrac{1}{2}(\delta(k-m) + \delta(k+m)\right) + \sum_{m=1}^{M} b_{m}\left(\dfrac{1}{2}(\delta(k+m) - \delta(k-m))\right).$$
Making substitutions, $a_m + b_m = A_m, a_m - b_m = B_m$, we obtain
$$\mathcal{E}(k) = \dfrac{1}{4N^2} \Big|  4a_0\mathcal{S'}(k) + \sum_{m=1}^{M} \big( A_m \mathcal{S'}(k+m) +  B_m \mathcal{S'}(k-m)\big) \Big|^2. $$
The reconstruction error can then be upper bounded as follows:
\begin{equation}
\label{error_ub_1}
\mathcal{E}(k) \le \dfrac{1}{4N} \Big[ 4a_{0}^{2}\mathcal{P'}(k) +  \sum_{m=1}^{M} \big( A_m^{2} \mathcal{P'}(k+m) + B_m^{2} \mathcal{P'}(k-m) \big) \Big] .
\end{equation}
In the case of a single sinusoidal function, $cos( 2 \pi f x)$, using triangle inequality, this becomes~\citep{Heck:2013} 
\begin{equation}
\label{recon:upb}
\mathcal{E}(k) \le \dfrac{1}{4N} \Big[ \mathcal{P'}(k+f) + \mathcal{P'}(k-f) + 2\; \sqrt[]{\mathcal{P'}(k+f)\mathcal{P'}(k-f)}  \Big].
\end{equation}

Even though this is only an upper bound, we will see that it accurately predicts the characteristics of the sampling error and provides useful guidelines.

The above analysis implies that to assess the quality of the sample designs, one can analyze their spectral behavior. More specifically, the above analysis suggests that to minimize the reconstruction error: (a) the power spectra of the sample design should be close to zero, and (b) for errors to be broadband white noise (uniform over frequencies), the power spectra should be a constant. Note that, in several applications, e.g., image reconstruction, most relevant information is predominantly at low frequencies. In such scenarios, this naturally leads to the following criteria for sample designs: (a) the spectrum should be close to zero for low frequencies which indicates the range of frequencies that can be represented with almost no error, (b) the spectrum should be a constant for high frequencies or contain minimal amount of oscillations in the power spectrum. 

\subsection{Effect of PCF Characteristics on Sampling Performance}
\label{g-pcf}
Based on the two criteria discussed above, we assess the effect of the shape of the PCF on the quality of space-filling designs in the spectral domain. Note that, PCFs of the samples constructed in practice (Figure \ref{point_samp}) often demonstrates the following characteristics: (a) presence of a zero-region characterized by $r_{min}$, (b) a large peak around $r_{min}$, and (c) damped oscillations. To model and analyze these characteristics, we consider the following parametric PCF family:
\begin{eqnarray}
\label{gpcf}
&&G(r)=G(r-r_{\text{min}})+a\left(G(r-r_{\text{min}})-G(r-r_{\text{min}}-\delta)\right)\\\nonumber
&& \qquad + \frac{a}{4r}\exp(-r/2)\sin(c\times r-c)G(r-r_{\text{min}})
\end{eqnarray}
where $G(r-r_{\text{min}})$ is the Step function, peak width $\delta\geq 0$ and the peak height $a\geq 1$ and last term in~\eqref{gpcf} corresponds to damped oscillations. This family is a generalization of Step PCF, with additional parameterization of peak height and oscillations in the PCF.

\subsubsection{Effect of Peak Height on Spectral Properties}

In order to study the impact of increasing peak height in the PCF on the PSD characteristics, we conduct an empirical study. We compute the PSD of a sample design with the following parameters: $N=195, r_{\text{min}}=0.02, \delta=0.005$. Note that, we vary the PCF peak height $a$, which actually reflects the behavior of existing coverage based PDS algorithms. As shown in Figure~\ref{psdvsa}, increasing $a$ results in both significantly \textit{higher} low frequency power and \textit{larger} high frequency oscillations. As expected, the PSD of the Step PCF (or $a=1$) performs the best, i.e., the spectrum is close to zero for low frequencies and constant for high frequencies.

\subsubsection{Effect of Disk Radius on Spectral Properties}
Next, we study the importance of choosing an appropriate $r_{\text{min}}$ (or coverage $\rho$) while generating sample distributions. In Figure \ref{psdvsrmin}, we show the PSD for $N=195$ and $a$ = 1, with varying disk radius values $r_{\text{min}}$. For a fixed sample budget, as we increase the radius, we observe two contrasting changes in the PSD: $(i)$ the spectrum tends to be close to zero at low frequencies and $(ii)$ an increase in oscillations for high frequencies. Consequently, there is a trade-off between low frequency power and high frequency oscillations in power spectra which can be controlled by varying $r_{\text{min}}$. However, the increase in oscillations are less significant compared to the gain in the zero-region. Furthermore, in several applications, low frequency content is more informative, and hence one may still attempt to maximize $r_{min}$ or coverage.

\subsubsection{Effect of Oscillations on Spectral Properties}
Finally, we study the effect of oscillations in the PCF on the power distribution in the spectral domain. In Figure~\ref{psdvsc}, we plot the PSD for $a=1$ with varying amounts of oscillations controlled via the parameter $c$. It can be seen that introducing oscillations in the PCF results in significantly \textit{higher} low frequency power and \textit{larger} high frequency oscillations. As expected, the PSD of the Step PCF (or $c=0$) behaves the best.

In summary, the discussion in this section suggests that the PCF of an ideal space-filling spectral design should have the following three properties: $(a)$ large $r_{min}$, $(b)$ small peak height, and $(c)$ low oscillations. Since, the Step PCF satisfies these three properties, it is expected to be a good space-filling spectral design.
Next, we consider the problem of optimizing the parameter of the Step PCF design, i.e. $r_{min}$.

\begin{figure}[t!]
\vspace*{-0.2in}
  \centering
  \subfigure[]{
    \includegraphics[width=0.3\columnwidth, clip = true]{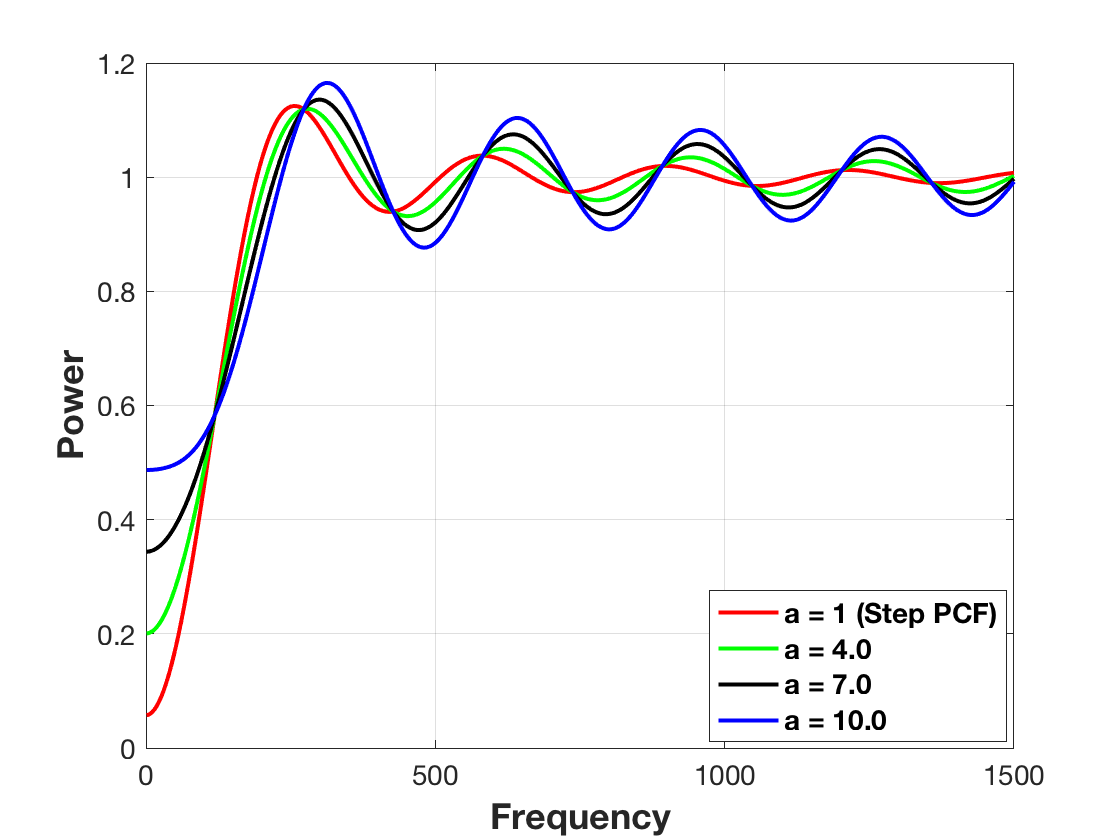}
    \label{psdvsa}
}
\subfigure[]{

    \includegraphics[width=0.3\columnwidth, clip = true]{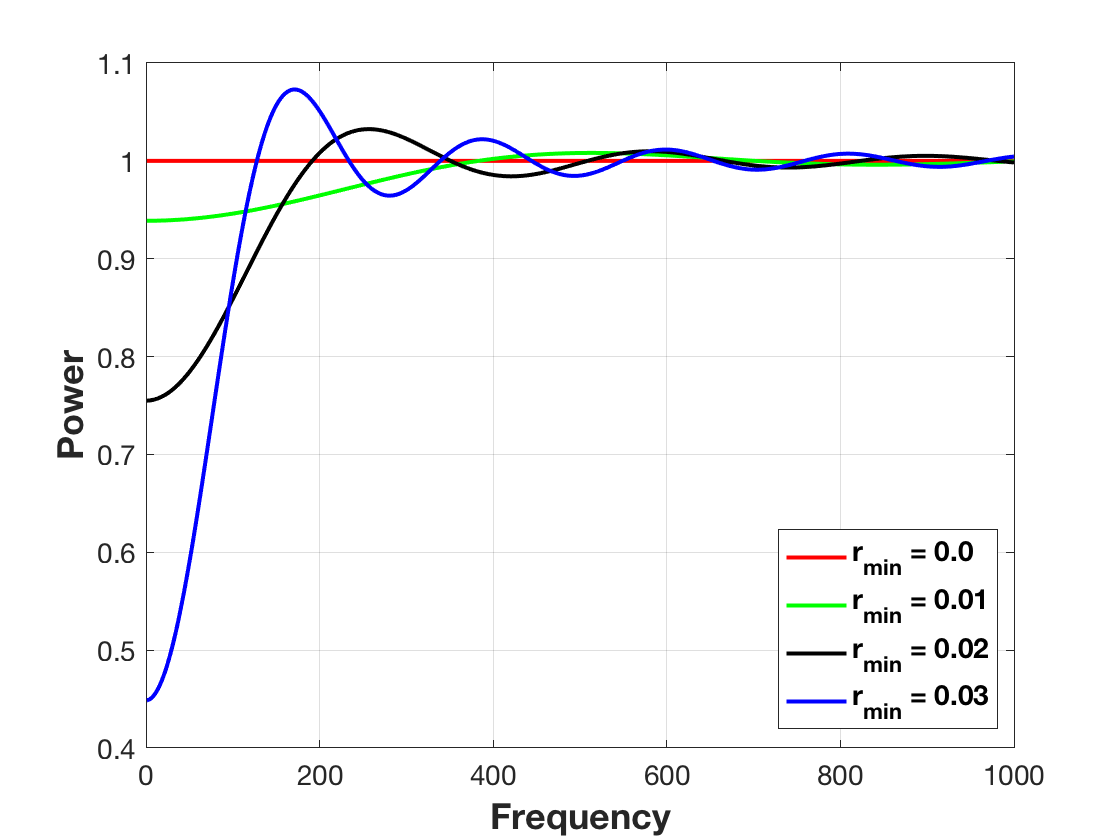}
    \label{psdvsrmin}
}
\subfigure[]{

    \includegraphics[width=0.3\columnwidth, clip = true]{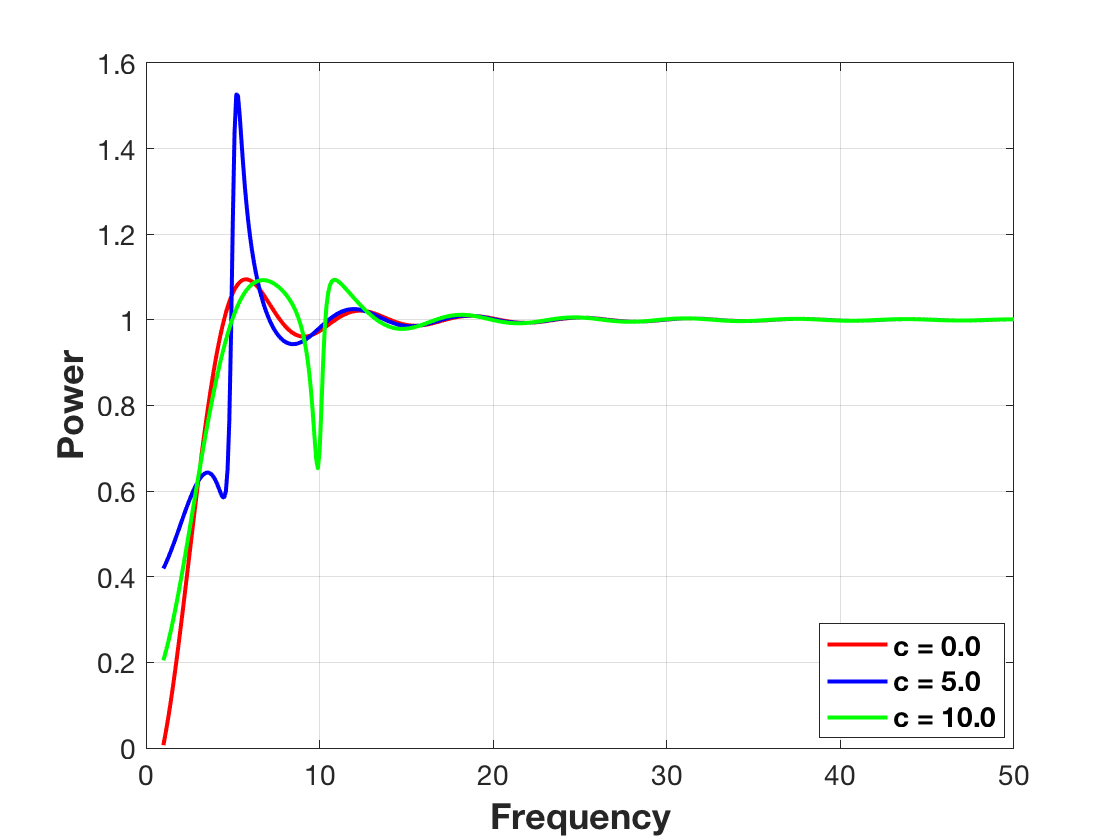}
    \label{psdvsc}
}
    \caption{(a) Effect of peak height in the PCF on power spectra, (b) Effect of disk radius in the PCF on power spectra, (c) Effect of oscillations in the PCF on power spectra.}
    
    \label{pdsanalys}
\end{figure}




%% file: opt_pds.tex
The proposed space-filling metric enjoys mathematical tractability and is supported by theoretical results as defined in Section~\ref{sfsd}. This enables us to obtain new insights for optimizing Step PCF based space-filling spectral designs. In particular, (a) For a fixed $r_{\text{min}}$, we obtain the maximum number of point samples in any arbitrary dimension $d$, (b) For a fixed sampling budget $N$, we derive the maximum achievable $r_{\text{min}}$ in arbitrary dimension $d$. 

\subsection{Case 1: Fixed $r_{\text{min}}$}

The problem of finding the maximum number of point samples in a Step PCF based space-filling spectral design with a given disk radius $r_{\text{min}}$ can be formalized as follows:
\begin{equation}
\begin{split}
{\mathrm{maximize}}\quad& N \\
\mbox{subject to}  
\quad &  P(k)\geq 0,\;\forall k\\
\quad &  G(r-r_{\text{min}})\geq 0,\;\forall r,
\end{split}
\end{equation}where $P(k)=1-\frac{N}{V} \left(\dfrac{2\pi r_{\text{min}}}{k}\right)^{\frac{d}{2}}J_{\frac{d}{2}}(kr_{\text{min}})$. Note that, a space-filling spectral design has to satisfy realizability constraints as defined in Definition~\ref{real}.

\begin{proposition}
\label{prp1}
For a fixed disk radius $r_{\text{min}}$, the maximum number
of point samples possible for a realizable Step PCF based space-filling spectral design in the sampling region with volume $V$ is given by
$$N=\dfrac{V\Gamma\left(\frac{d}{2}+1\right)}{\pi^{\frac{d}{2}}r_{\text{min}}^d}.$$
\end{proposition}
\begin{proof}
Using the definition of the Step PCF function, the constraint $G(r-r_{\text{min}})$ is trivially satisfied. Note that, the constraint $P(k)\geq 0,\;\forall k$ is equivalent to $\underset{k}{\min}\; P(k)\geq 0$. In other words,
\begin{eqnarray*}
&& \underset{k}{\min}\; 1-\rho \left(\dfrac{2\pi r_{\text{min}}}{k}\right)^{\frac{d}{2}}J_{\frac{d}{2}}(kr_{\text{min}})\geq 0\\
&\Leftrightarrow& \underset{k}{\max}\;\rho \left(\dfrac{2\pi r_{\text{min}}}{k}\right)^{\frac{d}{2}}J_{\frac{d}{2}}(kr_{\text{min}})\leq 1\\
&\Leftrightarrow& \rho \left({2\pi }\right)^{\frac{d}{2}}r_{\text{min}}^d\;\underset{k}{\max}\;\left(\dfrac{J_{\frac{d}{2}}(kr_{\text{min}})}{(kr_{\text{min}})^{\frac{d}{2}}}\right)\leq 1\\
&\Leftrightarrow& \rho \left({2\pi }\right)^{\frac{d}{2}}r_{\text{min}}^d\;\dfrac{1}{2^{\frac{d}{2}}\Gamma\left(\frac{d}{2}+1\right)}\leq 1\\
&\Leftrightarrow& N \leq \dfrac{V\Gamma\left(\frac{d}{2}+1\right)}{\left({\pi }\right)^{\frac{d}{2}}r_{\text{min}}^d}
\end{eqnarray*}
where, we have used the fact that $J_v(x)\approx {(x/2)^v}/{\Gamma(v+1)}$. 
\end{proof}

Note that, for the $2$-dimensional case, we have $\frac{J_{1}(kr_{\text{min}})}{kr_{\text{min}}}=\text{jinc}(kr_{\text{min}})$. Now using the fact that $\text{jinc}(x)$ has the maximum value equal to $1/2$, for a fixed disk radius $r_{\text{min}}$, the maximum number of point samples possible in a $2$-d Step PCF based space-filling spectral design is given by $$N={V}/{\pi(r_{\text{min}})^2},$$
which again corroborates our bound in Proposition~\ref{prp1}.

\subsection{Case 2: Fixed $N$}

Alternately, we can also derive the bound for the disk radius of Step PCF with a fixed sampling budget $N$ as follows:
\begin{equation}
\begin{split}
{\mathrm{maximize}}\quad& r_{\text{min}} \\
\mbox{subject to}  
\quad &  P(k)\geq 0,\;\forall k\\
\quad &  G(r-r_{\text{min}})\geq 0,\;\forall r
\end{split}
\end{equation}

\begin{proposition}
\label{prp2}
For a fixed sampling budget $N$, 
the maximum possible disk radius $r_{\text{min}}$ for a realizable Step PCF based space-filling spectral design
in the sampling region with volume $V$ is given by
$$r_{\text{min}}=\sqrt[d]{\dfrac{V\Gamma\left(\frac{d}{2}+1\right)}{\pi^{\frac{d}{2}}N}}.$$
\end{proposition}
\begin{proof}
The proof is similar to the one in Proposition~\ref{prp1}.
\end{proof}

\subsection{Relative Radius of Step PCF}
As mentioned before, the current literature characterizes coverage by the fraction $\rho$ of the maximum possible
radius $r_{max}$ for $N$ samples to cover the sampling domain,
such that $r_{min} = \rho r_{max}$. 
The maximum possible disk radius is achieved by the deterministic hexagonal lattice~\citep{Hex-lattice} and can be approximated in a $d$ dimensional sampling region as $r_{max} \approx \sqrt[\leftroot{-3}\uproot{3}d]{\tfrac{A_d}{C_d N}}$. Here, $A_d$ is the hypervolume of the sampling domain and $C_d=V_d/r^d$ with $V_d$ being the hypervolume of a hypersphere with radius $r$.
Note that, a uniformly distributed point
set can have a relative radius of $0$, and the relative radius of a hexagonal lattice equals $1$ (in $2$-d). Next, we derive a closed-form expression for the relative radius of Step PCF based design.

\begin{proposition}
\label{prp_rho}
For a fixed sampling budget $N$, the maximum relative radius $\rho$ for Step PCF based space-filling spectral design in the sampling region with volume $V$ is given by $\rho = \dfrac{1}{2\;\sqrt[d]{\eta_d}} $ where $\eta_d$ is maximal density of a sphere packing in $d$-dimensions.
\end{proposition}
\begin{proof}
Let us denote by $r_{max} = \underset{r}{\arg\min}\; \eta_d$, then, the maximal density of a sphere packing with $N$ samples in $d$-dimensions is given by 
\begin{eqnarray}
&& \eta_d = \dfrac{N\pi^{d/2}}{\Gamma(1+\frac{d}{2})} \dfrac{r_{max}^d}{V}\\
&\Leftrightarrow & \eta_d = \left(\frac{r_{max}}{r_{min}}\right)^d \label{equality1} \\
&\Leftrightarrow & \rho = \dfrac{1}{2\;\sqrt[d]{\eta_d}}
\end{eqnarray}
where equality in \eqref{equality1} uses Proposition~\eqref{prp2}.
\end{proof}

For $d=2$ and $3$, the relative radius simplifies to:
$$\rho = 0.5 \sqrt[2]{\dfrac{\pi \sqrt{3}}{6}},\;\textbf{for}\; d =2,\; \text{and} $$
$$\rho = 0.5 \sqrt[3]{\dfrac{\pi \sqrt{2}}{6}},\;\textbf{for}\; d =3 .$$

Note that, finding the maximal density of a sphere packing for an arbitrary high dimension (except in $d=2,3$ and recently in $8,24$~\citep{sp8, sp24}) is an open problem. 
Note that, best known packings are
often lattices, thus, we use the best known lattices to be an approximation of $r_{max}$ in our analysis\footnote{We use relative radius as a metric only for analysis and not for design optimization.}.

In Figure~\ref{rhovsd}, we plot the relative radius $\rho=r_{min}/r_{max}$ of Step PCF for different dimensions $d$. It is interesting to notice that the relative radius of Step PCF based designs increases as the dimension $d$ increases, i.e., Step PCF based designs approach a more regular pattern. Further, note that, for a fixed sampling budget both $r_{min}$ and $r_{max}$ increase as the number of dimensions increases. The Step PCF based designs maintain randomness by keeping the PCF flat, but this comes at a cost: the disk radius $r_{min}$ of these patterns is very small (as can be seen from Figure~\ref{rhovsd}). For several applications, covering the space better (by trading-off randomness) is more important. In the next section, we will propose a new class of space-filling spectral designs that can achieve a much higher $r_{min}$ at the small cost of compromising randomness by introducing a single peak into an otherwise flat PCF.

\begin{figure}[t!]
\vspace*{-0.2in}
  \centering
    \includegraphics[width=0.5\columnwidth, clip = true]{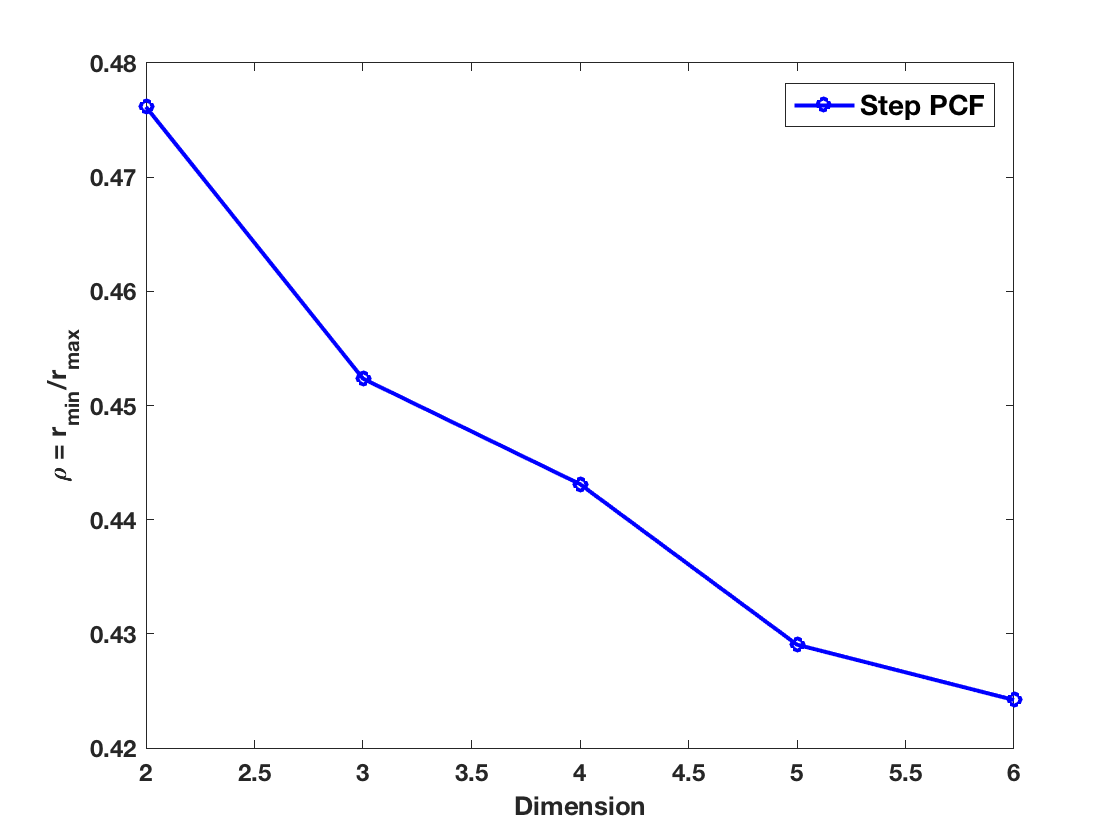}
    \caption{Relative radius $\rho=r_{min}/r_{max}$ of Step PCF based space-filling spectral design for different dimensions $d$.}    
\label{rhovsd}
\end{figure}


%% file: stair.tex
To improve the coverage of Step PCF base space-filling spectral design, in this section, we propose a novel space-filling spectral design which systematically trades-off randomness with coverage of the resulting samples. Note that, the randomness property can be relaxed either by increasing the peak height of the PCF, or by increasing the amounts oscillations in the PCF (as discussed in Section~\ref{g-pcf}). For simplicity\footnote{In our initial experiments, we found that increasing the peak height alone is sufficient for trading-off randomness to maximize coverage, and performs better than trading-off randomness by increasing oscillations in the PCF.}, we adopt the former strategy and use only the peak height parameter. More specifically, as an alternative to Step PCF, we design the following generalization which we refer as the Stair PCF design.

\subsection{Stair PCF based Space-filling Spectral Design}
Now, we define the proposed {\it Stair} PCF based space-filling design and quantify the gains achieved in the coverage characteristics (\textit{i.e.} $r_{min}$).

\vspace{0.1in}

\noindent \textbf{Stair PCF in the Spatial Domain}: The Stair PCF construction is defined as follows:

\begin{align}
  \label{eqn:stair}
  & G(r;r_0,r_1,P_0)=f(r-r_1)+P_0\left(f(r-r_0)-f(r-r_1)\right), \\
  \nonumber  & \text{with } f(r-r_0) = \left\{ \begin{array}{rll}
      0  & \mbox{if}\ r\leq r_0 \\
      1  & \mbox{if}\ r>r_0
    \end{array} \right\}, \\
\nonumber & \text{where } r_0\le r_1 \text{ and } P_0 \geq 1. 
\end{align}
This family of space-filling spectral designs has three interesting properties:
\begin{itemize}
\item except for a single peak in the region $r_0 \leq r\leq r_1$, the PCF is flat, thus, does not compromise randomness entirely,
\item both the height and width of the peak can be optimized to maximize coverage,
\item the Step PCF based spectral design can be derived as as a special case of this construction.
\end{itemize}
A representative example of Stair PCF is shown in Figure~\eqref{stair_fig}.

\vspace{0.1in}

\noindent \textbf{Stair PCF in the Spectral Domain}: Following the analysis in the earlier sections, we derive the power spectral density of Stair PCF based space-filling spectral designs. 

\begin{proposition}
\label{psd_stair}
The power spectral density of a Stair PCF based space-filling spectral designs, $G(r;r_0,r_1,P_0)$, with $N$ samples in the sampling region with volume $V$ is given by
$$P(k)=1-\frac{N}{V} P_0\left(\dfrac{2\pi r_{0}}{k}\right)^{\frac{d}{2}}J_{\frac{d}{2}}(kr_{0})-\frac{N}{V} (1-P_0)\left(\dfrac{2\pi r_{1}}{k}\right)^{\frac{d}{2}}J_{\frac{d}{2}}(kr_{1}).
$$
\end{proposition}
\begin{proof}
Using results from Section~\ref{sec:psd}, we have
\begin{equation}
\label{PSDht}
P(k)=1+\frac{N}{V}\;{(2\pi)^{\frac{d}{2}}} k^{1-\frac{d}{2}}H_{\frac{d}{2}-1}\left(r^{\frac{d}{2}-1}(G({r})-1)\right). 
\end{equation}
To derive the PSD of a Stair function, we first evaluate the Hankel transform of $f(r)=(G(r)-1)$ where $G(r)$ is a Stair function.
\begin{eqnarray*}
&&H_{\frac{d}{2}-1}\left(r^{\frac{d}{2}-1}(G(r)-1)\right)= \int_{0}^{\infty}r^{\frac{d}{2}} J_{\frac{d}{2}-1}(kr)\left(G(r)-1\right) dr\\
&&= -P_0 \int_{0}^{r_0} r^{\frac{d}{2}} J_{\frac{d}{2}-1}(kr) dr - (1-P_0)\int_{0}^{r_1}r^{\frac{d}{2}} J_{\frac{d}{2}-1}(kr) dr\\
&&= -P_0 \dfrac{^{r_0^\frac{d}{2}}}{k}J_{\frac{d}{2}}(kr_0) - (1-P_0)\dfrac{^{r_1^\frac{d}{2}}}{k}J_{\frac{d}{2}}(kr_1)
\end{eqnarray*}
Using this expression in (\ref{PSDht}),
\begin{equation}
\label{PSDstep}
P(k)=1-\frac{N}{V} P_0\left(\dfrac{2\pi r_{0}}{k}\right)^{\frac{d}{2}}J_{\frac{d}{2}}(kr_{0})-\frac{N}{V} (1-P_0)\left(\dfrac{2\pi r_{1}}{k}\right)^{\frac{d}{2}}J_{\frac{d}{2}}(kr_{1}).
\end{equation}
\end{proof}

\begin{figure*}[t!]
\vspace*{-0.2in}
  \centering
  \subfigure[]{\includegraphics[width=0.3\columnwidth, clip = true]{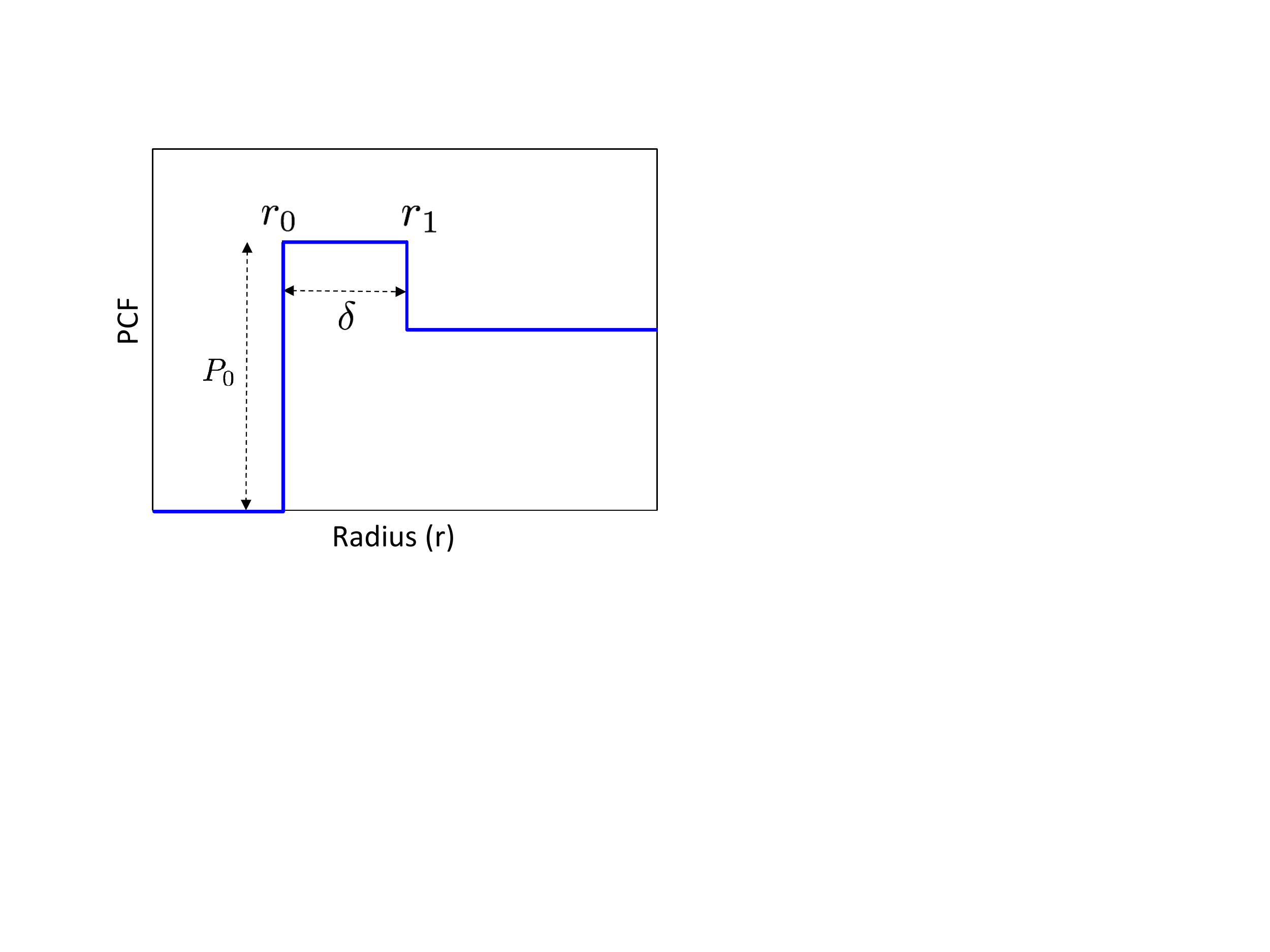}
  \label{stair_fig}}
   \subfigure[$d=2$]{
   	\includegraphics[
   	width=0.3\textwidth,clip=true]{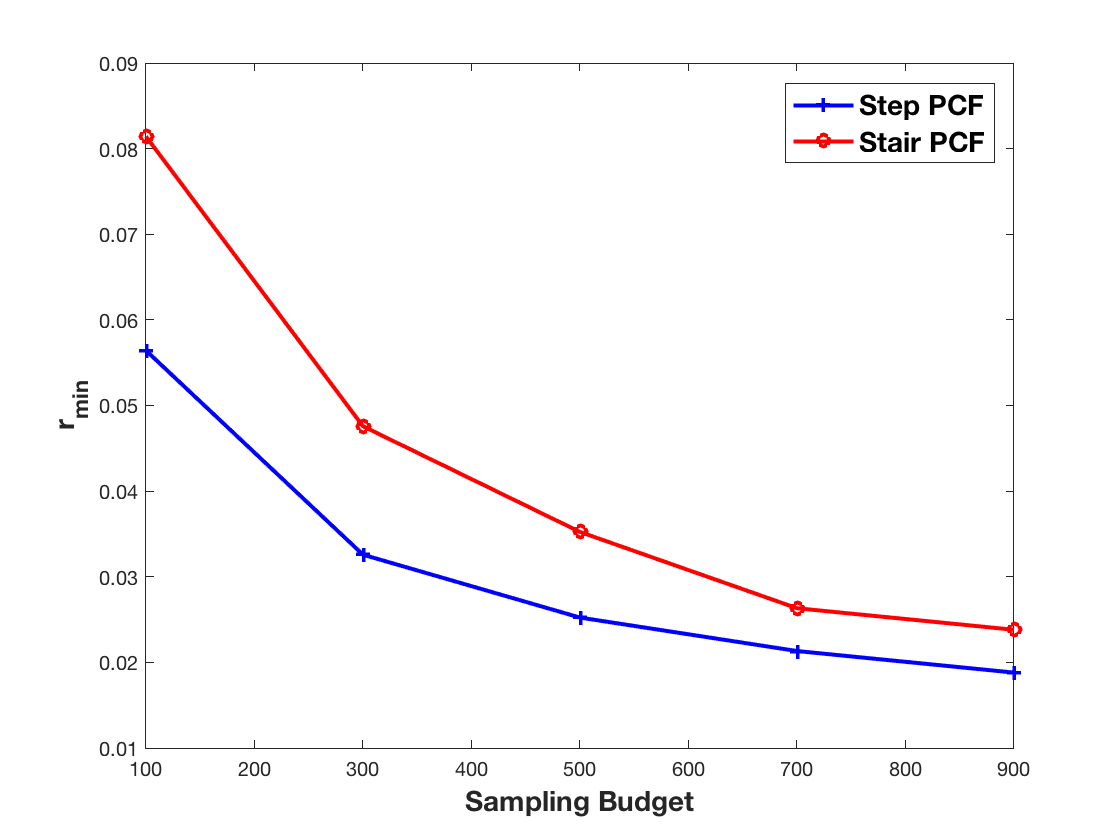}
   	\label{fr2} } 
   \subfigure[$d=3$]{
   	\includegraphics[%
   	width=0.3\textwidth,clip=true]{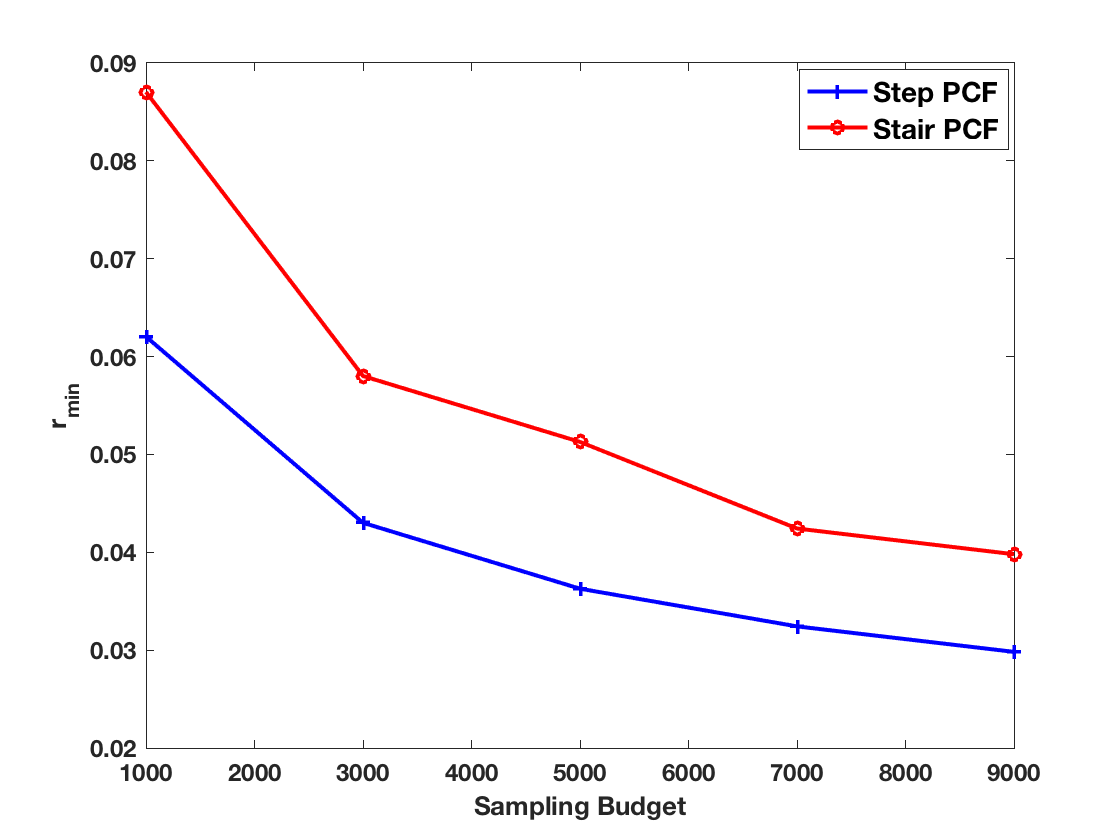}
   	\label{fr3} }
   \vfill
   \subfigure[$d=4$]{
   	\includegraphics[
   	width=0.3\textwidth,clip=true]{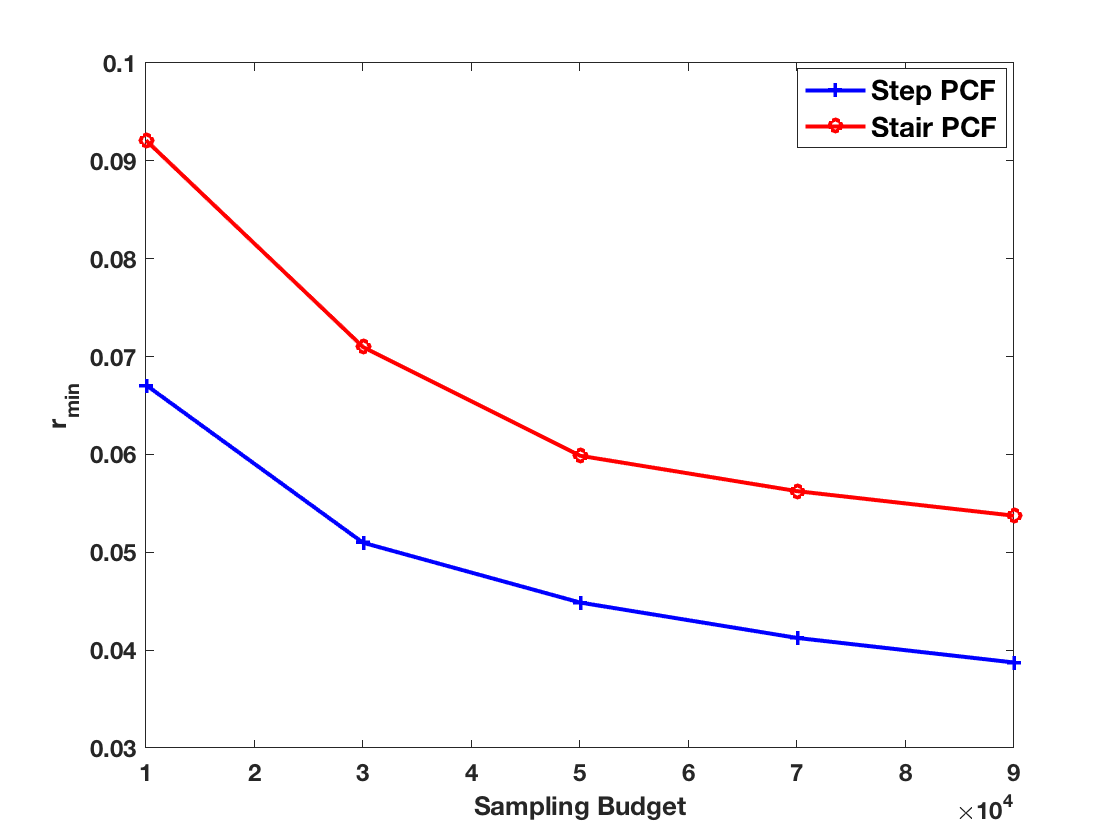}
   	\label{fr4} }
   \subfigure[$d=5$]{
   	\includegraphics[
   	width=0.3\textwidth,clip=true]{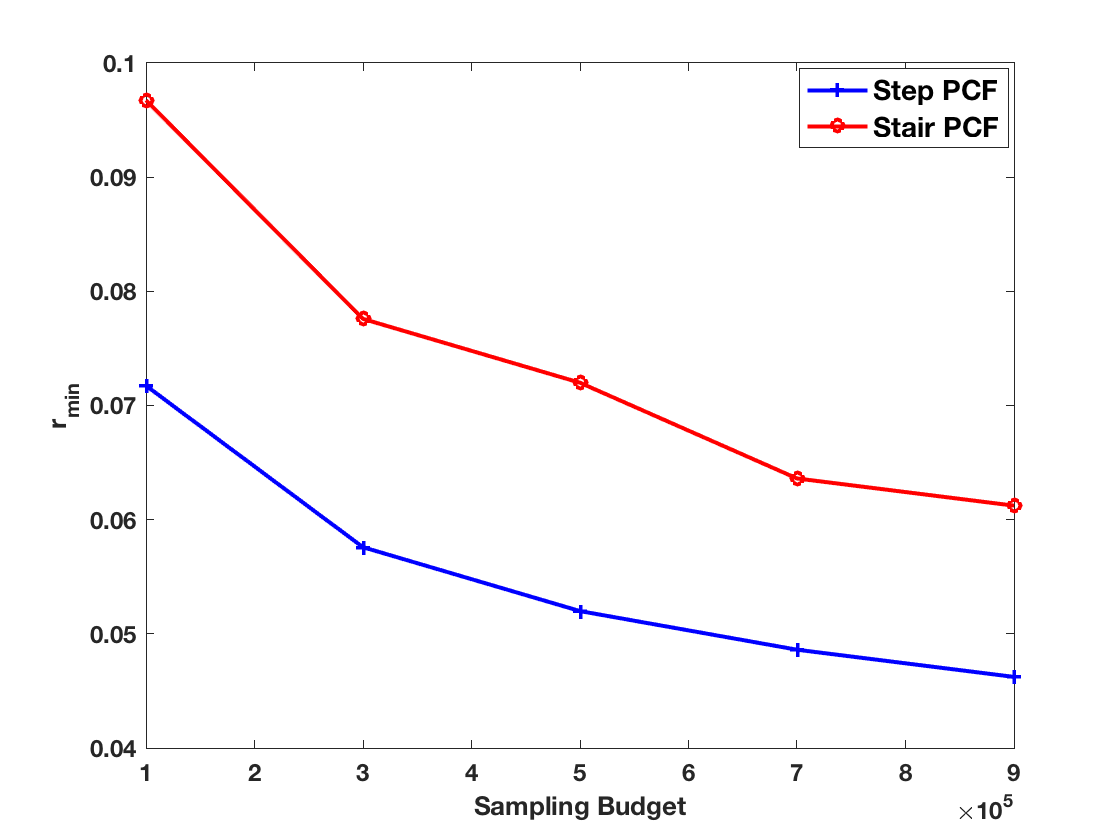}
   	\label{fr5} }
   \subfigure[$d=6$]{
   	\includegraphics[
   	width=0.3\textwidth,clip=true]{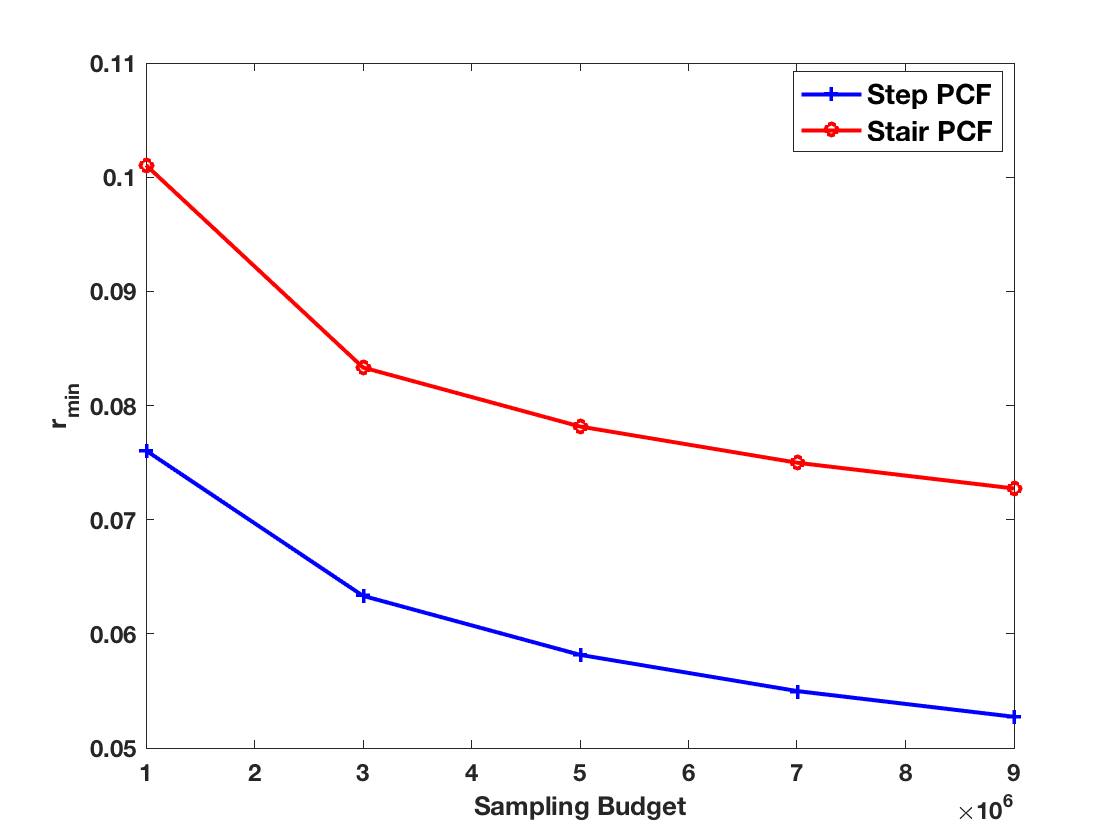}
   	\label{fr6} }

    \caption{(a) Pair correlation function of Stair PCF based designs, (b)-(f) Maximum Disk Radius For Step and Stair PCF for dimensions $2$ to $6$.}    

\end{figure*}

Next, we empirically evaluate the gain in coverage achieved by Stair PCF based designs compared to the Step PCF based designs.

\subsection{Coverage Gain with Stair PCF}
Ideally, the optimal Stair PCF should be obtained by simultaneously maximizing $r_{0}$ ($:=r_{min}$) and minimizing $P_0$. Furthermore, not all PCFs in the Stair PCF family are realizable. Due to the realizability conditions, the parameters cannot be adjusted independently. The main challenge, therefore, is to find the combinations of the three parameters $(r_0,r_1,P_0)$ that is realizable and yield a good sample design. Unlike Step PCF, the closed form expression for the optimal parameters $(r_0,r_1,P_0)$ are difficult to obtain, and, therefore, we explore this family of PCF patterns empirically by searching configurations for which:
\begin{itemize}
\item the disk radius $r_0$ is as high as possible, and
\item the PCF is flat with minimal increase in the peak height $P_0$.
\end{itemize}

\subsubsection{Disk Radius $r_{min}$ vs. Sample Budget $N$}
In this section, we show the increase in coverage (or $r_{min}$) obtained by compromising randomness by increasing peak height in the PCF. We constrain the peak height to be below $P_0\leq 1.5$ and analyze the gain in $r_{min}$ due to this small compromise in randomness. Furthermore, we assume that $r_{min}^{step} \leq r_0 \leq 2\times r_{min}^{step}$ and $r_0 \leq r_1 \leq 1.5 \times r_0$. In Figures~\ref{fr2} through~\ref{fr6}, we compare the maximum $r_{min}$ achieved by the Step and Stair PCF designs, for varying sample sizes in dimensions $2$ to $6$. It can be seen that introducing a small peak in the PCF results in a significant increase in the coverage. This gain can be observed for all sampling budgets in all dimensions. Furthermore, as expected, for low sampling budgets maximal gain is observed, and should decrease with increasing $N$ as the $r_{min}$ for both the families will asymptotically (in $N$) converge to zero.

\subsubsection{Relative Radius $\rho$ vs. Dimension $d$}
In this section, we study the increase in relative radius $\rho$ due to the introduction of a peak in the PCF. Again, we assume that $P_0\leq 1.5$, $r_{min}^{step} \leq r_0 \leq 2\times r_{min}^{step}$ and $r_0 \leq r_1 \leq 1.5 \times r_0$. In Figure~\ref{stair_den}, we show the maximum $\rho=r_{min}/r_{max}$ achieved by the Step and Stair PCFs for different dimensions $d$. For Stair PCF, we do not have a closed form expression of $\rho$, thus, we obtain the maximum achievable $r_{min}$ empirically for various sampling budgets and plot the mean (with standard deviation) behavior of the $\rho$. It can be seen that introducing a small peak in the PCF results in a significant increase in the relative radius. This gain can be observed
at all sampling budgets in all dimensions. This also corroborates 
the recommendation of using $0.65 \le \rho \le 0.85$ in practice for coverage based designs and suggests that in higher dimensions $\rho$ should be higher.

\begin{figure}[t!]
  \centering
  \subfigure[]{\includegraphics[width=0.45\textwidth, clip = true]{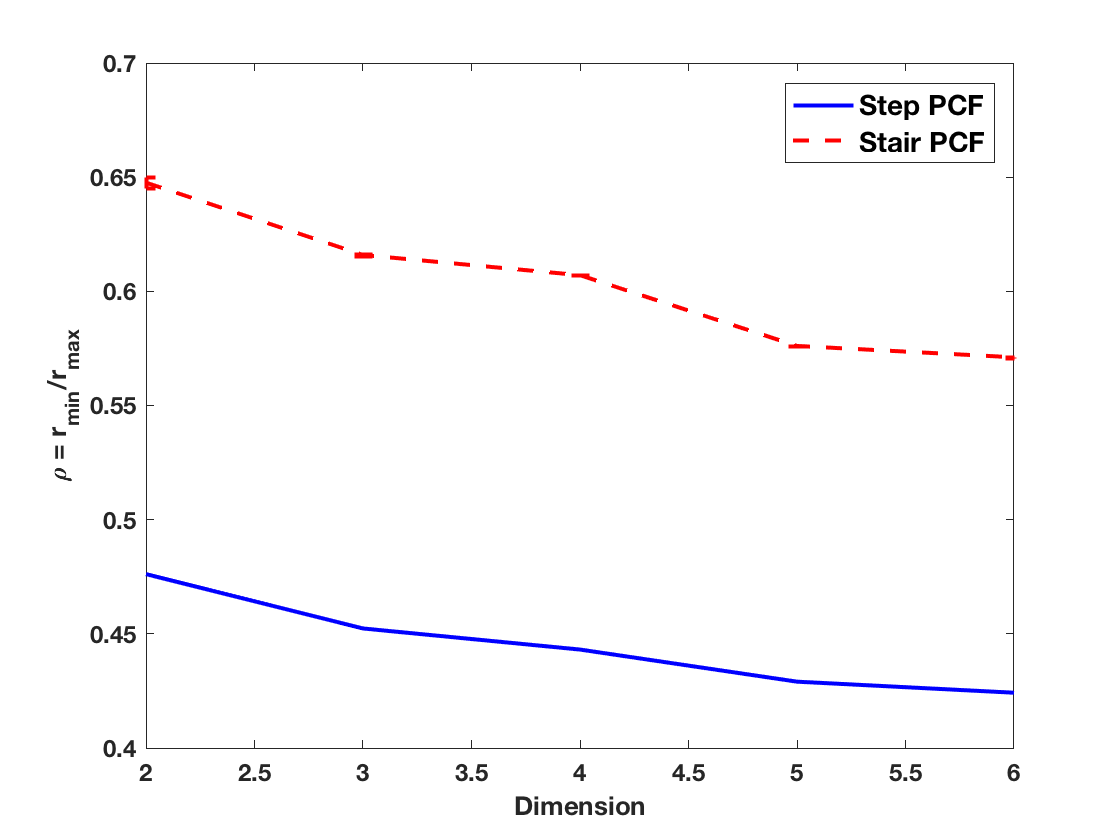}
  \label{stair_den}}
\subfigure[]{\includegraphics[width=0.45\textwidth, clip = true]{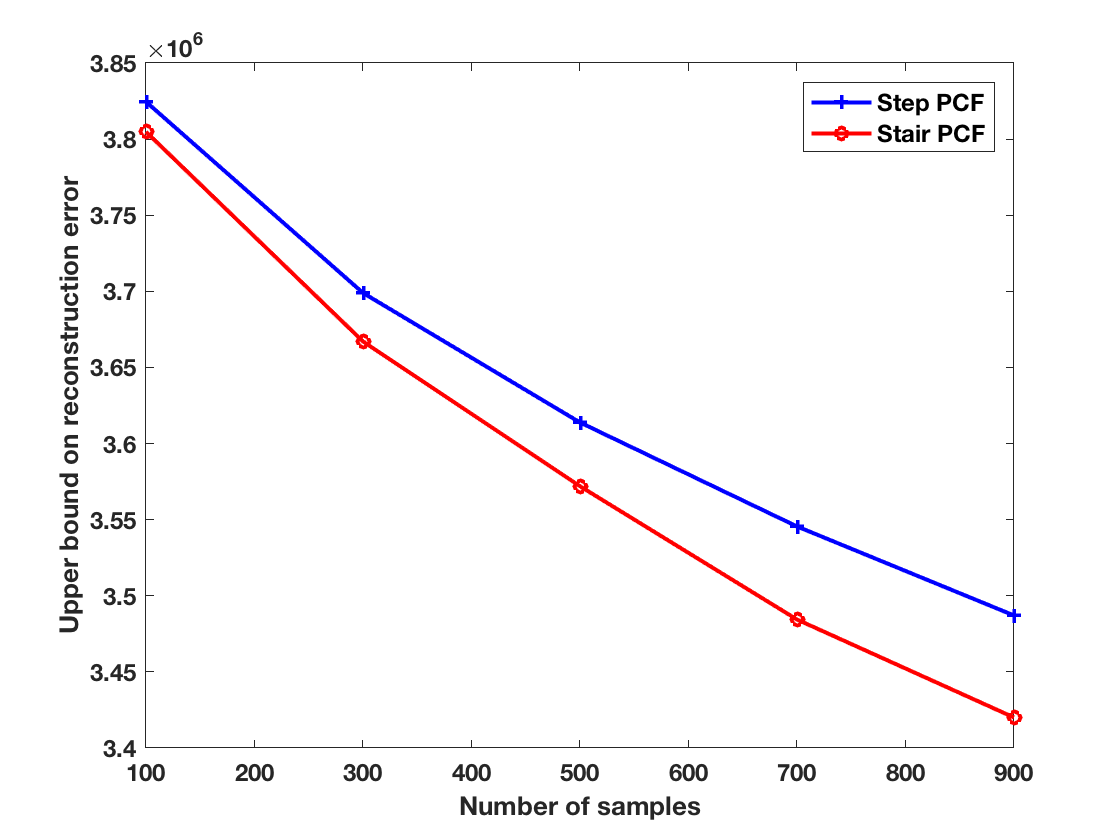}
	\label{recon_up}}
    
    \caption{(a) Gain in the relative radius $\rho$ achieved with the Stair PCF constructions, in comparison to the Step PCF constructions; (b) Upper bound on the reconstruction error of Step and Stair PCF based constructions.}
\end{figure}

\subsubsection{Analysis of Reconstruction Error Upper Bound}
We also assess the reconstruction quality of the Step and Stair PCF based spectral designs, on the class of periodic functions considered in Section \ref{sec:anal-per}, for varying sampling budgets. Here, we consider the setup where $0\leq k\leq 1000$ and $0\leq f\leq 1000$. In Figure~\ref{recon_up}, we plot the average reconstruction error upper bounds as given in~\eqref{recon:upb} for Step and Stair PCF. As expected, for both sample designs, the reconstruction error decreases with an increase in the sampling budget. More interestingly, the reconstruction error of Stair PCF is lower compared to the reconstruction error of Step PCF, thus showing the effectiveness of increased coverage in sample designs.


%% file: synthesis.tex
In this section, we describe the proposed approach for synthesizing sample designs that match the optimal (Stair or Step) PCF characteristics.
Existing approaches for PCF matching such as~\citep{Oztireli:2012, Kailkhura:2016:SBN} rely on kernel density estimators to evaluate the PCF of a point set. 
A practical limitation of these approaches is the lack of an efficient PCF estimator in high dimensions. 
More specifically, these estimators are biased
due to lack of an appropriate edge correction strategy. This bias in PCF estimation arises  due to the fact that sample hyper-spheres used in calculating point-pattern statistics may fall partially outside the study region and will produce a biased estimate of the PCF unless a correction is applied. The effect of this bias is barely noticeable in $2$ dimensions and hence existing PCF matching algorithms have ignored this. However, this problem becomes severe in higher dimensions, thus, making the matching algorithm highly inaccurate. To address this crucial limitation, we introduce an edge corrected estimator for computing the PCF of sample designs in arbitrary dimensions. Following this, we describe a gradient descent based optimization technique to synthesize samples that match the desired PCF.

\subsection{PCF Estimation in High Dimensions with Edge Correction}

In order to create an unbiased PCF estimator, we propose to employ an edge corrected kernel density estimator, defined as follows: 

\begin{equation}
  \hat{G}(r)=\frac{V_W}{\gamma_W}\frac{V_W}{N}\frac{1}{S_E
    (N-1)}\sum\limits_{i=1}^{N}\underset{i\neq
    j}{\sum\limits_{j=1}^{N}}k\left(r-|x_i-x_j|\right) \end{equation}
where $k(.)$ denotes the kernel function; here we use the classical Gaussian kernel
\begin{equation}
  k(z)=\frac{1}{\sqrt{\pi}\sigma}\exp\left(-\frac{z^2}{2\sigma^2}\right).
\end{equation}In the above expression, $V_W$ indicates the volume of the sampling region. When the sampling region is a hyper-cube with length $1$, we have $V_W = 1$. Let $S_E$ denote the area of hyper-sphere with radius $r$ which is given by
$$S_E = \dfrac{dr^{d-1}\pi^{\frac{d}{2}}}{\Gamma(1+\frac{d}{2})}.$$
Also, we denote the surface area of the sampling region by $S_W$, which is expressed as
$$S_W = r^{d-1}\sin^{d-2}\phi_1\sin^{d-3}\phi_2\cdots\sin\phi_{d-2}.$$

The term $\frac{V_W}{\gamma_W}$ performs edge correction to handle the
unboundedness of the estimator, where $\gamma_W$ is an isotropic set covariance function given by
\begin{equation}
\gamma_W=\dfrac{1}{S_E} \underset{ \underset{0\leq\phi_{i}\leq \pi,i=1\; \text{to}\; d-2}{0\leq\phi_{d-1}\leq 2\pi}}{\int} S_W \gamma d\phi_1\cdots d\phi_{d-1} 
\end{equation}
where $\gamma = \prod_{p=1}^d (1-|x^p|)$ with
\begin{eqnarray*}
&& x^1 = r \cos \phi_1\\
&& x^2 = r \sin \phi_1\cos \phi_2\\
&& x^3 = r \sin \phi_1 \sin \phi_2 \cos \phi_3\\
&& \vdots\\
&& x^{d-1} = r \sin \phi_1 \cdots \sin \phi_{d-2} \cos \phi_{d-1}\\
&& x^{d} = r \sin \phi_1 \cdots \sin \phi_{d-2} \sin \phi_{d-1}.
\end{eqnarray*}

\begin{figure*}[t!]
	\vspace*{-0.2in}
	\centering
	\subfigure[]{\includegraphics[width=0.45\columnwidth, clip = true]{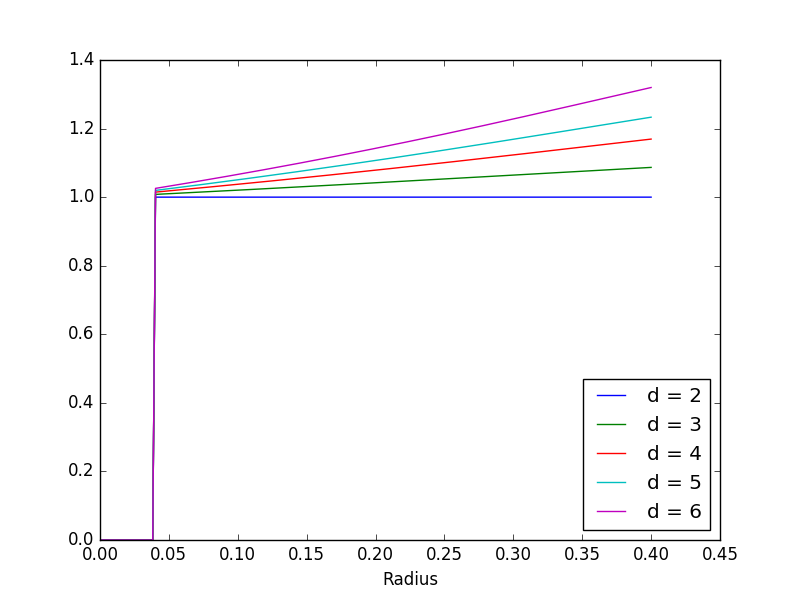}
		\label{mismatch}}
	\subfigure[$d=2$]{
		\includegraphics[
		width=0.45\textwidth,clip=true]{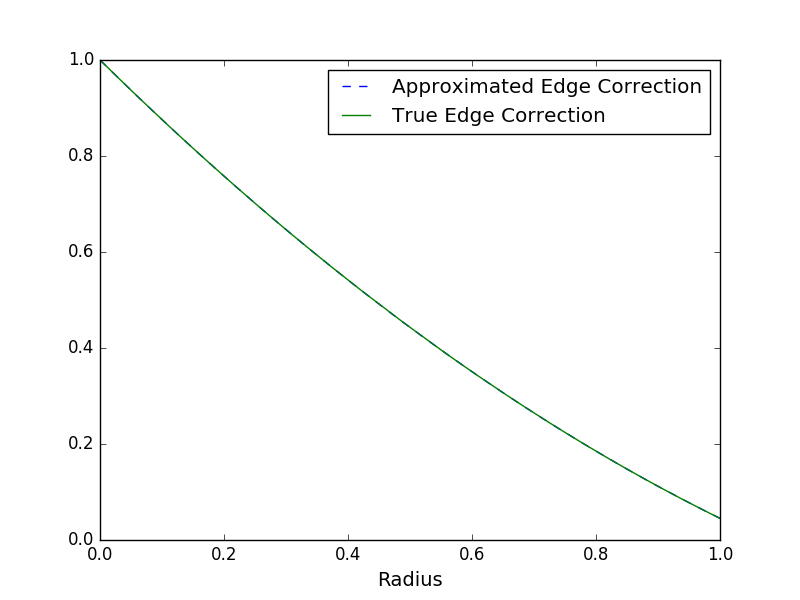}
		\label{ec2} } 
	\subfigure[$d=3$]{
		\includegraphics[%
		width=0.45\textwidth,clip=true]{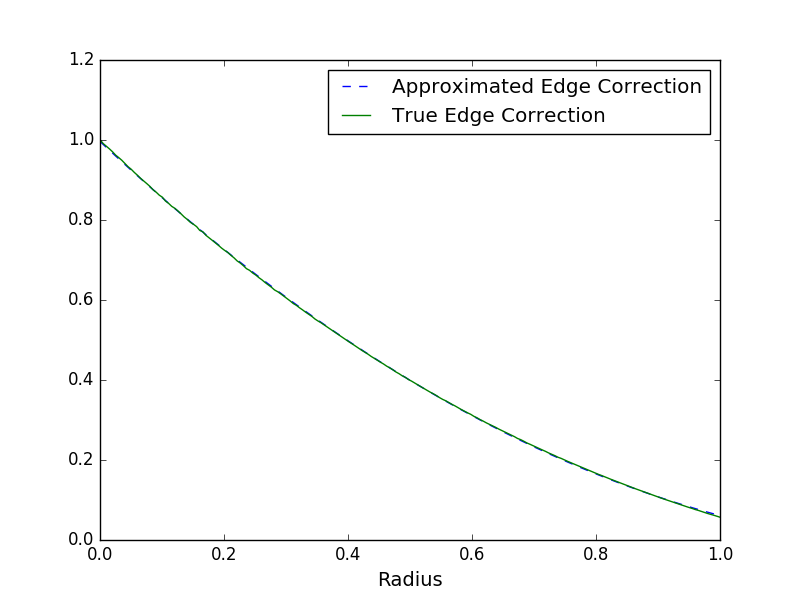}
		\label{ec3} }
	\subfigure[$d=4$]{
		\includegraphics[
		width=0.45\textwidth,clip=true]{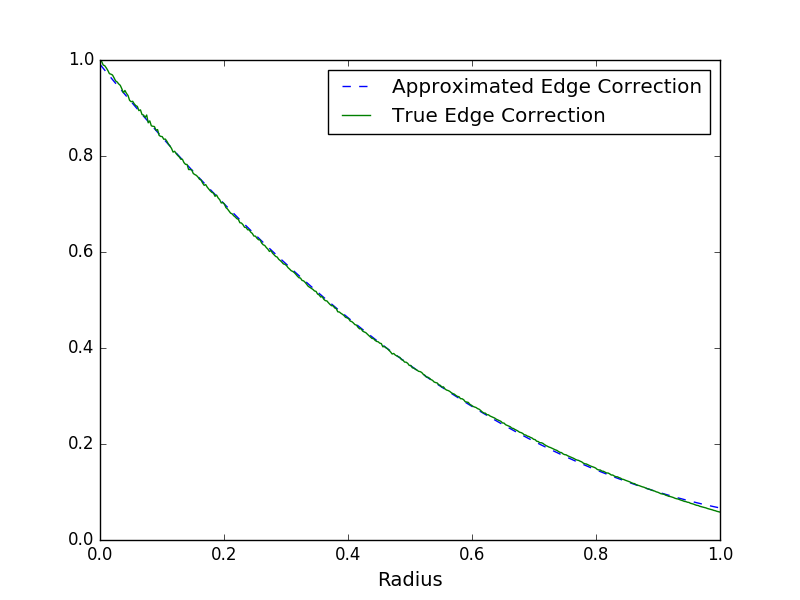}
		\label{ec4} }
	\subfigure[$d=5$]{
		\includegraphics[
		width=0.45\textwidth,clip=true]{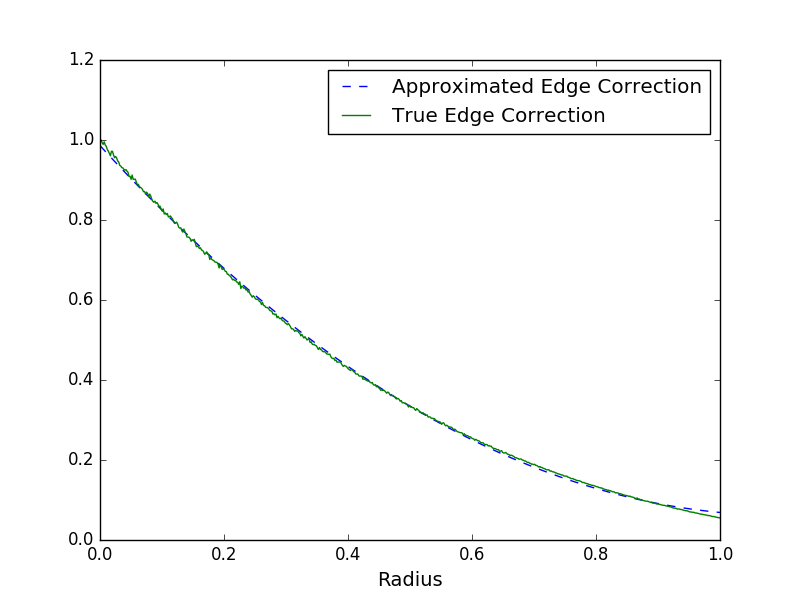}
		\label{ec5} }
	\subfigure[$d=6$]{
		\includegraphics[
		width=0.45\textwidth,clip=true]{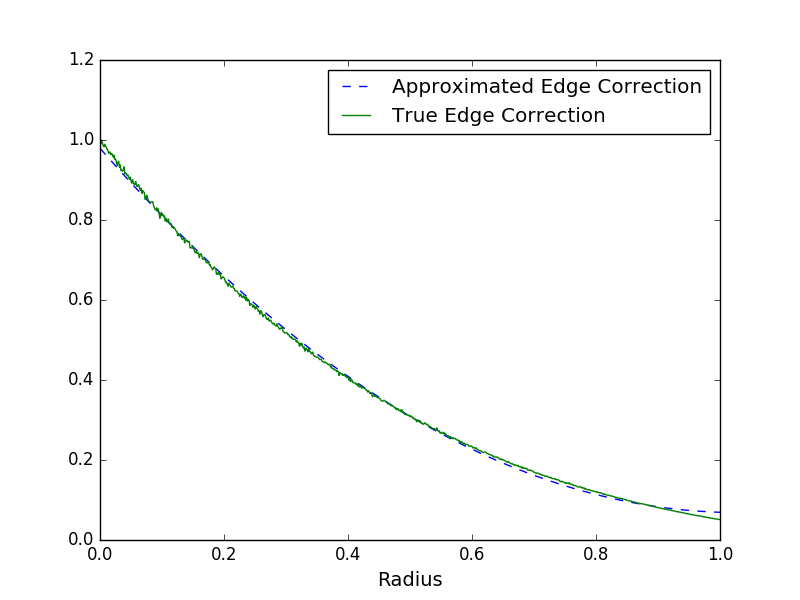}
		\label{ec6} }
	
	\caption{(a) Incorrect PCF estimation due to the use of an approximate edge correction factor, (b)-(f) Effectiveness of the approximation edge correction, obtained using polynomial regression, in comparison to the true edge correction from the evaluation of a multi-dimensional integral, for dimensions $2$ to $6$.}    
	
\end{figure*}

In Figure~\ref{mismatch}, we show that by using an approximate edge correction factor (using the same factor as $d=2$), the PCF is wrongly estimated. Moreover, as the dimension increases, the estimated PCF moves farther away from the true PCF very quickly. 

Note that, the calculation of the correct edge correction factor requires the evaluation of a multi-dimensional integral which is computationally expensive in high dimensions. In this paper, we provide a closed form approximation of $\gamma_W$ (using polynomial regression of order $2$) in different dimensions $d=2$ to $6$ when $r \leq 1.0$. More specifically, we have the following approximation $\hat{\gamma_W} = 1 - a_1 r +a_2 r^2$ where $a_1$ and $a_2$ are as given below.

\begin{center}
    \begin{tabular}{| l | l | l | l | l | l |}
    \hline
    Dimension & $d=2$ & $d=3$ & $d=4$ & $d=5$ & $d=6$ \\ \hline
    $a_1$ & $4/\pi$ & $1.47$ & $1.63$ & $1.75$ & $1.89$  \\ \hline
    $a_2$ & $1/\pi$ & $0.54$ & $0.72$ & $0.87$ & $1.04$ \\ \hline
    \end{tabular}
\end{center}It can be observed from Figures~\ref{ec2} through~\ref{ec6} that the proposed approximations are quite tight.

\subsection{Synthesis Algorithm}
The underlying idea of the proposed algorithm is to iteratively transform an initial random input sample design such that its PCF matches the target PCF. In particular, we propose a non-linear least squares formulation to optimize for the desired space-filling properties. Let us denote the target PCF by $G^*(r)$. We discretize the radius $r$ into $m$ points $\{r_j\}_{j=1}^{m}$ and minimize the sum of the weighted squares of errors between the target PCF $G^*(r_j)$ and the curve-fit function (kernel density estimator of PCF) $G(r_j)$ over $m$ points. This scalar-valued goodness-of-fit measure is referred to as the
\textit{chi-squared error} criterion and can be posed as a non-linear weighted least squares problem as follows.
$$\arg\min\;\; \sum\limits_{j=1}^{M}\left(\dfrac{G(r_j)-G^*(r_j)}{w_j}\right)^2,$$
where $w_j$ indicates the weight (importance) assigned to the fitting error at radius $r_j$.
This optimization problem can be efficiently solved using a variant of gradient descent algorithm (discussed next), that in our experience converges quickly. 
In the simplest cases of uniform weights the
solution tends to produce a higher fitting error at lower radii
$r_j$. To address this challenge we use a
non-uniform distribution for the weights $\{w_j\}$. These weights 
are initialized to be uniform and are updated in an adaptive fashion in the gradient descent iterations. The weight $w_j$ at gradient descent iteration $t+1$ is given by~\citep{Kailkhura:2016:SBN}:
$$w_j=\dfrac{1}{|G^{t}(r_j)-G^*(r_j)|}$$
where $G^t(r_j)$ is the value of the PCF at radius $r_j$ during the gradient descent iteration $t$.
Note that, PCF matching is a highly non-convex problem. We found that the following trick further helps solve PCF matching problem more efficiently.

\subsubsection{One Sided PCF Smoothing}  
We propose to perform one sided smoothing of the target PCF which is given as follows:

\[ \hat{G}^*(r) = \left\{ \begin{array}{rll}
				(cr)^b  & \mbox{if}\ r< r_{\text{min}} \\
			   1  & \mbox{if}\ r \geq r_{\text{min}}.
				\end{array}\right. 
\]
where $c$ is some pre-specified constant and $b>1$ is the smoothing constant obtained via cross-validation. More specifically, we add polynomial noise in the low radius region of the PCF. This can also be interpreted as polynomial approximation of the PCF in the low radii regime. We have noticed that sometimes adding a controlled amount of Gaussian noise instead of polynomial noise also improves the performance.

\subsubsection{Edge Corrected Gradient Descent}  
The non-linear least squares problem is solved iteratively using gradient descent. Starting with a random point set $X=\{x_i\}_{i=1}^N$, we iteratively
update $x_i$ in the negative gradient direction of the objective
function. At each iteration $k$, this can be formally stated as
$$x_i^{k+1}=x_i^k-\lambda\dfrac{\Delta_i}{|\Delta_i|},$$
where $\lambda$ is the step size and $\Delta_i=\{\Delta_i^k\}_{k=1}^d$ in the normalized edge corrected gradient is given by

\begin{align}
\nonumber \Delta_i^p=\underset{i\neq l}{\sum}&\dfrac{(x_l^p-x_i^p)}{|x_l-x_i|}\sum_{j=1}^{m} \dfrac{G(r_j)^k-G^*(r_j)}{w_j(1-a_1r_j+a_2r_j^2)r_j^{d-1}} \\
&\left(|x_l-x_i|-r_j\right)   k\left(r_j-|x_i-x_l|\right). \label{gradient}
\end{align}
We re-evaluate the PCF $G(r_j)^k$ of the updated
point set after each iteration using the unbiased estimator from the
previous section.

\begin{figure}[t!]
\centering
\subfigure[]{
\includegraphics[%
  width=0.45\textwidth,clip=true]{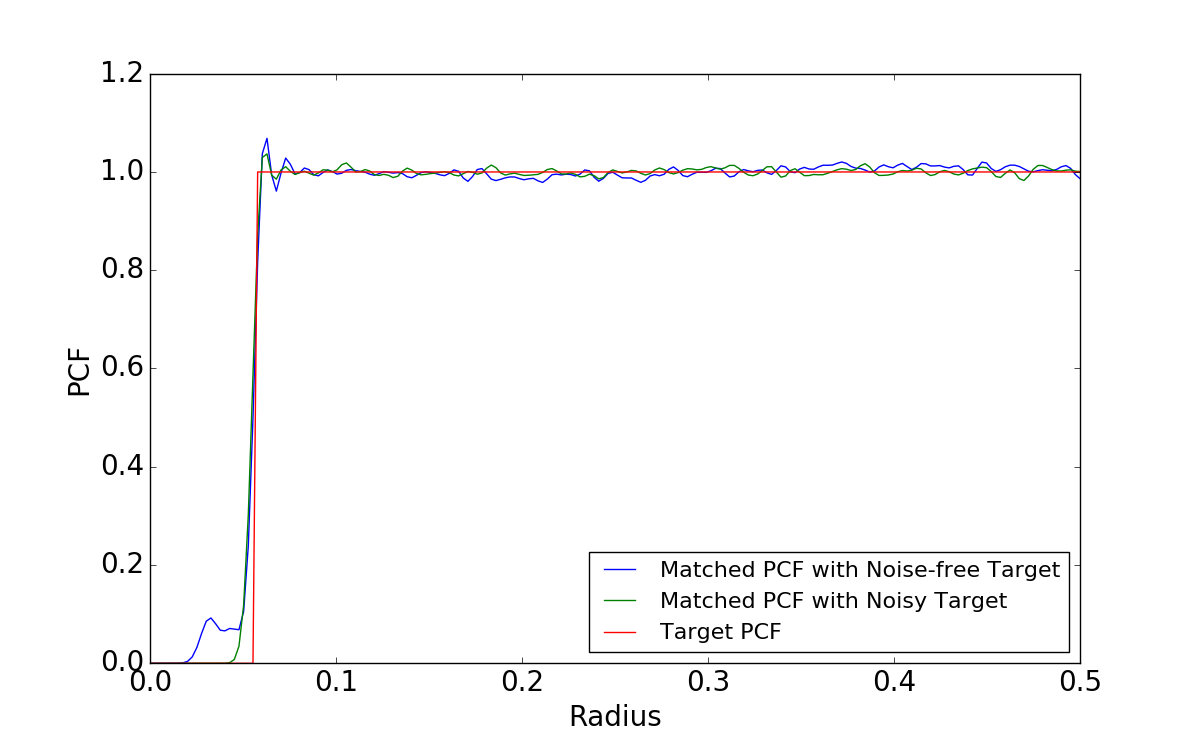}
\label{pcf_noise2}}
\subfigure[] {
\includegraphics[%
width=0.45\textwidth,clip=true]{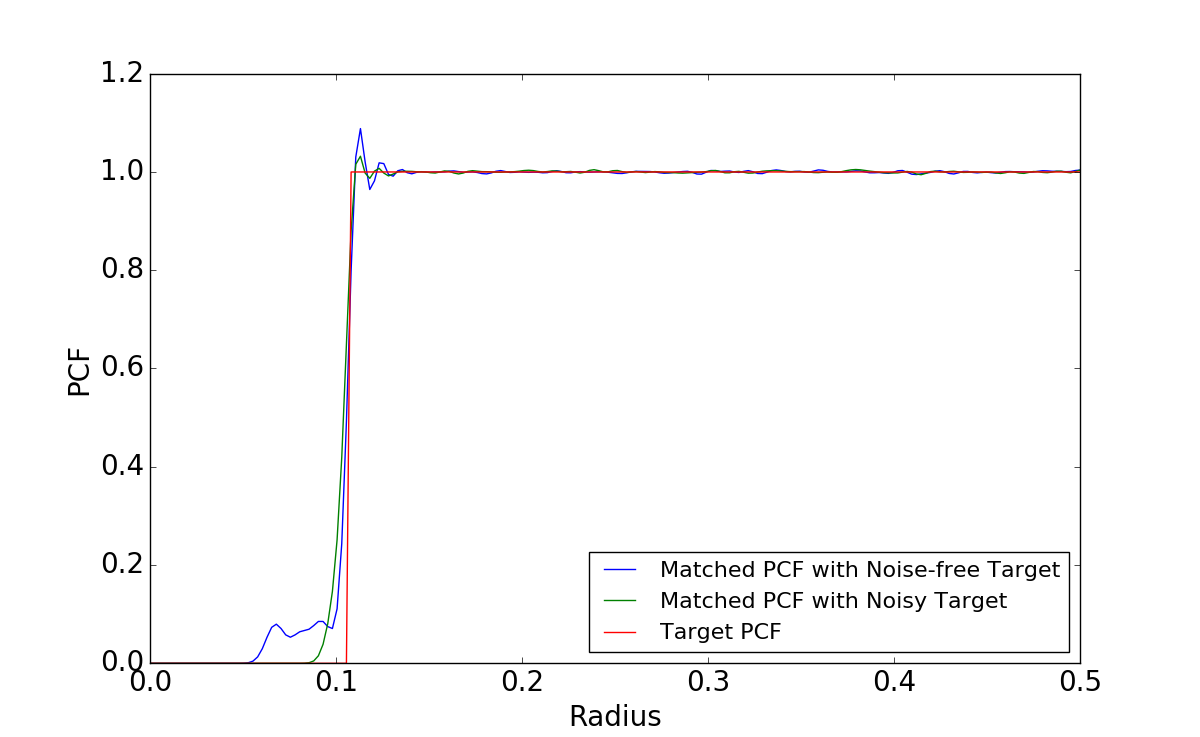}
\label{pcf_noise3} }
\subfigure[]{
\includegraphics[%
width=0.45\textwidth,clip=true]{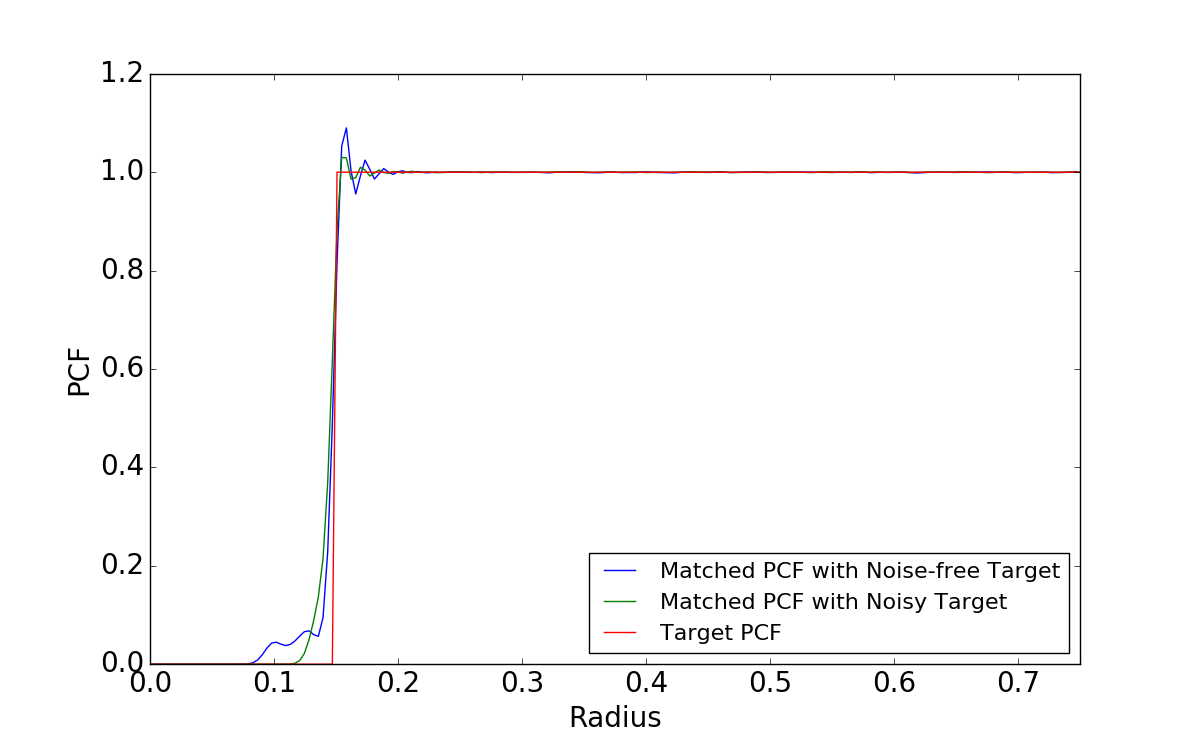}
\label{pcf_noise4}}
\subfigure[]{
\includegraphics[%
width=0.45\textwidth,clip=true]{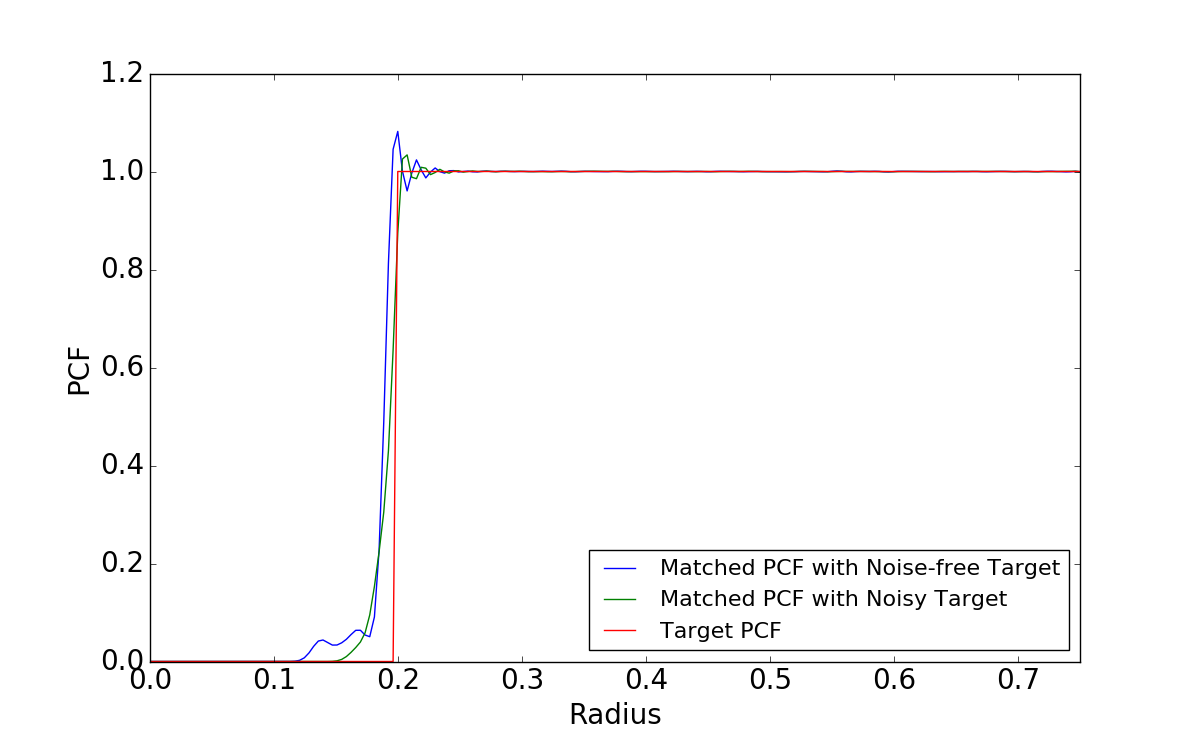}
\label{pcf_noise5}}
\subfigure[]{
\includegraphics[%
width=0.45\textwidth,clip=true]{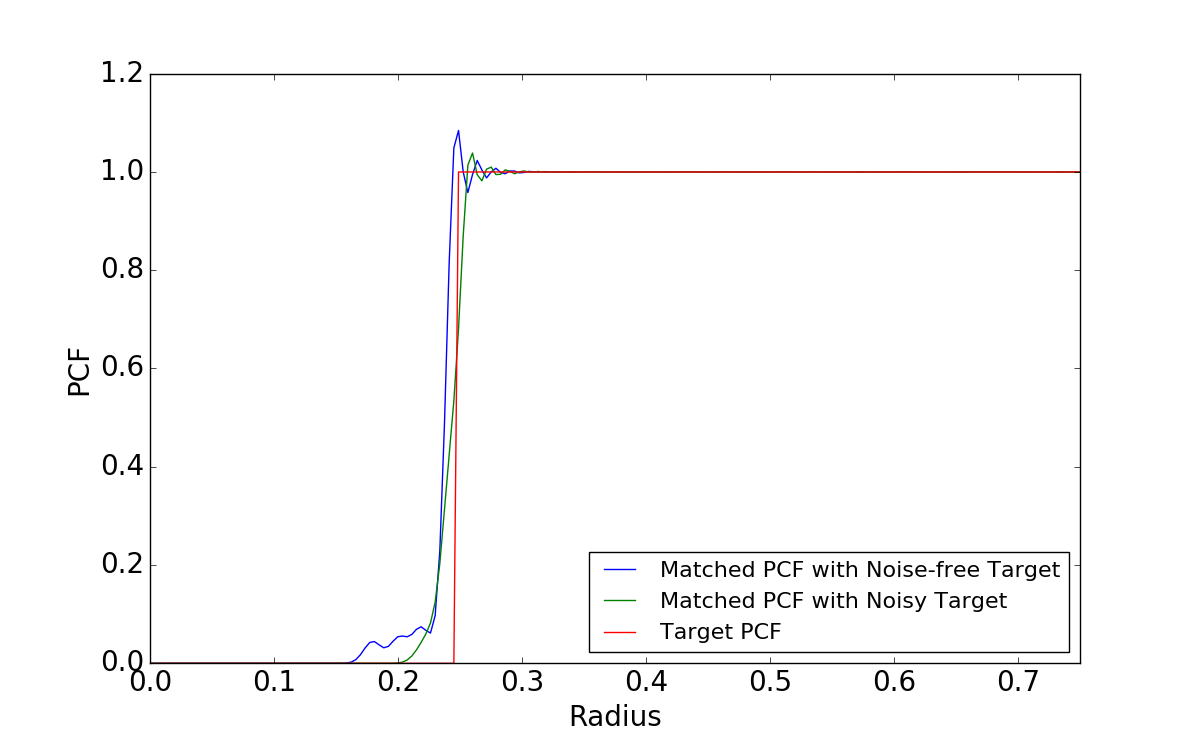}
\label{pcf_noise6}}
\caption{Step PCF synthesis using one sided PCF smoothing technique. \subref{pcf_noise2} $d=2$ \subref{pcf_noise3} $d=3$ \subref{pcf_noise4} $d=4$ \subref{pcf_noise5} $d=5$ \subref{pcf_noise6} $d=6$.}
\label{pcf_noise}
\vspace{-0.1in}
\end{figure}

\begin{figure}[t!]
\centering
\subfigure[]{
\includegraphics[%
  width=0.45\textwidth,clip=true]{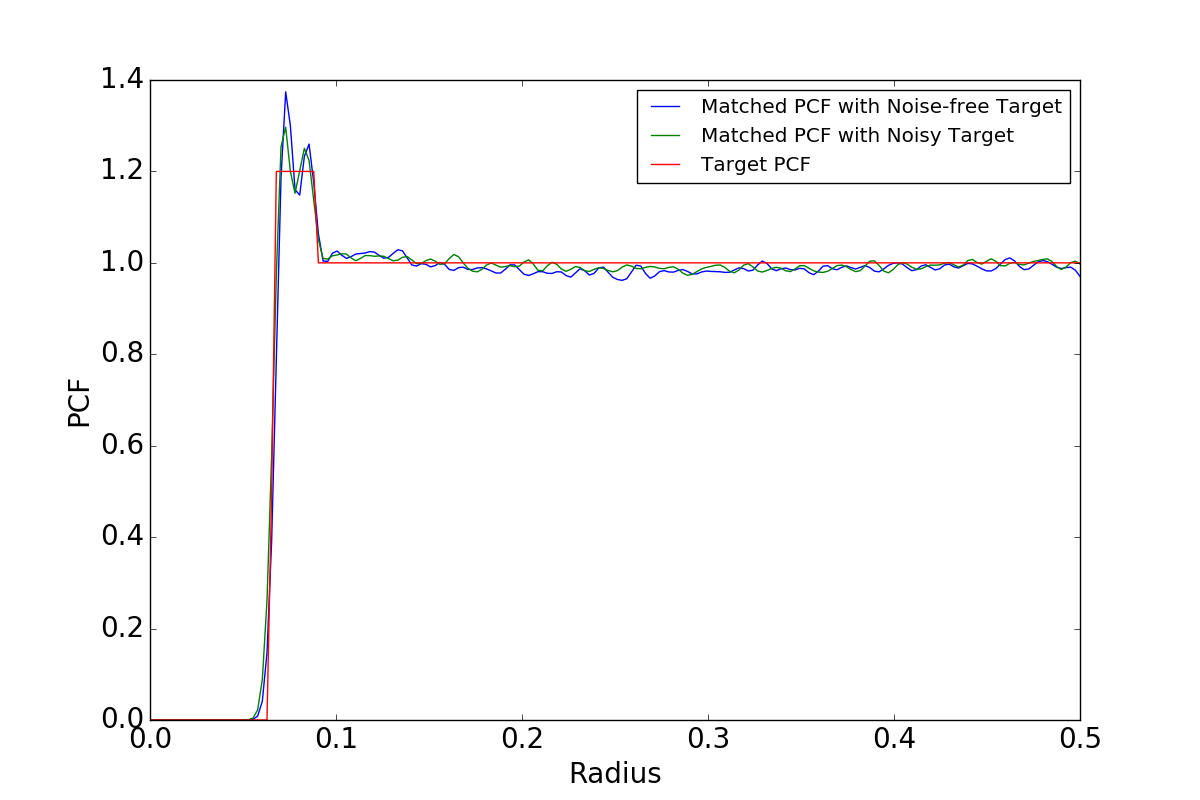}
\label{spcf_noise2}}
\subfigure[] {
\includegraphics[%
width=0.45\textwidth,clip=true]{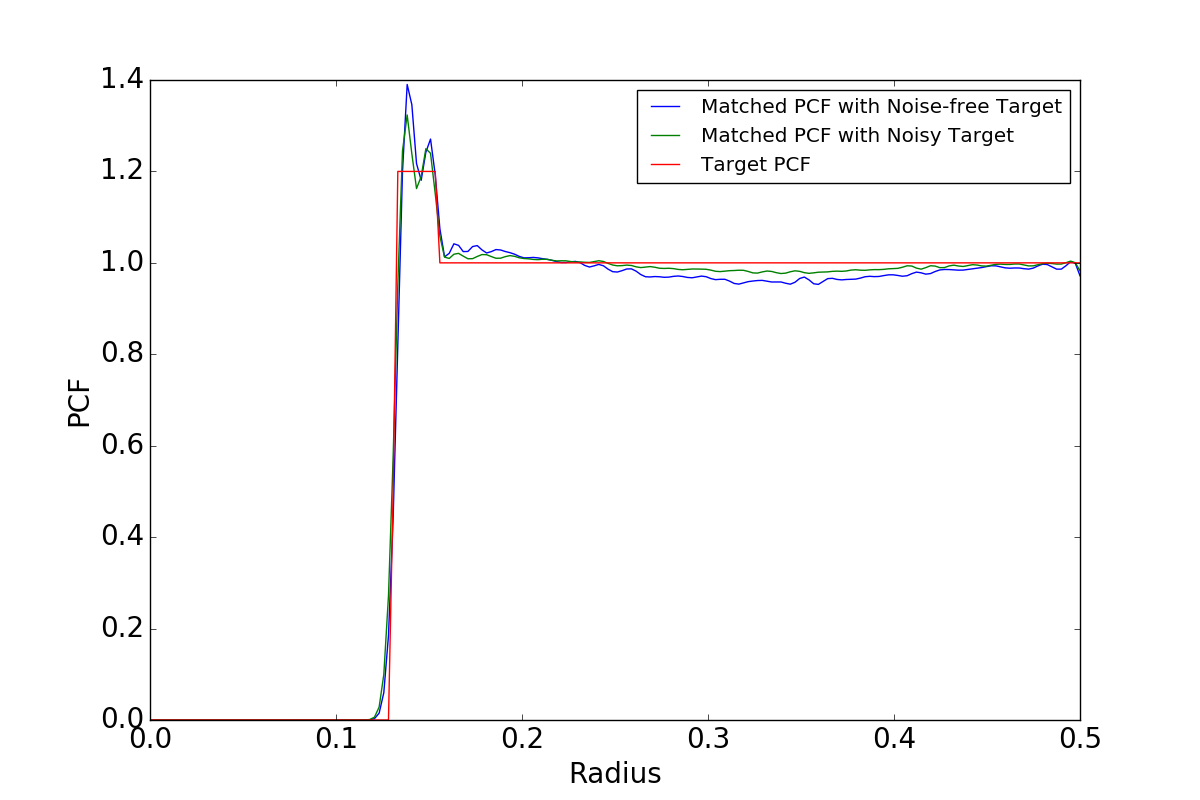}
\label{spcf_noise3} }
\subfigure[]{
\includegraphics[%
width=0.45\textwidth,clip=true]{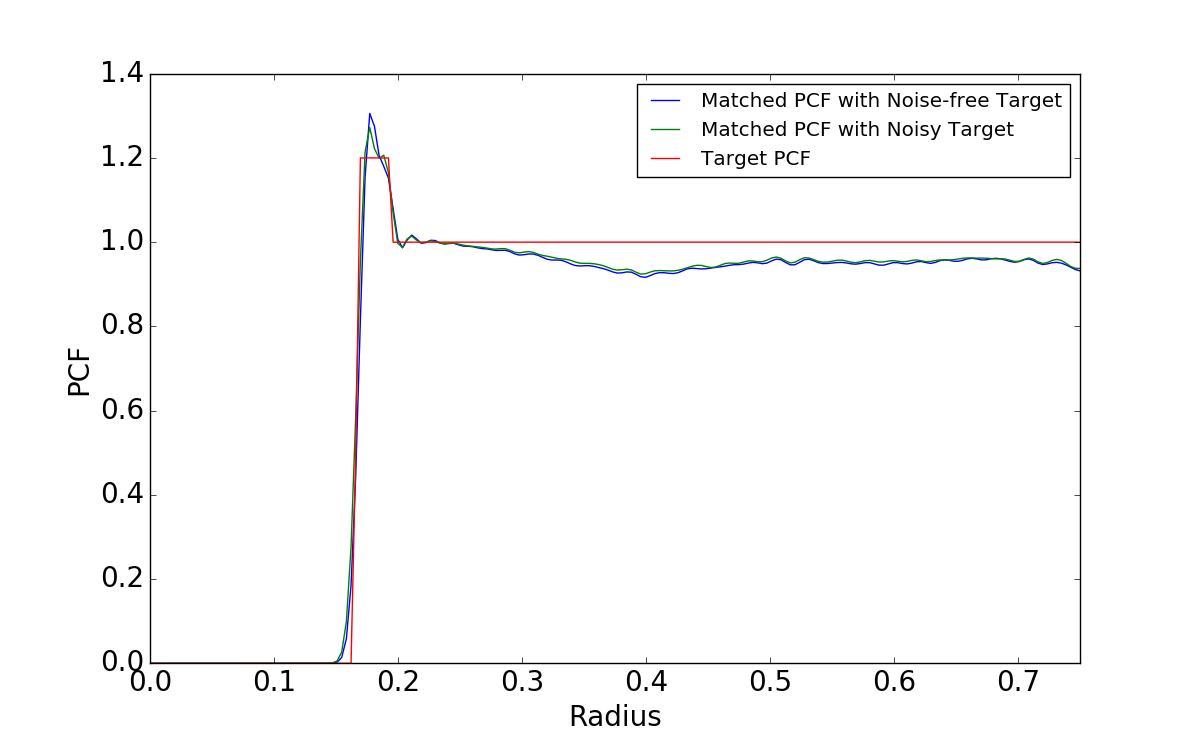}
\label{spcf_noise4}}
\subfigure[]{
\includegraphics[%
width=0.45\textwidth,clip=true]{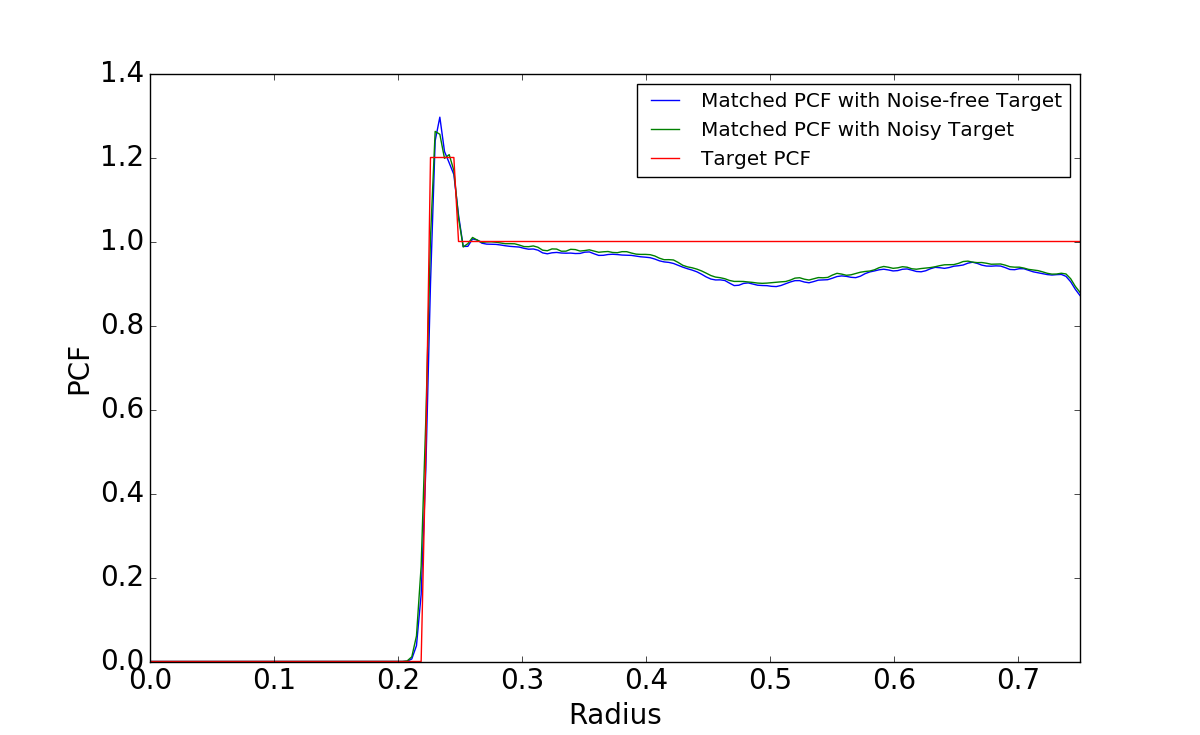}
\label{spcf_noise5}}
\subfigure[]{
\includegraphics[%
width=0.45\textwidth,clip=true]{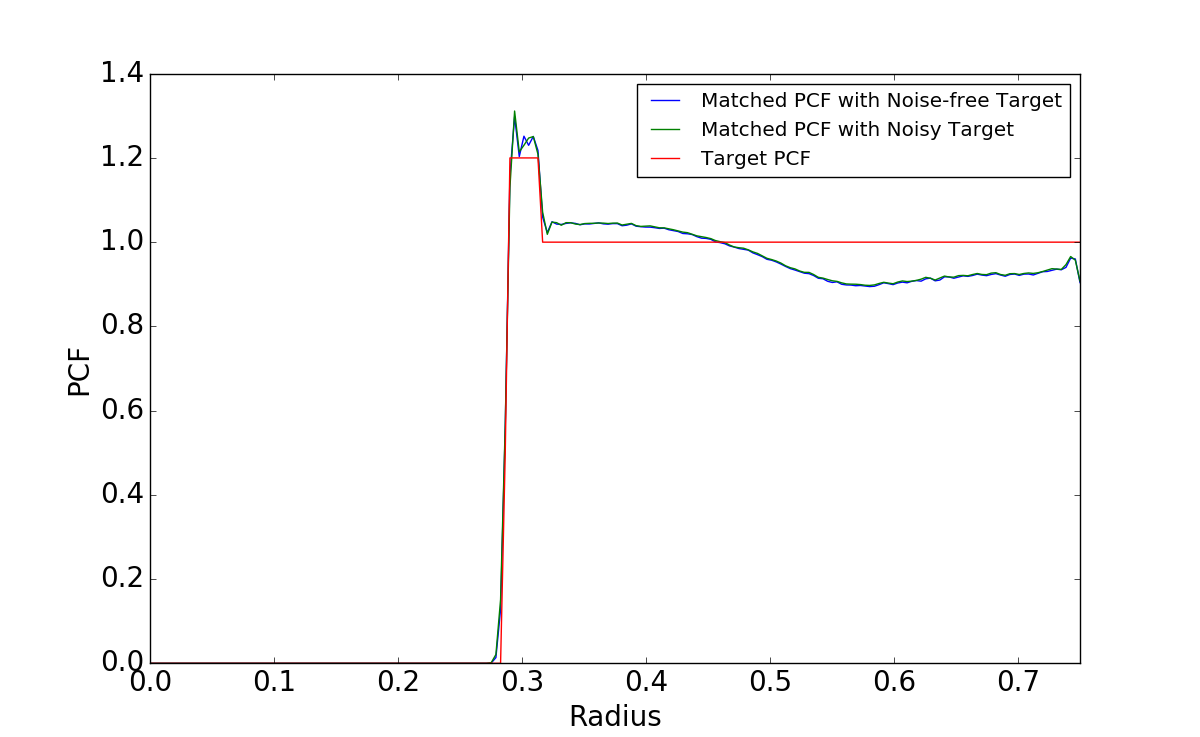}
\label{spcf_noise6}}
\caption{Stair PCF synthesis using one sided PCF smoothing technique. \subref{spcf_noise2} $d=2$ \subref{spcf_noise3} $d=3$ \subref{spcf_noise4} $d=4$ \subref{spcf_noise5} $d=5$ \subref{spcf_noise6} $d=6$.}
\label{spcf_noise}
\vspace{-0.1in}
\end{figure}

In Figure~\ref{pcf_noise}, we compare the behavior of the proposed PCF matching algorithm with and without the one sided PCF smoothing. The target PCF is designed using a Step PCF design with $r_{min}$ as given in Proposition~\ref{prp2}. PCF matching is carried out with varying sampling budget, $N=100,\;200,\; 400,\;600,\;800$ for $d=2,\;3,\;4,\;5,\;6$, respectively. The variances of the Gaussian kernel were set at $\sigma^2=0.0065,\; 0.007,\;0.01,\;0.01,\;0.01$ for $d=2,\;3,\;4,\;5,\;6$, respectively and the step size for the gradient descent algorithm was fixed at $0.001$. The value of $b$ was obtained using cross-validation. The initial point set was generated randomly (uniform) in the unit hyper-cube and matching was carried for $100$ gradient descent iterations. It can be observed that the proposed algorithm produces an accurate fit to the target, and that the smoothing actually leads to improved performance. 

In Figure~\ref{spcf_noise}, we demonstrate the synthesis of a Stair PCF based spectral design, using parameters $P_0=1.2,\delta = 0.025$. Similar to the previous case, PCF matching is carried out with varying sampling budget, $N=100,\;200,\; 400,\;600,\;800$, for $d=2,\;3,\;4,\;5,\;6$ respectively. The variances of the Gaussian kernel were set at $\sigma^2=0.0065,\; 0.007,\;0.01,\;0.01,\;0.01$ for $d=2,\;3,\;4,\;5,\;6$, respectively and the step size for the gradient descent algorithm was fixed at $0.001$. We found that matching the Stair PCF is more challenging for a gradient descent optimization compared to the Step PCF. When a random point set is used for initialization, reaching convergence takes much longer. However, choosing the initial point set intelligently improves the quality of matching significantly. In all our experiments, we used the maximal PDS \citep{Ebeida:2012} to initialize the optimization and matching was carried for $100$ gradient descent iterations. We observed that another reasonable choice for the initialization is a regular grid sample, and interestingly in most cases it matches the performance of the MPDS initialization. Furthermore, one sided PCF smoothing does not provide significant improvements in this case, particularly in higher dimensions.

%% file: exp.tex
In this section, we evaluate the qualitative performance of proposed space-filling spectral designs and present comparisons to popularly adopted space-filling designs, such as LHS, QMC and MPDS. Note that, currently, there does not exist any PDS synthesis approach which can generate sample sets with a desired size $N$ while achieving user-specified spatial characteristics (e.g. relative radius). In all PDS synthesis approaches, there is no control over the number of samples generated by the algorithm which makes the use of these algorithms difficult in practice. However, the proposed approach can control both $N$ and $r_{min}$ simultaneously. For our qualitative comparison, we perform three empirical studies, in dimensions $2$ to $6$ : (a) image reconstruction, (b) regression on several benchmark optimization functions, and (c) surrogate modeling for an inertial confinement fusion (ICF) simulation code.

\begin{table}[htb!]
	\centering
	\caption{Impact of different space-filling designs on image reconstruction performance. In all cases, we show the reconstructed images and their PSNR values.}
	\label{table:recon}
\begin{tabularx}{\textwidth}{@{}lYYYYYY@{}}
	\toprule
	$\boldsymbol{\alpha}$ & \textbf{Sobol} & \textbf{Halton} & \textbf{LHS} & \textbf{MPDS} & \textbf{Step} & \textbf{Stair}\\
	\hline
	
	 \colorbox{gray!25}{0.001}&\raisebox{-.5\height}{\includegraphics[trim = 3cm 2.5cm 3cm 0cm ,width=0.12\textwidth,clip=True]{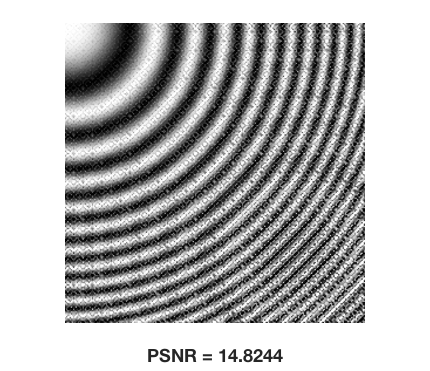}}&\raisebox{-.5\height}{\includegraphics[trim = 3cm 2.5cm 3cm 0cm ,width=0.12\textwidth,clip=True]{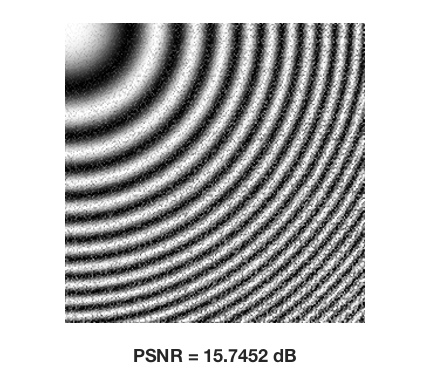}}&\raisebox{-.5\height}{\includegraphics[trim = 3cm 2.5cm 3cm 0cm ,width=0.12\textwidth,clip=True]{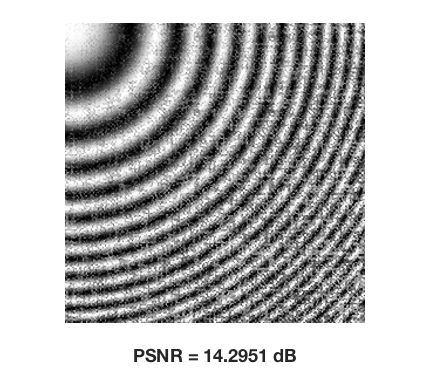}}&\raisebox{-.5\height}{\includegraphics[trim = 3cm 2.5cm 3cm 0cm ,width=0.12\textwidth,clip=True]{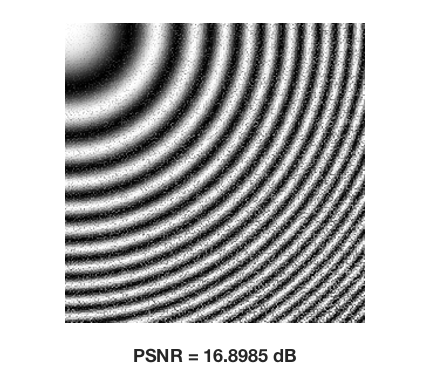}}&\raisebox{-.5\height}{\includegraphics[trim = 3cm 2.5cm 3cm 0cm ,width=0.12\textwidth,clip=True]{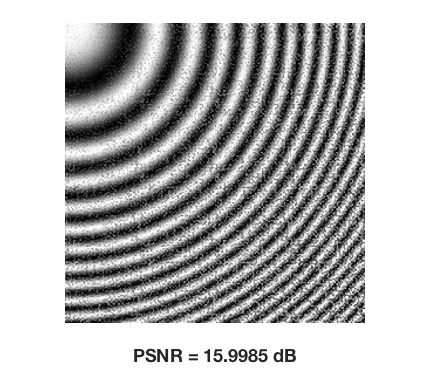}}&\raisebox{-.5\height}{\includegraphics[trim = 3cm 2.5cm 3cm 0cm ,width=0.12\textwidth,clip=True]{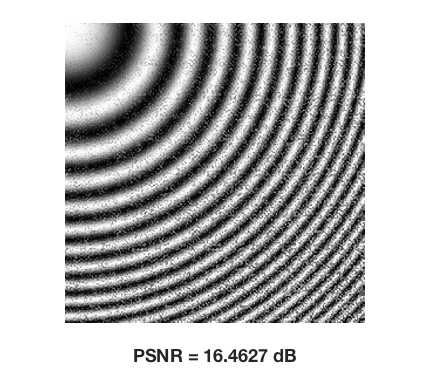}}\\
	 &\footnotesize 14.82 dB & \footnotesize 15.75 dB & \footnotesize 14.29 dB &\footnotesize 16.90 dB & \footnotesize 16.00 dB & \footnotesize 16.46 dB \\
	 
	 \hline
	 
	 \colorbox{gray!25}{0.002}&\raisebox{-.5\height}{\includegraphics[trim = 3cm 2.5cm 3cm 0cm ,width=0.12\textwidth,clip=True]{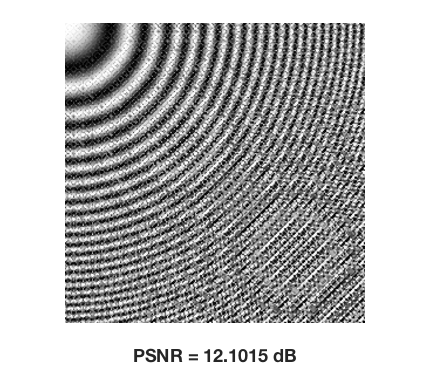}}&\raisebox{-.5\height}{\includegraphics[trim = 3cm 2.5cm 3cm 0cm ,width=0.12\textwidth,clip=True]{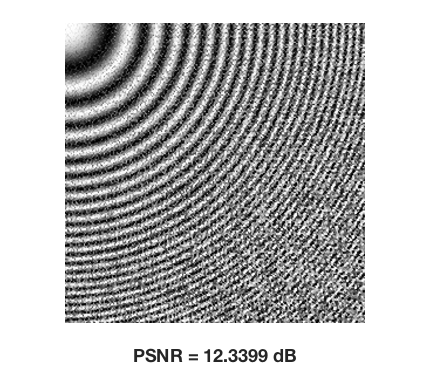}}&\raisebox{-.5\height}{\includegraphics[trim = 3cm 2.5cm 3cm 0cm ,width=0.12\textwidth,clip=True]{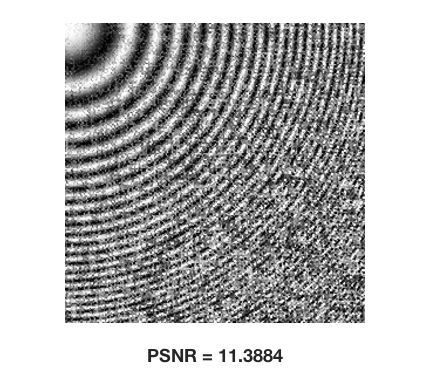}}&\raisebox{-.5\height}{\includegraphics[trim = 3cm 2.5cm 3cm 0cm ,width=0.12\textwidth,clip=True]{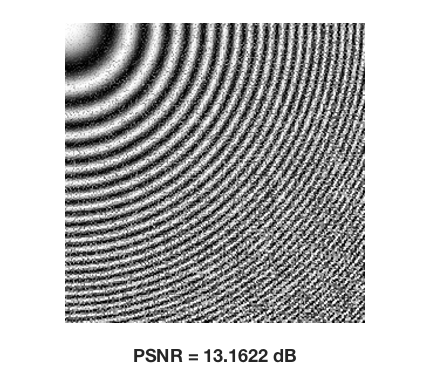}}&\raisebox{-.5\height}{\includegraphics[trim = 3cm 2.5cm 3cm 0cm ,width=0.12\textwidth,clip=True]{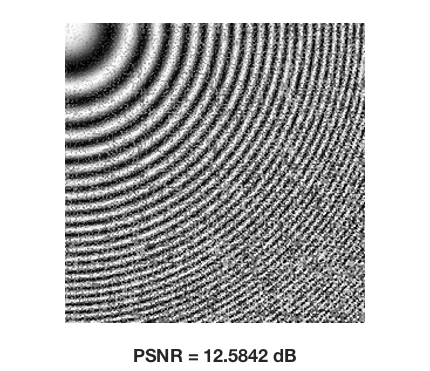}}&\raisebox{-.5\height}{\includegraphics[trim = 3cm 2.5cm 3cm 0cm ,width=0.12\textwidth,clip=True]{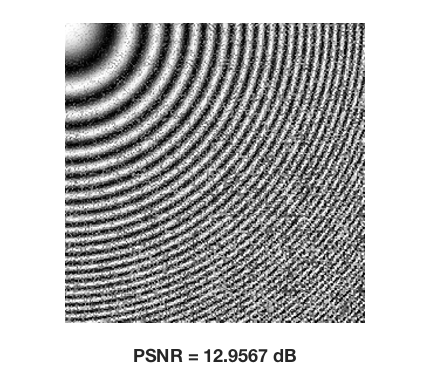}}\\
	 &\footnotesize 12.10 dB & \footnotesize 12.34 dB & \footnotesize 11.39 dB &\footnotesize 13.16 dB & \footnotesize 12.58 dB & \footnotesize 12.96 dB \\
	 
	 \hline
	 
	 \colorbox{gray!25}{0.003}&\raisebox{-.5\height}{\includegraphics[trim = 3cm 2.5cm 3cm 0cm ,width=0.12\textwidth,clip=True]{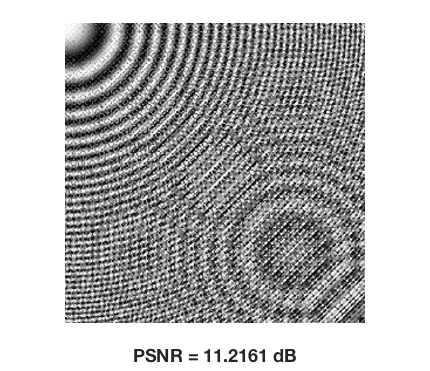}}&\raisebox{-.5\height}{\includegraphics[trim = 3cm 2.5cm 3cm 0cm ,width=0.12\textwidth,clip=True]{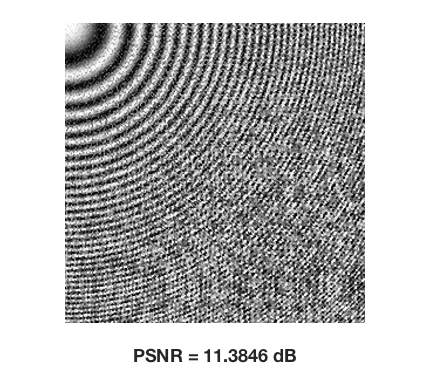}}&\raisebox{-.5\height}{\includegraphics[trim = 3cm 2.5cm 3cm 0cm ,width=0.12\textwidth,clip=True]{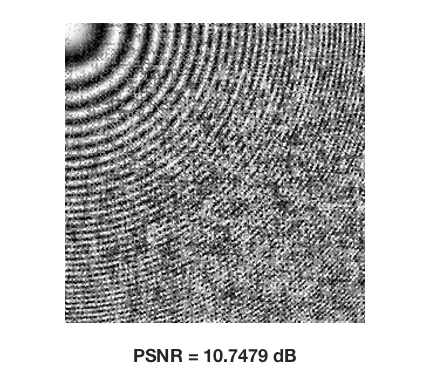}}&\raisebox{-.5\height}{\includegraphics[trim = 3cm 2.5cm 3cm 0cm ,width=0.12\textwidth,clip=True]{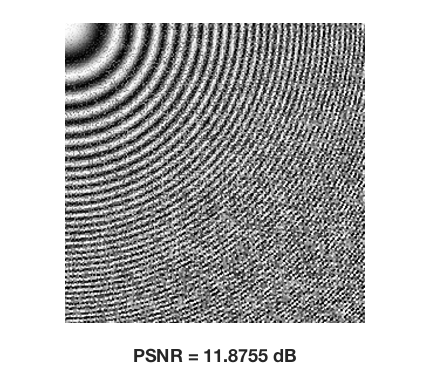}}&\raisebox{-.5\height}{\includegraphics[trim = 3cm 2.5cm 3cm 0cm ,width=0.12\textwidth,clip=True]{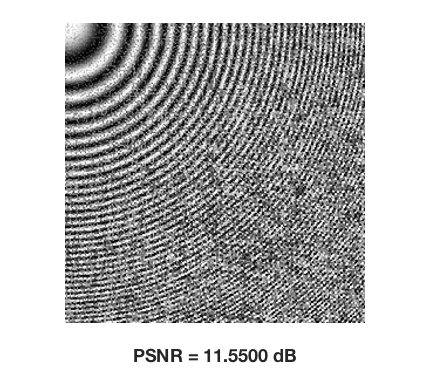}}&\raisebox{-.5\height}{\includegraphics[trim = 3cm 2.5cm 3cm 0cm ,width=0.12\textwidth,clip=True]{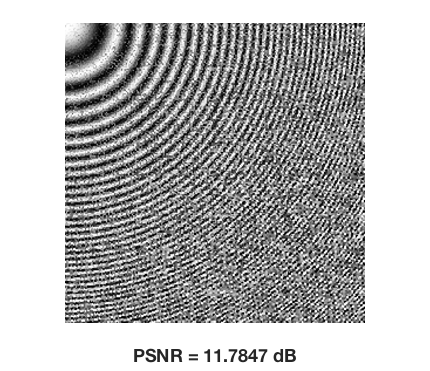}}\\
	 &\footnotesize 11.22 dB & \footnotesize 11.38 dB & \footnotesize 10.75 dB &\footnotesize 11.88 dB & \footnotesize 11.55 dB & \footnotesize 11.78 dB \\
	 
	 \hline
	 
	\colorbox{gray!25}{ 0.004}&\raisebox{-.5\height}{\includegraphics[trim = 3cm 2.5cm 3cm 0cm ,width=0.12\textwidth,clip=True]{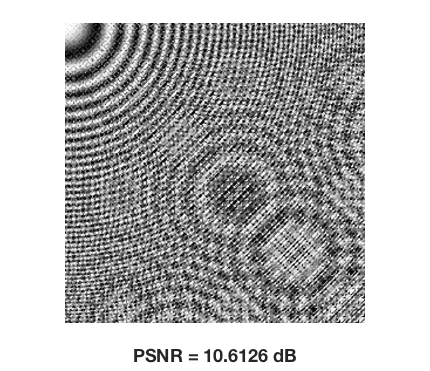}}&\raisebox{-.5\height}{\includegraphics[trim = 3cm 2.5cm 3cm 0cm ,width=0.12\textwidth,clip=True]{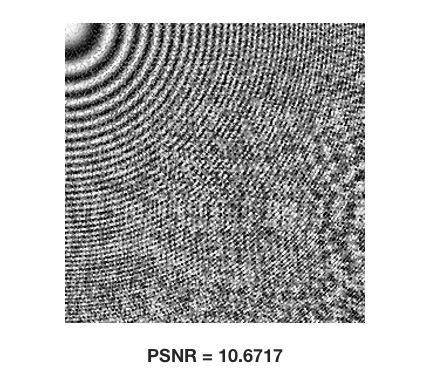}}&\raisebox{-.5\height}{\includegraphics[trim = 3cm 2.5cm 3cm 0cm ,width=0.12\textwidth,clip=True]{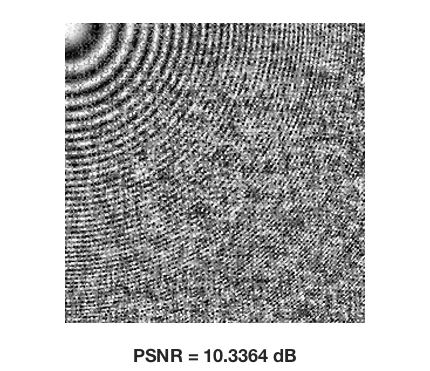}}&\raisebox{-.5\height}{\includegraphics[trim = 3cm 2.5cm 3cm 0cm ,width=0.12\textwidth,clip=True]{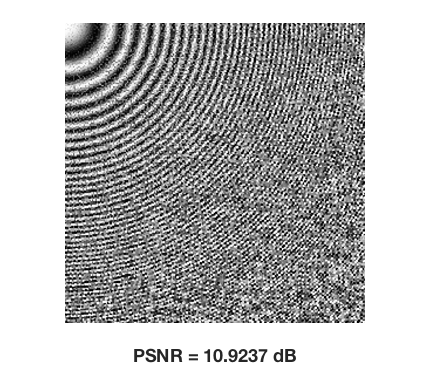}}&\raisebox{-.5\height}{\includegraphics[trim = 3cm 2.5cm 3cm 0cm ,width=0.12\textwidth,clip=True]{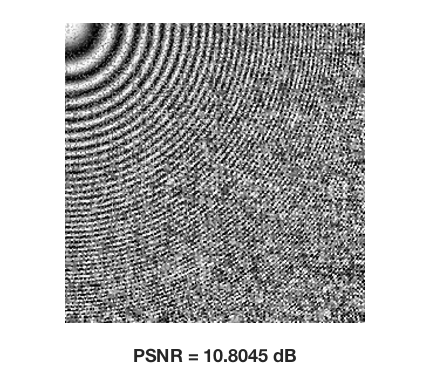}}&\raisebox{-.5\height}{\includegraphics[trim = 3cm 2.5cm 3cm 0cm ,width=0.12\textwidth,clip=True]{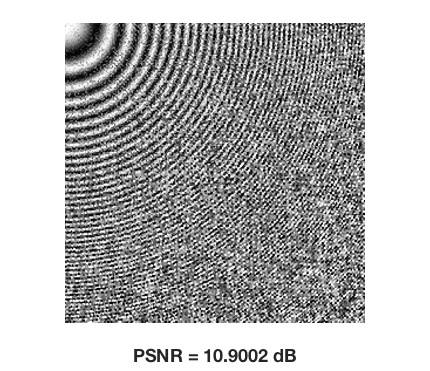}}\\
	 &\footnotesize 10.61 dB & \footnotesize 10.67 dB & \footnotesize 10.34 dB &\footnotesize 10.92 dB & \footnotesize 10.80 dB & \footnotesize 10.90 dB \\
	 
	 \hline
	 
	 \colorbox{gray!25}{0.005}&\raisebox{-.5\height}{\includegraphics[trim = 3cm 2.5cm 3cm 0cm ,width=0.12\textwidth,clip=True]{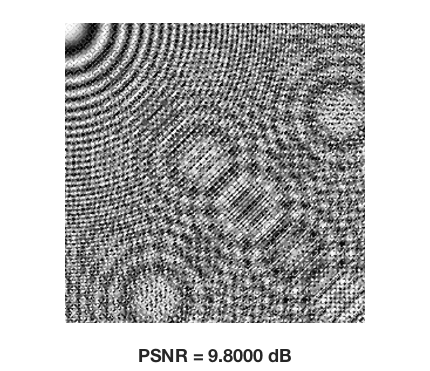}}&\raisebox{-.5\height}{\includegraphics[trim = 3cm 2.5cm 3cm 0cm ,width=0.12\textwidth,clip=True]{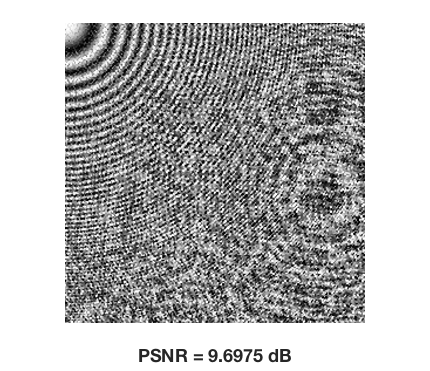}}&\raisebox{-.5\height}{\includegraphics[trim = 3cm 2.5cm 3cm 0cm ,width=0.12\textwidth,clip=True]{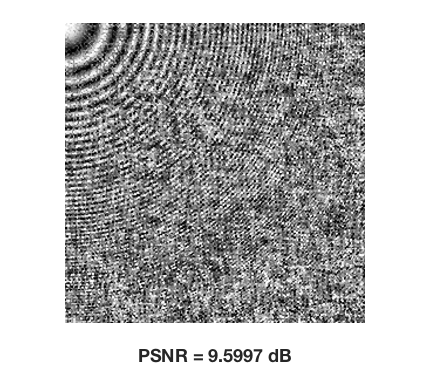}}&\raisebox{-.5\height}{\includegraphics[trim = 3cm 2.5cm 3cm 0cm ,width=0.12\textwidth,clip=True]{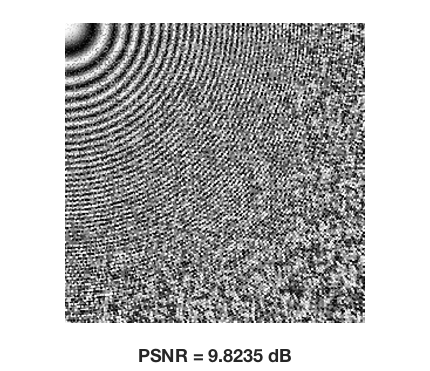}}&\raisebox{-.5\height}{\includegraphics[trim = 3cm 2.5cm 3cm 0cm ,width=0.12\textwidth,clip=True]{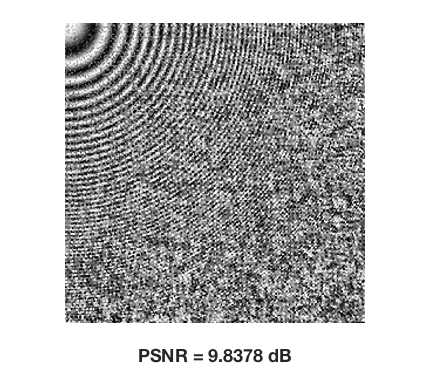}}&\raisebox{-.5\height}{\includegraphics[trim = 3cm 2.5cm 3cm 0cm ,width=0.12\textwidth,clip=True]{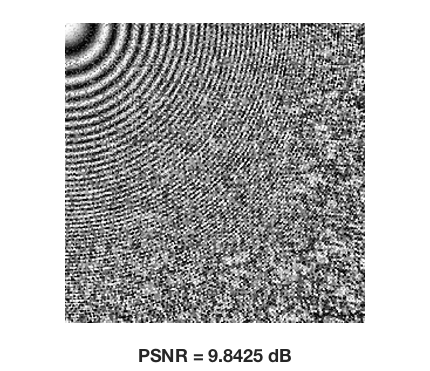}}\\
	 &\footnotesize 9.80 dB & \footnotesize 9.70 dB & \footnotesize 9.59 dB &\footnotesize 9.82 dB & \footnotesize 9.83 dB & \footnotesize 9.84 dB \\
	 
	 \hline
	 
	 \colorbox{gray!25}{0.006}&\raisebox{-.5\height}{\includegraphics[trim = 3cm 2.5cm 3cm 0cm ,width=0.12\textwidth,clip=True]{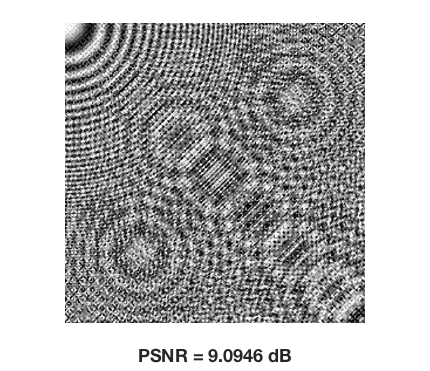}}&\raisebox{-.5\height}{\includegraphics[trim = 3cm 2.5cm 3cm 0cm ,width=0.12\textwidth,clip=True]{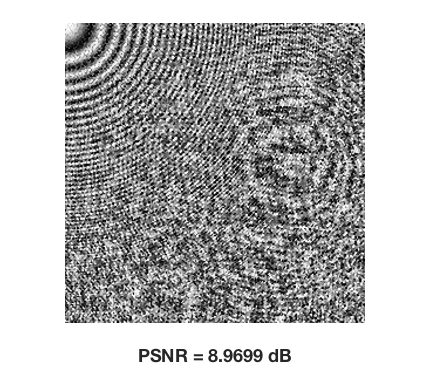}}&\raisebox{-.5\height}{\includegraphics[trim = 3cm 2.5cm 3cm 0cm ,width=0.12\textwidth,clip=True]{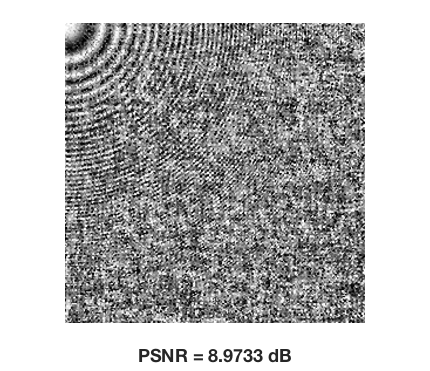}}&\raisebox{-.5\height}{\includegraphics[trim = 3cm 2.5cm 3cm 0cm ,width=0.12\textwidth,clip=True]{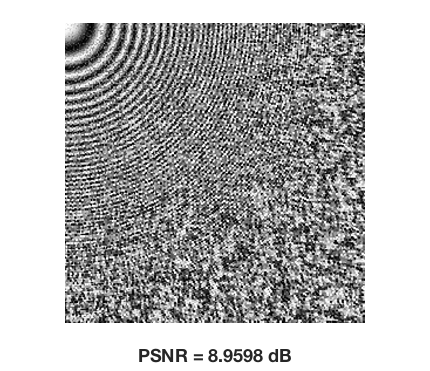}}&\raisebox{-.5\height}{\includegraphics[trim = 3cm 2.5cm 3cm 0cm ,width=0.12\textwidth,clip=True]{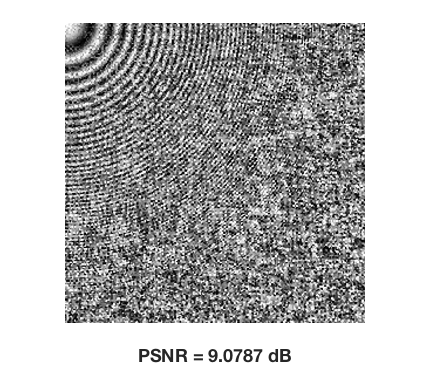}}&\raisebox{-.5\height}{\includegraphics[trim = 3cm 2.5cm 3cm 0cm ,width=0.12\textwidth,clip=True]{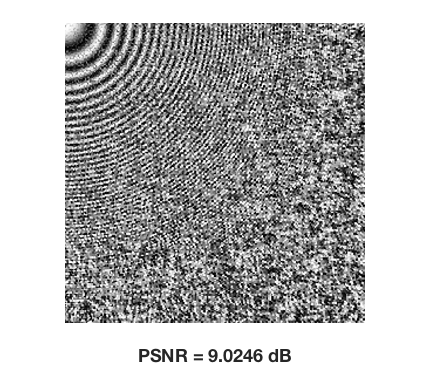}}\\
	 &\footnotesize 9.09 dB & \footnotesize 8.97 dB & \footnotesize 8.97 dB &\footnotesize 8.96 dB & \footnotesize 9.08 dB & \footnotesize 9.02 dB \\
	 
	  \hline
	  
	  \colorbox{gray!25}{0.007}&\raisebox{-.5\height}{\includegraphics[trim = 3cm 2.5cm 3cm 0cm ,width=0.12\textwidth,clip=True]{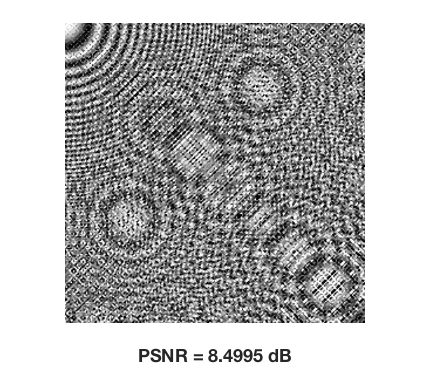}}&\raisebox{-.5\height}{\includegraphics[trim = 3cm 2.5cm 3cm 0cm ,width=0.12\textwidth,clip=True]{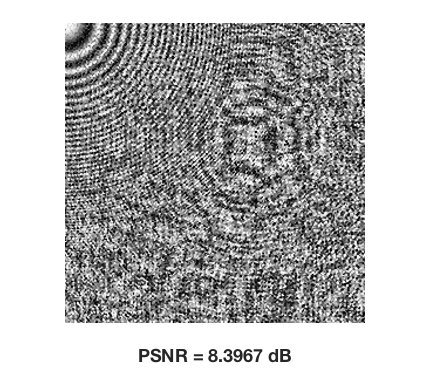}}&\raisebox{-.5\height}{\includegraphics[trim = 3cm 2.5cm 3cm 0cm ,width=0.12\textwidth,clip=True]{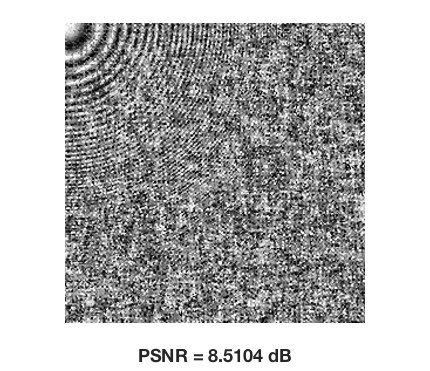}}&\raisebox{-.5\height}{\includegraphics[trim = 3cm 2.5cm 3cm 0cm ,width=0.12\textwidth,clip=True]{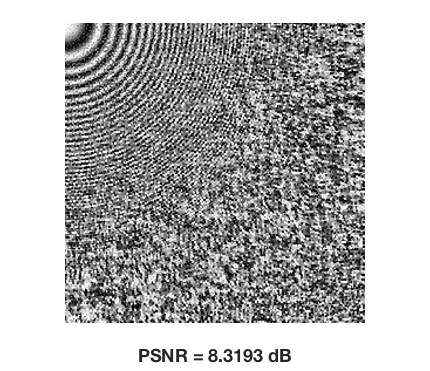}}&\raisebox{-.5\height}{\includegraphics[trim = 3cm 2.5cm 3cm 0cm ,width=0.12\textwidth,clip=True]{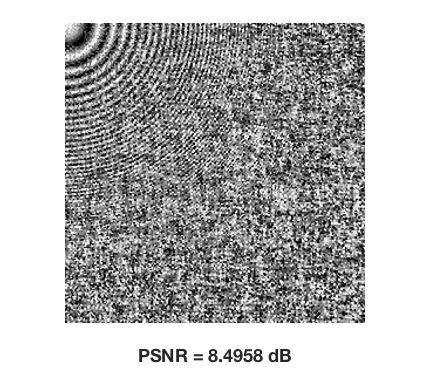}}&\raisebox{-.5\height}{\includegraphics[trim = 3cm 2.5cm 3cm 0cm ,width=0.12\textwidth,clip=True]{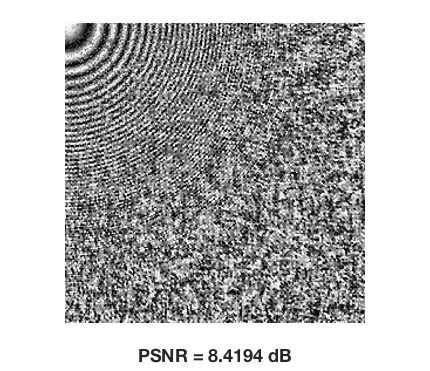}}\\
	  &\footnotesize 8.49 dB & \footnotesize 8.39 dB & \footnotesize 8.51 dB &\footnotesize 8.32 dB & \footnotesize 8.49 dB & \footnotesize 8.42 dB \\
	 
	\bottomrule
	
\end{tabularx}
\end{table}

\subsection{Image Reconstruction}
In this experiment, we consider the problem of designing sample distributions for image reconstruction. More specifically, we consider the commonly used zone plate test function:
$$z(r)=(1+\cos(\alpha r^2))/2,$$ 
with varying levels of complexity (or frequency content) $\alpha$. Note that, we choose the zone plate for our study over natural images, since it shows the response for a wide range of frequencies and aliasing effects that are not masked by image features. For all zone plate renderings in this paper, we have tiled toroidal sets of $1000$ $2$-dimensional points over the image-plane and utilized a Lanczos filter with a support of width $4$ for resampling. Further, we also report the peak signal-to-noise ratio (PSNR) as a quantitative error measure:
$$\text{PSNR} = 20 \log_{10} \dfrac{1}{\text{MSE}},$$ where MSE is the mean squared error. However, it is well known in the image processing community that PSNR can be a weak surrogate for visual quality and, therefore, we also show the reconstructed images. 

Table~\ref{table:recon} illustrates the reconstructions obtained using different space-filling designs, for varying values of $\alpha$. It can be observed from the results that the QMC sequences produce a large amount of aliasing artifacts in the high frequency regions, which can be directly linked to the oscillations in their corresponding PCFs. On the other hand, LHS design recovers a small amount of low-frequencies, and maps most of the frequencies to white noise due to its small $r_{min}$ and near-constant PCF.    
In contrast, sample designs which attempt to trade-off between coverage and randomness properties, i.e., MPDS and the proposed spectral space-filling designs (as seen in Figure~\ref{pcf_samp}), have superior reconstruction quality. These designs reduce the aliasing artifacts, have cleaner low frequency content (upper left corner)) and map all high frequencies (bottom right corner) to white noise.
More interestingly, we see that for low complexity cases, i.e., lower $\alpha$, the MPDS performs the best followed by the proposed Stair and Step PDS respectively. For moderately complex images, the Stair PDS performs the best followed by the Step and the MPDS. Finally, for highly complex images, the Step PDS performs the best followed by the Stair and the MPDS. 
These observations corroborate our discussion in Section~\ref{g-pcf} that an increase in $r_{min}$ (coverage) in the PCF results in an increase in the range of low frequencies that can be recovered without aliasing, and equivalently reduction in the amount of oscillations (or an increase in randomness) in the PCF leads to reduced oscillations in the PSD, which in turn indicates a systematic mapping of high frequency content to white noise.

\subsection{Regression Modeling for Benchmark Optimization Functions}
\label{sm:bf}
In this study, we consider the problem of fitting regression models to analytical functions and perform a comparative study of different sample designs, in terms of their generalization performance. More specifically, we consider a set of benchmark analytic functions between dimensions $2$ and $6$, that are commonly used in global optimization tests~\citep{opt:fun}. They are chosen due to their diversity in terms of their complexity and shapes. We compare the performance of proposed space-filling spectral designs (Step, Stair) with coverage based designs (MPDS), low-discrepancy designs (Halton and Sobol), latin hypercube sampling and random sampling. Appendix~\ref{apd:first} lists the set of functions used in our experiments. In each case, we fit a random forest regressor with $30$ trees and repeated for $20$ independent realizations of sample designs. We evaluate the generalization performance on $10^6$ regular grid based test samples. Finally, we employ $3$ popular quality metrics to quantify the performance of the resulting regression models: mean squared error (MSE), relative average absolute error (AAE), and the R$^2$-statistic. The metrics are defined as follows:
\begin{equation}
MSE(\mathbf{y},\hat {\mathbf{y}}) = \frac{\sum_{i=1}^N (y_i - \hat{y}_i)^2}{N},
\end{equation}
\begin{equation}
AAE(\mathbf{y},\hat {\mathbf{y}}) = \frac{\sum_{i=1}^N |y_i - \hat{y}_i|}{N*STD(\mathbf{y})},
\end{equation}
\begin{equation}
R^2(\mathbf{y},\hat {\mathbf{y}}) = 1-\frac{\sum_{i=1}^N (y_i - \hat{y}_i)^2}{\sum_{i=1}^N (y_i - MEAN(\mathbf{y}))^2}
\end{equation}
where $\mathbf{y} = f(\mathbf{x})$ are the true function values and $\hat {\mathbf{y}}$ are the predicted values. 

Tables~\ref{reg2} through~\ref{reg6} show the performance of different space-filling designs for various analytic functions in dimensions $2$ to $6$, respectively. We see that, for $d=2$ (Table~\ref{reg2}), LHS and Halton sequences perform better compared to the rest of the sample designs on most of the test functions. However, on some functions, e.g., GoldsteinPrice, Stair PCF and MPDS perform better. Therefore, none of the sample designs consistently guarantee superior performance. For $d=3$ (Table~\ref{reg3}), we see that Stair PCF design and MPDS (followed by Sobol sequences) perform consistently better compared to the rest of the approaches. As we go higher in dimensions, i.e., $d>3$, we notice a significant gain in the performance of Stair PCF based space-filling spectral designs. Interesting, the amount of performance gain of Stair PCF based design increases as we go higher in dimensions. The reason for the poor performance of QMC sequences and LHS for $d>3$ is due to their poor space-filling properties in high dimensions~\citep{lds:hd}. In comparison, both space-filling spectral designs and MPDS have good space-filling properties. We found that Stair PCF design and MPDS have similar coverage characteristics ($r_{min}$). However, the difference in their performance can be attributed to the fact that MPDS designs have significantly more oscillations in their PCF compared to an equivalent Stair PCF based space-filling spectral design, i.e. violation of the randomness objective.   

\begin{table}[htb!]
	\centering
	\caption{Impact of sample design on generalization performance of regression models fit to benchmark analytical functions in $2$ dimensions. LHS and Halton sequences perform slightly better compared to rest of the sample designs.}\label{reg2}
	\begin{tabularx}{\textwidth}{@{}YYYY@{}}
		\toprule
		\textbf{Function} & \textbf{MSE} & \textbf{AAE} & \textbf{R2-Statistic}\\
		\hline
		
		\footnotesize \colorbox{gray!25}{GoldsteinPrice}&\raisebox{-.5\height}{\includegraphics[trim = 1cm 0cm 0cm 0cm,width=0.25\textwidth,clip=True]{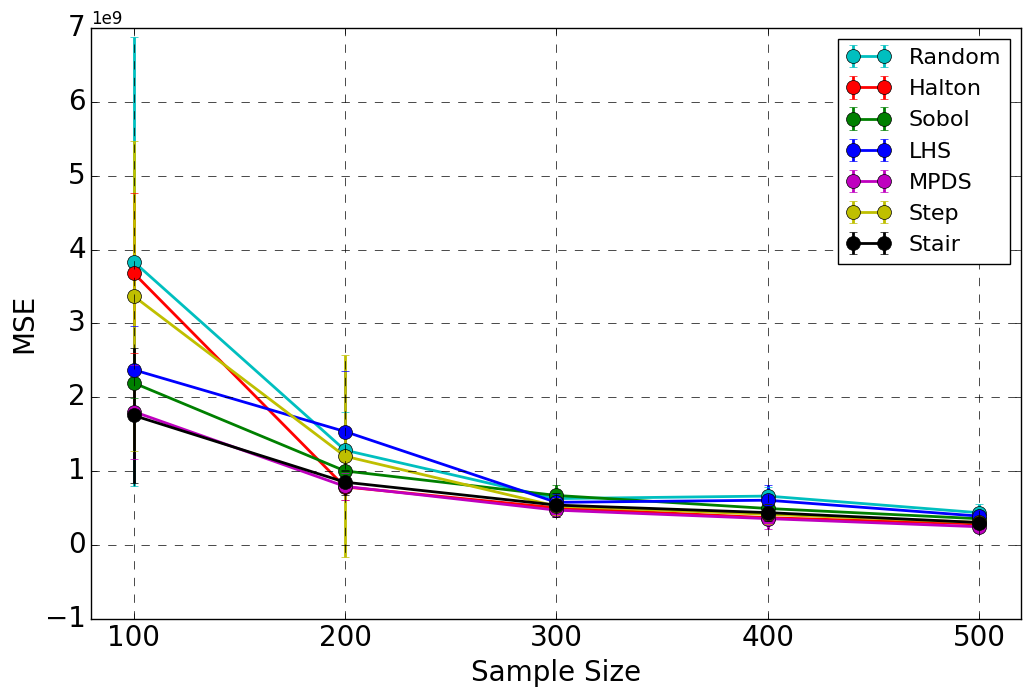}}&\raisebox{-.5\height}{\includegraphics[trim = 1cm 0cm 0cm 0cm ,width=0.25\textwidth,clip=True]{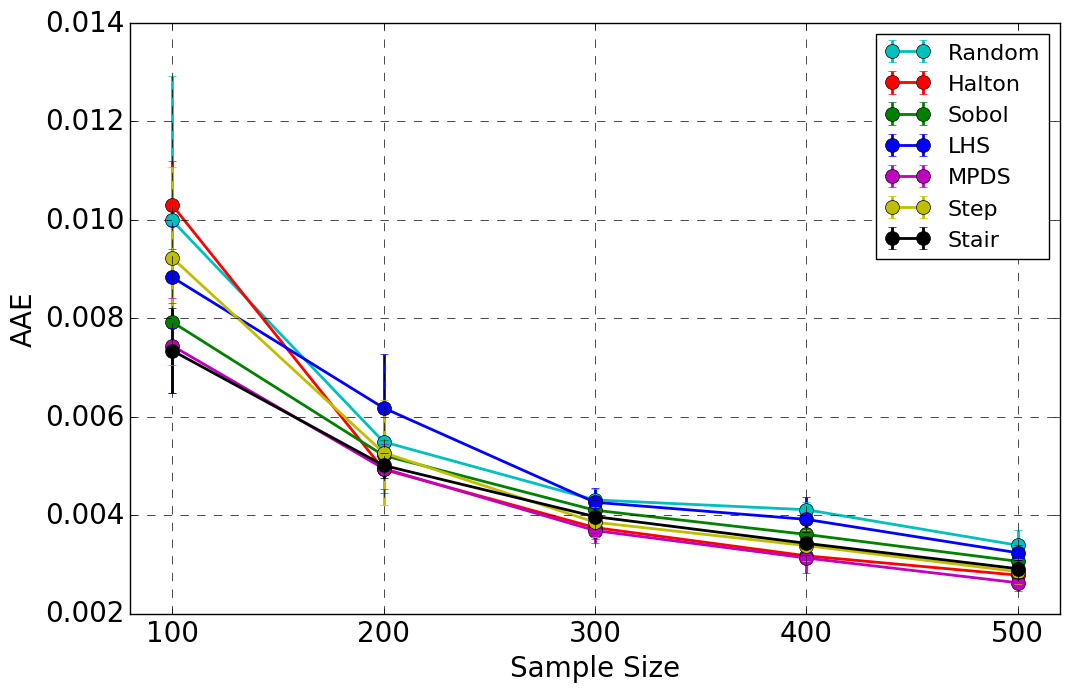}}&\raisebox{-.5\height}{\includegraphics[trim = 1cm 0cm 0cm 0cm ,width=0.25\textwidth,clip=True]{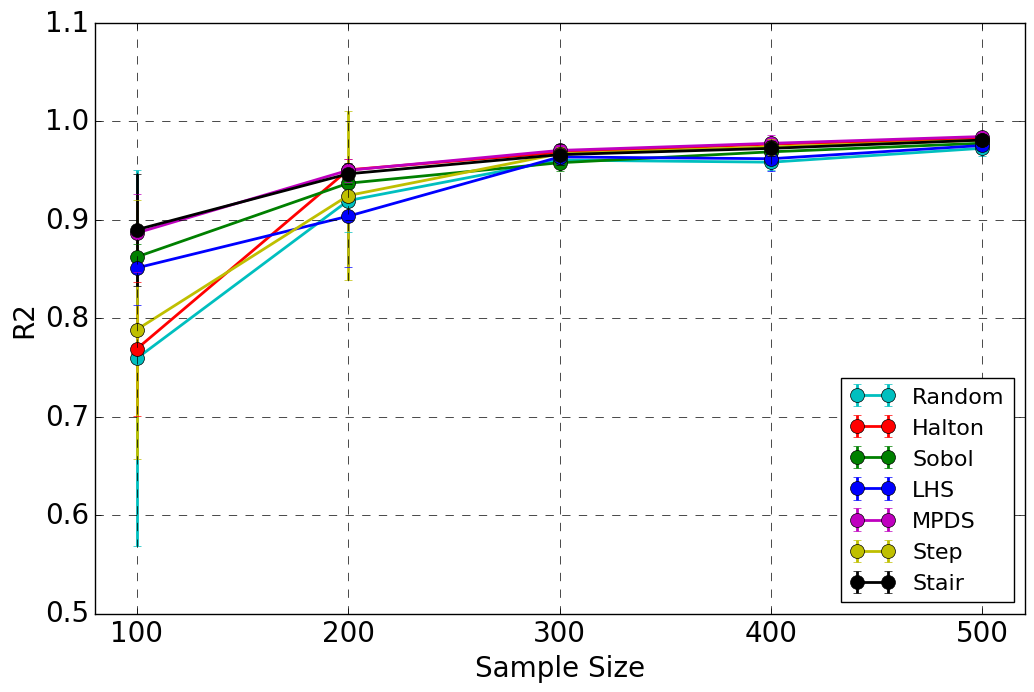}} \\
		
		\hline
		
		\footnotesize \colorbox{gray!25}{Chichinadze}&\raisebox{-.5\height}{\includegraphics[trim = 1cm 0cm 0cm 0cm,width=0.25\textwidth,clip=True]{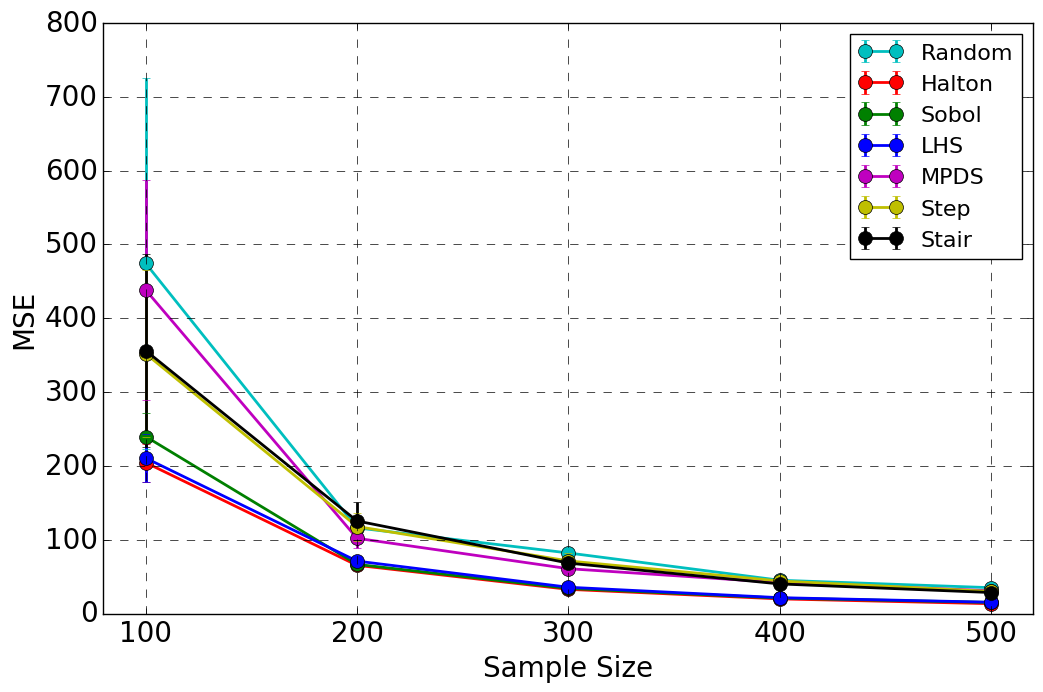}}&\raisebox{-.5\height}{\includegraphics[trim = 1cm 0cm 0cm 0cm ,width=0.25\textwidth,clip=True]{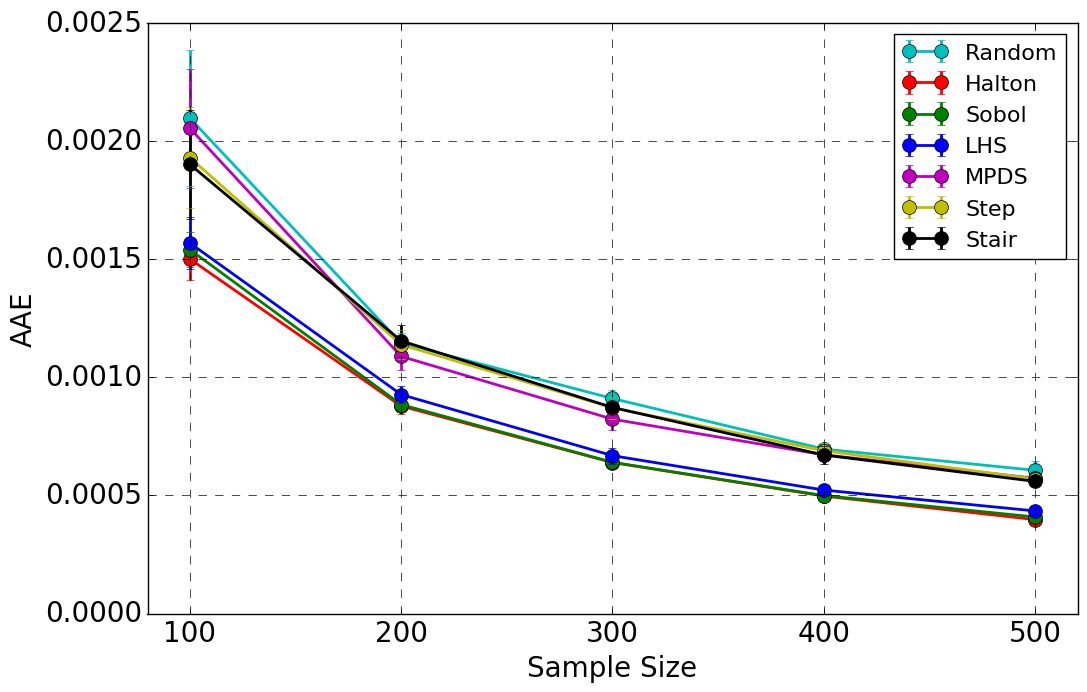}}&\raisebox{-.5\height}{\includegraphics[trim = 1cm 0cm 0cm 0cm ,width=0.25\textwidth,clip=True]{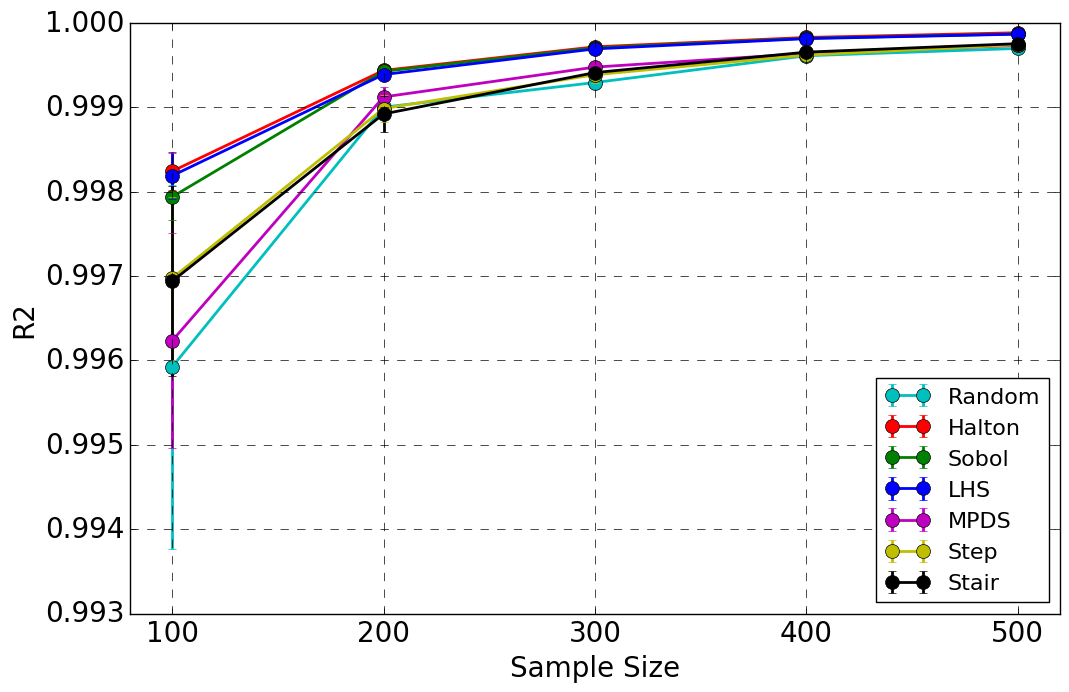}} \\
		
		\hline
		
		\footnotesize \colorbox{gray!25}{Rosenbrock}&\raisebox{-.5\height}{\includegraphics[trim = 1cm 0cm 0cm 0cm,width=0.25\textwidth,clip=True]{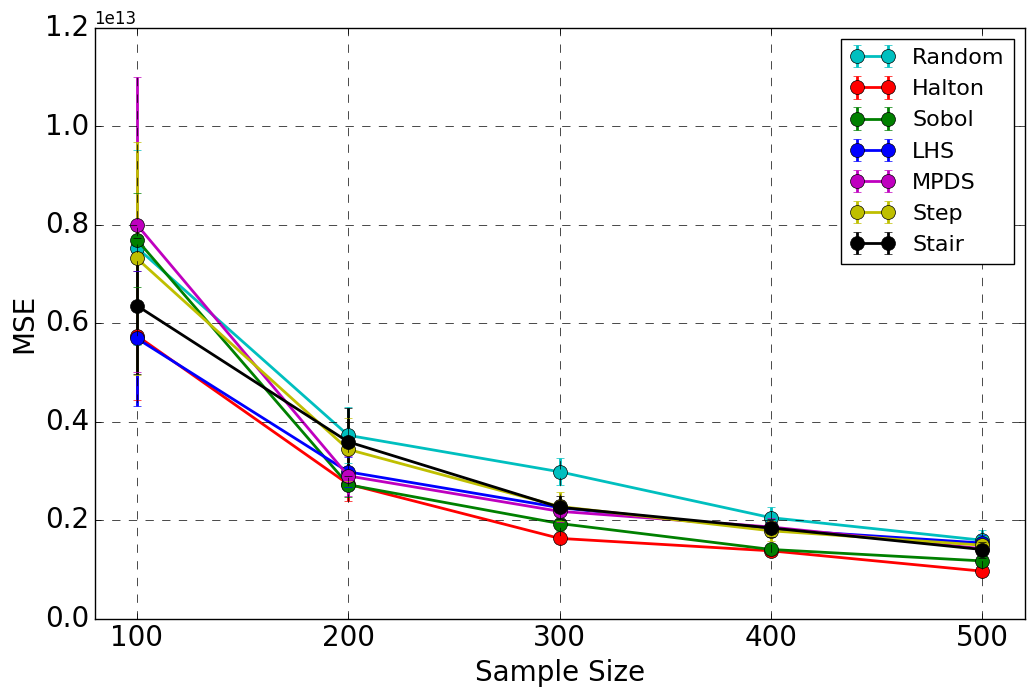}}&\raisebox{-.5\height}{\includegraphics[trim = 1cm 0cm 0cm 0cm ,width=0.25\textwidth,clip=True]{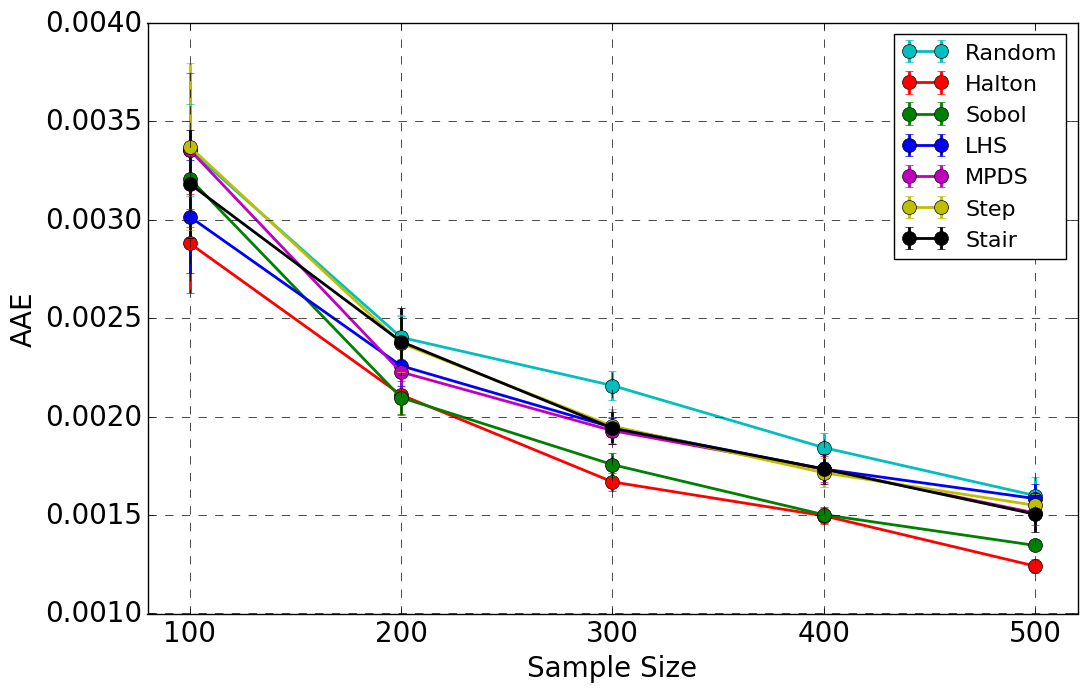}}&\raisebox{-.5\height}{\includegraphics[trim = 1cm 0cm 0cm 0cm ,width=0.25\textwidth,clip=True]{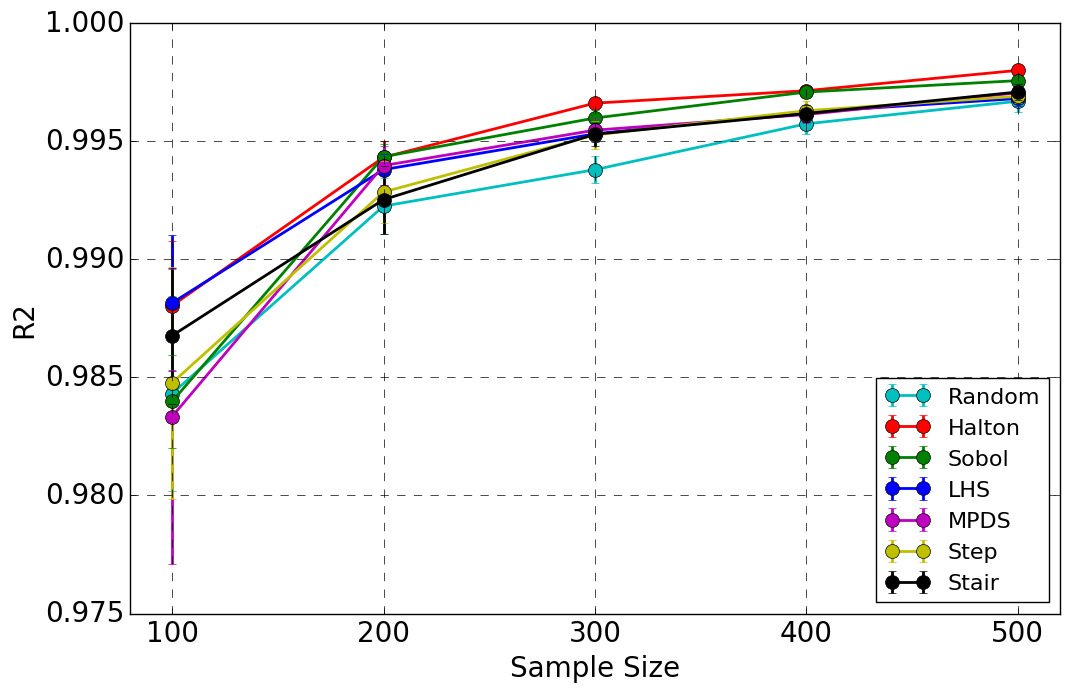}} \\
		
		\hline
		
		\footnotesize \colorbox{gray!25}{Cube}&\raisebox{-.5\height}{\includegraphics[trim = 1cm 0cm 0cm 0cm,width=0.25\textwidth,clip=True]{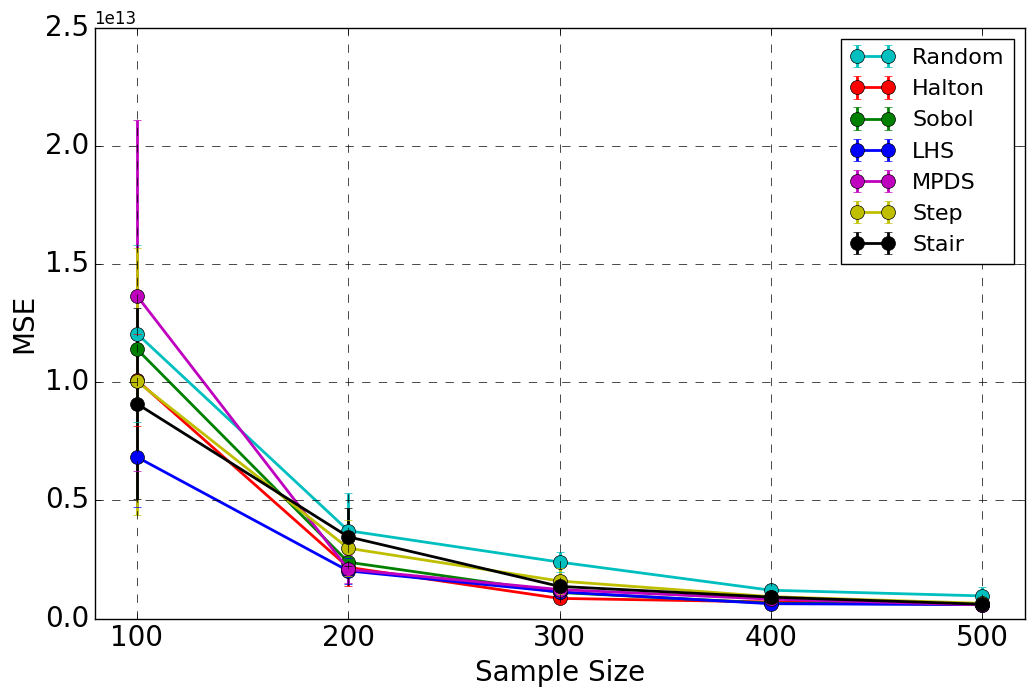}}&\raisebox{-.5\height}{\includegraphics[trim = 1cm 0cm 0cm 0cm ,width=0.25\textwidth,clip=True]{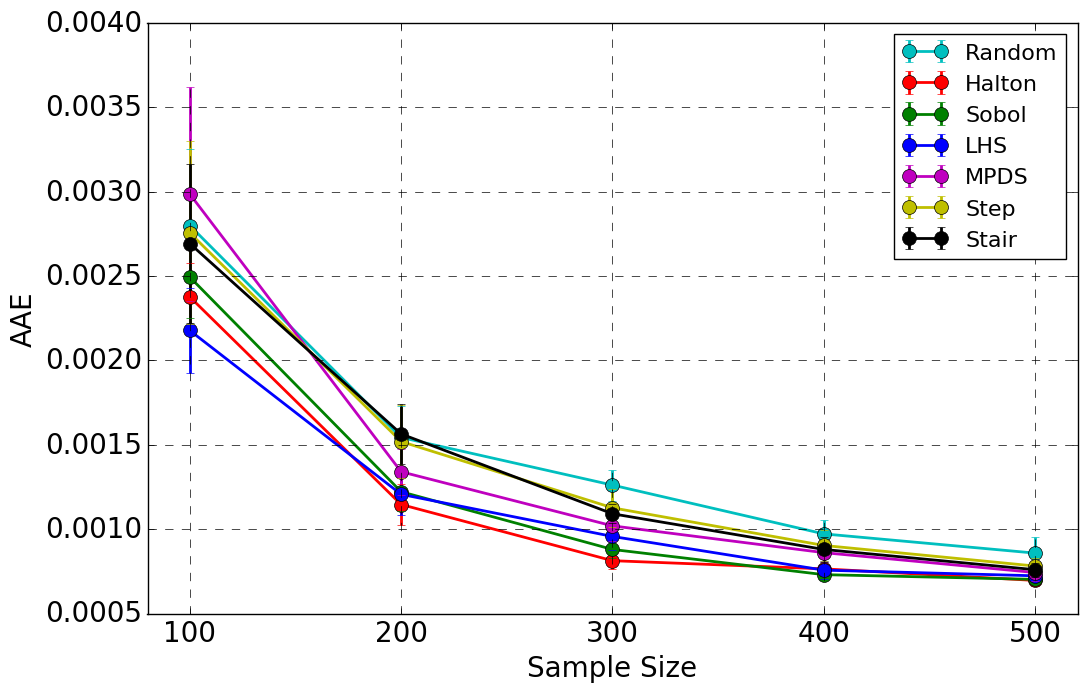}}&\raisebox{-.5\height}{\includegraphics[trim = 1cm 0cm 0cm 0cm ,width=0.25\textwidth,clip=True]{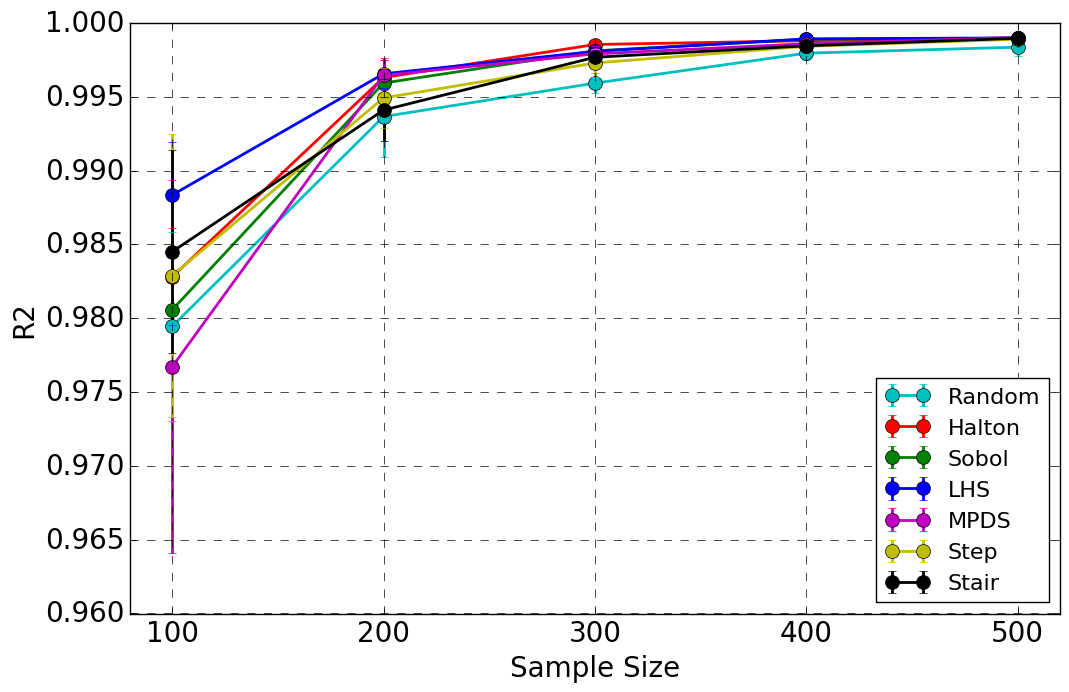}} \\
		
		\bottomrule
		
	\end{tabularx}
\end{table}

\begin{table}[th]
	\centering
	\caption{Impact of sample design on generalization performance of regression models fit to benchmark analytical functions in $3$ dimensions. While the Stair PCF and MPDS designs are consistently better than the other methods, the amount of performance gain is minimal.}\label{reg3}
	\begin{tabularx}{\textwidth}{@{}YYYY@{}}
		\toprule
		\textbf{Function} & \textbf{MSE} & \textbf{AAE} & \textbf{R2-Statistic}\\
		\hline
		
		\footnotesize \colorbox{gray!25}{BoxBetts}&\raisebox{-.5\height}{\includegraphics[trim = 1cm 0cm 0cm 0cm,width=0.25\textwidth,clip=True]{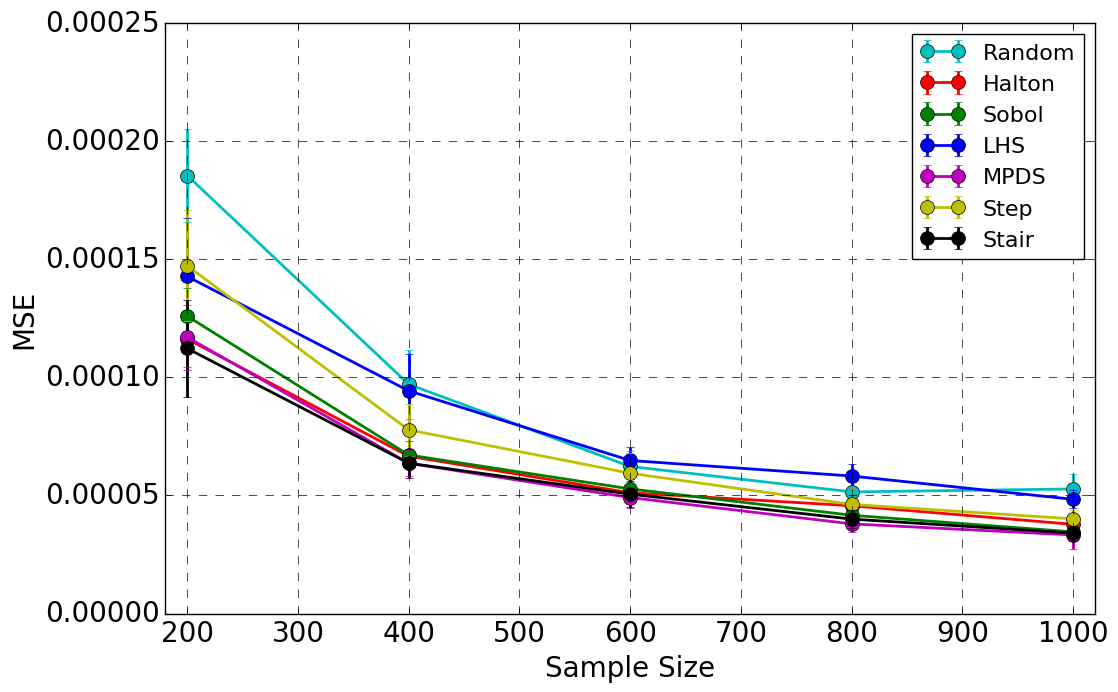}}&\raisebox{-.5\height}{\includegraphics[trim = 1cm 0cm 0cm 0cm ,width=0.25\textwidth,clip=True]{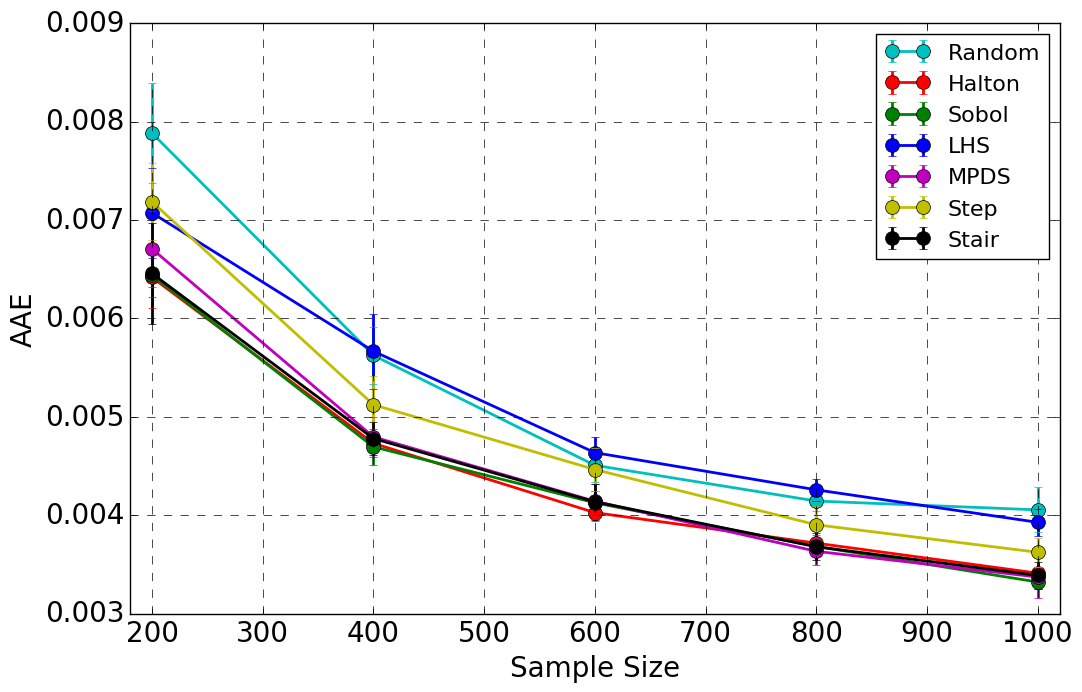}}&\raisebox{-.5\height}{\includegraphics[trim = 1cm 0cm 0cm 0cm ,width=0.25\textwidth,clip=True]{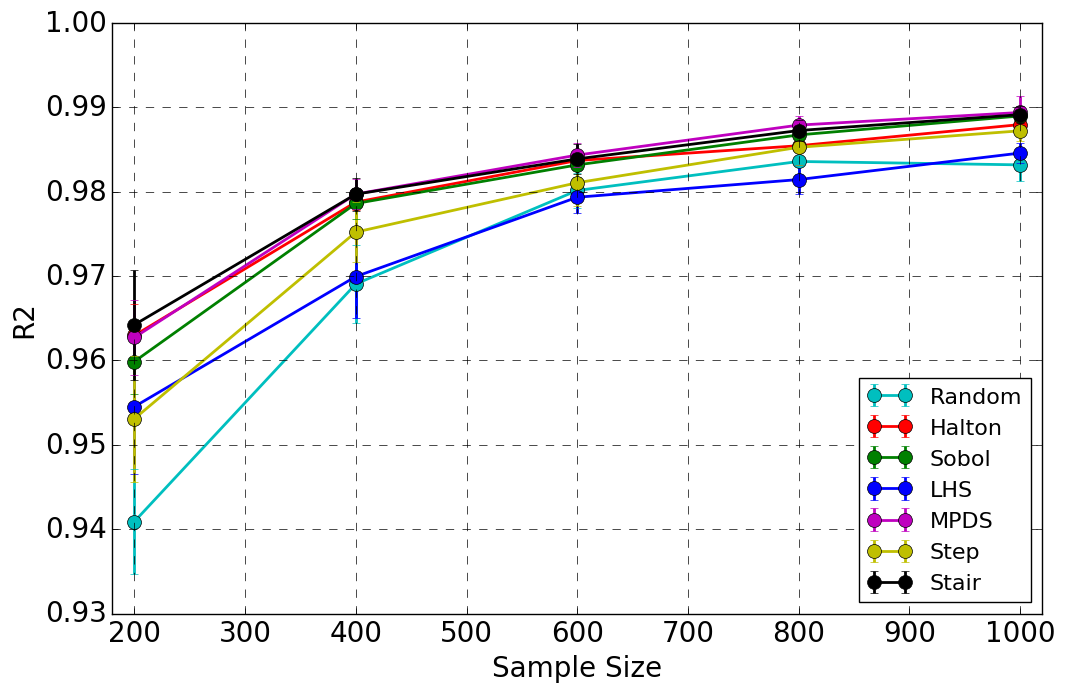}} \\
		
		\hline
		
		\footnotesize \colorbox{gray!25}{Hartmann3}&\raisebox{-.5\height}{\includegraphics[trim = 1cm 0cm 0cm 0cm,width=0.25\textwidth,clip=True]{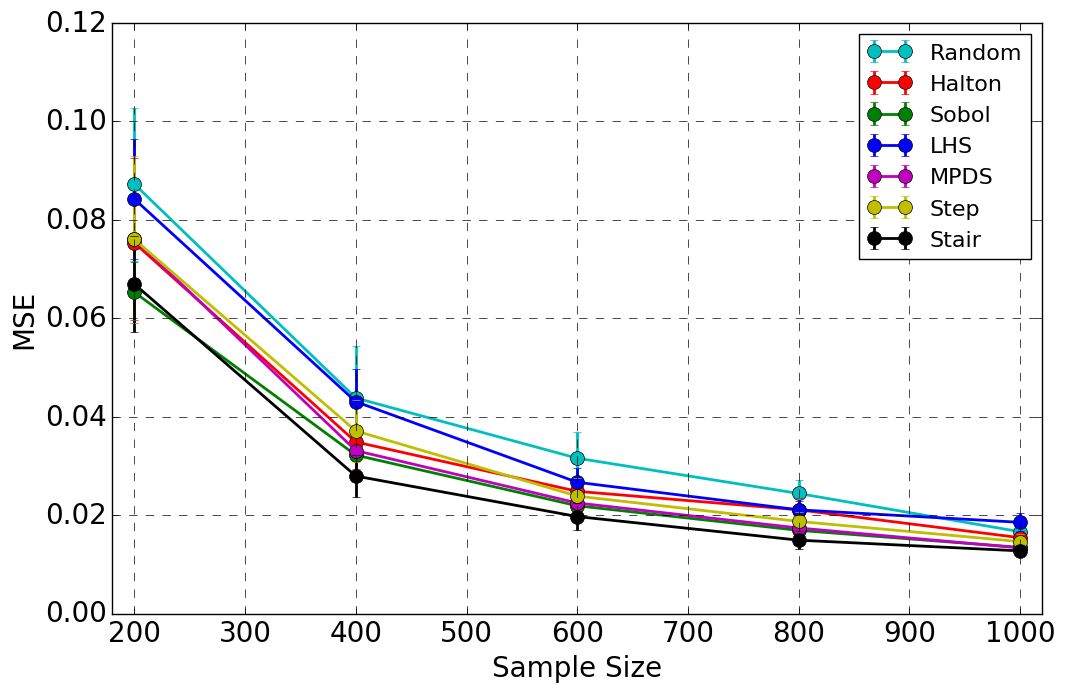}}&\raisebox{-.5\height}{\includegraphics[trim = 1cm 0cm 0cm 0cm ,width=0.25\textwidth,clip=True]{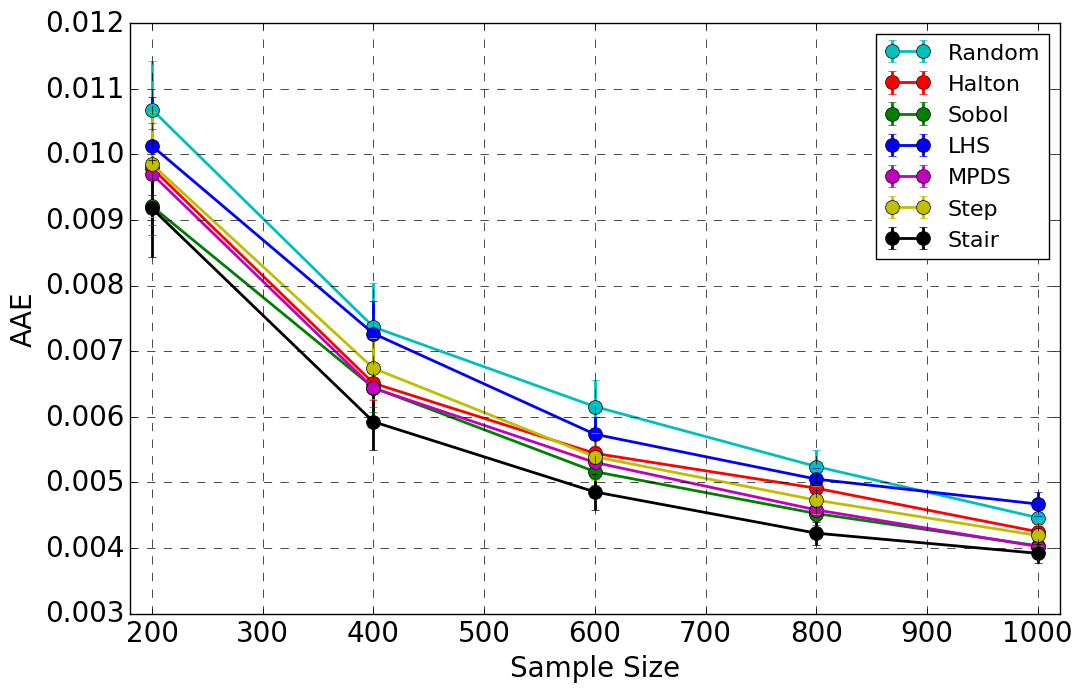}}&\raisebox{-.5\height}{\includegraphics[trim = 1cm 0cm 0cm 0cm ,width=0.25\textwidth,clip=True]{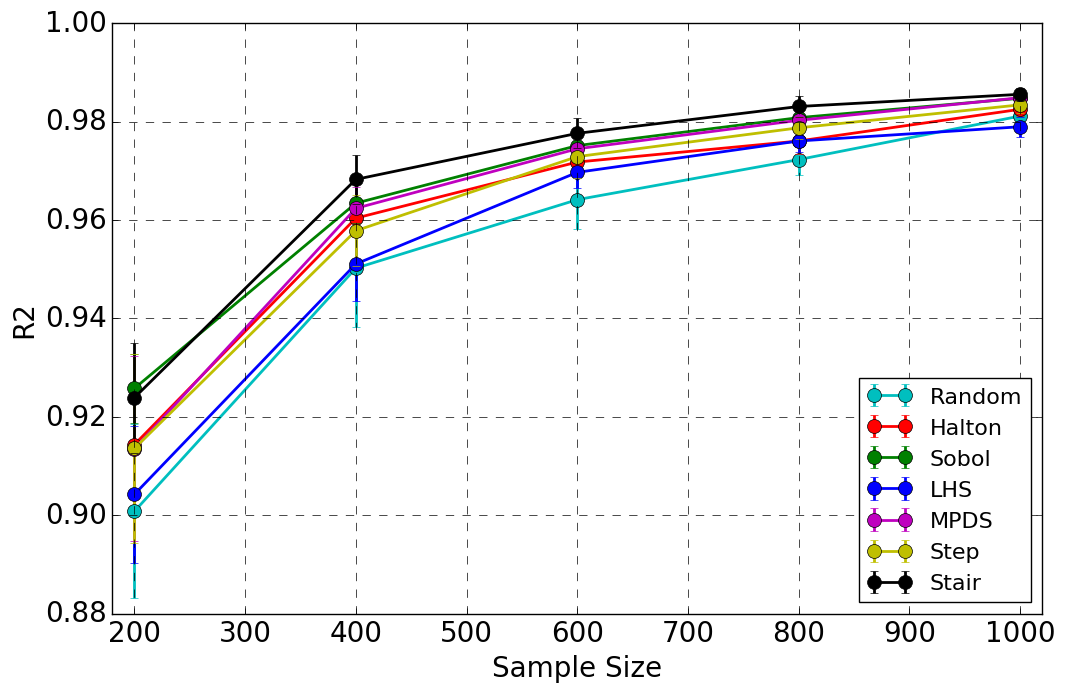}} \\
		
		\hline
		
		\footnotesize \colorbox{gray!25}{Wolfe}&\raisebox{-.5\height}{\includegraphics[trim = 1cm 0cm 0cm 0cm,width=0.25\textwidth,clip=True]{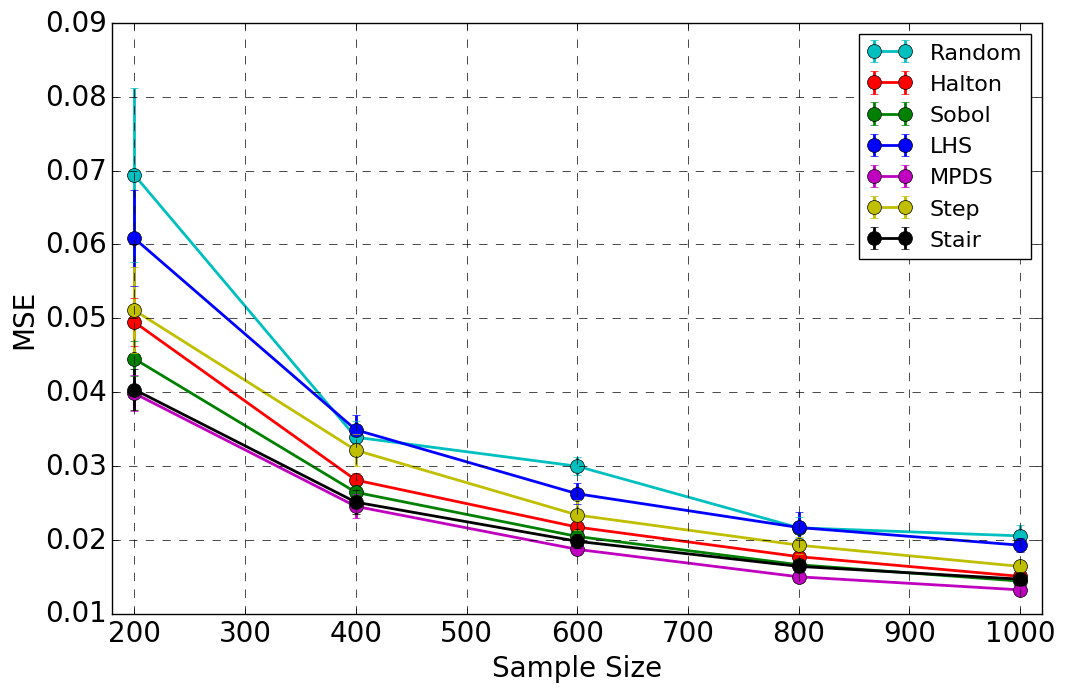}}&\raisebox{-.5\height}{\includegraphics[trim = 1cm 0cm 0cm 0cm ,width=0.25\textwidth,clip=True]{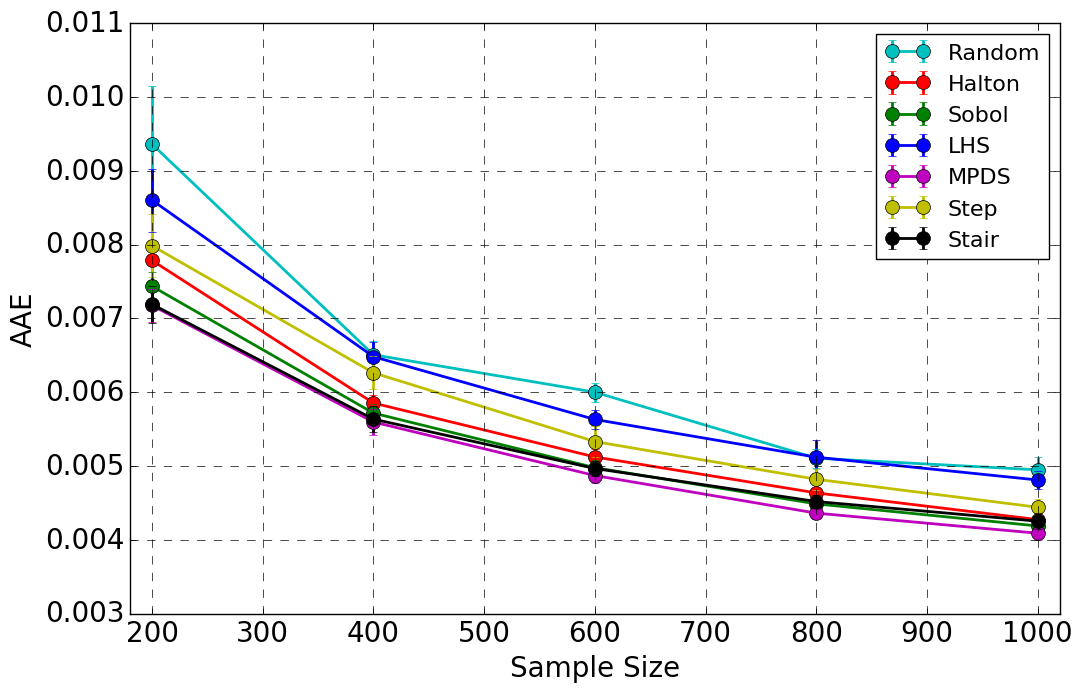}}&\raisebox{-.5\height}{\includegraphics[trim = 1cm 0cm 0cm 0cm ,width=0.25\textwidth,clip=True]{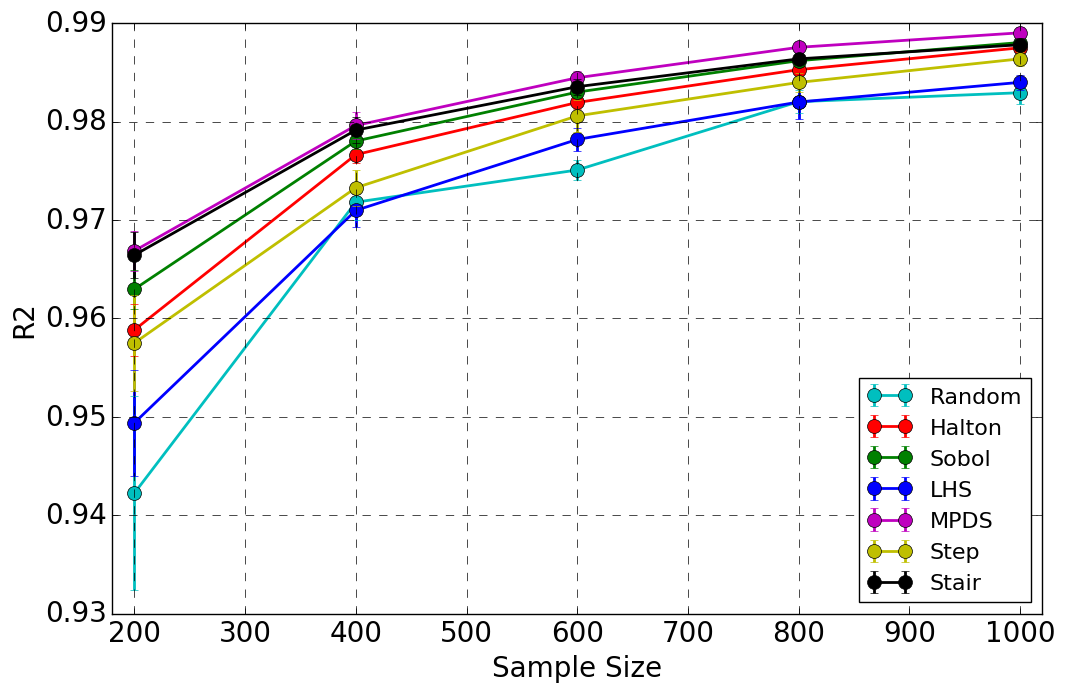}} \\
		
		\hline
		
		\footnotesize \colorbox{gray!25}{Helical Valley}&\raisebox{-.5\height}{\includegraphics[trim = 1cm 0cm 0cm 0cm,width=0.25\textwidth,clip=True]{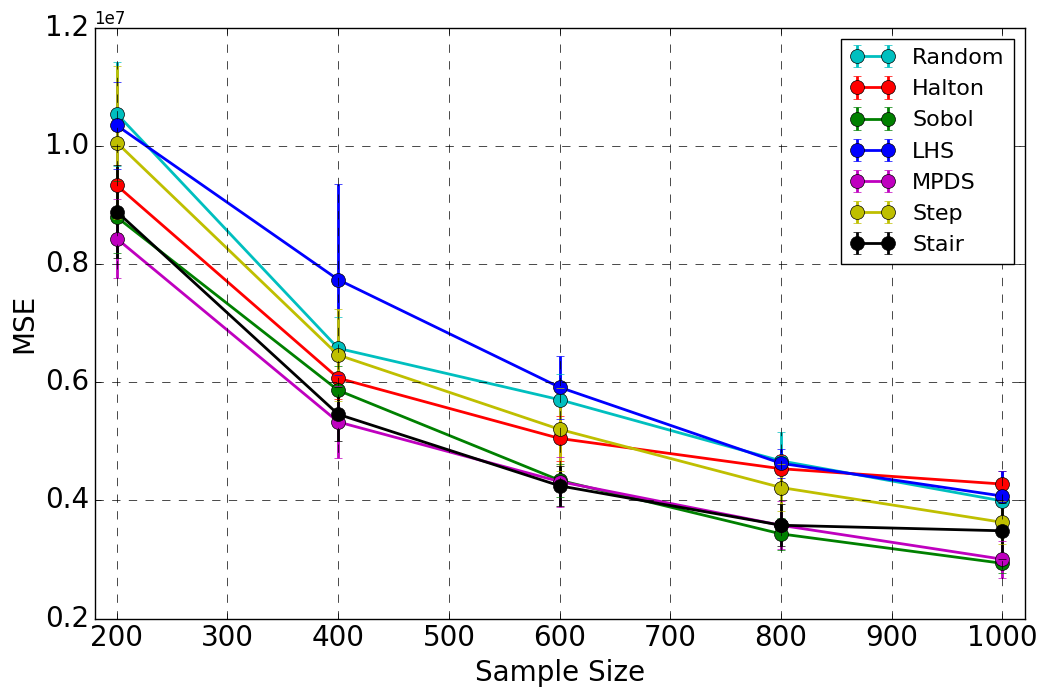}}&\raisebox{-.5\height}{\includegraphics[trim = 1cm 0cm 0cm 0cm ,width=0.25\textwidth,clip=True]{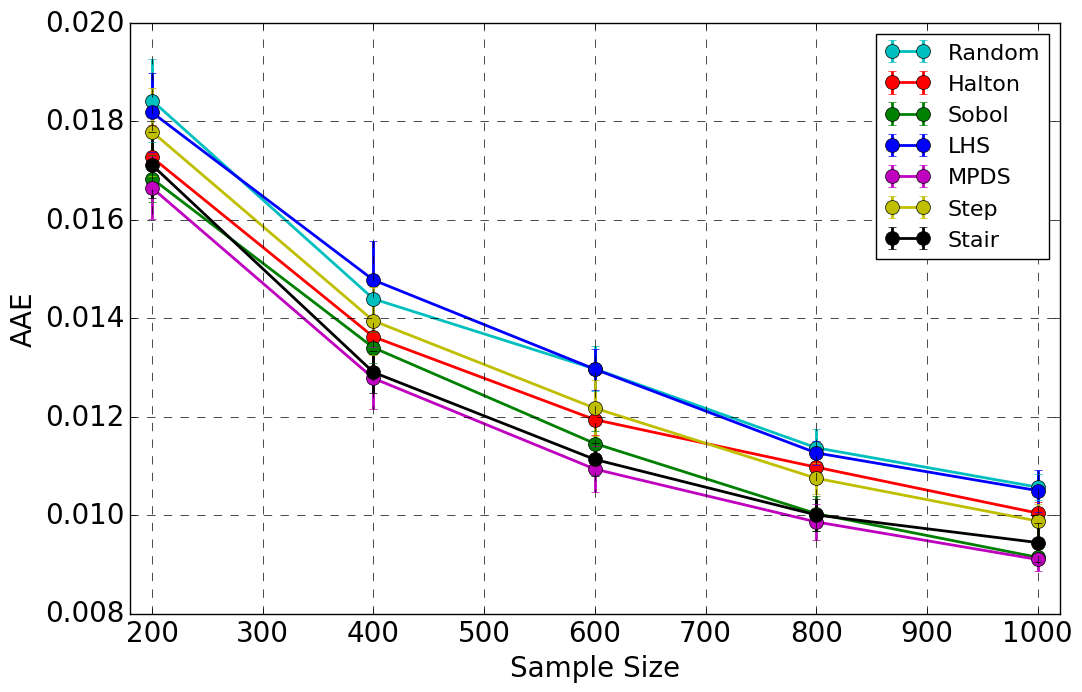}}&\raisebox{-.5\height}{\includegraphics[trim = 1cm 0cm 0cm 0cm ,width=0.25\textwidth,clip=True]{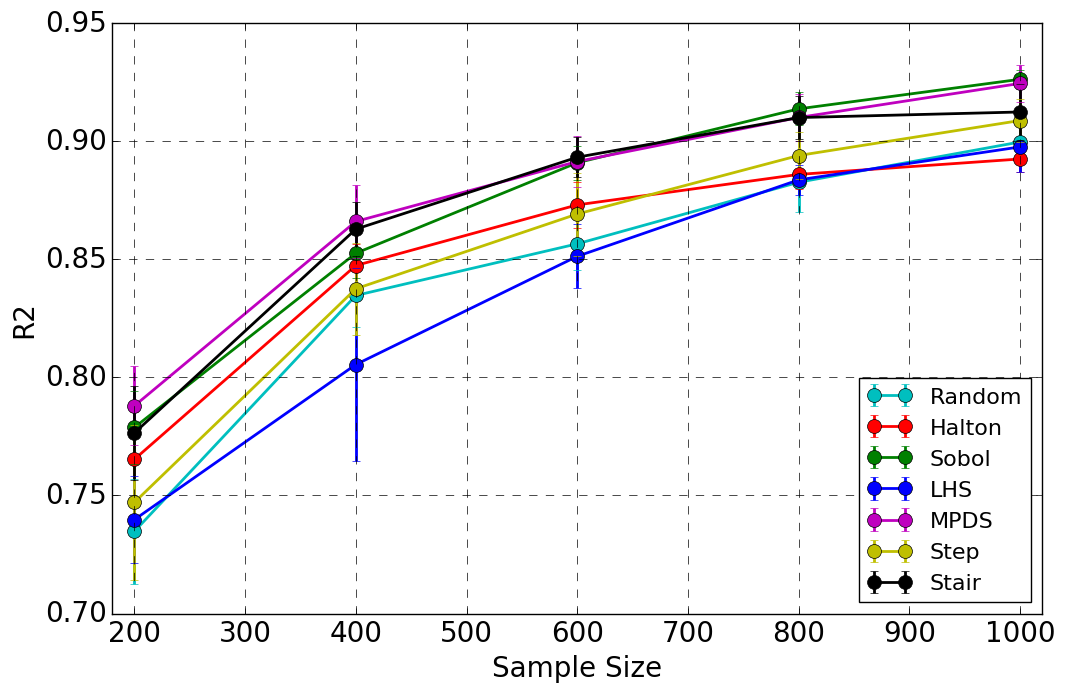}} \\
		
		\bottomrule
		
	\end{tabularx}
\end{table}

\begin{table}[th]
	\centering
	\caption{Impact of sample design on generalization performance of regression models fit to benchmark analytical functions in $4$ dimensions. Stair PCF and MPDS designs demonstrate appreciable gains over popular sample design choices.}\label{reg4}
	\begin{tabularx}{\textwidth}{@{}YYYY@{}}
		\toprule
		\textbf{Function} & \textbf{MSE} & \textbf{AAE} & \textbf{R2-Statistic}\\
		\hline
		
		\footnotesize \colorbox{gray!25}{DeVilliersGlasser01}&\raisebox{-.5\height}{\includegraphics[trim = 1cm 0cm 0cm 0cm,width=0.25\textwidth,clip=True]{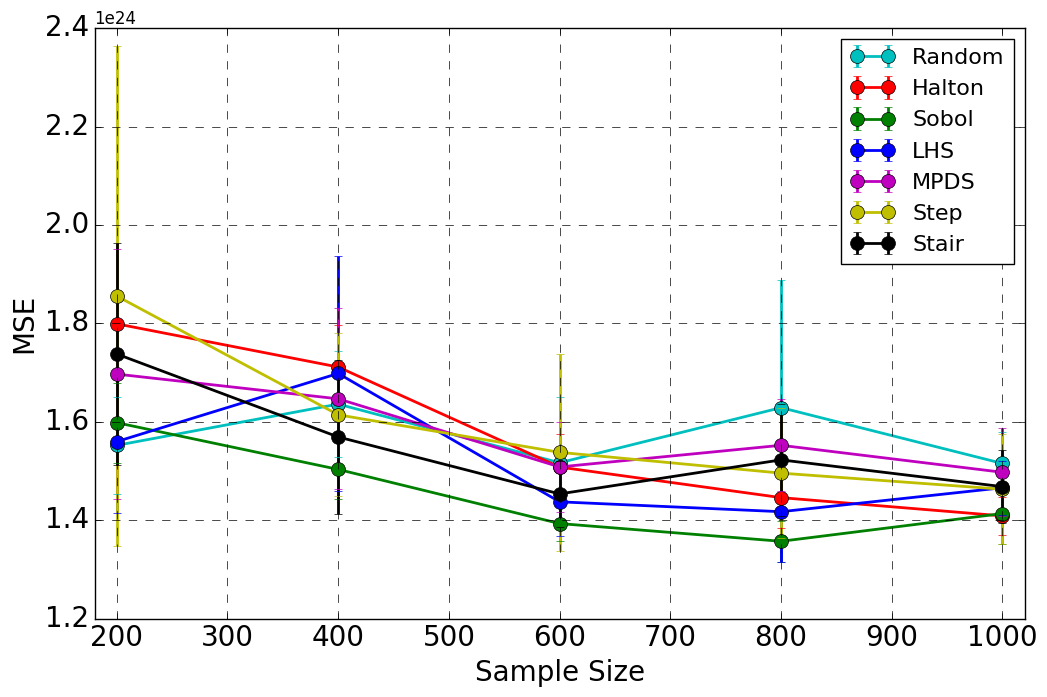}}&\raisebox{-.5\height}{\includegraphics[trim = 1cm 0cm 0cm 0cm ,width=0.25\textwidth,clip=True]{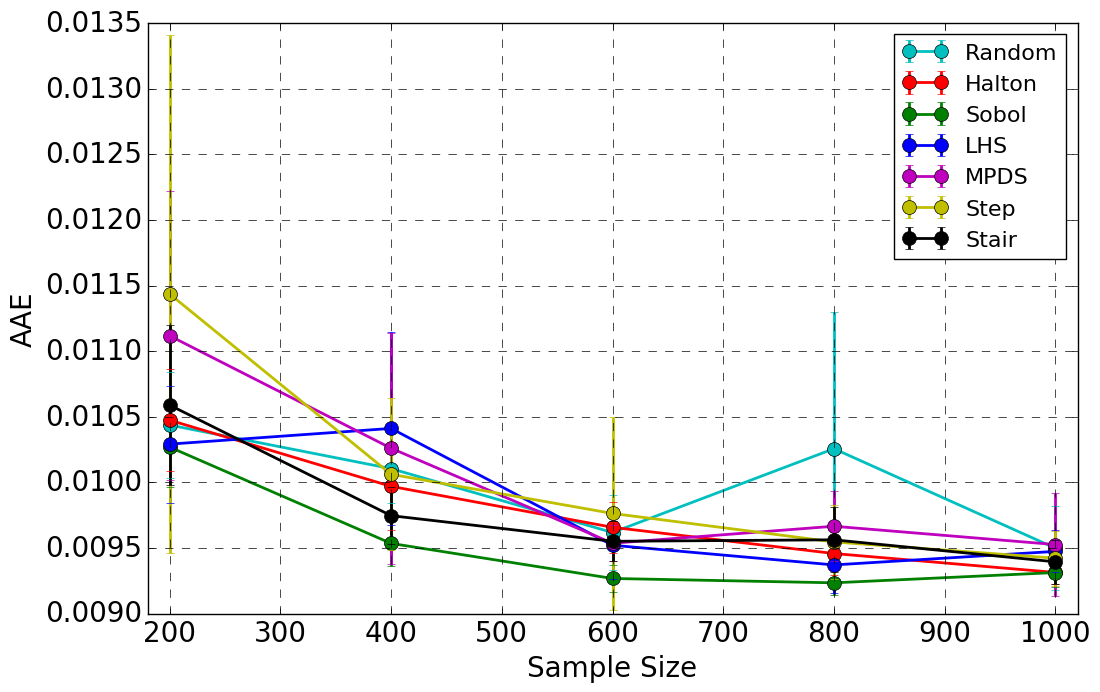}}&\raisebox{-.5\height}{\includegraphics[trim = 1cm 0cm 0cm 0cm ,width=0.25\textwidth,clip=True]{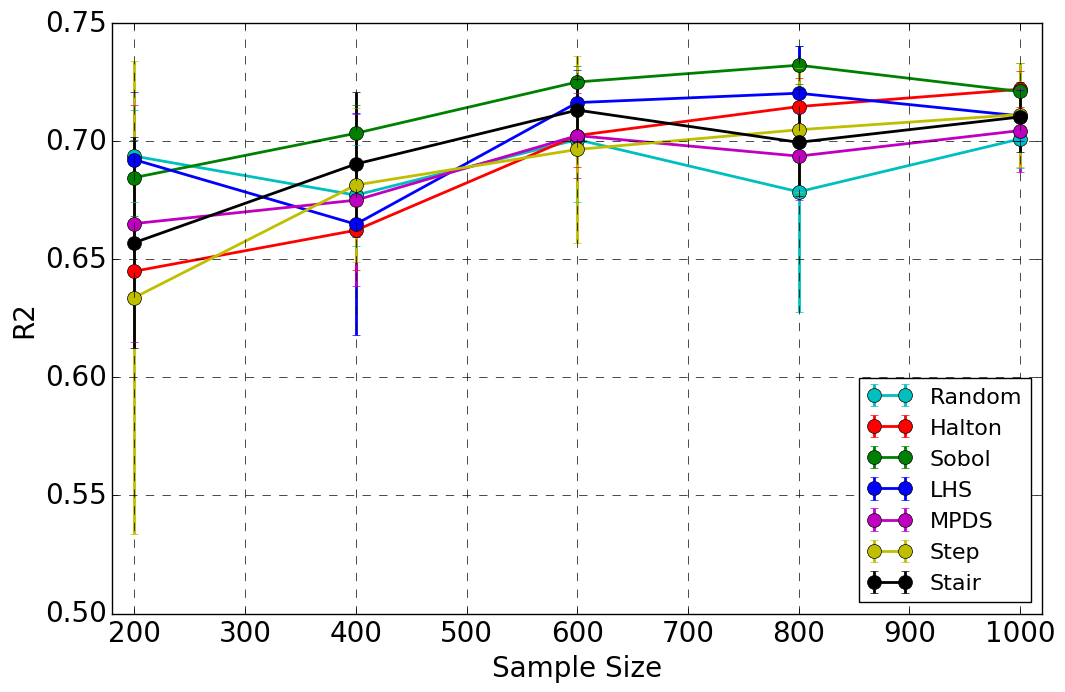}} \\
		
		\hline
		
		\footnotesize \colorbox{gray!25}{Colville}&\raisebox{-.5\height}{\includegraphics[trim = 1cm 0cm 0cm 0cm,width=0.25\textwidth,clip=True]{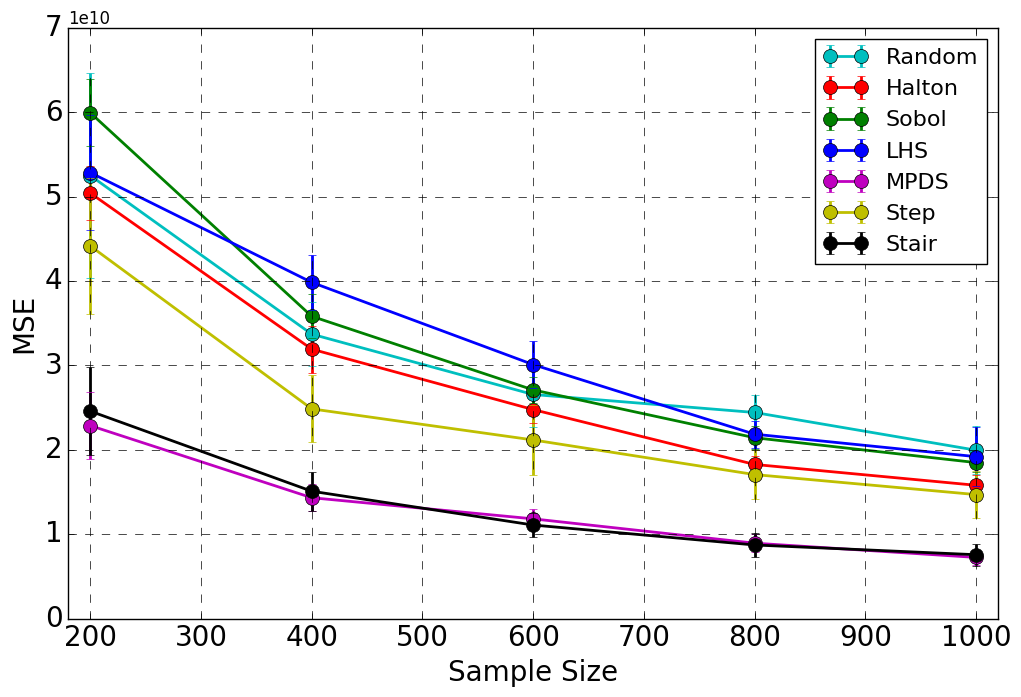}}&\raisebox{-.5\height}{\includegraphics[trim = 1cm 0cm 0cm 0cm ,width=0.25\textwidth,clip=True]{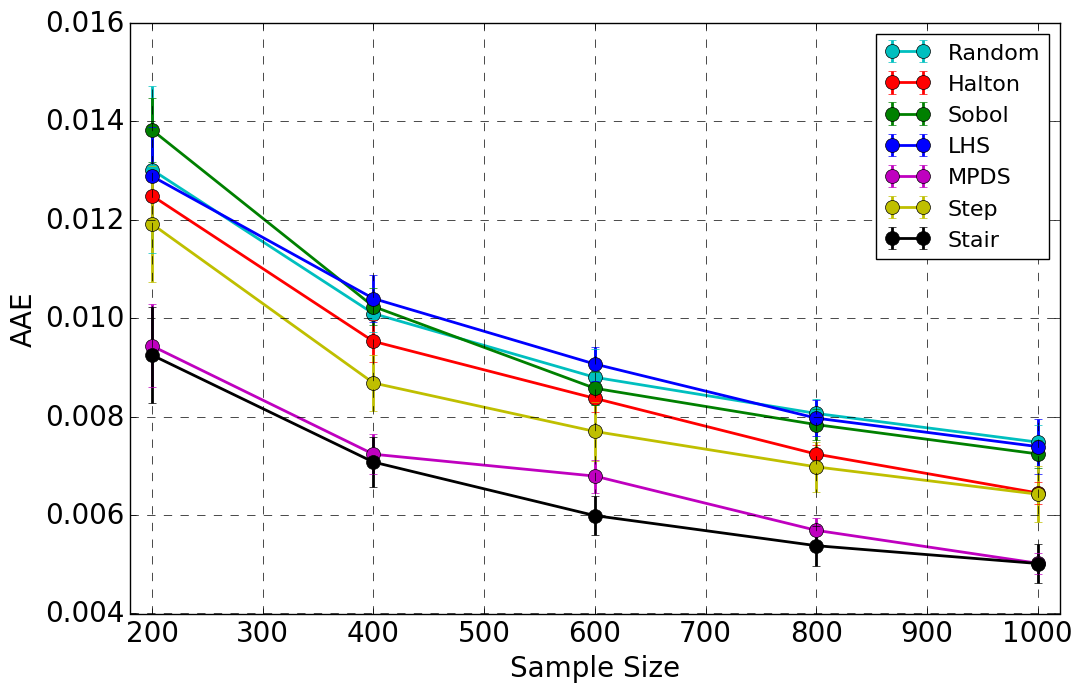}}&\raisebox{-.5\height}{\includegraphics[trim = 1cm 0cm 0cm 0cm ,width=0.25\textwidth,clip=True]{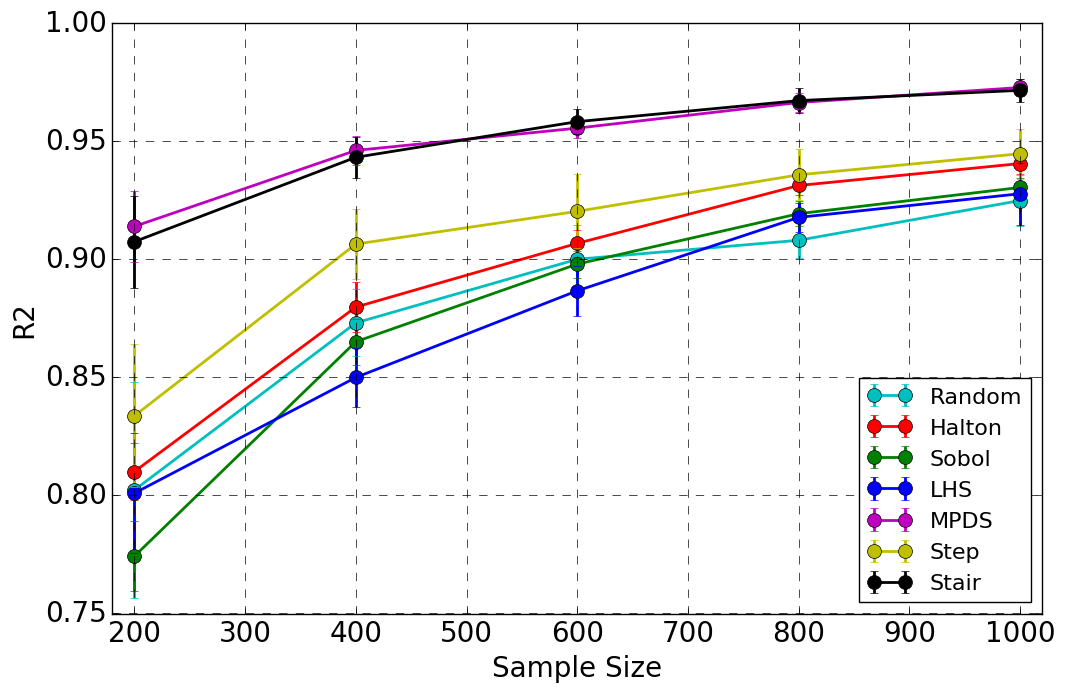}} \\
		
		\hline
		
		\footnotesize \colorbox{gray!25}{Powell}&\raisebox{-.5\height}{\includegraphics[trim = 1cm 0cm 0cm 0cm,width=0.25\textwidth,clip=True]{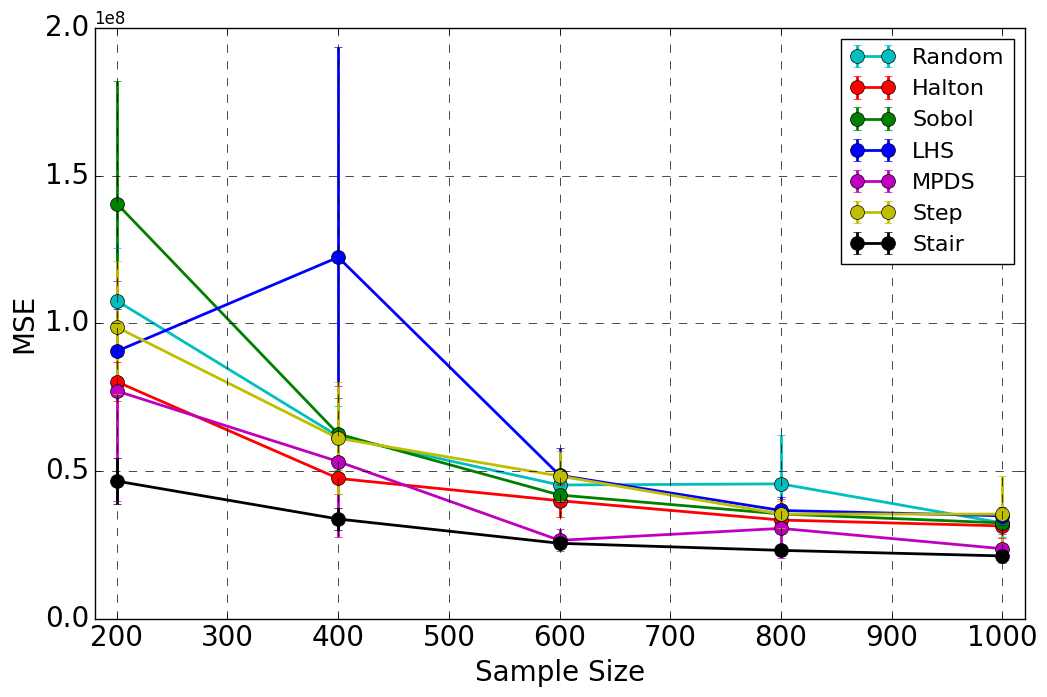}}&\raisebox{-.5\height}{\includegraphics[trim = 1cm 0cm 0cm 0cm ,width=0.25\textwidth,clip=True]{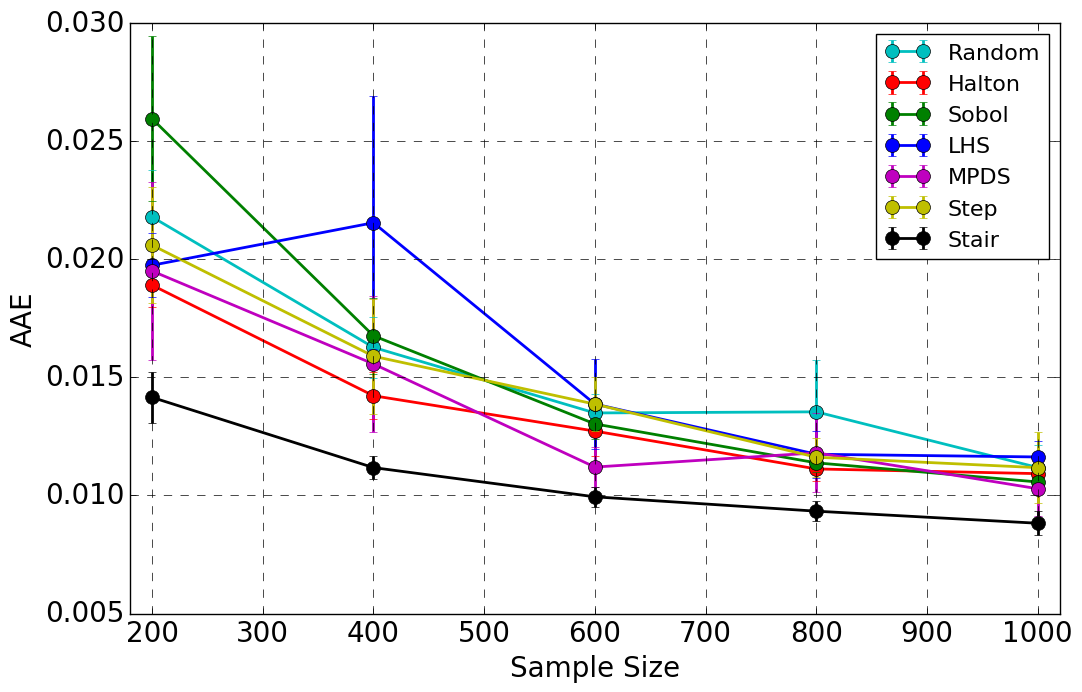}}&\raisebox{-.5\height}{\includegraphics[trim = 1cm 0cm 0cm 0cm ,width=0.25\textwidth,clip=True]{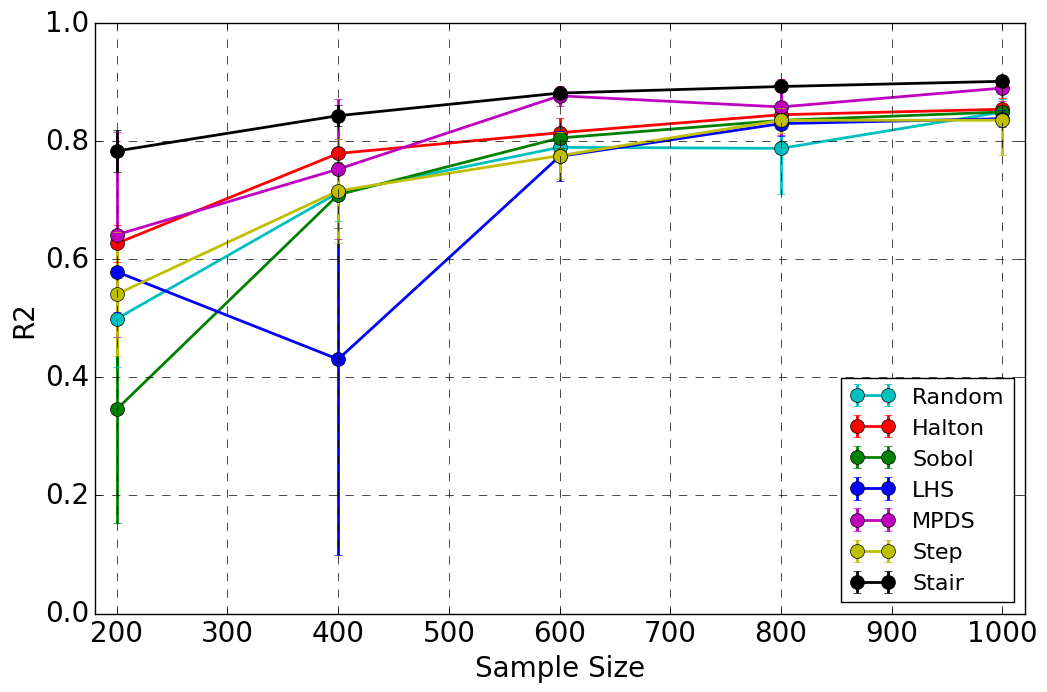}} \\

		\bottomrule
		
	\end{tabularx}
\end{table} 

\begin{table}[th]
	\centering
	\caption{Impact of sample design on generalization performance of regression models fit to benchmark analytical functions in $5$ dimensions. In higher dimensions, conventional methods such as the LHS and QMC perform very poorly, while Stair PCF design significantly outperforms all competing methods, because of improved trade-off between coverage and randomness properties.}\label{reg5}
	\begin{tabularx}{\textwidth}{@{}YYYY@{}}
		\toprule
		\textbf{Function} & \textbf{MSE} & \textbf{AAE} & \textbf{R2-Statistic}\\
		\hline
		
		\footnotesize \colorbox{gray!25}{Dolan}&\raisebox{-.5\height}{\includegraphics[trim = 1cm 0cm 0cm 0cm,width=0.25\textwidth,clip=True]{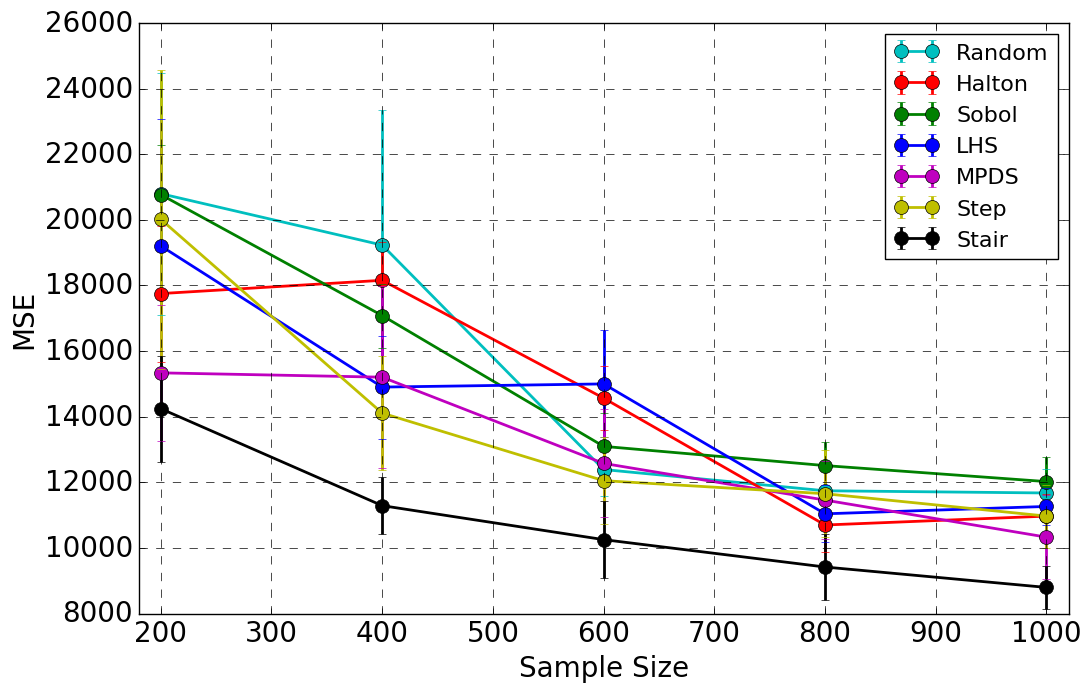}}&\raisebox{-.5\height}{\includegraphics[trim = 1cm 0cm 0cm 0cm ,width=0.25\textwidth,clip=True]{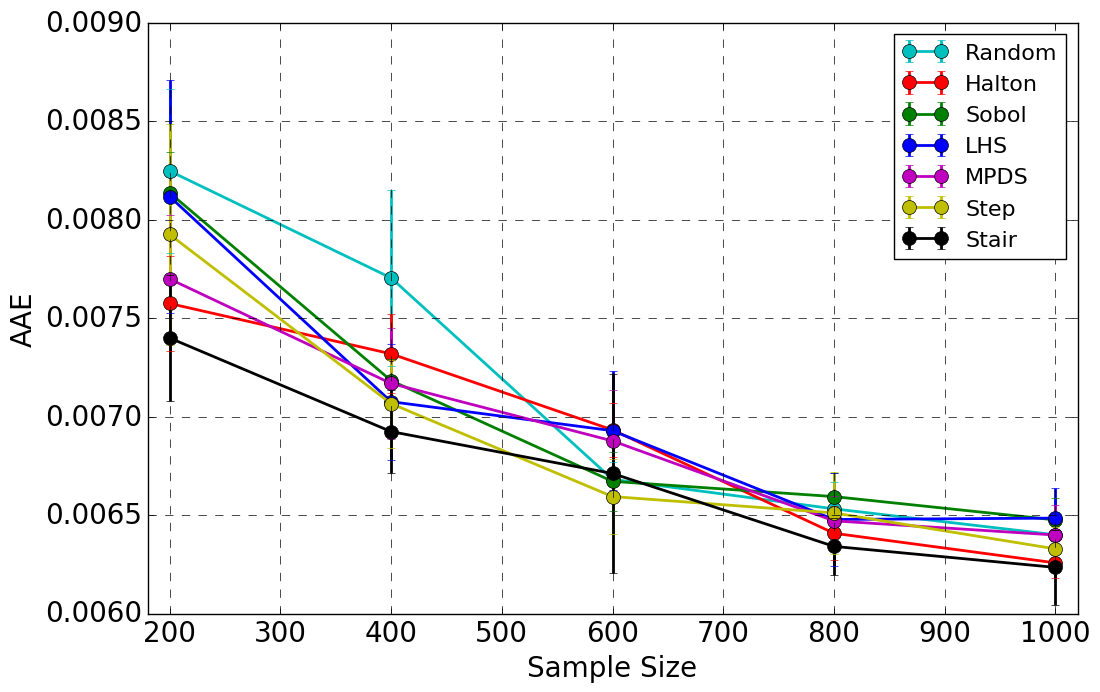}}&\raisebox{-.5\height}{\includegraphics[trim = 1cm 0cm 0cm 0cm ,width=0.25\textwidth,clip=True]{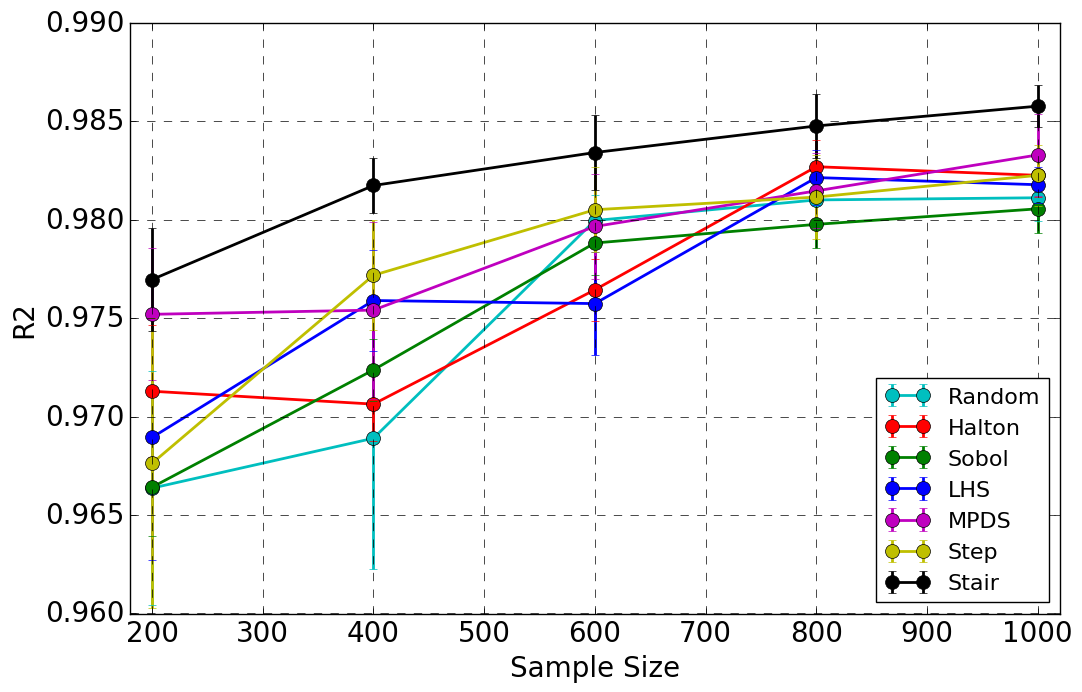}} \\
		
		\hline
		
		\footnotesize \colorbox{gray!25}{BiggsExp05}&\raisebox{-.5\height}{\includegraphics[trim = 1cm 0cm 0cm 0cm,width=0.25\textwidth,clip=True]{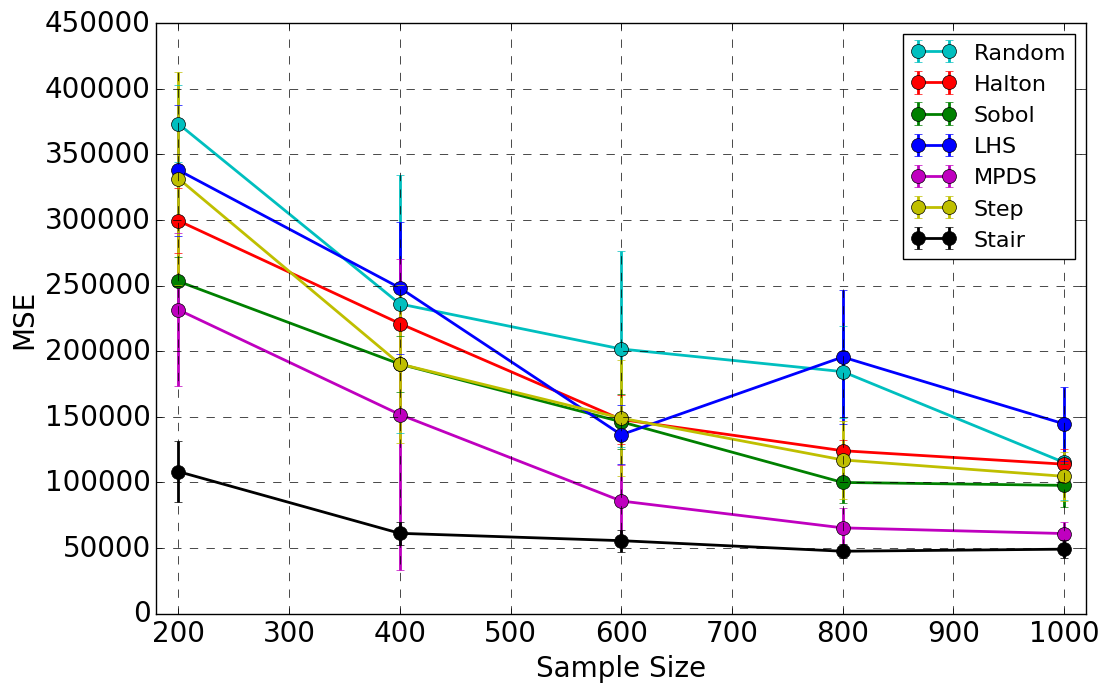}}&\raisebox{-.5\height}{\includegraphics[trim = 1cm 0cm 0cm 0cm ,width=0.25\textwidth,clip=True]{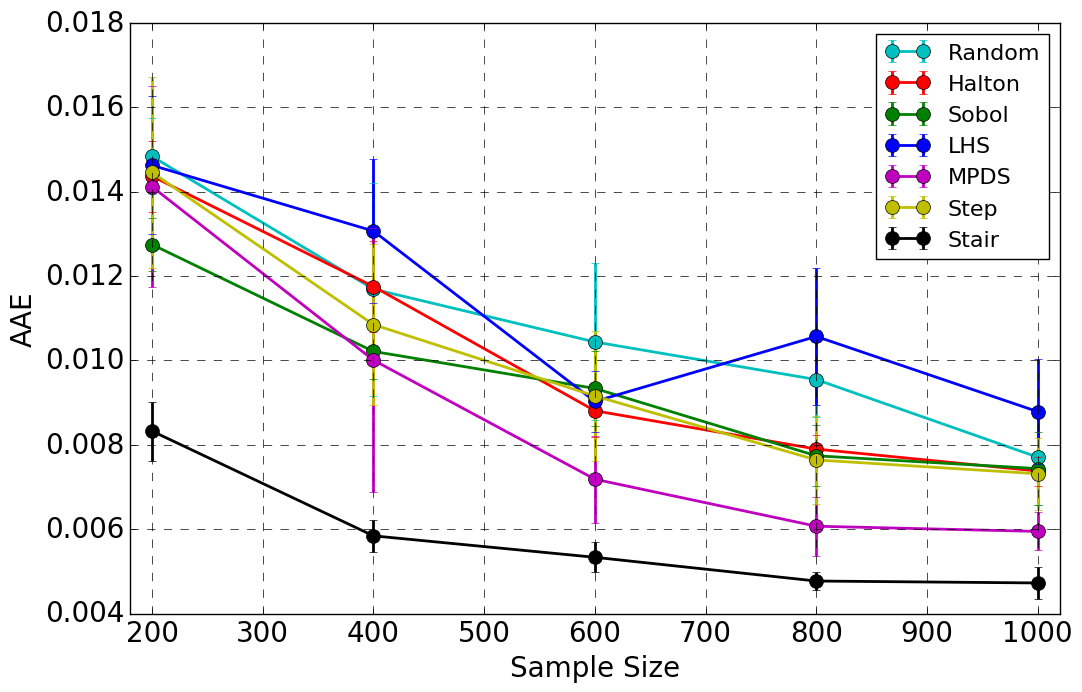}}&\raisebox{-.5\height}{\includegraphics[trim = 1cm 0cm 0cm 0cm ,width=0.25\textwidth,clip=True]{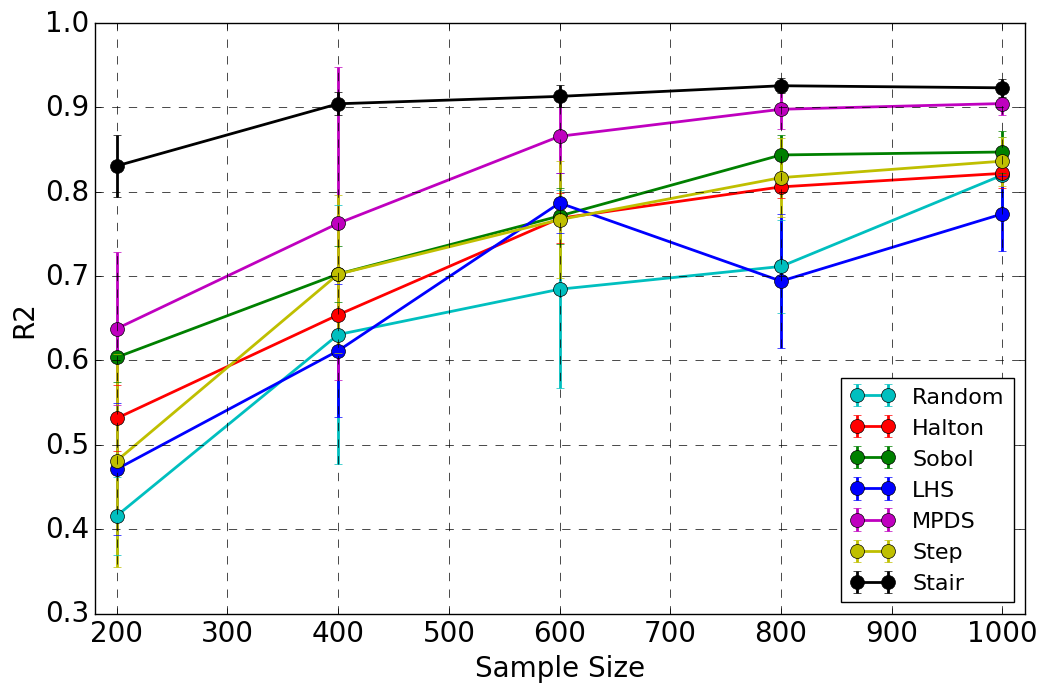}} \\

		\bottomrule
		
	\end{tabularx}
\end{table} 

\begin{table}[th]
	\centering
	
	\caption{Impact of sample design on generalization performance of regression models fit to benchmark analytical functions in $6$ dimensions. Even with highly complex functions such as the \textit{Hartmann6}, the proposed Stair PCF based spectral space-filling design produces more accurate regression models, thus evidencing the importance of improved space-filling characteristics.}\label{reg6}
	\begin{tabularx}{\textwidth}{@{}YYYY@{}}
		\toprule
		\textbf{Function} & \textbf{MSE} & \textbf{AAE} & \textbf{R2-Statistic}\\
		\hline
		
		\footnotesize \colorbox{gray!25}{Trid}&\raisebox{-.5\height}{\includegraphics[trim = 1cm 0cm 0cm 0cm,width=0.25\textwidth,clip=True]{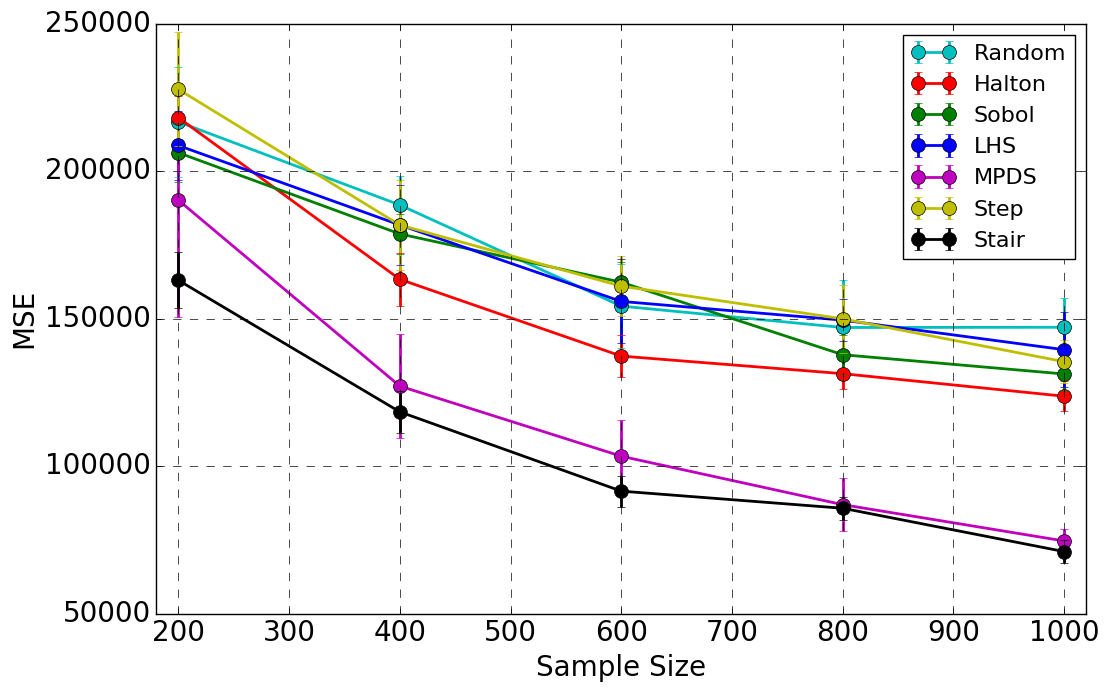}}&\raisebox{-.5\height}{\includegraphics[trim = 1cm 0cm 0cm 0cm ,width=0.25\textwidth,clip=True]{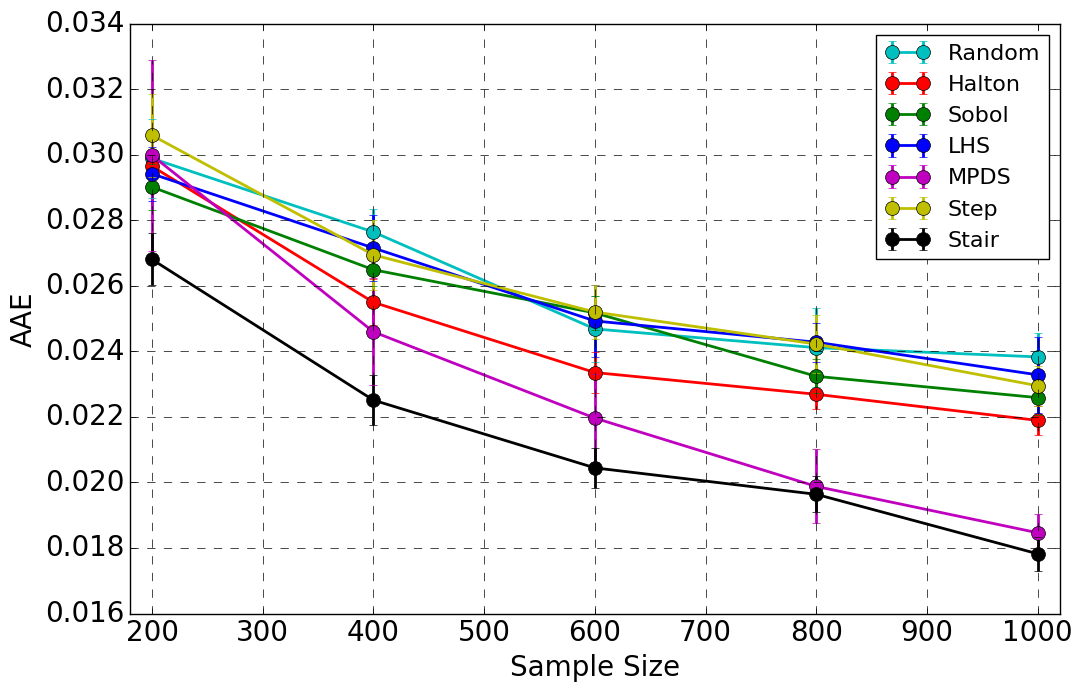}}&\raisebox{-.5\height}{\includegraphics[trim = 1cm 0cm 0cm 0cm ,width=0.25\textwidth,clip=True]{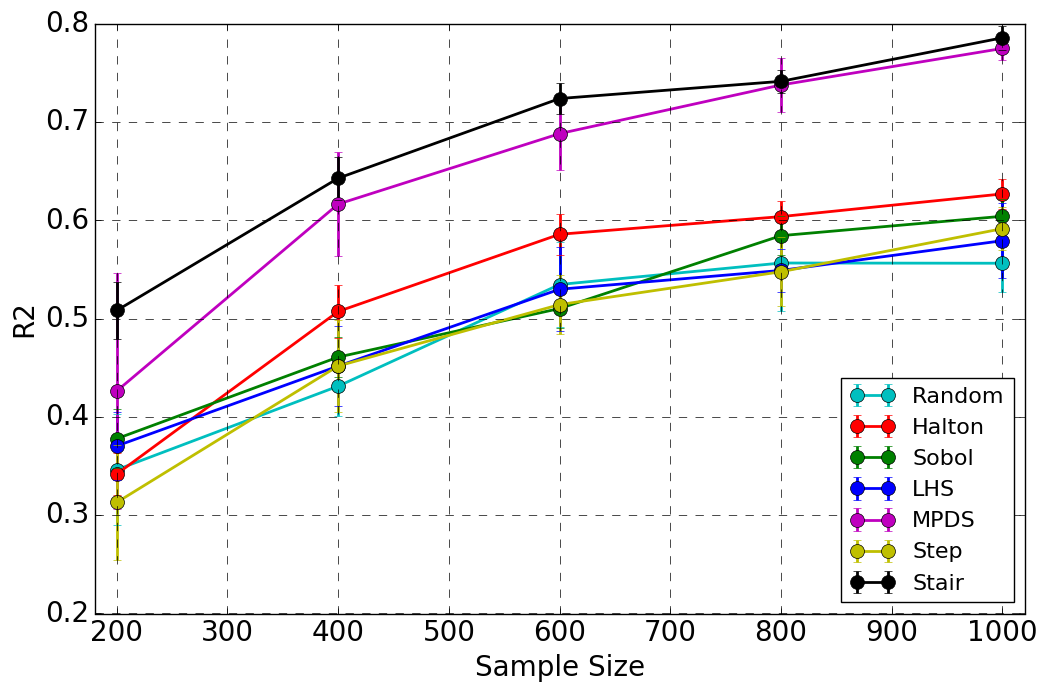}} \\
		
		\hline
		
		\footnotesize \colorbox{gray!25}{Hartmann6}&\raisebox{-.5\height}{\includegraphics[trim = 1cm 0cm 0cm -2 cm,width=0.25\textwidth,clip=True]{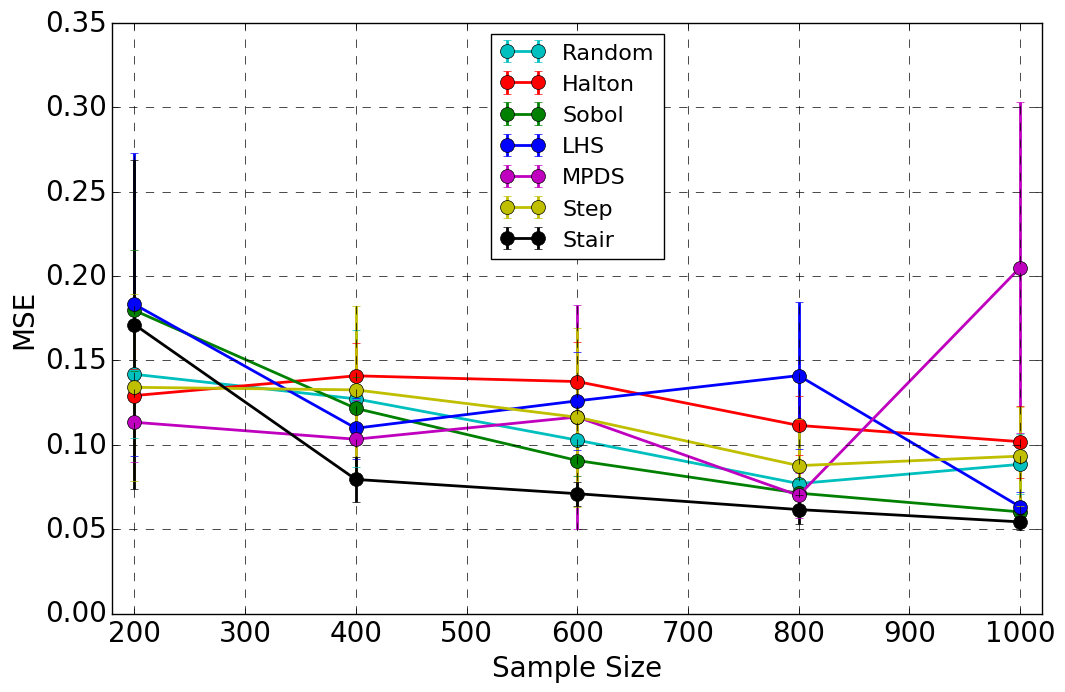}}&\raisebox{-.5\height}{\includegraphics[trim = 1cm 0cm 0cm -2 cm ,width=0.25\textwidth,clip=True]{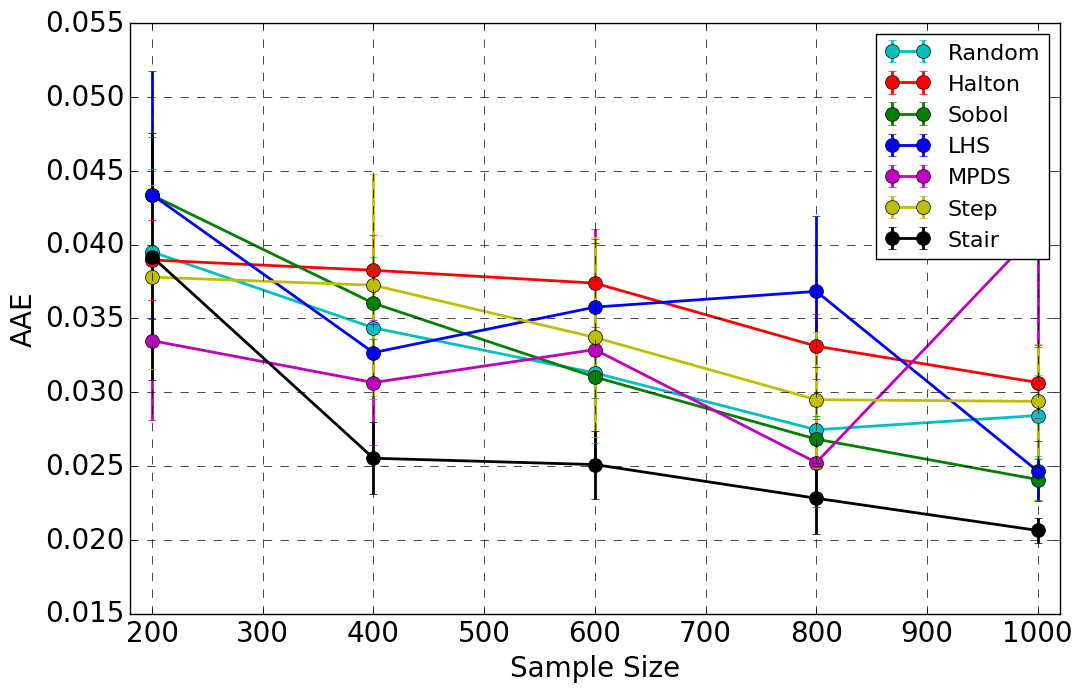}}&\raisebox{-.5\height}{\includegraphics[trim = 1cm 0cm 0cm -2 cm ,width=0.25\textwidth,clip=True]{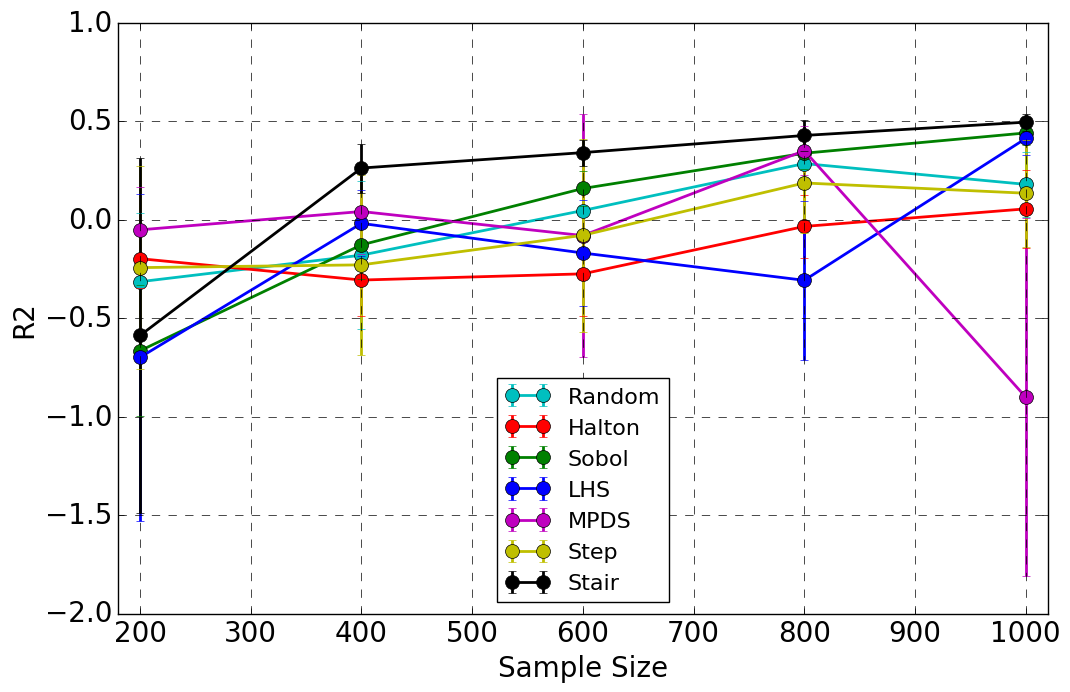}} \\

		\bottomrule
		
	\end{tabularx}
\end{table}

\begin{table}[th]
	\centering
	
	\caption{Performance of surrogate models for the NIF $1-$d HYDRA simulator using different sample design techniques, with varying number of input parameters. While the conventional sample designs achieve reasonable performance in low dimensions, the proposed Stair PCF based design is consistently superior as the dimension of the input space grows.}\label{nif:exp}
	\begin{tabularx}{\textwidth}{@{}YYYY@{}}
		\toprule
		\textbf{Function} & \textbf{MSE} & \textbf{AAE} & \textbf{R2-Statistic}\\
		\hline
		
		\multicolumn{4}{l}{\cellcolor{blue!10}\footnotesize \textbf{Parameter Space Dimension = 2}} \\
		
		\footnotesize \colorbox{gray!25}{PEAK fusion power}&\raisebox{-.5\height}{\includegraphics[trim = 1cm 0cm 0cm 0cm,width=0.25\textwidth,clip=True]{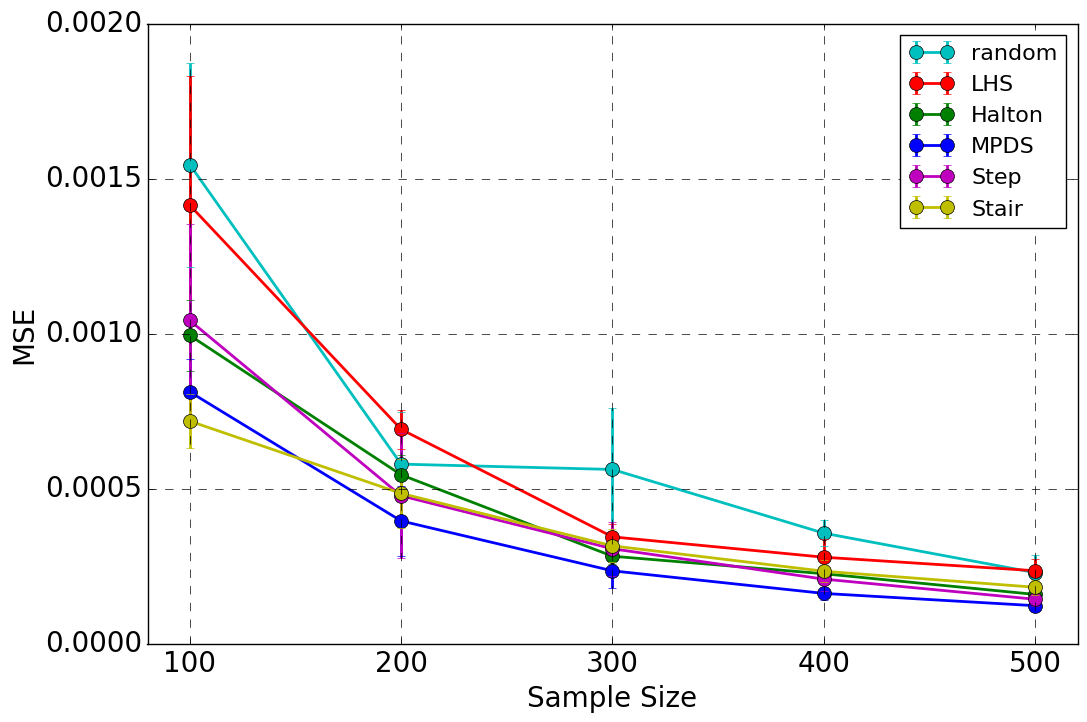}}&\raisebox{-.5\height}{\includegraphics[trim = 1cm 0cm 0cm 0cm ,width=0.25\textwidth,clip=True]{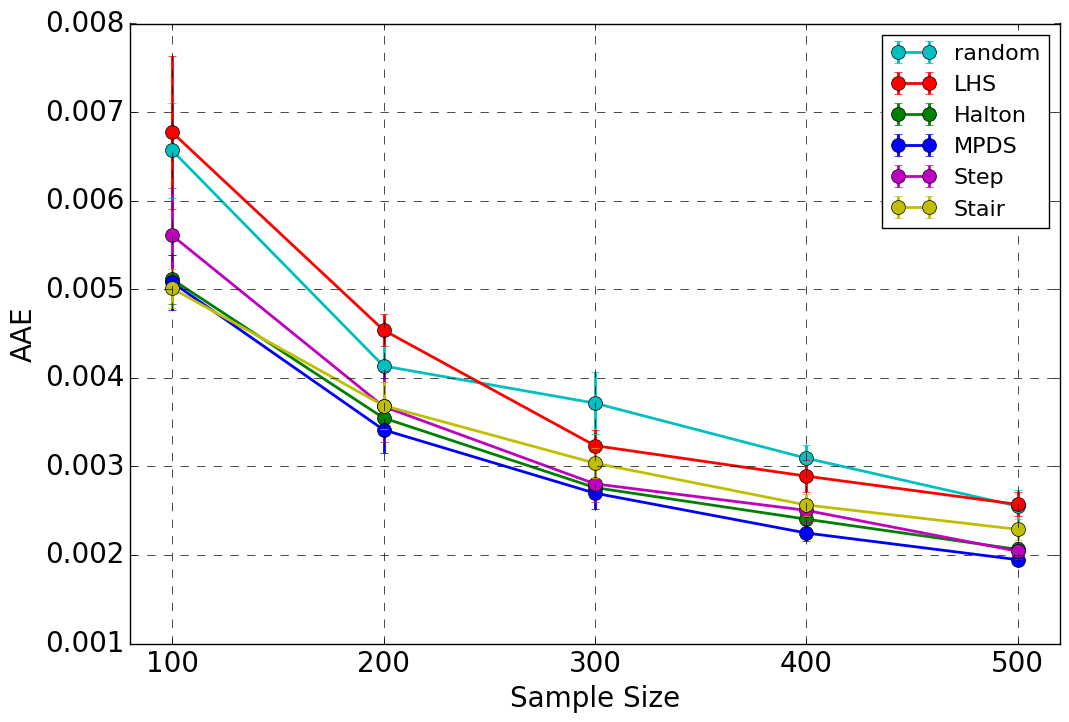}}&\raisebox{-.5\height}{\includegraphics[trim = 1cm 0cm 0cm 0cm ,width=0.25\textwidth,clip=True]{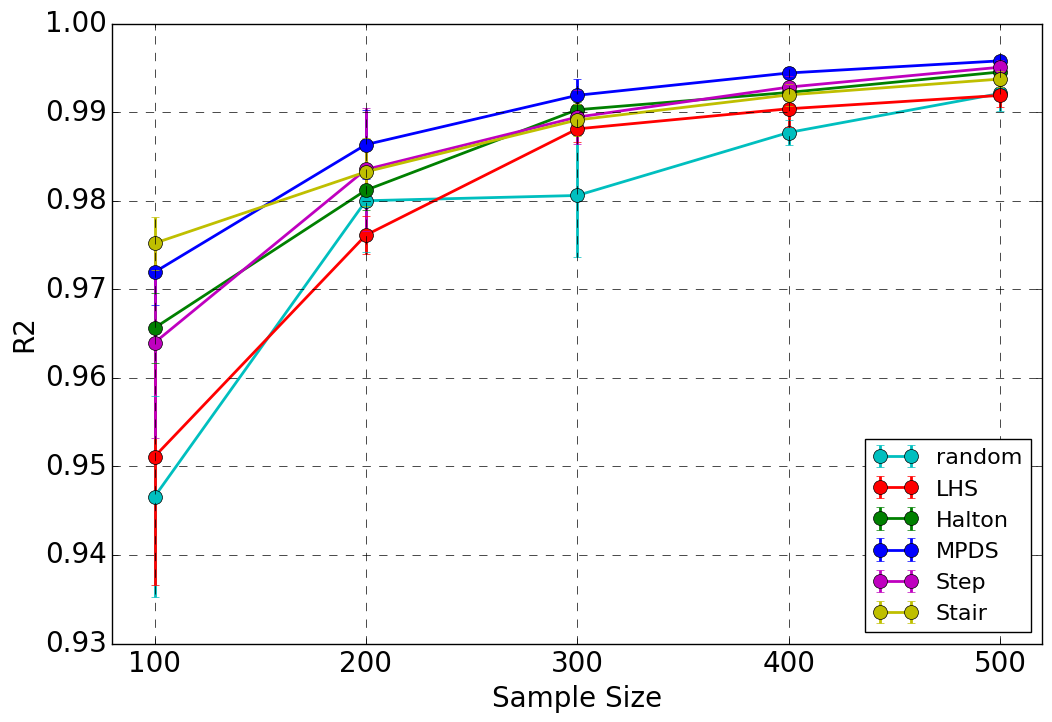}} \\
		
		\hline
		
		\multicolumn{4}{l}{\cellcolor{blue!10} \footnotesize \textbf{Parameter Space Dimension = 3}} \\
		
		\footnotesize \colorbox{gray!25}{Radiation energy}&\raisebox{-.5\height}{\includegraphics[trim = 1cm 0cm 0cm 0cm,width=0.25\textwidth,clip=True]{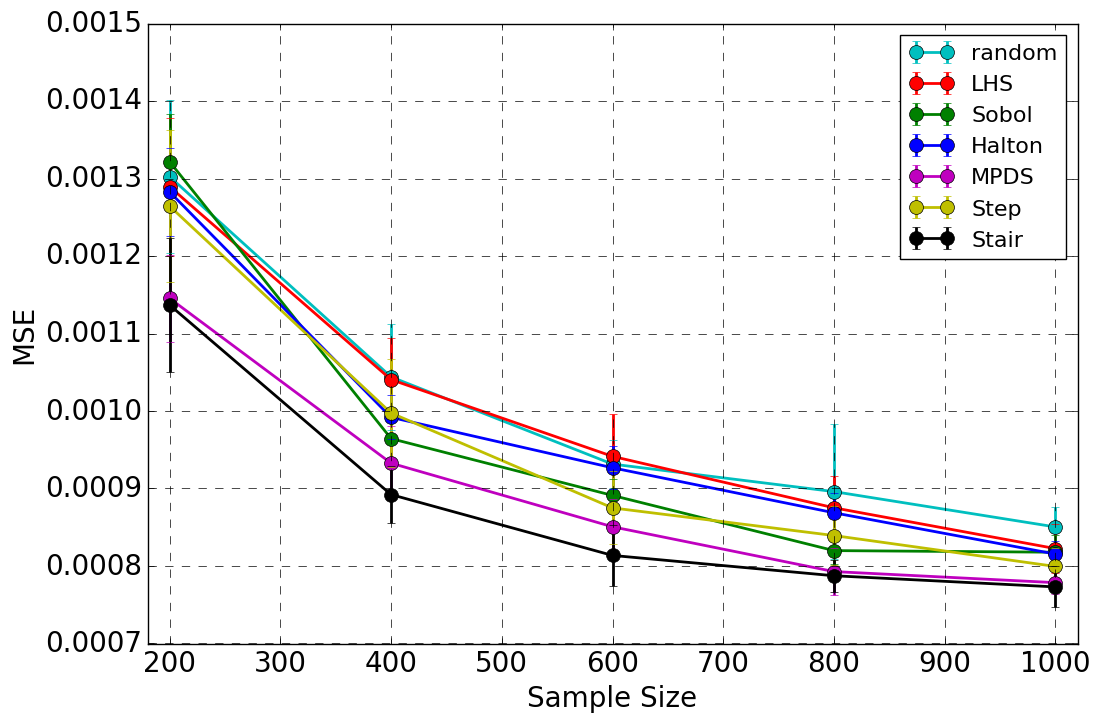}}&\raisebox{-.5\height}{\includegraphics[trim = 1cm 0cm 0cm 0cm ,width=0.25\textwidth,clip=True]{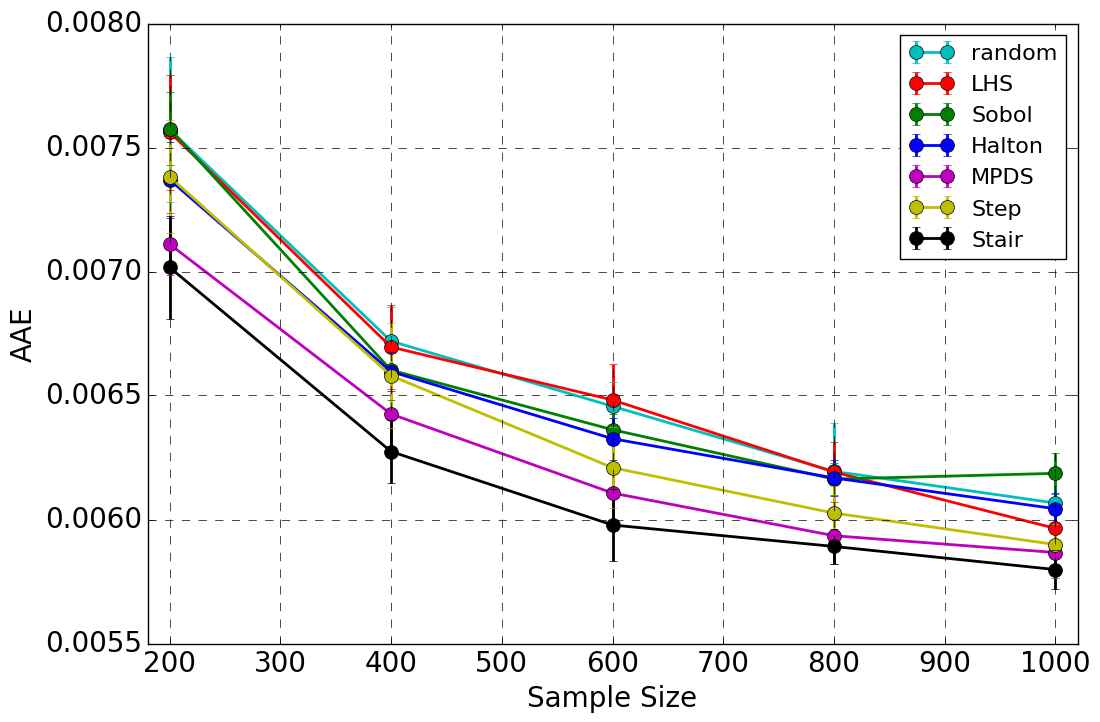}}&\raisebox{-.5\height}{\includegraphics[trim = 1cm 0cm 0cm 0cm ,width=0.25\textwidth,clip=True]{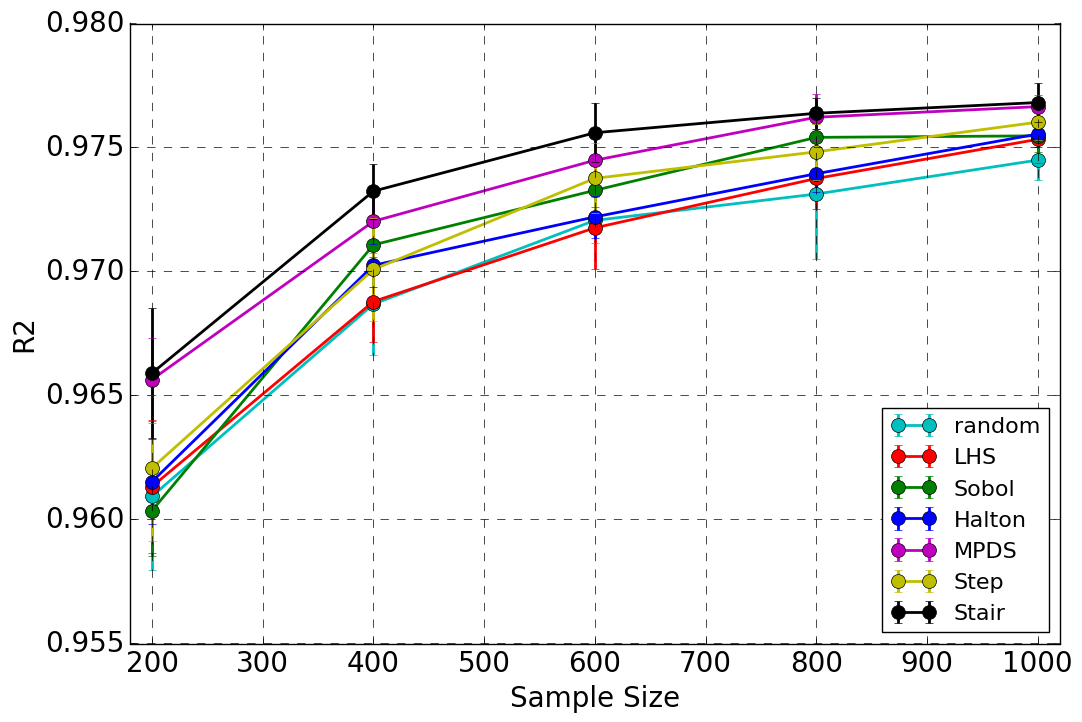}} \\
		
		\hline
		
		\multicolumn{4}{l}{\cellcolor{blue!10}\footnotesize \textbf{Parameter Space Dimension = 4}} \\
		
		\footnotesize \colorbox{gray!25}{MINradius shock}&\raisebox{-.5\height}{\includegraphics[trim = 1cm 0cm 0cm 0cm,width=0.25\textwidth,clip=True]{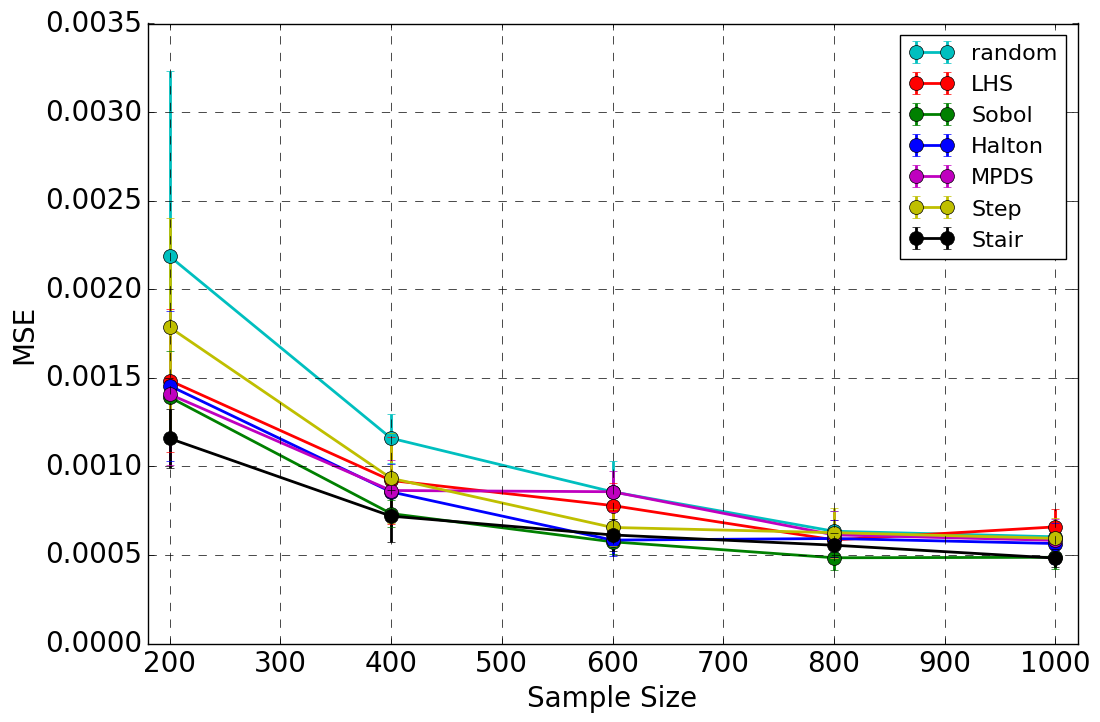}}&\raisebox{-.5\height}{\includegraphics[trim = 1cm 0cm 0cm 0cm ,width=0.25\textwidth,clip=True]{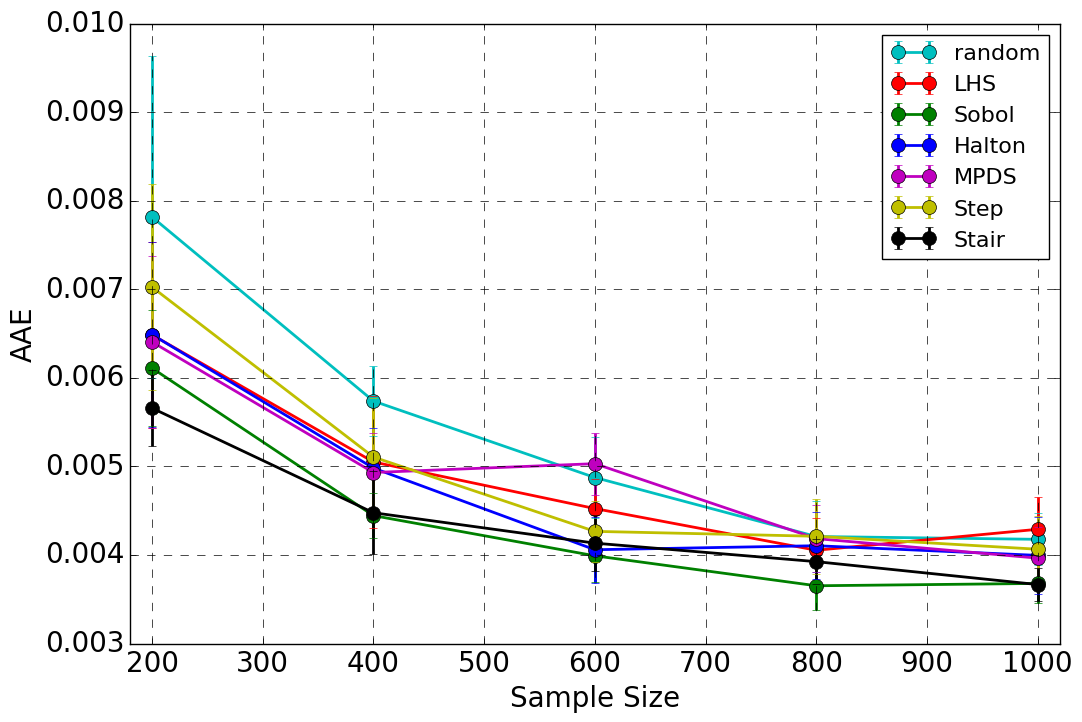}}&\raisebox{-.5\height}{\includegraphics[trim = 1cm 0cm 0cm 0cm ,width=0.25\textwidth,clip=True]{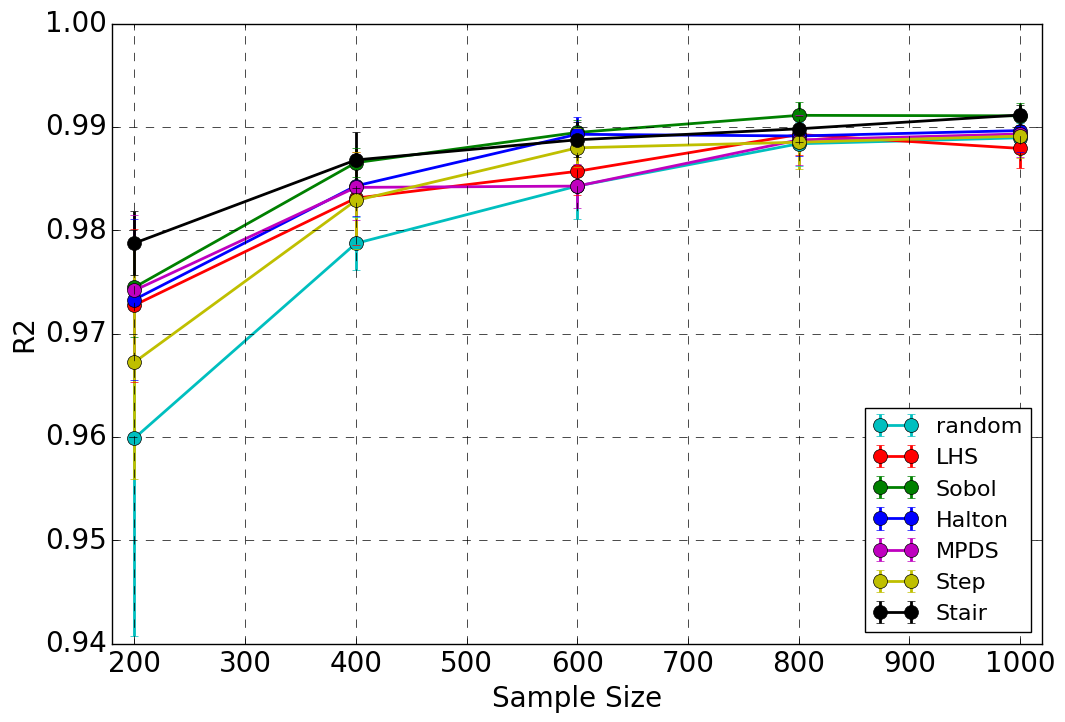}} \\
		
		\hline
		
		 \multicolumn{4}{l}{\cellcolor{blue!10}\footnotesize \textbf{Parameter Space Dimension = 5}} \\
		
		\footnotesize \colorbox{gray!25}{MINradius shock}&\raisebox{-.5\height}{\includegraphics[trim = 1cm 0cm 0cm 0cm,width=0.25\textwidth,clip=True]{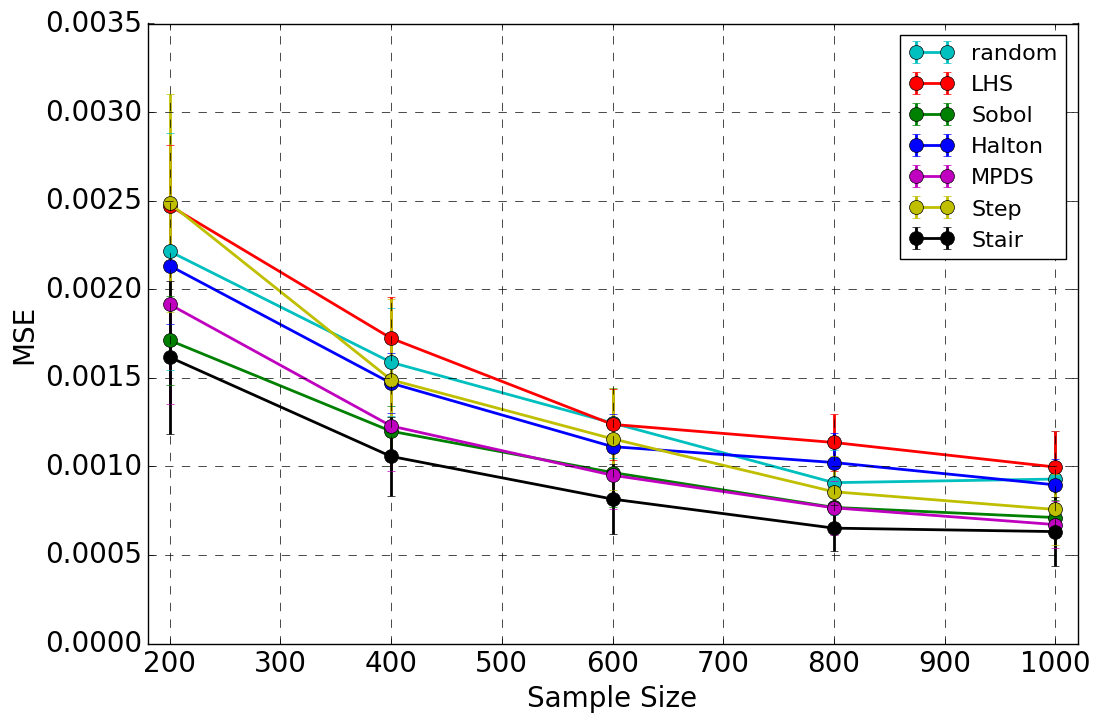}}&\raisebox{-.5\height}{\includegraphics[trim = 1cm 0cm 0cm 0cm ,width=0.25\textwidth,clip=True]{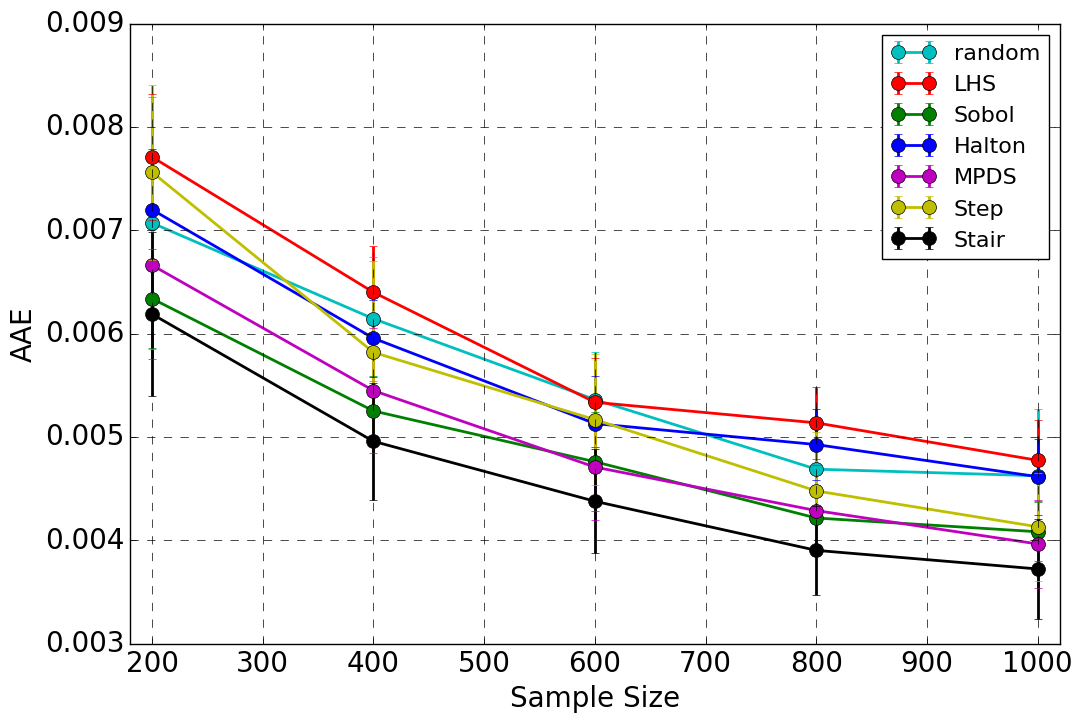}}&\raisebox{-.5\height}{\includegraphics[trim = 1cm 0cm 0cm 0cm ,width=0.25\textwidth,clip=True]{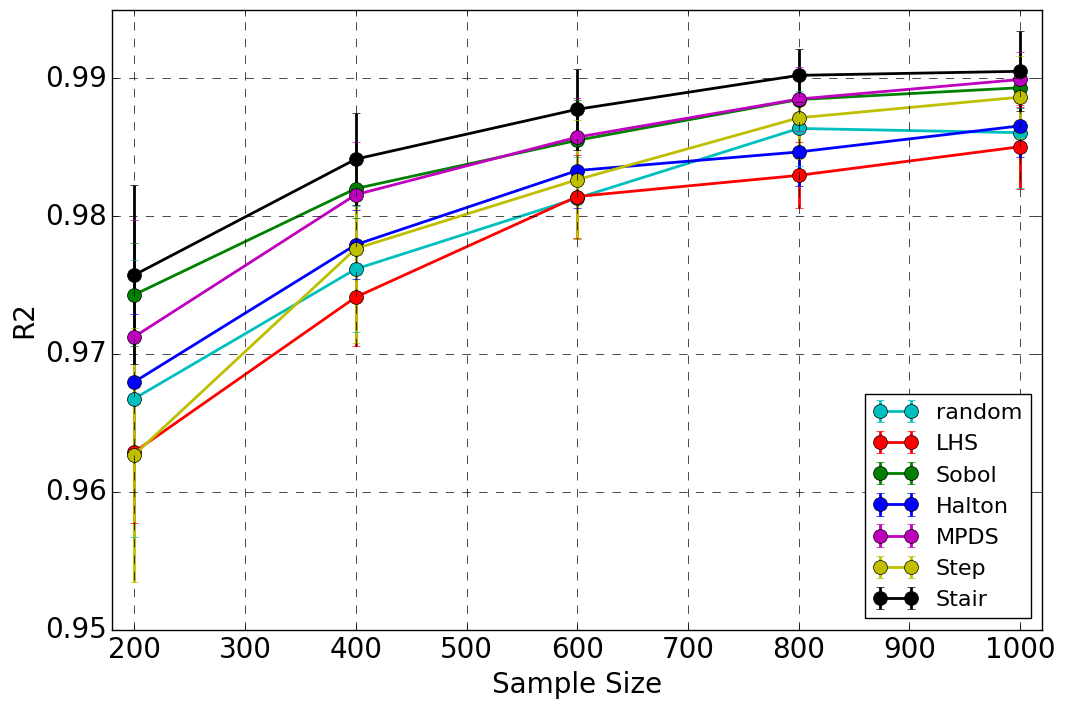}} \\
		
		\hline
		
		\multicolumn{4}{l}{\cellcolor{blue!10}\footnotesize \textbf{Parameter Space Dimension = 6}} \\
		
		\footnotesize \colorbox{gray!25}{MAXpressure}&\raisebox{-.5\height}{\includegraphics[trim = 1cm 0cm 0cm 0cm,width=0.25\textwidth,clip=True]{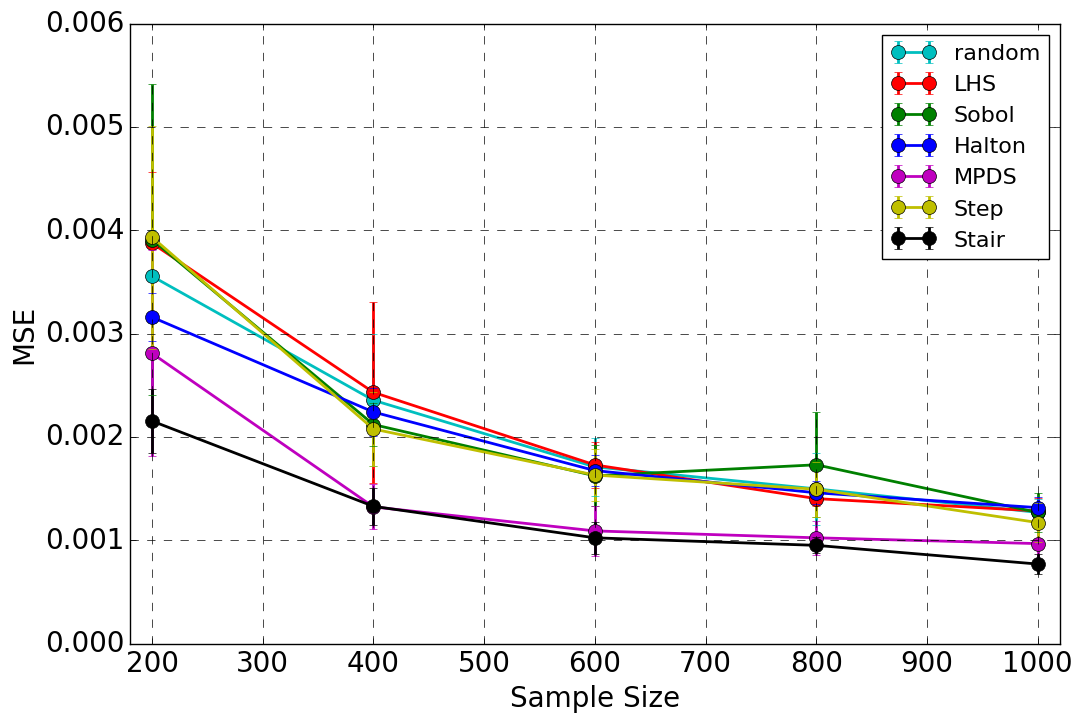}}&\raisebox{-.5\height}{\includegraphics[trim = 1cm 0cm 0cm 0cm ,width=0.25\textwidth,clip=True]{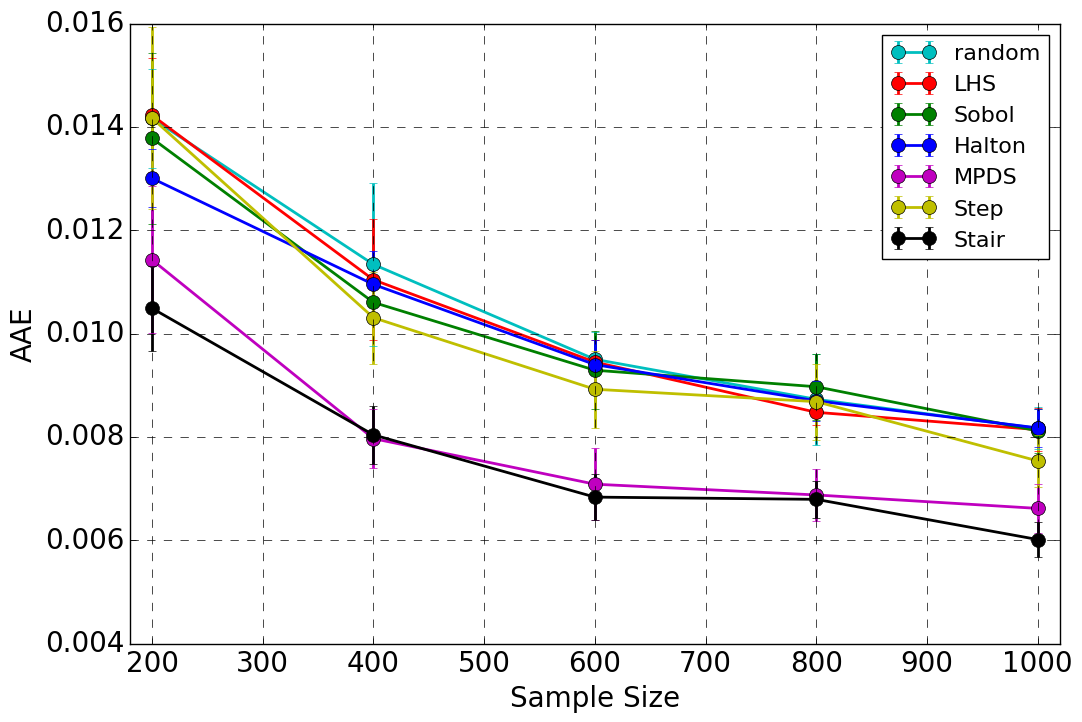}}&\raisebox{-.5\height}{\includegraphics[trim = 1cm 0cm 0cm 0cm ,width=0.25\textwidth,clip=True]{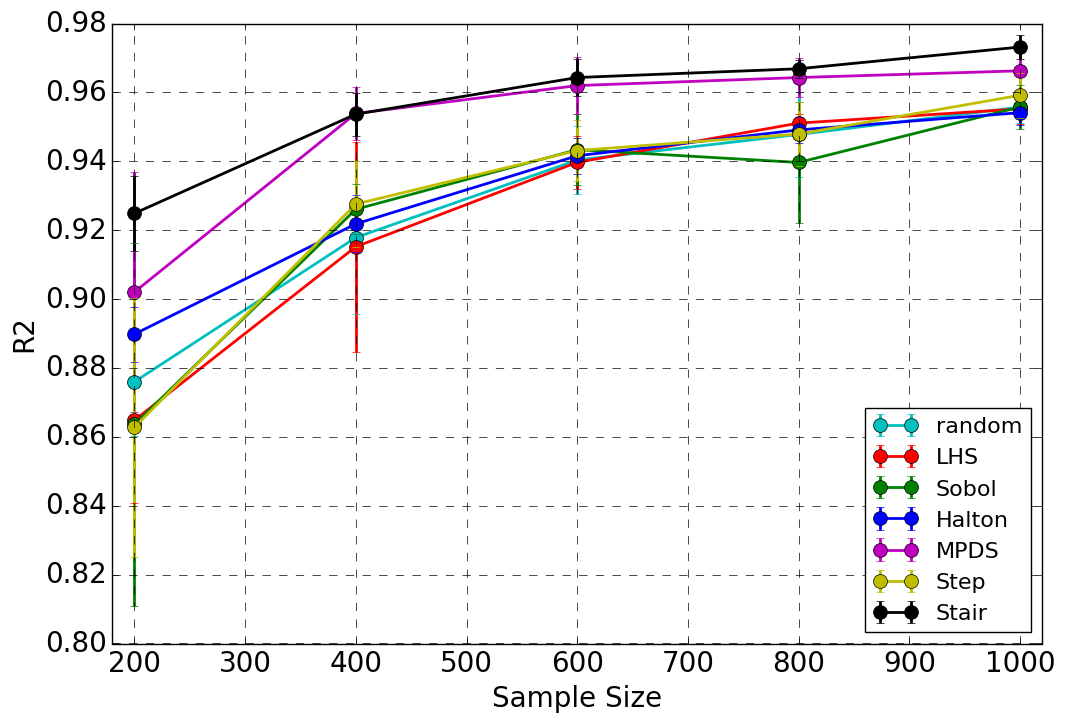}} \\
		
		\bottomrule
		
	\end{tabularx}
\end{table} 

\subsection{Surrogate Model Design for an Inertial Confinement Fusion (ICF) Simulator}

In this subsection, we consider the problem of designing surrogate models for an inertial confinement fusion (ICF) simulator developed at the National Ignition Facility (NIF). The NIF is aimed at demonstrating inertial confinement fusion (ICF), that is, thermonuclear ignition and energy gain in a laboratory setting. The goal is to focus $192$ beams of the most energetic
laser built so far onto a tiny capsule containing frozen deuterium. Under the right conditions, the resulting pressure will collapse the target to the point
of ignition where hydrogen starts to fuse and produce massive amounts of energy, effectively creating a small star which can be harnessed for energy production. Though significant progress has been made, the ultimate
goal of ``ignition" has not yet been reached. 

NIF employs an adaptive pipeline: perform experiments, use post-shot simulations to understand the experimental results, and design new experiments with parameter settings that are expected to improve performance. From an analysis viewpoint, the goal is to search the parameter space to find the region that leads to near-optimal performance. The dataset considered here is a so called engineering or macro-physics simulation ensemble in which an implosion is simulated using different input parameters, such as, laser power, pulse shape etc. From these simulations, scientists extract a set of drivers, physical quantities believed to determine the behavior of the resulting implosion. These drivers are then analyzed with respect to the energy yield to better understand how to optimize future experiments. As one can expect, the success of this pipeline heavily depends on the quality of samples used for post-shot simulations. 

We use the NIF $1$-d HYDRA simulator~\citep{hydra} and compare the performance of proposed space-filling spectral designs with existing approaches (random, LHS, Halton and Sobol and MPDS). 
For each simulation run, a large number of output quantities, such as peak velocity, yield, etc., are computed, and subsequently used to describe the resulting implosion. We vary the number of input parameters between $2$ and $6$, and fix the remaining variables to their default values. In each case, we fit a random forest regressor with $30$ trees and repeated for $20$ independent realizations of sample designs. We evaluated the reconstruction performance on $10^5$ regular grid based test samples using the metrics in the previous experiment.

Table~\ref{nif:exp} shows the regression performance of the different sample designs for various output quantities in dimensions $2$ to $6$. We observe that regression error patterns are consistent with our observations in Section~\ref{sm:bf}. The proposed Stair PCF based design consistently performs the best (followed by MPDS) for $d\geq 3$. Furthermore, the performance gain with the Stair PCF based design improves as we go higher in dimensions. This performance gain can be credited to their ability to achieve better space-filling properties in high dimensions by intelligently balancing the trade-off between coverage and randomness, and the effectiveness of the quality metric (PCF) adopted for design and optimization.

%% file: conclusion.tex
In this work, we considered the problem of constructing high quality space-filling designs. 
We proposed the use of pair correlation function (PCF) to quantify the space-filling property and systematically traded-off coverage and randomness in sample designs in arbitrary dimensions. 
Next, we linked PCF to the power spectral density (PSD) to analyze the objective measure of the design performance.  
Using the insights provided by this spatial-spectral analysis, we proposed novel space-filling spectral designs. 
We also provided an efficient PCF estimator to evaluate the space-filling properties of sample designs in arbitrary dimensions. 
Next, we devised a gradient descent based optimization algorithm to generate high quality space-filling designs. 
Superiority of proposed space-filling spectral designs were shown 
on two different applications in $2$ to $6$ dimensions: a) image reconstruction and b) surrogate
modeling on several benchmark optimization functions and an inertial confinement fusion (ICF) simulation code.
There are still many interesting questions that remain to be explored in the future work such as an analysis of the problem for non-linear manifolds. Note that, some analytical methodologies used in this paper are certainly exploitable for studying and designing space-filling designs in different manifolds. Other questions such as PCF parameterizations for other variants of space-filling designs and optimization approaches to synthesize them can also be investigated.

%% file: appendix.tex
\subsection{2 Dimensional Functions}
\textbf{Rosenbrock:} $\sum_{p=1}^2\left( 100((x^p)^2-x^{p+1})+(x^p-1)^2\right)$\\
\textbf{Cube:} $100(x^2-(x^1)^3)^2 + (1-x^1)^2$\\
\textbf{Chichinadze:} $(x^1)^2-12 x^1 + 8\sin(2.5\pi x^1)+10\cos(0.5\pi x^1)+11-0.2\dfrac{\sqrt{5}}{\exp(0.5(x^2-0.5)^2)}$\\
\textbf{GoldsteinPrice:} $\left(1+(x^1+x^2+1)^2(19-14 x^1+3(x^1)^2-14 x^2 +6 x^1 x^2 +3 (x^2)^2)\right)\\
\qquad \left(30+(2 x^1-3 x^2)^2 (18-32 x^1 + 12 (x^1)^2 + 48 x^2 -36 x^1 x^2 + 27 (x^2)^2)\right)$\\

\subsection{3 Dimensional Functions}
\textbf{BoxBetts:} $\sum_{p=1}^3 g(x^i)^2; g(x) = \exp(-0.1(p+1)x^1)-\exp(-0.1(p+1)x^2)-(\exp(-0.1(p+1)) -\exp(-(p+1)) x^3)$\\
\textbf{HelicalValley:} $100 (x^3-10\psi(x^1,x^2))^2 +(\sqrt{(x^1)^2+(x^2)^2}-1)^2 + (x^3)^2; 2\pi \psi(x^1,x^2) = \tan^{-1} (x^2/x^1)\; \textbf{if}\; x^1 \geq 0,\;\textbf{and}\; \pi+\tan^{-1}(x^2/x^1)\;\textbf{otherwise.}$\\
\textbf{Wolfe:} $\frac{4}{3} ((x^1)^2+(x^2)^2-x^1 x^2)^{0.75} + x^3$\\
\textbf{Hartmann3:} $-\sum_{i=1}^4 c_i \exp(-\sum_{j=1}^3 a_{ij}(x^j-p_{ij})^2)$\\

\subsection{4 Dimensional Functions}
\textbf{DeVilliersGlasser01:} $\sum_{i=1}^{24} \left(x^1 (x^2)^{0.1(i-1)} \sin(x^3 (0.1(i-1))+x^4)-y_i\right)^2$\\
\textbf{Powell:} $(x^3+10 x^1)^2 + 5(x^2-x^4)^2 + (x^1 -2 x^2)^4 + 10(x^3-x^4)^4$\\
\textbf{Colville:} $(x^1-1)^2 + 100 ((x^1)^2-x^2)^2 + 10.1 (x^2 -1)^2 +(x^3-1)^2 + 90 ((x^3)^2-x^4)^2+10.1 (x^4-1)^2+19.8 \frac{x^4-1}{x^2}$\\

\subsection{5 Dimensional Functions}
\textbf{BiggsExp05:}  $\sum_{i=1}^{11} (x^3 e^{-t_i x^1} - x^4 e^{-t_i x^2} + 3 e^{-t_i x^5} - y_i)^2;\;t_i = 0.1i;\;y_i = e^{-t_i} - 5e^{-10 t_i} + 3e^{-4 t_i}$\\  
\textbf{Dolan:} $|(x61+1.7 x^2)\sin(x^1) - 1.5 x^3 - 0.1 x^4 \cos(x^5-x^1)+0.2 (x^5)^2-x^2 -1 |$\\

\subsection{6 Dimensional Functions}
\textbf{Trid:} $\sum_{p=1}^6 (x^p-1)^2-\sum_{p=2}^6 x^p x^{p-1}$\\
\textbf{Hartmann6:} $-\sum_{i=1}^4 c_i \exp(-\sum_{j=1}^6 a_{ij}(x^j-p_{ij})^2)$\\

%% file: jmlr-paper.bbl
\begin{thebibliography}{58}
\providecommand{\natexlab}[1]{#1}
\providecommand{\url}[1]{\texttt{#1}}
\expandafter\ifx\csname urlstyle\endcsname\relax
  \providecommand{\doi}[1]{doi: #1}\else
  \providecommand{\doi}{doi: \begingroup \urlstyle{rm}\Url}\fi

\bibitem[Anirudh et~al.(2017)Anirudh, Kailkhura, Thiagarajan, and
  Bremer]{Anirudh:2017}
R.~Anirudh, B.~Kailkhura, J.~J. Thiagarajan, and P.~T. Bremer.
\newblock Poisson disk sampling on the grassmannnian: Applications in subspace
  optimization.
\newblock In \emph{2017 IEEE Conference on Computer Vision and Pattern
  Recognition Workshops (CVPRW)}, pages 690--698, July 2017.
\newblock \doi{10.1109/CVPRW.2017.98}.

\bibitem[Balzer et~al.(2009)Balzer, Schl\"{o}mer, and Deussen]{Balzer:2009}
Michael Balzer, Thomas Schl\"{o}mer, and Oliver Deussen.
\newblock Capacity-constrained point distributions: A variant of lloyd's
  method.
\newblock \emph{ACM Trans. Graph.}, 28\penalty0 (3):\penalty0 86:1--86:8, July
  2009.
\newblock ISSN 0730-0301.
\newblock \doi{10.1145/1531326.1531392}.

\bibitem[Bridson(2007)]{Bridson:2007}
Robert Bridson.
\newblock Fast poisson disk sampling in arbitrary dimensions.
\newblock In \emph{ACM SIGGRAPH 2007 Sketches}, SIGGRAPH '07, New York, NY,
  USA, 2007. ACM.
\newblock \doi{10.1145/1278780.1278807}.

\bibitem[Caflisch(1998)]{caflisch_1998}
Russel~E. Caflisch.
\newblock Monte carlo and quasi-monte carlo methods.
\newblock \emph{Acta Numerica}, 7:\penalty0 1–49, 1998.
\newblock \doi{10.1017/S0962492900002804}.

\bibitem[Cohn et~al.(2017)Cohn, Kumar, Miller, Radchenko, and Viazovska]{sp24}
Henry Cohn, Abhinav Kumar, Stephen~D Miller, Danylo Radchenko, and Maryna
  Viazovska.
\newblock The sphere packing problem in dimension $24$.
\newblock \emph{Annals of Mathematics}, 185\penalty0 (3):\penalty0 1017--1033,
  2017.

\bibitem[Cook(1986)]{Cook:1986}
Robert~L. Cook.
\newblock Stochastic sampling in computer graphics.
\newblock \emph{ACM Trans. Graph.}, 5\penalty0 (1):\penalty0 51--72, January
  1986.
\newblock ISSN 0730-0301.
\newblock \doi{10.1145/7529.8927}.
\newblock URL \url{http://doi.acm.org/10.1145/7529.8927}.

\bibitem[de~Goes et~al.(2012)de~Goes, Breeden, Ostromoukhov, and
  Desbrun]{deGoes:2012}
Fernando de~Goes, Katherine Breeden, Victor Ostromoukhov, and Mathieu Desbrun.
\newblock Blue noise through optimal transport.
\newblock \emph{ACM Trans. Graph.}, 31\penalty0 (6):\penalty0 171:1--171:11,
  November 2012.
\newblock ISSN 0730-0301.
\newblock \doi{10.1145/2366145.2366190}.

\bibitem[Dipp{\'e} and Wold(1985)]{Dippe:1985}
Mark A.~Z. Dipp{\'e} and Erling~Henry Wold.
\newblock Antialiasing through stochastic sampling.
\newblock \emph{SIGGRAPH Comput. Graph.}, 19\penalty0 (3):\penalty0 69--78,
  July 1985.
\newblock ISSN 0097-8930.
\newblock \doi{10.1145/325165.325182}.
\newblock URL \url{http://doi.acm.org/10.1145/325165.325182}.

\bibitem[Dunbar and Humphreys(2006)]{Dunbar:2006}
Daniel Dunbar and Greg Humphreys.
\newblock A spatial data structure for fast poisson-disk sample generation.
\newblock \emph{ACM Trans. Graph.}, 25\penalty0 (3):\penalty0 503--508, July
  2006.
\newblock ISSN 0730-0301.
\newblock \doi{10.1145/1141911.1141915}.

\bibitem[Ebeida et~al.(2011)Ebeida, Davidson, Patney, Knupp, Mitchell, and
  Owens]{Ebeida:2011}
Mohamed~S. Ebeida, Andrew~A. Davidson, Anjul Patney, Patrick~M. Knupp, Scott~A.
  Mitchell, and John~D. Owens.
\newblock Efficient maximal poisson-disk sampling.
\newblock \emph{ACM Trans. Graph.}, 30\penalty0 (4):\penalty0 49:1--49:12, July
  2011.
\newblock ISSN 0730-0301.
\newblock \doi{10.1145/2010324.1964944}.

\bibitem[Ebeida et~al.(2012)Ebeida, Mitchell, Patney, Davidson, and
  Owens]{Ebeida:2012}
Mohamed~S. Ebeida, Scott~A. Mitchell, Anjul Patney, Andrew~A. Davidson, and
  John~D. Owens.
\newblock A simple algorithm for maximal poisson-disk sampling in high
  dimensions.
\newblock \emph{Computer Graphics Forum}, 31\penalty0 (2pt4):\penalty0
  785--794, 2012.
\newblock ISSN 1467-8659.
\newblock \doi{10.1111/j.1467-8659.2012.03059.x}.

\bibitem[Ebeida et~al.(2014)Ebeida, Patney, Mitchell, Dalbey, Davidson, and
  Owens]{Ebeida2014}
Mohamed~S. Ebeida, Anjul Patney, Scott~A. Mitchell, Keith~R. Dalbey, Andrew~A.
  Davidson, and John~D. Owens.
\newblock K-d darts: Sampling by k-dimensional flat searches.
\newblock \emph{ACM Trans. Graph.}, 33\penalty0 (1):\penalty0 3:1--3:16,
  February 2014.
\newblock ISSN 0730-0301.
\newblock \doi{10.1145/2522528}.

\bibitem[Fisher(1935)]{fisher}
Ronald~A Fisher.
\newblock The design of experiments. 1935.
\newblock \emph{Oliver and Boyd, Edinburgh}, 1935.

\bibitem[Gamito and Maddock(2009)]{Gamito:2009}
Manuel~N. Gamito and Steve~C. Maddock.
\newblock Accurate multidimensional poisson-disk sampling.
\newblock \emph{ACM Trans. Graph.}, 29\penalty0 (1):\penalty0 8:1--8:19,
  December 2009.
\newblock ISSN 0730-0301.
\newblock \doi{10.1145/1640443.1640451}.

\bibitem[Garud et~al.(2017)Garud, Karimi, and Kraft]{doe:review}
Sushant~S. Garud, Iftekhar~A. Karimi, and Markus Kraft.
\newblock Design of computer experiments: A review.
\newblock \emph{Computers and Chemical Engineering}, 106\penalty0 (Supplement
  C):\penalty0 71 -- 95, 2017.
\newblock ISSN 0098-1354.
\newblock ESCAPE-26.

\bibitem[Geng et~al.(2013)Geng, Zhang, Wang, and Wang]{Bo}
Bo~Geng, HuiJuan Zhang, Heng Wang, and GuoPing Wang.
\newblock Approximate poisson disk sampling on mesh.
\newblock \emph{Science China Information Sciences}, 56\penalty0 (9):\penalty0
  1--12, 2013.
\newblock ISSN 1674-733X.
\newblock \doi{10.1007/s11432-011-4322-8}.

\bibitem[Guo et~al.(2014)Guo, Yan, Bao, Dong, Zhang, and Wonka]{Guo}
Jianwei Guo, Dong-Ming Yan, Guanbo Bao, Weiming Dong, Xiaopeng Zhang, and Peter
  Wonka.
\newblock Efficient triangulation of poisson-disk sampled point sets.
\newblock \emph{The Visual Computer}, 30\penalty0 (6-8):\penalty0 773--785,
  2014.
\newblock ISSN 0178-2789.
\newblock \doi{10.1007/s00371-014-0948-z}.

\bibitem[Halton(1964)]{Halton:1964}
J.~H. Halton.
\newblock Algorithm 247: Radical-inverse quasi-random point sequence.
\newblock \emph{Commun. ACM}, 7\penalty0 (12):\penalty0 701--702, December
  1964.
\newblock ISSN 0001-0782.
\newblock \doi{10.1145/355588.365104}.
\newblock URL \url{http://doi.acm.org/10.1145/355588.365104}.

\bibitem[Heck et~al.(2013)Heck, Schl\"{o}mer, and Deussen]{Heck:2013}
Daniel Heck, Thomas Schl\"{o}mer, and Oliver Deussen.
\newblock Blue noise sampling with controlled aliasing.
\newblock \emph{ACM Trans. Graph.}, 32\penalty0 (3):\penalty0 25:1--25:12, July
  2013.
\newblock ISSN 0730-0301.

\bibitem[Hou et~al.(2013)Hou, Zhang, Li, Lai, and Ding]{Hou2013}
Wenguang Hou, Xuming Zhang, Xin Li, Xudong Lai, and Mingyue Ding.
\newblock Poisson disk sampling in geodesic metric for \{DEM\} simplification.
\newblock \emph{International Journal of Applied Earth Observation and
  Geoinformation}, 23:\penalty0 264 -- 272, 2013.
\newblock ISSN 0303-2434.
\newblock \doi{http://dx.doi.org/10.1016/j.jag.2012.09.008}.

\bibitem[Illian et~al.(2008)Illian, Penttinen, Stoyan, and Stoyan]{illian2008}
Janine Illian, Antti Penttinen, Helga Stoyan, and Dietrich Stoyan.
\newblock \emph{Statistical analysis and modelling of spatial point patterns},
  volume~70.
\newblock John Wiley \& Sons, 2008.

\bibitem[Ip et~al.(2013)Ip, Yal\c{c}in, Luebke, and Varshney]{Amitabh}
Cheuk~Yiu Ip, M.~Adil Yal\c{c}in, David Luebke, and Amitabh Varshney.
\newblock Pixelpie: Maximal poisson-disk sampling with rasterization.
\newblock In \emph{Proceedings of the 5th High-Performance Graphics
  Conference}, HPG '13, pages 17--26, New York, NY, USA, 2013. ACM.
\newblock ISBN 978-1-4503-2135-8.
\newblock \doi{10.1145/2492045.2492047}.

\bibitem[Jamil and Yang(2013)]{opt:fun}
Momin Jamil and Xin-She Yang.
\newblock A literature survey of benchmark functions for global optimisation
  problems.
\newblock \emph{International Journal of Mathematical Modelling and Numerical
  Optimisation}, 4\penalty0 (2):\penalty0 150--194, 2013.

\bibitem[Jin et~al.(2005)Jin, Chen, and Sudjianto]{lhs:dist}
Ruichen Jin, Wei Chen, and Agus Sudjianto.
\newblock An efficient algorithm for constructing optimal design of computer
  experiments.
\newblock \emph{Journal of Statistical Planning and Inference}, 134\penalty0
  (1):\penalty0 268 -- 287, 2005.
\newblock ISSN 0378-3758.
\newblock \doi{https://doi.org/10.1016/j.jspi.2004.02.014}.
\newblock URL
  \url{http://www.sciencedirect.com/science/article/pii/S0378375804001922}.

\bibitem[Kailkhura et~al.(2016{\natexlab{a}})Kailkhura, Thiagarajan, Bremer,
  and Varshney]{Kailkhura:2016}
B.~Kailkhura, J.~J. Thiagarajan, P.~T. Bremer, and P.~K. Varshney.
\newblock Theoretical guarantees for poisson disk sampling using pair
  correlation function.
\newblock In \emph{2016 IEEE International Conference on Acoustics, Speech and
  Signal Processing (ICASSP)}, pages 2589--2593, March 2016{\natexlab{a}}.
\newblock \doi{10.1109/ICASSP.2016.7472145}.

\bibitem[Kailkhura et~al.(2016{\natexlab{b}})Kailkhura, Thiagarajan, Bremer,
  and Varshney]{Kailkhura:2016:SBN}
Bhavya Kailkhura, Jayaraman~J. Thiagarajan, Peer-Timo Bremer, and Pramod~K.
  Varshney.
\newblock Stair blue noise sampling.
\newblock \emph{ACM Trans. Graph.}, 35\penalty0 (6):\penalty0 248:1--248:10,
  November 2016{\natexlab{b}}.
\newblock ISSN 0730-0301.
\newblock \doi{10.1145/2980179.2982435}.
\newblock URL \url{http://doi.acm.org/10.1145/2980179.2982435}.

\bibitem[Koehler and Owen(1996)]{owen:book}
JR~Koehler and AB~Owen.
\newblock Computer experiments.
\newblock \emph{Handbook of statistics}, 13:\penalty0 261--308, 1996.

\bibitem[Lagae and Dutré(2008)]{Lagae:2008}
Ares Lagae and Philip Dutré.
\newblock A comparison of methods for generating poisson disk distributions.
\newblock \emph{Computer Graphics Forum}, 27\penalty0 (1):\penalty0 114--129,
  2008.
\newblock ISSN 1467-8659.
\newblock \doi{10.1111/j.1467-8659.2007.01100.x}.

\bibitem[Leary et~al.(2003)Leary, Bhaskar, and Keane]{oa2:lhs}
Stephen Leary, Atul Bhaskar, and Andy Keane.
\newblock Optimal orthogonal-array-based latin hypercubes.
\newblock 2003.
\newblock URL \url{https://eprints.soton.ac.uk/22393/}.

\bibitem[L'Ecuyer and Lemieux(2005)]{rqmc}
Pierre L'Ecuyer and Christiane Lemieux.
\newblock Recent advances in randomized quasi-monte carlo methods.
\newblock \emph{Modeling uncertainty}, pages 419--474, 2005.

\bibitem[Marinak et~al.(2001)Marinak, Kerbel, Gentile, Jones, Munro, Pollaine,
  Dittrich, and Haan]{hydra}
MM~Marinak, GD~Kerbel, NA~Gentile, O~Jones, D~Munro, S~Pollaine, TR~Dittrich,
  and SW~Haan.
\newblock Three-dimensional hydra simulations of national ignition facility
  targets.
\newblock \emph{Physics of Plasmas}, 8\penalty0 (5):\penalty0 2275--2280, 2001.

\bibitem[McCool and Fiume(1992)]{McCool:1992}
Michael McCool and Eugene Fiume.
\newblock Hierarchical poisson disk sampling distributions.
\newblock In \emph{Proceedings of the Conference on Graphics Interface '92},
  pages 94--105, San Francisco, CA, USA, 1992. Morgan Kaufmann Publishers Inc.
\newblock ISBN 0-9695338-1-0.
\newblock URL \url{http://dl.acm.org/citation.cfm?id=155294.155306}.

\bibitem[McKay(1992)]{McKay:1992}
Michael~D. McKay.
\newblock Latin hypercube sampling as a tool in uncertainty analysis of
  computer models.
\newblock In \emph{Proceedings of the 24th Conference on Winter Simulation},
  WSC '92, pages 557--564, New York, NY, USA, 1992. ACM.
\newblock ISBN 0-7803-0798-4.
\newblock \doi{10.1145/167293.167637}.

\bibitem[Morokoff and Caflisch(1994)]{qmc:ana}
William~J. Morokoff and Russel~E. Caflisch.
\newblock Quasi-random sequences and their discrepancies.
\newblock \emph{SIAM J. Sci. Comput}, 15:\penalty0 1251--1279, 1994.

\bibitem[Morris and Mitchell(1995)]{lhs:maximin}
Max~D. Morris and Toby~J. Mitchell.
\newblock Exploratory designs for computational experiments.
\newblock \emph{Journal of Statistical Planning and Inference}, 43\penalty0
  (3):\penalty0 381 -- 402, 1995.
\newblock ISSN 0378-3758.
\newblock \doi{https://doi.org/10.1016/0378-3758(94)00035-T}.
\newblock URL
  \url{http://www.sciencedirect.com/science/article/pii/037837589400035T}.

\bibitem[Ostromoukhov et~al.(2004)Ostromoukhov, Donohue, and
  Jodoin]{Ostromoukhov:2004}
Victor Ostromoukhov, Charles Donohue, and Pierre-Marc Jodoin.
\newblock Fast hierarchical importance sampling with blue noise properties.
\newblock \emph{ACM Trans. Graph.}, 23\penalty0 (3):\penalty0 488--495, August
  2004.
\newblock ISSN 0730-0301.
\newblock \doi{10.1145/1015706.1015750}.
\newblock URL \url{http://doi.acm.org/10.1145/1015706.1015750}.

\bibitem[Owen(1995)]{Owen:1995}
Art~B. Owen.
\newblock \emph{Randomly Permuted (t,m,s)-Nets and (t, s)-Sequences}, pages
  299--317.
\newblock Springer New York, New York, NY, 1995.
\newblock ISBN 978-1-4612-2552-2.
\newblock \doi{10.1007/978-1-4612-2552-2_19}.
\newblock URL \url{https://doi.org/10.1007/978-1-4612-2552-2_19}.

\bibitem[Owen(2009)]{Sampling:survey}
Art~B Owen.
\newblock Monte carlo and quasi-monte carlo for statistics.
\newblock \emph{Monte Carlo and Quasi-Monte Carlo Methods 2008}, pages 3--18,
  2009.

\bibitem[Owen and Tribble(2005)]{owen2005quasi}
Art~B Owen and Seth~D Tribble.
\newblock A quasi-monte carlo metropolis algorithm.
\newblock \emph{Proceedings of the National Academy of Sciences of the United
  States of America}, 102\penalty0 (25):\penalty0 8844--8849, 2005.

\bibitem[\"{O}ztireli and Gross(2012)]{Oztireli:2012}
A.~Cengiz \"{O}ztireli and Markus Gross.
\newblock Analysis and synthesis of point distributions based on pair
  correlation.
\newblock \emph{ACM Trans. Graph.}, 31\penalty0 (6):\penalty0 170:1--170:10,
  November 2012.
\newblock ISSN 0730-0301.
\newblock \doi{10.1145/2366145.2366189}.

\bibitem[Packham(2015)]{lhs:var}
N.~Packham.
\newblock Combining latin hypercube sampling with other variance reduction
  techniques.
\newblock \emph{Wilmott}, 2015\penalty0 (76):\penalty0 60--69, 2015.
\newblock ISSN 1541-8286.
\newblock \doi{10.1002/wilm.10410}.
\newblock URL \url{http://dx.doi.org/10.1002/wilm.10410}.

\bibitem[Schl\"{o}mer et~al.(2011)Schl\"{o}mer, Heck, and
  Deussen]{Schlomer:2011}
Thomas Schl\"{o}mer, Daniel Heck, and Oliver Deussen.
\newblock Farthest-point optimized point sets with maximized minimum distance.
\newblock In \emph{Proceedings of the ACM SIGGRAPH Symposium on High
  Performance Graphics}, HPG '11, pages 135--142, New York, NY, USA, 2011. ACM.
\newblock ISBN 978-1-4503-0896-0.
\newblock \doi{10.1145/2018323.2018345}.
\newblock URL \url{http://doi.acm.org/10.1145/2018323.2018345}.

\bibitem[Schreiber(1943)]{Hex-lattice}
Edwin~W. Schreiber.
\newblock Mathematical snapshots*.
\newblock \emph{School Science and Mathematics}, 43\penalty0 (9):\penalty0
  795--799, 1943.
\newblock ISSN 1949-8594.
\newblock \doi{10.1111/j.1949-8594.1943.tb06003.x}.
\newblock URL \url{http://dx.doi.org/10.1111/j.1949-8594.1943.tb06003.x}.

\bibitem[Sobol(1967)]{Sobol:1967}
I.M Sobol.
\newblock On the distribution of points in a cube and the approximate
  evaluation of integrals.
\newblock \emph{USSR Computational Mathematics and Mathematical Physics},
  7\penalty0 (4):\penalty0 86 -- 112, 1967.
\newblock ISSN 0041-5553.
\newblock \doi{https://doi.org/10.1016/0041-5553(67)90144-9}.
\newblock URL
  \url{http://www.sciencedirect.com/science/article/pii/0041555367901449}.

\bibitem[Tang(1993)]{oa1:lhs}
Boxin Tang.
\newblock Orthogonal array-based latin hypercubes.
\newblock \emph{Journal of the American statistical association}, 88\penalty0
  (424):\penalty0 1392--1397, 1993.

\bibitem[viazovska(2017)]{sp8}
maryna~s viazovska.
\newblock the sphere packing problem in dimension 8.
\newblock \emph{annals of mathematics}, 185\penalty0 (3):\penalty0 991--1015,
  2017.

\bibitem[Wachtel et~al.(2014)Wachtel, Pilleboue, Coeurjolly, Breeden, Singh,
  Cathelin, de~Goes, Desbrun, and Ostromoukhov]{Wachtel:2014}
Florent Wachtel, Adrien Pilleboue, David Coeurjolly, Katherine Breeden, Gurprit
  Singh, Ga\"{e}l Cathelin, Fernando de~Goes, Mathieu Desbrun, and Victor
  Ostromoukhov.
\newblock Fast tile-based adaptive sampling with user-specified fourier
  spectra.
\newblock \emph{ACM Trans. Graph.}, 33\penalty0 (4):\penalty0 56:1--56:11, July
  2014.
\newblock ISSN 0730-0301.
\newblock \doi{10.1145/2601097.2601107}.

\bibitem[Wang and Sloan(2008)]{lds:hd}
Xiaoqun Wang and Ian~H. Sloan.
\newblock Low discrepancy sequences in high dimensions: How well are their
  projections distributed?
\newblock \emph{Journal of Computational and Applied Mathematics}, 213\penalty0
  (2):\penalty0 366 -- 386, 2008.
\newblock ISSN 0377-0427.
\newblock \doi{https://doi.org/10.1016/j.cam.2007.01.005}.
\newblock URL
  \url{http://www.sciencedirect.com/science/article/pii/S0377042707000374}.

\bibitem[Wei(2008)]{Wei:2008}
Li-Yi Wei.
\newblock Parallel poisson disk sampling.
\newblock \emph{ACM Trans. Graph.}, 27\penalty0 (3):\penalty0 20:1--20:9,
  August 2008.
\newblock ISSN 0730-0301.
\newblock \doi{10.1145/1360612.1360619}.
\newblock URL \url{http://doi.acm.org/10.1145/1360612.1360619}.

\bibitem[Wei(2010)]{Wei:2010}
Li-Yi Wei.
\newblock Multi-class blue noise sampling.
\newblock \emph{ACM Trans. Graph.}, 29\penalty0 (4):\penalty0 79:1--79:8, July
  2010.
\newblock ISSN 0730-0301.
\newblock \doi{10.1145/1778765.1778816}.

\bibitem[Xu et~al.(2014)Xu, Wang, and An]{Xu}
Qingtong Xu, Jing Wang, and Xuandong An.
\newblock A pipeline for surface reconstruction of 3-dimentional point cloud.
\newblock In \emph{Audio, Language and Image Processing (ICALIP), 2014
  International Conference on}, pages 822--826, July 2014.
\newblock \doi{10.1109/ICALIP.2014.7009909}.

\bibitem[Yan and Wonka(2012)]{Yan:2012}
Dong-Ming Yan and Peter Wonka.
\newblock Adaptive maximal poisson-disk sampling on surfaces.
\newblock In \emph{SIGGRAPH Asia 2012 Technical Briefs}, SA '12, pages
  21:1--21:4, New York, NY, USA, 2012. ACM.
\newblock ISBN 978-1-4503-1915-7.
\newblock \doi{10.1145/2407746.2407767}.

\bibitem[Yan and Wonka(2013)]{Yan:2013}
Dong-Ming Yan and Peter Wonka.
\newblock Gap processing for adaptive maximal poisson-disk sampling.
\newblock \emph{ACM Trans. Graph.}, 32\penalty0 (5):\penalty0 148:1--148:15,
  October 2013.
\newblock ISSN 0730-0301.
\newblock \doi{10.1145/2516971.2516973}.

\bibitem[Yellott(1983)]{Yellott:1983}
JI~Yellott.
\newblock Spectral consequences of photoreceptor sampling in the rhesus retina.
\newblock \emph{Science}, 221\penalty0 (4608):\penalty0 382--385, 1983.
\newblock ISSN 0036-8075.
\newblock \doi{10.1126/science.6867716}.
\newblock URL \url{http://science.sciencemag.org/content/221/4608/382}.

\bibitem[Ying et~al.(2013{\natexlab{a}})Ying, Li, and He]{Ying2013}
Xiang Ying, Zhenhua Li, and Ying He.
\newblock A parallel algorithm for improving the maximal property of poisson
  disk sampling in r2 and r3.
\newblock In \emph{Proceedings of the ACM SIGGRAPH Symposium on Interactive 3D
  Graphics and Games}, I3D '13, pages 179--179, New York, NY, USA,
  2013{\natexlab{a}}. ACM.
\newblock ISBN 978-1-4503-1956-0.
\newblock \doi{10.1145/2448196.2448227}.

\bibitem[Ying et~al.(2013{\natexlab{b}})Ying, Xin, Sun, and He]{Ying:2013}
Xiang Ying, Shi-Qing Xin, Qian Sun, and Ying He.
\newblock An intrinsic algorithm for parallel poisson disk sampling on
  arbitrary surfaces.
\newblock \emph{Visualization and Computer Graphics, IEEE Transactions on},
  19\penalty0 (9):\penalty0 1425--1437, Sept 2013{\natexlab{b}}.
\newblock ISSN 1077-2626.
\newblock \doi{10.1109/TVCG.2013.63}.

\bibitem[Ying et~al.(2014)Ying, Li, and He]{Ying2014}
Xiang Ying, Zhenhua Li, and Ying He.
\newblock A parallel algorithm for improving the maximal property of poisson
  disk sampling.
\newblock \emph{Computer-Aided Design}, 46:\penalty0 37 -- 44, 2014.
\newblock ISSN 0010-4485.
\newblock \doi{http://dx.doi.org/10.1016/j.cad.2013.08.016}.
\newblock 2013 \{SIAM\} Conference on Geometric and Physical Modeling.

\bibitem[Zhou et~al.(2012)Zhou, Huang, Wei, and Wang]{Zhou:2012}
Yahan Zhou, Haibin Huang, Li-Yi Wei, and Rui Wang.
\newblock Point sampling with general noise spectrum.
\newblock \emph{ACM Trans. Graph.}, 31\penalty0 (4):\penalty0 76:1--76:11, July
  2012.
\newblock ISSN 0730-0301.
\newblock \doi{10.1145/2185520.2185572}.

\end{thebibliography}
